\documentclass[12pt]{article}
\usepackage{setspace}
\doublespace 

\usepackage[utf8]{inputenc}
\usepackage{amsfonts}
\usepackage{amsgen,amsmath,amstext,amsbsy,amsopn,amssymb,subcaption,mathrsfs}
\usepackage[dvips]{graphicx}
\usepackage{comment}
\usepackage{bbm}

\usepackage[shortlabels]{enumitem}
\usepackage{natbib}
\usepackage{booktabs}
\usepackage{amsthm}
\usepackage{hyperref}
\usepackage{blkarray}
\usepackage{algorithm}
\usepackage{algorithmic}
\usepackage{url}
\usepackage{longtable}
\usepackage{stmaryrd}
\usepackage{mwe}
\usepackage{mathtools}
\DeclareMathAlphabet\mathbfcal{OMS}{cmsy}{b}{n}

\makeatletter
\newcommand*{\rom}[1]{\expandafter\@slowromancap\romannumeral #1@}
\makeatother

\usepackage[dvipsnames]{xcolor}

\newcommand\fro[1]{\| #1 \|_{\rm{F}}}

\newcommand{\inp}[2]{\langle #1,#2\rangle}

\newcommand{\argmax}{\mathop{\rm arg\max}}

\newcommand{\PP}{\mathbb{P}}

\def\calA{{\mathcal A}}

\def\calE{{\mathcal E}}
\def\calF{{\mathcal F}}

\def\calP{{\mathcal P}}

\def\calX{{\mathcal X}}

\def\EE{{\mathbb E}}

\def\OO{{\mathbb O}}
\def\PP{{\mathbb P}}

\def\RR{{\mathbb R}}

\def\indicator{\mathbf{1}}
\def\hat{\widehat}
\def\eps{\varepsilon}
\def\tilde{\widetilde}

\def\uoff{\textsf{u-off}}
\def\uipw{\textsf{IPW}}

\newtheorem{Theorem}{Theorem}
\newtheorem{Assumption}{Assumption}

\newtheorem{Lemma}{Lemma}

\theoremstyle{plain}
\newtheorem{Corollary}{Corollary}

\title{Online Policy Learning and Inference by Matrix Completion}

 \author{Congyuan Duan$^1$, Jingyang Li$^2$ and Dong Xia$^1$\\
  $^1${\small Department of Mathematics, Hong Kong University of Science and Technology, Hong Kong SAR}\\
  $^2${\small Department of Statistics, University of Michigan, Ann Arbor, USA} }
\date{(\today)}

\begin{document}

\maketitle


\begin{abstract}
Is it possible to make online decisions when personalized covariates are unavailable? We take a collaborative-filtering approach for decision-making based on collective preferences. By assuming low-dimensional  \emph{latent} features, we formulate the \emph{covariate-free} decision-making problem as a matrix completion bandit. 
We propose a policy learning procedure that combines an $\varepsilon$-greedy policy for decision-making with an online gradient descent algorithm for bandit parameter estimation. Our novel two-phase design balances policy learning accuracy and regret performance. For policy inference, we develop an online debiasing method based on inverse propensity weighting and establish its asymptotic normality. Our methods are applied to data from the San Francisco parking pricing project, revealing intriguing discoveries and outperforming the benchmark policy.
\end{abstract}

\section{Introduction}\label{sec:intro}

Over recent decades, the demand for personalized policy design has surged across various fields, including clinical trials and healthcare \citep{murphy2003optimal, kim2011battle, bertsimas2017personalized}, online advertising and marketing \citep{bottou2013counterfactual, he2012timing, bertsimas2020predictive}, revenue management \citep{chen2022statistical}, and online news recommendation \citep{li2010contextual, ban2019big}. Personalized policies leverage individual characteristics to incorporate covariate heterogeneity into decision-making, such as tailoring medicine dosages based on patients' lab results \citep{murphy2003optimal} or recommending news according to users' profiles \citep{li2010contextual}. However, policy learning is challenging due to its online nature, where data arrives sequentially and are collected adaptively. Online learning algorithms aim to make informed decisions that maximize rewards while continuously refining decision rules. 

The contextual bandit (CB) model \citep{langford2007epoch} offers a convenient framework for online personalized decision-making. A decision maker selects from a finite set of actions (\emph{arms}) $\calA$. At each time step $t$ within a horizon $T$, a request with context $x_t$ arrives, prompting the choice of action $a_t \in \calA$, which yields a reward $r_t$ with expectation $\EE[r_t] = f_{a_t}(x_t)$. If the reward functions $f_a(x)$ are known for all $a \in \calA$, the \emph{optimal policy} is to select $a_t^{\ast}:= \arg\max_{a \in \calA} f_a(x_t)$. However, when these functions are unknown, they must be learned from data. 
 The decision maker aims to maximize cumulative rewards \citep{bastani2020online,li2010contextual,zhang2021statistical}, learn the optimal policy \citep{agarwal2014taming,perchet2013multi,bastani2021mostly,gur2022smoothness,may2012optimistic}, and make inferences about the optimal policy \citep{hadad2021confidence,chen2022online,chen2021statistical,zhang2021statistical}. Once an action $a_t$ is chosen, the rewards for other actions remain unobserved. To achieve these goals, the decision maker must balance \emph{exploration} -  gathering data for each arm to learn the reward functions - and \emph{exploitation} - using current knowledge to maximize rewards. This balance is known as the exploration-exploitation dilemma \citep{lattimore2020bandit}.

However, in some scenarios, personalized covariates are unavailable, rendering individual-specific information inaccessible. Consequently, traditional contextual bandit methods that depend on these covariates become inapplicable. To address these \emph{covariate-free} challenges, we explore personalized decision-making by identifying \emph{latent} similarities among users or entities involved in the decision-making process. This approach is analogous to collaborative filtering \citep{su2009survey}, which generates personalized recommendations by leveraging the preferences and behaviors of a large population. It does not require detailed information about the items or users; instead, it relies on collective preferences to suggest items that individuals may find appealing. To illustrate this concept, we provide two motivating examples from transportation science and revenue management, offering practical insights into our study.

\emph{Parking pricing (\emph{SFpark}, \cite{sfpark})}. The San Francisco government plans to implement dynamic pricing for parking lots across various blocks and time periods (e.g., hourly, weekly). For simplicity, assume two pricing options: high and low. High prices aim to deter parking in overcrowded blocks, while low prices seek to attract drivers to underutilized areas. The objective is to maintain moderate occupancy rates across more blocks during most times. In this context, personalized covariates for blocks or hours are unavailable, but latent similarities exist, such as blocks in the same busy area or morning hours exhibiting similar patterns. If setting a high price in one block during morning hours effectively achieves the target occupancy, it is reasonable to apply high prices to adjacent blocks in the same area and time period to obtain comparable results.

\emph{Supermarket discount (\cite{walmart})}.A supermarket manager must decide whether to offer discounts on various products across different time periods (weekly, annually). While discounts may reduce per-unit profits, they can increase overall revenue by attracting more customers. The objective is to develop a policy that identifies the optimal timing and products for discounts to maximize total annual profits. In this context, personalized covariates for products or time periods are unavailable. However, latent similarities exist - products within the same category or time periods within the same season or holiday exhibit correlated sale trends. For example, if discounting a product during a specific holiday boosts profits, it is reasonable to extend discounts to other products in the same category during that holiday. 

To address these \emph{covariate-free} challenges, we propose the matrix completion bandit (MCB) model, which captures similarities between current and historical requests or rewards through low-rank matrices. Consider the parking pricing problem, where the objective is to set high or low prices for $d_1$ blocks across $d_2$ hours to achieve target occupancy rates. We represent the occupancy rates of high and low pricing with two unknown $d_1\times d_2$  matrices, $M_1$	and $M_0$, respectively. For each block-hour pair $(i,j)$, if $M_1(i,j)>M_0(i,j)$, a high price should be set to better reach the occupancy rate target; otherwise, a low price is preferred. We assume that both $M_0$ and $M_1$ are low-rank, indicating that blocks and hours can be described by low-dimensional latent features, such as geographical proximity or close daytime periods. 

In the MCB framework, learning the optimal policy involves estimating these underlying matrices. Moreover, to evaluate whether a high price significantly outperforms a low price in improving occupancy rates (i.e., policy inference), we perform statistical inference on $M_1(i,j)-M_0(i,j)$. This formulation links policy learning with online matrix completion and policy inference with entrywise inference. However, policy learning and inference via online matrix completion are far from straightforward. Most existing matrix completion literature focuses on independent and offline observations, whereas in online decision-making, observations arrive sequentially, and actions are typically dependent on historical data. This adaptively collected data necessitates novel estimation and inference methodologies. Furthermore, high-dimensional scenarios, where the matrix dimensions $d_1d_2$ 
are comparable to or even significantly exceed the decision horizon $T$, introduce additional technical challenges.

Balancing regret performance with inference efficiency is another major challenge in policy learning and inference. Accurate estimation and efficient inference typically require sufficient sample sizes for each arm. For example, random exploration allocates a constant fraction of samples to each arm, resulting in an effective sample size of $O(T)$ per arm and achieving optimal convergence rates. However, this approach leads to a linear regret of $O(T)$, which is highly undesirable in bandit problems. In contrast, most bandit algorithms aim to minimize  regret by selecting suboptimal arms with decreasing probability. This strategy can substantially reduce the effective sample size for suboptimal arms, resulting in convergence rates slower than $O(T^{-1/2})$. A recent work \citet{simchi2023multi} demonstrates that the product of estimation error and the square root of regret remains of constant order in the \emph{worst case}.

\subsection{Main contributions}

We address \emph{covariate-free} online personalized decision-making using a collaborative filtering approach by formulating it as a matrix completion bandit (MCB) problem. We analyze policy learning accuracy, regret performance, and inference efficiency, highlighting the trade-offs among these aspects. Our contributions are summarized as follows. 

\emph{General convergence of $\eps$-greedy and online gradient descent algorithms.} We propose an online gradient descent algorithm combined with an $\varepsilon$-greedy strategy to learn the MCB parameters, establishing its convergence under general step size and exploration probability schedules. A key contribution is the derivation of sharp entry-wise error rates for the estimated MCB parameters, demonstrating that errors are uniformly distributed across all matrix entries. By using constant exploration probability and step size (assuming a known horizon $T$) and under classical offline matrix completion conditions, our algorithm achieves statistically optimal estimators (up to logarithmic factors) in both Frobenius and sup-norms. These sup-norm error rates are essential for regret analysis and policy inference. We obtain these error rates through advanced spectral analysis and martingale techniques.

\emph{Tradeoff between policy learning and regret performance.} Using a constant exploration probability leads to trivial regret performance. We propose an \emph{exploration-then-commit} (ETC) schedule, where exploration probabilities decay geometrically in the algorithm's later stages. This approach balances policy learning and regret performance: when achieving a regret upper bound of $\widetilde{O}(T^{1-\gamma})$ with some $\gamma\in [0,1/2]$, the estimator's error rate is $\tilde{O}_p(1/\sqrt{T^{1-\gamma}})$. This finding aligns with the minimax lower bound established by \citet{simchi2023multi}. The optimal regret achievable by our algorithm is $\tilde{O}(T^{2/3})$, consistent with existing results in online high-dimensional bandit algorithms \citep{hao2020high, ma2023high}.

\emph{Online inference framework by IPW-based debiasing.} We introduce a general framework for online statistical inference of the optimal policy.  Gradient descent estimators are typically biased due to implicit regularization \citep{NEURIPS2019_c0c783b5, chen2020noisy}. Building on the IPW-based online debiasing method in \citet{han2022online}, which addresses policy inference with matrix-valued covariates and constant exploration probabilities, we extend this approach to covariate-free policy inference and allow exploration probabilities to diminish over time. We establish the asymptotic normality of the studentized estimator, enabling the construction of valid confidence intervals. 

\emph{Real data analysis.} We evaluate the practical merit of MCB using the \emph{SFpark} and supermarket discount datasets, which have been widely studied in transportation and operations management research. In the \emph{SFpark} dataset, we address the dynamic pricing problem for parking lots across various blocks and time periods to minimize over- and under-occupancy rates. Our algorithm learns optimal prices for each block and time period and makes inferences to assess policy effectiveness. Compared to benchmark and alternative pricing strategies, our method achieves a higher proportion of blocks and time periods meeting the target occupancy rates. Additionally, our method is applied to the superstore discount data, yielding interesting findings.

\subsection{Related works}


\emph{Policy evaluation and inference with adaptively collected data.}  Policy evaluation with adaptively collected data is a closely related and challenging area, as estimators often exhibit bias and inflated variance. Existing approaches address these issues through de-biasing and reweighting techniques. For instance, \citet{bottou2013counterfactual, wang2017optimal, su2020doubly} utilize importance sampling to balance the trade-off between bias and variance, while \citet{luedtke2016statistical, hadad2021confidence, zhan2021off} design specific weights for IPW-based estimators to stabilize variance. However, these works do not simultaneously consider policy inference and regret performance and are typically limited to low-dimensional contexts.

\emph{Matrix completion and inference.} \citet{candes2010matrix} pioneered exact matrix completion through convex programming, while \citet{keshavan2010matrix} introduced a computationally efficient non-convex matrix factorization method that achieves statistical optimality under sub-Gaussian noise. Over the past decade, numerous approaches have been developed, focusing on optimization algorithms and their statistical performance \citep{koltchinskii2011nuclear, hastie2015matrix, ge2016matrix, ma2018implicit}. Online matrix completion and regret analysis have been explored by \citet{jin2016provable} and \citet{cai2023online}. Statistical inference for noisy matrix completion remains challenging, with recent advancements by \citet{chen2019inference, xia2021statistical, yan2021inference, chernozhukov2023inference}. For example, \citet{chen2019inference} and \citet{yan2021inference} employ leave-one-out analysis for entry-wise inference, while \citet{xia2021statistical} utilizes double-sample debiasing and low-rank projection for linear forms inference. More recently, \citet{ma2023multiple} developed a symmetric data aggregation method to control the false discovery rate in multiple testing for noisy matrix completion. These methods generally assume independent observations. In this paper, we develop inferential methods for matrix completion with adaptively collected data, inspired by the estimation of average treatment effects in causal matrix completion \citep{athey2021matrix, bai2021matrix, xiong2023large, choi2023matrix}. 

\emph{A/B Testing.} A/B testing is a widely used technique for the online and adaptive evaluation of strategies or services. It involves assigning users to control or treatment groups, collecting outcome data, and making statistical inference to assess differences between groups. The existing literature primarily focuses on developing estimators of average treatment effect that minimize bias and variance and on characterizing their asymptotic distributions for statistical inference \citep{tang2010overlapping, johari2017peeking, yang2017framework, kohavi2020trustworthy, shi2023dynamic, wu2024nonstationary}. However, A/B testing typically does not address policy learning algorithms, thereby avoiding the study of the trade-off between regret performance and inference efficiency.

\emph{Low-rank matrix bandit}.  \cite{han2022online} investigates the low-rank matrix bandit problem under the assumption that the covariate matrix $X_t$ consists of i.i.d. $N(0,1)$ entries, a scenario referred to as \emph{matrix sensing} (\cite{candes2011tight}). In contrast, we focus on \emph{covariate-free} decision-making through collaborative filtering, formulated as a \emph{}matrix completion problem (\cite{candes2012exact}). The methodologies, theories, and assumptions underlying these two problems are strikingly different.
For instance, deriving the highly non-trivial sup-norm error rate and establishing the incoherence property are uniquely critical yet challenging for analyzing regret performance and policy inference in matrix completion. Furthermore, we design new algorithms to explore the trade-off between regret performance and inference efficiency. Our novel two-phase design provides a straightforward approach to balancing these two objectives.

\section{Collaborative Filtering and Matrix Completion Bandits}\label{sec:CF-policy}
Let $\|\cdot\|$ denote the $\ell_2$-norm for vectors and the spectral norm for matrices, and $\|\cdot\|_{\mathrm{F}}$ represent the Frobenius norm. We use $\|\cdot\|_{\max}$ for the sup-norm, defined as the maximum absolute entry, and define the 2-max norm of a matrix $A$ by $\|A\|_{2,\max} := \max_{j} \|e_j^{\top}A\|$, representing the maximum row-wise $\ell_2$-norm. Let $\OO^{d \times r} := \big\{U \in \mathbb{R}^{d \times r} : U^{\top}U = I_{r}\big\}$, where $I_{r}$ is the identity matrix of dimension $r\times r$. For a matrix $U \in \OO^{d \times r}$, we denote by $U_{\perp} \in \OO^{d \times (d-r)}$ its orthogonal complement, such that $(U, U_{\perp})$ forms a $d\times d$ orthogonal matrix.

We model the collaborative filtering (CF) policy as a \emph{covariate-free} decision-making strategy. Consider $d_1d_2$ distinct requests arranged in a $d_1 \times d_2$ grid, indexed by $(j_1, j_2)$ where $j_1 \in [d_1]$ and $j_2 \in [d_2]$. Specially, we allow high dimensional scenario where $d_1d_2\gg T$. There are $K$ arms (actions), each associated with an unknown reward matrix $M_k\in\RR^{d_1\times d_2}$ for $k \in [K]$. Upon receiving request $(j_1, j_2)$, selecting action $k$ yields an expected reward $[M_k]_{j_1j_2}$. The optimal CF policy that maximizes expected reward is
$a^{\ast}(j_1, j_2) := \arg\max_{k\in[K]} \ [M_k]_{j_1j_2}$ for $j_1 \in [d_1], j_2 \in [d_2].$
The matrices $M_k$ are unknown and must be estimated to learn the optimal policy.  Requests arrive sequentially, and after choosing an action (say, $k$), only a single (noisy) entry from $M_k$ is observed. A fundamental assumption in collaborative filtering is that each $M_k$ is low-rank, implying a low-dimensional structure in their row and column spaces. 


We formulate the matrix completion bandit (MCB) within the trace regression framework \citep{koltchinskii2011nuclear}. For simplicity, we consider the two-armed case ($K=2$), with extensions to $K$ arms discussed in Appendix \ref{sec:karm}. Let $\{(X_t, a_t, r_t)\}_{t=1}^T$ denote a sequence of observations where each $d_1 \times d_2$ matrix $X_t$ is independently and \emph{uniformly} sampled from the orthonormal basis $\calX=\big\{e_{j_1}e_{j_2}^{\top}: j_1\in[d_1], j_2\in[d_2]\big\}$.  This framework applies to scenarios such as sequential pricing of parking blocks or weekly product discounts in a supermarket. Moreover, our model can be extended to accommodate general covariate distributions, as detailed in Appendix \ref{non-uniform}. The action $a_t\in\{0,1\}$ yields a reward according to linear function
$r_t=\langle M_{a_t}, X_t\rangle+\xi_t$, 
where $\xi_t$ is sub-Gaussian noise independent of $X_t$. The noise variance depends on the action: $\mathrm{Var}(\xi_t \mid a_t=1) = \sigma_1^2$ and $\mathrm{Var}(\xi_t \mid a_t=0) = \sigma_0^2$. Let $\mathcal{F}_t$ denote the filtration generated by $\{(X_\tau, a_\tau, r_\tau) : \tau \leq t\}$. Since the action $a_t$ is chosen based on $\mathcal{F}_{t-1}$, the data $\{(X_t, a_t, r_t)\}_{t=1}^T$ are adaptively collected and thus dependent. We assume without loss of generality that $d_1 \geq d_2$ and that both $M_1$ and $M_0$ share a common low rank $r \ll d_2$.

Let $\hat M_{1,t-1}$ and $\hat M_{0, t-1}$ denote the estimated MCB parameters up to time $t-1$. A \emph{greedy} strategy selects the action $a_t=\arg\max_{a\in\{0,1\}} \langle \hat M_{a, t-1}, X_t\rangle$ that potentially maximizes the expected reward. While this approach can maximize cumulative rewards when estimates are sufficiently accurate, it may lead to suboptimal performance when the estimators are not accurate. To address this, we employ an $\varepsilon$-greedy strategy \citep{sutton2018reinforcement, lattimore2020bandit}, which balances exploration and exploitation by introducing randomness into the decision-making process. Specifically, the action $a_t=1$ is chosen according to the probability $
\pi_t:=\PP\big(a_t=1|X_t, \calF_{t-1}\big):=(1-\eps_t)\mathbbm{1}\big(\inp{\hat M_{1, t-1}-\hat M_{0,t-1}}{X_t}>0\big) +\eps_t/2,
$
where $\varepsilon_t$ is a non-increasing sequence of exploration probabilities. This means the algorithm selects the greedy action with probability $1-\varepsilon_t/2$ (exploitation) and the alternative action with probability $\varepsilon_t/2$ (exploration). As more data is collected, $\varepsilon_t$ decreases, allowing the algorithm to increasingly favor exploitation as the parameter estimates become more accurate.

\section{Policy Learning and Regret Performance}\label{sec:policy-learning}
Our primary objective is to estimate the bandit parameters $M_1$ and $M_0$ during online decision-making, which underpins policy learning and inference. We introduce an $\varepsilon$-greedy bandit algorithm that leverages stochastic gradient descent for sequential parameter estimation. Furthermore, we establish the convergence of these online estimators and analyze the regret performance of our algorithm.

\subsection{\texorpdfstring{$\varepsilon$}{epsilon}-Greedy bandit algorithm with stochastic gradient descent}\label{sec:grad-desc}
We begin by introducing key notations and assumptions.  For each $i=0,1$, let $M_i=L_i\Lambda_i R_i^{\top}$ denote the singular value decomposition (SVD), where $L_i\in\OO^{d_1\times r}$ and $R_i\in\OO^{d_2\times r}$ contain the left and right singular vectors, respectively. The diagonal matrix $\Lambda_i={\rm diag}(\lambda_{i,1},\cdots,\lambda_{i,r})$ holds the singular values in non-increasing order. We define the signal strength as $\lambda_{\min}:=\min\{\lambda_{0,r}, \lambda_{1,r}\}$. Also denote $\lambda_{\max}:=\max\{\lambda_{0,1}, \lambda_{1,1}\}$ and the condition number $\kappa:=\lambda_{\max}/\lambda_{\min}$.  
The balanced decomposition of $M_i$ is written as $M_i=U_iV_i^{\top}$, where $U_i^{\top}U_i=V_i^{\top}V_i=\Lambda_i$, achievable by setting $U_i=L_i\Lambda_i^{1/2}$ and $V_i=R_i\Lambda_i^{1/2}$.  Matrix completion becomes ill-posed if the underlying matrix is spikied, meaning that a few entries dominate. The \emph{incoherence} parameter \citep{candes2012exact} for $M_i$ is defined as
$
\mu(M_i):=\max\big\{\sqrt{d_1/r}\|L_i\|_{2,\max},\ \sqrt{d_2/r}\|R_i\|_{2,\max} \big\}.
$
We assume that $\max\{\mu(M_0), \mu(M_1)\}\leq \mu$.  For clarity, we treat $\mu$ and $\kappa$ as bounded constants in our main results, ensuring that most entries of $M_0$ ($M_1$, similarly) have comparable magnitudes and that both matrices are well-conditioned.

At time $t$, a classical offline matrix completion method \citep{keshavan2009matrix} is to estimate $M_1$ by minimizing the sum of squares 
$
\bar{\mathscr{L}}_{1,t}(U,V):=\sum_{\tau=1}^t \ell_{1,\tau}(U,V):=\sum_{\tau=1}^t \mathbbm{1}(\{a_{\tau}=1\}) (r_{\tau}-\big<X_{\tau}, UV^{\top}\big>)^2
$
subject to $U\in\RR^{d_1\times r}$ and $V\in\RR^{d_2\times r}$. This is a non-convex program and can be locally optimized by gradient descent algorithms with guaranteed convergence and statistical performance \citep{burer2003nonlinear, zheng2016convergence} if the observations are i.i.d.. However, since data is adaptively collected under MCB, the propensities $\pi_{\tau}:=\PP(a_{\tau}=1| X_{\tau},\mathcal{F}_{\tau-1})$ are unequal.  The impact to this can be easily seen in the conditional expected loss $\EE [\ell_{1,t}(U,V)|\{\calF_{t-1},X_t, r_t\}]=\pi_t(r_t-\langle X_t, UV^{\top}\rangle)^2$, which likely places more significant weight on the entries where $M_1$ is larger than $M_0$. To mitigate the bias issue, we consider using the weighted sum of squares
\begin{align}\label{eq:MCB-weighted-L}
	\min_{U,V}\ \mathscr{L}^{\pi}_{1,t}(U,V)=\sum_{\tau=1}^t \frac{\mathbbm{1}(a_{\tau}=1)}{\pi_{\tau}}\cdot \big(r_{\tau}-\big<X_{\tau}, UV^{\top}\big>\big)^2\quad {\rm s.t.}\quad U^{\top}U=V^{\top}V,
\end{align}
where the last constraint enforces a balanced factorization and is for algorithmic stability \citep{jin2016provable}.  When estimating $M_0$, we define the loss $\mathscr{L}_{0,t}^{\pi}$ as in (\ref{eq:MCB-weighted-L}), but replace the weight $\mathbbm{1}(a_{\tau}=1)/\pi_{\tau}$ with $\mathbbm{1}(a_{\tau}=0)/(1-\pi_{\tau})$.  Re-solving this non-convex optimization problem with each new observation is computationally intensive. Therefore, we employ an online gradient descent algorithm to update the estimates of $M_1$ and 
$M_0$ incrementally as new data arrives, thereby enhancing computational efficiency and facilitating statistical inference. The detailed implementations of $\eps$-greedy policy and online  gradient decent can be found in Algorithm~\ref{alg:mcb1}.


\begin{algorithm}
		\caption{$\eps$-greedy two-arm MCB with online gradient descent}\label{alg:mcb1}
\begin{algorithmic}
\STATE{\textbf{Input}: exploration probabilities $\{\varepsilon_t\}_{t\geq 1}$; step sizes $\{\eta_t\}_{t\geq 1}$; initializations with balanced factorization ${\hat M}_{0,0}=\hat U_{0,0}\hat V_{0,0}^{\top}$, ${\hat M}_{1,0}=\hat U_{1,0}\hat V_{1,0}^{\top}$.}
\STATE{\textbf{Output}: $\hat{M}_{0,T}$, $\hat{M}_{1,T}$.}
			\FOR{$t= 1,2,\cdots, T$}  
        \STATE Observe a new request $X_t$; \\
        \STATE Calculate $\pi_t=(1-\varepsilon_t)\mathbbm{1}\big(\inp{{\hat M}_{1,t-1} - {\hat M}_{0,t-1}}{X_t}>0\big)+\frac{\varepsilon_t}{2}$; \\
        \STATE Sample an action $a_t\sim \text{Bernoulli}(\pi_t)$ and get a reward $r_t$; \\
		  \FOR{$i= 0,1$}
            \STATE Update by
            \begin{align*}
             \begin{pmatrix}
        \tilde U_{i,t} \\
        \tilde V_{i,t}
        \end{pmatrix} = \begin{pmatrix}
        \hat U_{i,t-1} \\
        \hat V_{i,t-1} 
        \end{pmatrix} - \frac{\mathbbm{1}(a_t=i)\eta_t}{i\pi_{i,t}+(1-i)(1-\pi_t)}\cdot\begin{pmatrix}
        \big(\inp{\hat U_{i,t-1}\hat V_{i,t-1}^{\top}}{X_t}-r_t\big)X_t\hat V_{i,t-1} \\
        \big(\inp{\hat U_{i,t-1}\hat V_{i,t-1}^{\top}}{X_t}-r_t)X_t^{\top}\hat U_{i,t-1}
        \end{pmatrix}.
        \end{align*}	
        Set $\hat{U}_{i,t}=\hat L_{i,t}\hat \Lambda_{i,t}^{1/2}$ and $\hat{V}_{i,t}=\hat R_{i,t}\hat \Lambda_{i,t}^{1/2}$, where $\hat L_{i,t}\hat \Lambda_{i,t}\hat R_{i,t}^{\top}$ is the thin SVD of $\hat {M}_{i,t}=\tilde U_{i,t}\tilde V_{i,t}^{\top}$. 
        \ENDFOR
        \ENDFOR  
\end{algorithmic}
\end{algorithm}

Algorithm~\ref{alg:mcb1} usually requires good initializations $\widehat{M}_{1,0}$ and $\widehat{M}_{0,0}$. In practice, these can be obtained using historical data or a preliminary forced sampling procedure to collect sufficient data for each arm. Subsequently, offline convex matrix completion algorithms, such as those proposed by \citet{mazumder2010spectral} and \citet{hastie2015matrix}, can be applied. Additionally, by generating scree plots of $\widehat{M}_{1,0}$ and $\widehat{M}_{0,0}$, the rank $r$ in Algorithm \ref{alg:mcb1} can be determined in a data-driven manner. If the scree plot does not clearly reveal the elbow point, we recommend selecting a slightly larger $r$ to prevent information loss due to underestimation of the rank.

A common choice for the exploration probability is $\varepsilon_t\asymp t^{-\gamma}$ for some $\gamma\in [0,1)$. Instead of continuously decreasing $\varepsilon_t$ over the entire decision horizon, we partition the learning process into two phases. Accurate estimation of the bandit parameters requires a sufficiently large sample size for each arm. However, rapidly decaying $\varepsilon_t$ 	
can cause the probability of selecting suboptimal arms to diminish too quickly, leading to inadequate data and biased estimates. 
In the first phase, we set $\varepsilon_t$ to a constant to ensure adequate exploration and obtain reasonably accurate estimators. In the second phase, we switch to a geometrically decaying schedule for $\varepsilon_t$, achieving non-trivial regret bounds. Moreover, this two-phase approach facilitates policy inference, as will be demonstrated in Section \ref{sec:inference}. Our method resembles the ``Explore-Then-Commit" (ETC) algorithm \citep{garivier2016explore, lattimore2020bandit,perchet2016batched}. Unlike typical ETC algorithms, which set $\varepsilon_t\equiv 0$ and adopt purely greedy decisions in the second phase, our approach continues to incorporate exploration during the second phase. This ongoing exploration is crucial for effective policy inference, as demonstrated in Section~\ref{sec:inference}.



\begin{Theorem} \label{thm:MCB-conv}
Suppose that the horizon $T\leq d_1^{100}$ and the initializations are incoherent, satisfying $\|\hat M_{0,0}-M_0\|_{\rm F}+\|\hat M_{1,0}-M_1\|_{\rm F}\leq c_0\lambda_{\min}$ for some small constant $c_0>0$.
Fix some $\gamma\in [0,1)$, $\eps\in(0, 1)$, and set $T_0:=C_0T^{1-\gamma}$ for a sufficiently large constant $C_0>0$.  When $t\leq T_0$, set $\eps_t\equiv\eps$ and $\eta_t\equiv \eta:=c_1d_1d_2\log(d_1)/(T^{1-\gamma}\lambda_{\max})$; when $T_0<t\leq T$, set $\eps_t=c_2t^{-\gamma}$ and $\eta_t=\eps_t \eta$, where $c_1, c_2>0$ are numerical constants. Suppose that the horizon and signal-to-noise ratio (SNR) satisfy
$$
T\geq C_1r^3d_1^{1/(1-\gamma)} \log^2 d_1\quad {\rm and}\quad  \frac{\lambda_{\min}^2}{\sigma_0^2+\sigma_1^2}\geq C_2\frac{rd_1^2d_2\log^2d_1}{T^{1-\gamma}},
$$
for some large constants $C_1, C_2>0$ depending on $C_0, c_0,  c_1,c_2$ only. There exist constant $C_3,C_4,C_5 >0$ such that, for both $i\in\{0,1\}$, with probability at least $1-8td_1^{-200}$, the output of Algorithm~\ref{alg:mcb1} satisfies 
    \begin{align*}
        &\big\|\hat M_{i, t} - M_i \big\|_{\rm F}^2\leq C_3\fro{\widehat{M}_{i,0}-M_i}^2 \left(1-\frac{c_1\log d_1}{4\kappa T^{1-\gamma}}\right)^t + C_4t\cdot \sigma_i^2\frac{rd_1^2d_2\log^4 d_1}{T^{2-2\gamma}}, \\
        &\big\|\hat M_{i,T} - M_i\big\|_{\max}^2\leq  C_3\frac{\lambda_{\min}^2r^3}{d_1d_2} \left(1-\frac{c_1\log d_1}{4\kappa T^{1-\gamma}}\right)^t +   C_4t\cdot \sigma_i^2\frac{rd_1\log^4 d_1}{T^{2-2\gamma}},
    \end{align*}
for all $t<T_0$.  Moreover, for all $t\geq T_0$, with probability at least $1-8d_1^{-100}$, we have 
    \begin{align*}
        \big\|\hat M_{i, t} - M_i \big\|_{\rm F}^2\leq C_5\sigma_i^2\frac{rd_1^2d_2\log^4 d_1}{T^{1-\gamma}} \quad {\rm and}\quad 
        \big\|\hat M_{i,t} - M_i\big\|_{\max}^2\leq C_5\sigma_i^2\frac{rd_1\log^4 d_1}{T^{1-\gamma}}.
    \end{align*} 
\end{Theorem} 

We assume $T\leq d_1^{100}$ solely for technical convenience, allowing the exponent to be replaced by any sufficiently large constant. For clarity, Theorem \ref{thm:MCB-conv} considers specific schedules for learning rates and exploration probabilities, while Appendix \ref{proof:completeconv} establishes convergence under more general schedules.  We also present a computational acceleration of Algorithm~\ref{alg:mcb1} in Appendix \ref{sec:practical}, demonstrating that each SVD step in Algorithm~\ref{alg:mcb1} can be executed within $O(r^3)$ flops.   

In Theorem \ref{thm:MCB-conv}, the parameter $\gamma$ determines the duration of phase one, which maintains a constant exploration probability. When the primary objective is to estimate the MCB parameters and learn the optimal policy, $\gamma$ can be set near zero. Under this setting, the horizon and signal strength conditions align with those in classical matrix completion literature \citep{keshavan2010matrix, candes2012exact}. Moreover, the final estimators $\hat M_{0,T}$ and $\hat M_{1,T}$ achieve statistical optimality up to logarithmic factors \citep{koltchinskii2011nuclear, ma2018implicit, cai2023online}. The constant exploration probability ensures $O(T)$ observations for each arm, making our results comparable to the offline case.


Algorithm~\ref{alg:mcb1} includes several tuning parameters: the initial exploration probability 
$\varepsilon$, the constants $C_0$ and $c_1$ determining the duration of and the step size in phase one, respectively. While the parameter $\gamma$ is critical for balancing regret performance and statistical inference efficiency (discussed in subsequent sections), the other parameters are relatively flexible. We recommend selecting a small $c_1$  and a sufficiently large $C_0$. As shown in Section~\ref{sec:inference}, a large 
$C_0$ may slightly impact inference efficiency and the width of confidence intervals but ensures accurate estimates at $T_0$. Similarly, a small $c_1$ reduces the estimation error of $\hat M_{i,t}$	as demonstrated in Theorem \ref{proof:completeconv} in Appendix~\ref{app:proofs_thms}. Our simulation studies in Section~\ref{sec:simulation} and Appendix~\ref{app:numerical} confirm that the method remains effective across various parameter settings.

\subsection{Regret performance} \label{sec:regret}
One of the decision maker's goals is to maximize the expected cumulative reward during the policy learning process. The maximum expected cumulative reward $\EE\big[\sum_{t\in[T]}\max_{i\in\{0,1\}}\langle M_i, X_t \rangle\big]$ is attainable from making decisions using the optimal policy $a^{\ast}$ defined in Section~\ref{sec:CF-policy}. Let $\{a_t, t\in[T]\}$ be the sequence of actions taken according to Algorithm~\ref{alg:mcb1}, which results in an expected cumulative reward $\EE\big[\sum_{t\in[T]} \langle M_{a_t}, X_t\rangle\big]$.  Here,  the expectation is taken with respect to $\calF_T$. 
The difference between the achieved expected cumulative reward and the maximum one is referred to as the \emph{regret}, formally defined by
$ R_T:=\EE\big[\sum_{t=1}^{T} \max_{i\in \{0,1\}}\inp{M_i}{X_t} - \inp{M_{a_t}}{X_t}\big].$


\begin{Theorem} \label{thm:regret}
  Suppose that  the conditions in Theorem \ref{thm:MCB-conv} hold, and denote by $\bar{m}:=\|M_0\|_{\max}+\|M_1\|_{\max}$ and $\bar\sigma:=\max\{\sigma_0, \sigma_1\}$. 
  Then there exists a numerical constant $C_6>0$ such that  the regret is upper bounded by 
\begin{align*}
        R_T\leq & C_6\bigg(\bar m r\cdot T^{1-\gamma}+\bar\sigma\cdot T^{(1+\gamma)/2}\sqrt{rd_1}\log^2 d_1 \bigg)
 \end{align*}
\end{Theorem}

Assuming $r=O(1)$, Theorem~\ref{thm:regret} shows that Algorithm \ref{alg:mcb1} achieves the regret upper bound $R_T=\tilde{O}\big( \bar m\cdot T^{1-\gamma}+\bar\sigma\cdot T^{(1+\gamma)/2}d_1^{1/2}\big)$, where logarithmic factors are hidden in $\tilde O(\cdot)$.  As discussed after Theorem~\ref{thm:MCB-conv}, the MCB algorithm can learn the optimal policy more accurately if the value $\gamma$ is closer to $0$.  However, setting $\gamma=0$ will lead to a trivial regret upper bound of $O(T)$. 
By choosing a value $\gamma>0$ so that $T^{1-3\gamma} = (\bar\sigma/\bar m)^2d_1$, we end up with the best regret upper bound $R_T = \tilde O\big((\bar m\bar \sigma^2)^{1/3}d_1^{1/3}T^{2/3}\big)$. 
The $O(T^{2/3})$ regret performance has been observed in \emph{online} high-dimensional bandit algorithms especially when the horizon $T$ is not too large. See, e.g., \cite{hao2020high}, \cite{ma2023high},  and references therein. 



It is worth noting that a classical ETC bandit algorithm may achieve a similar policy learning accuracy and regret performance. The classical ETC algorithm will make greedy decisions in the second phase without exploration. By setting $\eps_t=\eta_t=0$ for $T_0<t\leq T$, the estimated MCB parameters will not be further updated after phase one. The decisions made during phase two are only determined by the entry-wise difference between $\hat M_{0, T_0}$ and $\hat M_{1, T_0}$. It can be shown that even without exploration during phase two, the regret performance remains comparable to that in Theorem~\ref{thm:regret}. We adopt the proposed exploration probability schedule in Theorem~\ref{thm:MCB-conv} to facilitate policy inference, which typically requires an unbiased estimator for each arm parameter. If no exploration is made during phase two, there will be no chance to choose the estimated sub-optimal arm, then the inferential methods in Section~\ref{sec:inference} will be invalid because the debiasing procedure does not yield an unbiased estimator.

\section{Policy Inference by Inverse Propensity Weighting}\label{sec:inference}
 Suppose, in the \emph{SFpark} example, we want to evaluate whether high price (arm 1) is better than low price (arm 0) for block $i$ at hour $j$. Then, there exists a $d_1\times d_2$ matrix $Q=e_ie_j^{\top}$ such that the expected rewards under arm 0 and 1 may be written as $\langle M_0, Q\rangle$ and $\langle M_1, Q\rangle$, respectively. The policy evaluation process can be conveniently formulated as a hypothesis testing problem: 
\begin{align}\label{eq:test-H}
H_0: \langle M_1-M_0, Q\rangle= 0\quad {\rm v.s.}\quad H_1: \langle M_1-M_0, Q\rangle>0.
\end{align}
Rejecting the null hypothesis means that high price is significantly better than the low price. The matrix $Q$ can represent more general linear forms.

\subsection{Debiasing}\label{sec:debias}

Constructing a good estimator is the first step in testing hypotheses (\ref{eq:test-H}). The plug-in estimator $\langle \hat M_{1,T}-\hat M_{0, T}, Q\rangle$ suffers from potential bias caused by the implicit regularization of gradient descent, and $\hat M_{0, T}, \hat M_{1,T}$ may not be unbiased estimators.  Proper debiasing treatment is crucial for valid inference. The commonly used debiasing method \citep{chen2019inference, xia2021statistical} in \emph{offline} matrix completion considers 
\begin{align*}
\hat M_{1}^{\textsf{u-off}}:= \hat M_{1, T}+\frac{d_1d_2}{\big|\big\{t\in[T]: a_t=1 \big\}\big|}\sum_{t=1}^T \mathbbm{1}(a_t=1)\cdot \big(r_t-\langle\hat M_{1, T}, X_t \rangle\big)X_t.
\end{align*}
This method may not work for MCB because data are adaptively collected, and $\hat M_1^{\uoff}$ may still be a biased estimator of $M_{1}$. Moreover, the estimator $\hat M_1^{\uoff}$ requires access to historical data to compute $\langle \hat M_{1, T}, X_t\rangle X_t$. 

An \emph{online} debiasing method was recently proposed by \cite{han2022online} for the low-rank matrix regression bandit, which provides an unbiased estimator even if data is adaptively collected. The crucial ingredient is to include the inverse propensity weight (IPW), if modified into matrix completion setting:
\begin{align}\label{eq:debias-uon}
\hat M_{1}^{\uipw}:=\frac{1}{T}\sum_{t=1}^ T\hat M_{1, t-1}+\frac{d_1d_2}{T}\sum_{t=1}^T \frac{\mathbbm{1}(a_t=1)}{\pi_t} \big(r_t-\langle \hat M_{1, t-1}, X_t\rangle\big)X_t.
\end{align}
The estimator $\hat M_0^{\uipw}$ is similarly defined using $(1-\pi_t)^{-1}$ as the IPW. The unbiasedness of $\hat M_0^{\uipw}$ and $\hat M_1^{\uipw}$ can be easily verified by noticing the independence between $\hat M_{1,t-1}$ and $(X_t, r_t, a_t)$, conditioned on $\calF_{t-1}$. The debiasing procedure (\ref{eq:debias-uon}) is online in nature, since it does not require saving historical data. The method only needs the incremental updating of the two summands in (\ref{eq:debias-uon}) over time. Nevertheless, the debiasing procedure (\ref{eq:debias-uon}) accumulates all past estimates, but estimates from the early stage may not be sufficiently accurate, which can potentially inflates the variance.  As demonstrated in Theorem~\ref{thm:MCB-conv}, the estimates $\hat M_{1, T_0}$ and $\hat M_{0, T_0}$ obtained at the end of phase one satisfy $\|\hat M_{1, T_0}-M_1\|_{\max}+\|\hat M_{0, T_0}-M_0\|_{\max}=o_p\big(\sigma_0+ \sigma_1\big)$ under the specified horizon and SNR conditions, which is essential for statistical inference \citep{xia2021statistical}. Consequently, our IPW debiasing approach utilizes only the second phase estimates. Since phase one occupies a minimal portion of the horizon, we retain $O(T)$
 data points for constructing the debiased estimator.  Finally, we define 
\begin{align}\label{eq:debias-aw}
\hat M_0^{\uipw}:&= \frac{1}{T-T_0}\sum_{t=T_0+1}^ T\hat M_{0, t-1}+\frac{d_1d_2}{T-T_0}\sum_{t=T_0+1}^T \frac{\mathbbm{1}(a_t=0)}{1-\pi_t} \big(r_t-\langle \hat M_{0, t-1}, X_t\rangle\big)X_t\notag\\
\hat M_1^{\uipw}:&= \frac{1}{T-T_0}\sum_{t=T_0+1}^ T\hat M_{1, t-1}+\frac{d_1d_2}{T-T_0}\sum_{t=T_0+1}^T \frac{\mathbbm{1}(a_t=1)}{\pi_t} \big(r_t-\langle \hat M_{1, t-1}, X_t\rangle\big)X_t.
\end{align}

Given the low-rank assumption of MCB parameters, our final estimators are defined by the best rank-$r$ approximation of $\hat M_1^{\uipw}$ and $\hat M_0^{\uipw}$, denoted by $\hat M_1$ and $\hat M_0$, respectively. More precisely, $\hat M_1=\hat L_1\hat L_1^{\top}\hat M_1^{\uipw} \hat R_1\hat R_1^{\top}$,  where $\hat L_1$ and $\hat R_1$ consist of the left and right top-$r$ singular vectors of $\hat M_1^{\uipw}$, respectively. As explained in \cite{xia2021statistical}, spectral projection further reduces the variance at the cost of introducing a negligible bias.  Using the plug-in method, we propose $\langle \hat M_1-\hat M_0, Q\rangle$ as the point estimator for the linear form $\langle M_1- M_0, Q\rangle$.

\subsection{Asymptotic normality}

We now investigate the variances of $\langle \hat M_1, Q\rangle$ and $\langle \hat M_0, Q\rangle$, respectively, and prove that these test statistics are asymptotically normal under mild conditions. Take $\langle \hat M_1, Q\rangle$ as an example, the two key factors  determining the variance of $\langle \hat M_1, Q\rangle$ are: the properness of linear form $Q$ and the source of variances in estimating $\hat M_1$. \cite{xia2021statistical} and \cite{ma2023multiple} show that the variance of $\langle \hat M_1, Q\rangle$ is characterized by the alignment between $Q$ and $M_1$. More precisely, the quantity $\|\calP_{M_1}(Q)\|_{\rm F}$ determines the size of variance, where 
$
\calP_{M_1}(Q):=Q-L_{1\perp}L_{1\perp}^{\top}Q R_{1\perp}R_{1\perp}^{\top}
$
with $L_{1\perp}$ and $R_{1\perp}$ such that $(L_1, L_{1\perp})$ and $(R_1, R_{1\perp})$ are $d_1\times d_1$ and $d_2\times d_2$ orthogonal matrices, respectively. If $Q$ is orthogonal to $M_1$ in the sense that $L_1^{\top}Q=0$ and $QR_1=0$, statistical inference for $\langle M_1, Q\rangle$ becomes an ill-posed problem. In this case, $Q$ lies in the \emph{normal space} of the fixed-rank matrix manifold at point $M_1$, rendering the observed data non-informative for the linear form $\langle M_1, Q\rangle$. 

The proof of Theorem~\ref{thm:CLT} shows that the variance of $\inp{\widehat{M}_1-M_1}{Q}$ is mainly contributed by 
\begin{align*}
   \widehat{V}_{1}:=\frac{d_1d_2}{T-T_0}\sum_{t=T_0+1}^T \frac{\mathbbm{1}(a_t=1)}{\pi_t}\xi_t\inp{\hat L_1\hat L_1^{\top} X_t \hat R_1\hat R_1^{\top}}{Q}.
\end{align*}
The term $\widehat{V}_1$ is a \emph{weighted} sum of the samples $X_t$ assigned to arm $1$, with weights given by the inverse propensity scores $\pi_t^{-1}$. Unlike traditional offline settings where samples are i.i.d. and contribute equally to the variance, each sample here contributes differently to the overall variance. For any $\delta>0$, we define $\Omega_1(\delta):=\big\{X\in\calX: \inp{M_1-M_0}{X}>\delta \big\}$, $\Omega_0(\delta):=\big\{X\in\calX: \inp{M_0-M_1}{X} <\delta \big\}$, and $\Omega_{\emptyset}(\delta):=\big(\Omega_1(\delta)\cup\Omega_0(\delta)\big)^c$. Here, $\Omega_1(\delta)$  represents instances where arm 1 is optimal with an expected regret gap exceeding $\delta$. The parameter $\delta$ denotes the reward gap between the optimal and suboptimal arms. For simplicity, we abbreviate these sets as $\Omega_0, \Omega_1$, and $ \Omega_{\emptyset}$, respectively, omitting their dependence on $\delta$. 

The sets $\Omega_1$ and $\Omega_0$ contribute differently to the variance of $\inp{\hat M_1}{Q}$. 
After phase one, the online estimators $\widehat{M}_{1,t}$ and $\widehat{M}_{0,t}$ are sufficiently accurate so that the optimal arms for samples from $\Omega_1$ and $\Omega_0$ are identifiable. Consequently, as the exploration probability $\varepsilon_t$ diminishes, observations $X_t\in\Omega_1$ are used to construct $\hat M_1$ with propensities $\pi_t \to 1$, while $X_t\in\Omega_0$ have $\pi_t \to 0$ as $t\to\infty$. Specifically, with $\varepsilon_t\asymp t^{-\gamma}$,  each $X_t\in\Omega_1$ contributes $\asymp 1/T^2$ to the variance, and each $X_t\in\Omega_0$ contributes $\asymp t^{\gamma}/T^2$. We demonstrate these variance contributions through simulation experiments, presenting box plots of samples from $\Omega_1$ and $\Omega_0$ used in constructing $\widehat{M}_1$ and $\widehat{M}_0$, respectively, in Figure \ref{varianceplot}. 

\begin{figure} 
	\centering
		\includegraphics[scale=0.7]{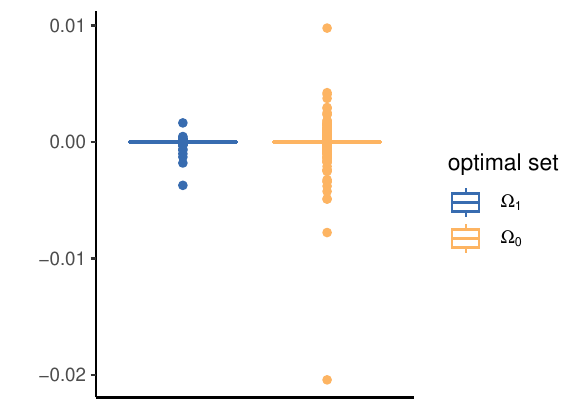}
		\includegraphics[scale=0.7]{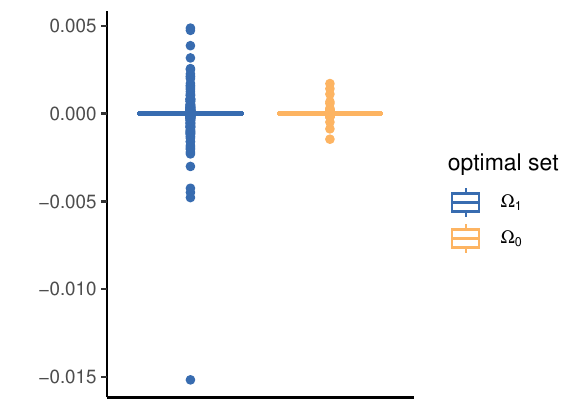}
    \caption{Box plots of $\mathbbm{1}(a_t=1)/\pi_t\cdot\xi_t\inp{\hat L_1\hat L_1^{\top}X_t\hat R_1\hat R_1^{\top}}{Q}$ (left) and $\mathbbm{1}(a_t=0)/(1-\pi_t)\cdot\xi_t\inp{\hat L_0\hat L_0^{\top}X_t\hat R_0\hat R_0^{\top}}{Q}$ (right), grouped by samples $X_t$ from $\Omega_1$ and $\Omega_0$ in a simulation with $T=60,000$ and $T_0=20,000$. Specifically, arm 1 was assigned to 18,761 samples from $\Omega_1$ and 1,281 from $\Omega_0$, while arm 0 was assigned to 1,285 samples from $\Omega_1$ and 18,673 from $\Omega_0$. This demonstrates that most samples assigned to arm $i$ originate from $\Omega_i$. However, samples from $\Omega_{1-i}$ exhibit significantly higher variance, which primarily determines the variance of $\inp{\widehat{M}_i}{Q}$. }
    \label{varianceplot} 
\end{figure}

We assume that these variances satisfy the following arm optimality conditions. Let $\calP_{\Omega_1}(M)$ denote the operator which zeros out the entries of $M$ except those in the set $\Omega_1$. 
\begin{Assumption}\label{assump:arm-opt}(Arm Optimality)
There exists a gap $\delta>0$ such that the sets $\Omega_0,  \Omega_1$, and $\Omega_{\emptyset}$ ensure 
$\fro{\calP_{\Omega_{\emptyset}}\calP_{M_i}(Q)}^2/\min\{\fro{\calP_{\Omega_{1}}\calP_{M_i}(Q)}^2, \fro{\calP_{\Omega_{0}}\calP_{M_i}(Q)}^2\}=o(1)$ as $d_1, d_2\to\infty$ for $i=0,1$.
\end{Assumption}

The arm optimality assumption is prevalent in bandit literature \citep{goldenshluger2013linear, bastani2020online}, typically requiring a constant gap between arms, holding with some probability. In contrast, Assumption~\ref{assump:arm-opt} allows the gap to diminish as $d_1, d_2\to\infty$, representing a weaker condition. Additionally, our assumption aligns with the margin condition \citep{audibert2007fast,qian2011performance} commonly used in bandit studies \citep{bastani2020online, chen2021statistical}, which requires that $\PP(|\inp{M_1-M_0}{X}|<\kappa)\leq C\kappa$ for some constant $C>0$ and all $\kappa>0$. Both conditions ensure that the proportion of samples near the decision boundary $\inp{M_1-M_0}{X}=0$ remains small, thereby reducing the risk of missing the optimal arm due to minor estimation errors in $\widehat{M}_{1,t}$ and $\widehat{M}_{0,t}$.  
Selecting the sub-optimal arms can lead to unstable variance during statistical inference, as samples near the decision boundary can contribute either $\asymp 1/T^2$ or $\asymp t^{\gamma}/T^2$ to the variance of $\inp{\widehat{M}_1}{Q}$ and $\inp{\widehat{M}_0}{Q}$, depending on the decaying propensity, even as $t\to\infty$. To ensure stable convergence of variance, we assume that the variance contributed by samples in $\Omega_{\emptyset}$ is negligible compared to those from $\Omega_0$ and $\Omega_1$. However, these samples do not impact policy learning, as they still facilitate the estimation of both  $M_1$ and $M_0$ despite potential selection of sub-optimal arms. Consequently, this assumption is not required in the earlier sections.

The following theorem instates the central limit theorem for $\langle \hat M_1, Q\rangle$ and $\langle \hat M_0, Q\rangle$. Here $\|Q\|_{\ell_1}$ denotes the vectorized $\ell_1$-norm, i.e., $\|Q\|_{\ell_1}=\sum_{X\in\calX} |\langle Q, X\rangle|$.

\begin{Theorem}\label{thm:CLT}
Suppose that the conditions in Theorem \ref{thm:MCB-conv} hold and Assumption~\ref{assump:arm-opt} holds with $\delta^2:=\delta_T^2:=C_1(\sigma_0^2+\sigma_1^2)(rd_1/T^{1-\gamma})\log^4d_1$ for some large constant $C_1>0$. Define $S_1^2=T^{-\gamma}\fro{\calP_{\Omega_1}\calP_{M_1}(Q)}^2+C_{\gamma}\fro{\calP_{\Omega_0}\calP_{M_1}(Q)}^2$ and $S_0^2=T^{-\gamma}\fro{\calP_{\Omega_0}\calP_{M_0}(Q)}^2+C_{\gamma}\fro{\calP_{\Omega_1}\calP_{M_0}(Q)}^2$, where $C_{\gamma}=2/(c_2(1+\gamma))$ and $c_2$ is defined in Theorem \ref{thm:MCB-conv}. Assume that for each $i=0,1$, 
\begin{align}\label{eq:CLT-cond}
\frac{r^2d_1\log^5d_1}{T^{1-\gamma}} + \frac{\sigma_i}{\lambda_{\min}}\sqrt{\frac{rd_1^2\log d_1}{T^{1-\gamma}}} \frac{\|Q\|_{\ell_1}}{S_i} + \frac{\sigma_i}{\lambda_{\min}}\sqrt{\frac{rd_1^2d_2\log d_1}{T^{1-\gamma}}} \rightarrow 0
\end{align}
as $d_1, d_2, T\to\infty$. Then, 
\begin{align}\label{eq:CLT}
    \frac{\langle \hat M_0, Q\rangle - \langle M_0, Q\rangle}{\sigma_0S_0 \sqrt{d_1d_2/T^{1-\gamma}}}\rightarrow N(0,1)\quad {\rm and}\quad 
    \frac{\langle \hat M_1, Q\rangle - \langle M_1, Q\rangle}{\sigma_1S_1 \sqrt{d_1d_2/T^{1-\gamma}}}\rightarrow N(0,1),
\end{align}
as $d_1,d_2, T\to\infty$. 
\end{Theorem}

Theorem~\ref{thm:CLT} shows that the variance of $\langle \hat M_1, Q\rangle$ consists of two components:
\begin{equation}\label{eq:var-two-terms}
\frac{\sigma_1^2d_1d_2}{T}\big\|\calP_{\Omega_1}\calP_{M_1}(Q)\big\|_{\rm F}^2+\frac{2\sigma_1^2d_1d_2}{c_2(1+\gamma)T^{1-\gamma}}\big\|\calP_{\Omega_0}\calP_{M_1}(Q)\big\|_{\rm F}^2,
\end{equation}
where the first term arises from exploitation with samples in $\Omega_1$, and the second from exploration with samples in $\Omega_0$. When the arm optimality sets are balanced, such that $\|\calP_{\Omega_1}\calP_{M_1}(Q)\|_{\rm F}$ and $\|\calP_{\Omega_0}\calP_{M_1}(Q)\|_{\rm F}$ are comparable, the variance is primarily determined by the second term. Consequently, the error rate of $\langle \hat M_1, Q\rangle$ decreases at the order $O\big(T^{-(1-\gamma)/2}\big)$, aligning with the convergence rate of the online estimators produced by Algorithm~\ref{alg:mcb1}.

Our method for policy inference only requires that $\|\widehat{M}_{i,t}-M_i\|_{\max}^2=o_p(\sigma_i^2)$ for $i\in\{0,1\}$. This condition is satisfied at time $T_0$ with a constant exploration probability during phase one, as established in Theorem~\ref{thm:MCB-conv}. For the learning phase after $T_0$, the exploration probability can follow any schedule $\varepsilon_t\asymp t^{-\gamma}$ with $\gamma\in[0,1)$, including a constant $\varepsilon_t$. 
Choosing an appropriate $\gamma$ to ensure the decay of $\varepsilon_t$ is crucial for achieving nontrivial regret while maintaining inference efficiency. Specifically, a constant $\varepsilon_t$ yields the most efficient, asymptotically normal estimator with variance $O(1/T)$ but results in a regret of $O(T)$, highlighting the trade-off between minimizing regret and optimizing inference efficiency. This trade-off is further supported by \cite{simchi2023multi}, which establishes a minimax lower bound indicating that the product of inference error and the square root of regret remains a constant order in the worst case: with regret $O(T^{1-\gamma})$, the estimation error will be $O(T^{-(1-\gamma)/2})$. To achieve a balance, a decaying $\varepsilon_t$ with $\gamma=1/3$ offers a regret upper bound of $O(T^{2/3})$ and a policy inference variance of $O(T^{-2/3})$.

\subsection{Policy inference}

The asymptotic normality established in Theorem~\ref{thm:CLT} depends on some unknown quantities such as $\sigma_0, \sigma_1, \|\mathcal{P}_{\Omega_1}\calP_{M_0}(Q)\|_{\rm F}, \|\mathcal{P}_{\Omega_0}\calP_{M_0}(Q)\|_{\rm F}, \|\mathcal{P}_{\Omega_1}\calP_{M_1}(Q)\|_{\rm F}$, and $\|\mathcal{P}_{\Omega_0}\calP_{M_1}(Q)\|_{\rm F}$. These quantities must be estimated for the purposes of constructing confidence intervals and testing hypothesis (\ref{eq:test-H}). Towards that end, we propose the following estimators:
\begin{align*}
    &\hat\sigma_i^2:=\frac{1}{T-T_0}\sum_{t=T_0+1}^T \frac{\mathbbm{1}(a_t=i)}{i\pi_t+(1-i)(1-\pi_t)} \big(r_t-\langle \hat M_{i, t-1}, X_t\rangle\big)^2 \\
    &\widehat{S}_i^2:= \left(1/T^{\gamma}\fro{\calP_{\widehat{\Omega}_{i,T}}\calP_{\hat M_i}(Q)}^2+C_{\gamma}\fro{\calP_{\widehat{\Omega}_{1-i,T}}\calP_{\hat M_i}(Q)}^2\right)b_T.
\end{align*}
where $\widehat{\Omega}_{i,T}=\{X\in \mathcal{X}: \inp{\widehat{M}_{i,T}-\widehat{M}_{1-i,T}}{X}>0\}$ for $i=0,1$ and $b_T=T/(T-T_0)$. A larger $T_0$ would result in slightly wider confidence intervals; however, as $T$ approaches infinity, its impact becomes negligible since $b_T$ converges to one. In the supplementary materials, we show that all these estimators are consistent, and the asymptotic normality still holds when replacing the standard deviations by its data-driven estimates. 

\begin{Theorem}
\label{thm:studentized-CLT}
Suppose that the conditions in Theorem~\ref{thm:CLT} hold. Then,  
\begin{align*}
    \frac{\langle \hat M_0, Q\rangle - \langle M_0, Q\rangle}{\hat\sigma_0 \widehat{S}_0 \sqrt{d_1d_2/T^{1-\gamma}}}\rightarrow N(0,1)\quad {\rm and}\quad 
    \frac{\langle \hat M_1, Q\rangle - \langle M_1, Q\rangle}{\hat\sigma_1 \widehat{S}_1 \sqrt{d_1d_2/T^{1-\gamma}}}\rightarrow N(0,1),
\end{align*}
as $d_1,d _2, T\to\infty$.  
\end{Theorem}

Finally,  we need to study the distribution of the point estimator $\langle \hat M_1-\hat M_0, Q\rangle=\langle \hat M_1, Q\rangle-\langle \hat M_0, Q\rangle$ and its studentized version to test between the two hypotheses in (\ref{eq:test-H}). A simple fact is that $\langle \hat M_1, Q\rangle$ and $\langle \hat M_0, Q\rangle$ are uncorrelated, and the following asymptotic normality holds. This immediately enables the classic Z-test and the construction of confidence intervals.

\begin{Corollary}\label{cor:m1-m0}
Suppose that the conditions in Theorem~\ref{thm:CLT} hold. Then, 
 \begin{align*}
      \frac{\langle \hat M_1 -\hat M_0, Q\rangle - \langle M_1-M_0, Q\rangle}{\sqrt{\big(\hat\sigma_0^2\widehat{S}_0^2+\hat\sigma_1^2\widehat{S}_1^2 \big)d_1d_2/T^{1-\gamma}}} \longrightarrow N(0,1),
 \end{align*}
as $d_1,d _2, T\to\infty$. 
\end{Corollary}

\section{Numerical Experiments}\label{sec:simulation}
In this section, we present several numerical studies to evaluate the practical performance of our methods.  Due to space constraints,  additional simulation results are provided in Appendix~\ref{app:numerical}.  

\subsection{Statistical inference}
We evaluate the performance of online inference using true matrices $M_1$ and $M_0$, each with dimensions $d_1=d_2=300$ and rank $r=2$. Matrix $M_1$ is generated by applying SVD to a random matrix with entries uniformly drawn from $U(-100,100)$. To create $M_0$, we add a perturbation matrix with entries from $U(-2,2)$ to the aforementioned matrix and then apply SVD to obtain a low-rank matrix. The largest and smallest singular values are approximately $\lambda_{\max}=1981.3$ and $\lambda_{\min}=1950.5$, respectively. The noise levels are set to $\sigma_1=\sigma_0=1$. Initial estimates $\widehat{M}_{1,0}$ and $\widehat{M}_{0,0}$ are obtained by applying the Soft-Impute algorithm \citep{mazumder2010spectral} from randomly sampled offline data.

\emph{Different decaying exploration probabilities.} We evaluate the inference procedure across four scenarios: (1) $\inp{M_1}{e_1e_5^{\top}}$, (2) $\inp{M_0}{e_1e_5^{\top}}$, (3)$\inp{M_0-M_1}{e_1e_5^{\top}}$  and (4) $\inp{M_0}{e_1e_5^{\top} - e_2e_2^{\top}}$. For each scenario, we perform 1,000 independent runs under two settings: $\gamma=\frac{1}{3}$ and $\gamma=\frac{1}{4}$, with a horizon $T=60,000$ and fixed $T_0=20,000$. During phase $\uppercase\expandafter{\romannumeral1}$, the exploration parameter is set to $\varepsilon=0.6$, and in phase $\uppercase\expandafter{\romannumeral2}$, it decays as $\varepsilon_t=10t^{-\gamma}$. The stepsize is $\eta=0.025d_1d_2\log d_1/(T^{1-\gamma}\lambda_{\max})$ for phase $\uppercase\expandafter{\romannumeral1}$ and $\eta_t=\varepsilon_t\eta$ for phase $\uppercase\expandafter{\romannumeral2}$.
Figure~\ref{gamma4} presents the density histograms of $(\inp{\widehat{M}_i}{Q} - \inp{M_i}{Q})/\widehat{\sigma}_i\widehat{S}_i\sqrt{d_1d_2/T^{1-\gamma}}$ for scenarios (1),(2),(4), and $((\inp{\widehat{M}_0}{Q} - \inp{\widehat{M}_1}{Q}) - (\inp{M_0}{Q} - \inp{M_1}{Q}))/\sqrt{(\widehat{\sigma}_1^2\widehat{S}_1^2+\widehat{\sigma}_0^2\widehat{S}_0^2)d_1d_2/T^{1-\gamma}}$ for scenario (3) under $\gamma=1/3$ (top four plots) and $\gamma=1/4$ (bottom four plots). The empirical distributions closely match the standard normal distribution (red curve).

\begin{figure} [h]
	\centering
	\begin{subfigure} {0.24\linewidth} 
		\includegraphics[width=1.1\linewidth]{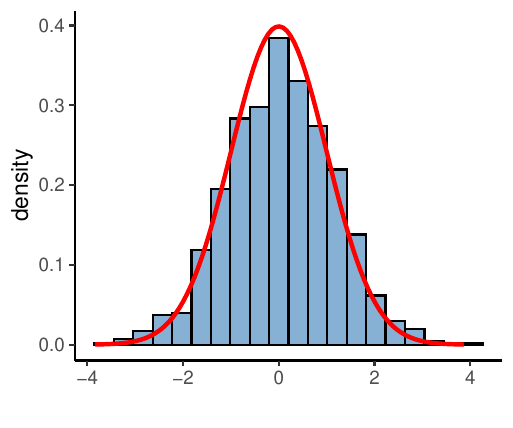}
        \caption{$\inp{M_1}{e_1e_5^{\top}}$}
    \end{subfigure}
    \begin{subfigure} {0.24\linewidth} 
		\includegraphics[width=1.1\linewidth]{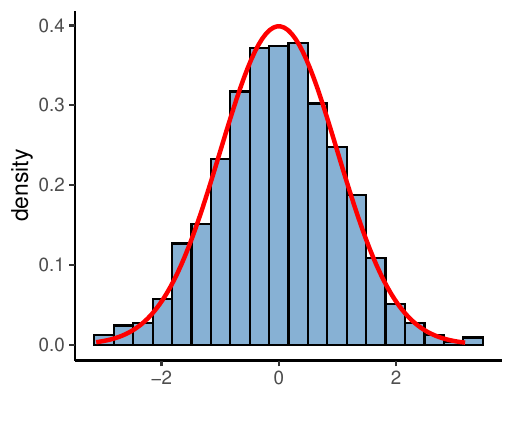}
        \caption{$\inp{M_0}{e_1e_5^{\top}}$}
    \end{subfigure}
    \begin{subfigure} {0.24\linewidth} 
		\includegraphics[width=1.1\linewidth]{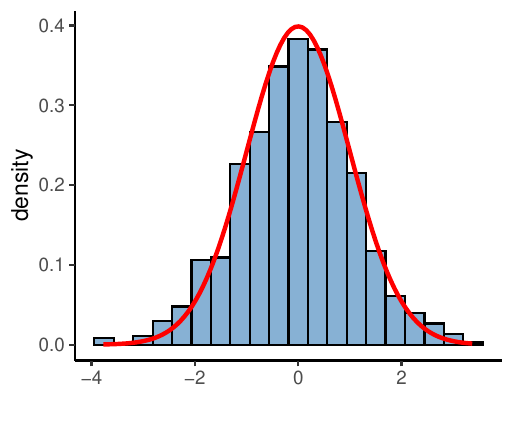}
        \caption{$\inp{M_0-M_1}{e_1e_5^{\top}}$}
    \end{subfigure}
    \begin{subfigure} {0.24\linewidth} 
		\includegraphics[width=1.1\linewidth]{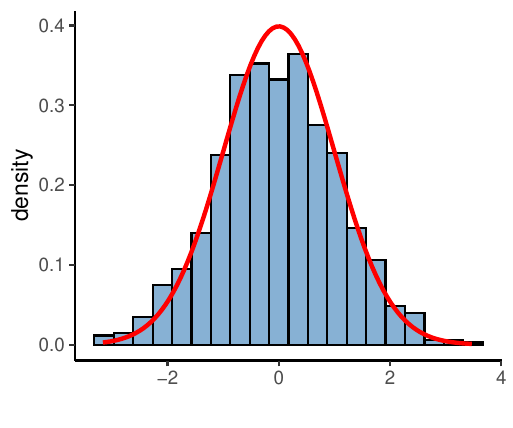}
        \caption{$\inp{M_0}{e_1e_5^{\top} - e_2e_2^{\top}}$}
    \end{subfigure}
	
	\begin{subfigure} {0.24\linewidth} 
		\includegraphics[width=1.1\linewidth]{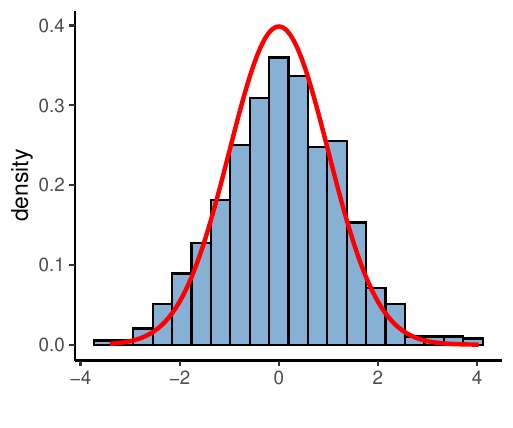}
        \caption{$\inp{M_1}{e_1e_5^{\top}}$}
    \end{subfigure}
    \begin{subfigure} {0.24\linewidth} 
		\includegraphics[width=1.1\linewidth]{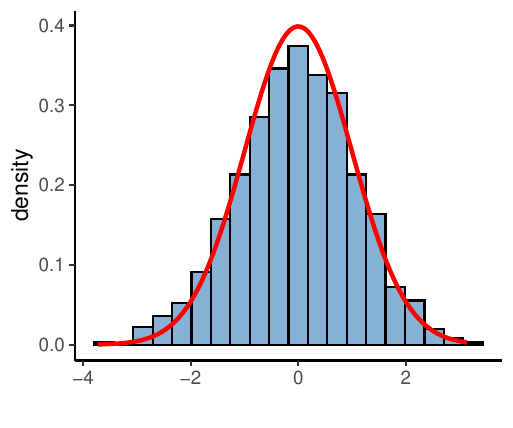}
        \caption{$\inp{M_0}{e_1e_5^{\top}}$}
    \end{subfigure}
    \begin{subfigure} {0.24\linewidth} 
		\includegraphics[width=1.1\linewidth]{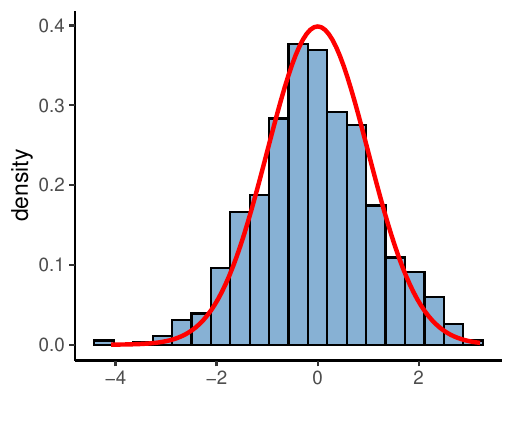}
        \caption{$\inp{M_0-M_1}{e_1e_5^{\top}}$}
    \end{subfigure}
    \begin{subfigure} {0.24\linewidth} 
		\includegraphics[width=1.1\linewidth]{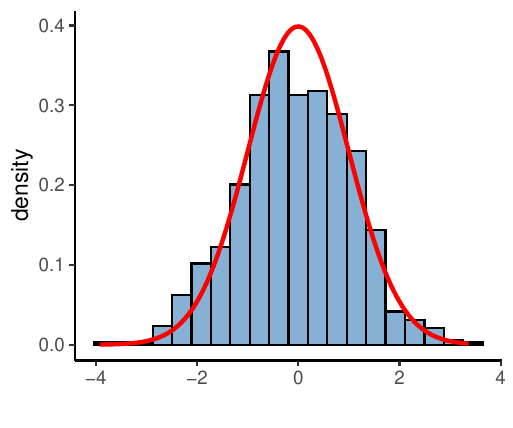}
        \caption{$\inp{M_0}{e_1e_5^{\top} - e_2e_2^{\top}}$}
    \end{subfigure}
	
    \caption{Empirical distributions under $\gamma=1/3$ (top) and $\gamma=1/4$ (bottom) with $T=60,000$ and $T_0=20,000$. The red curve represents the p.d.f. of standard normal distributions.}
    \label{gamma4} 
\end{figure}

\emph{Impact of phase I duration.} We assess the effect of phase I duration by setting $\gamma=1/3$  and a time horizon  $T=30,000$. Phase I lengths were fixed at either $T_0=8,000$ or $T_0=5,000$. As shown in Figure~\ref{T05}, both result in accurate normal approximation. 

\begin{figure} [h]
	\centering
	\begin{subfigure} {0.24\linewidth} 
		\includegraphics[width=1.1\linewidth]{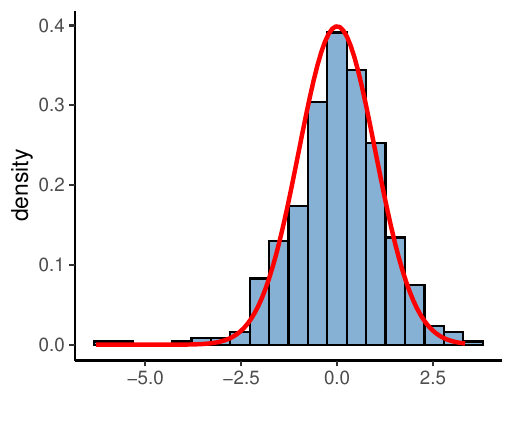}
        \caption{$\inp{M_1}{e_1e_5^{\top}}$}
    \end{subfigure}
    \begin{subfigure} {0.24\linewidth} 
		\includegraphics[width=1.1\linewidth]{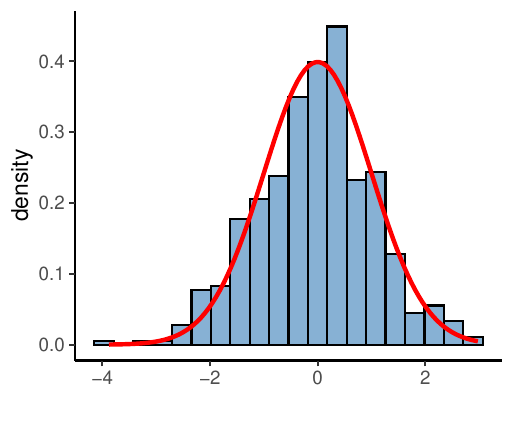}
        \caption{$\inp{M_0}{e_1e_5^{\top}}$}
    \end{subfigure}
    \begin{subfigure} {0.24\linewidth} 
		\includegraphics[width=1.1\linewidth]{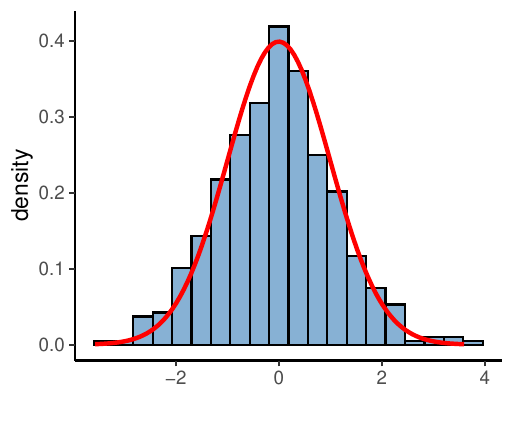}
        \caption{$\inp{M_0-M_1}{e_1e_5^{\top}}$}
    \end{subfigure}
    \begin{subfigure} {0.24\linewidth} 
		\includegraphics[width=1.1\linewidth]{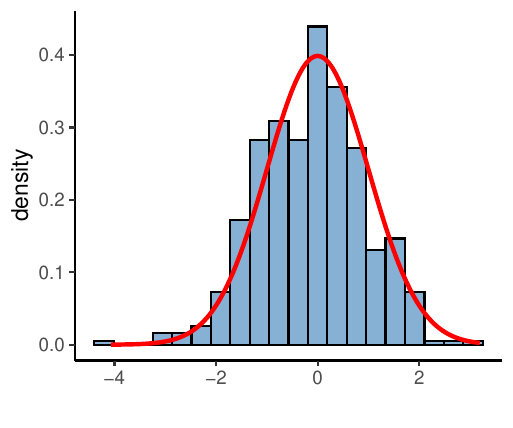}
        \caption{$\inp{M_0}{e_1e_5^{\top} - e_2e_2^{\top}}$}
    \end{subfigure}

	\centering
	\begin{subfigure} {0.24\linewidth} 
		\includegraphics[width=1.1\linewidth]{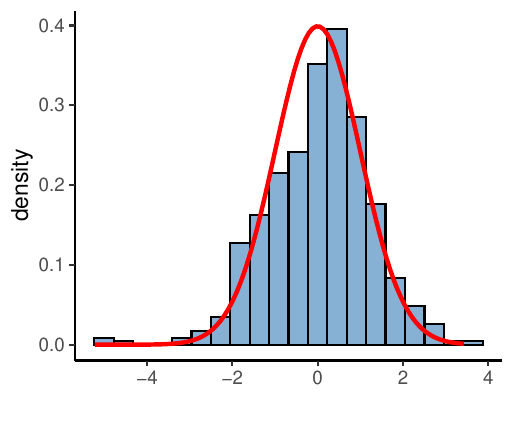}
        \caption{$\inp{M_1}{e_1e_5^{\top}}$}
    \end{subfigure}
    \begin{subfigure} {0.24\linewidth} 
		\includegraphics[width=1.1\linewidth]{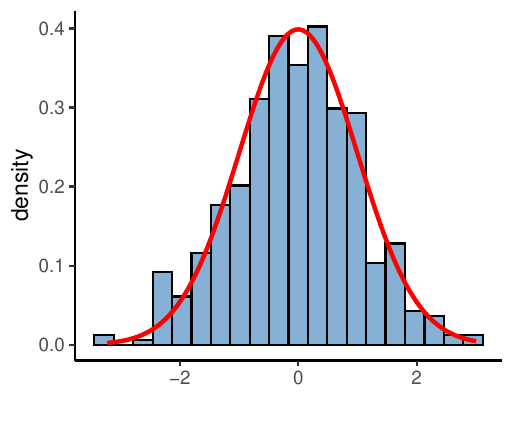}
        \caption{$\inp{M_0}{e_1e_5^{\top}}$}
    \end{subfigure}
    \begin{subfigure} {0.24\linewidth} 
		\includegraphics[width=1.1\linewidth]{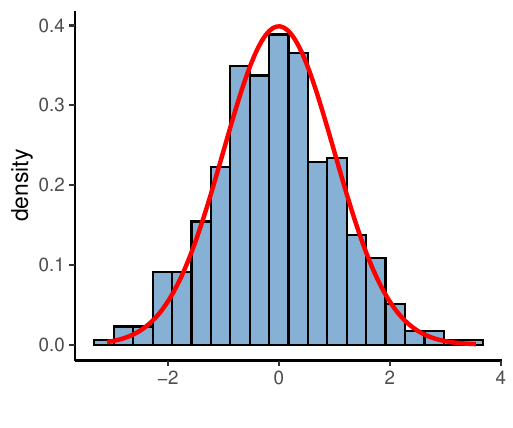}
        \caption{$\inp{M_0-M_1}{e_1e_5^{\top}}$}
    \end{subfigure}
    \begin{subfigure} {0.24\linewidth} 
		\includegraphics[width=1.1\linewidth]{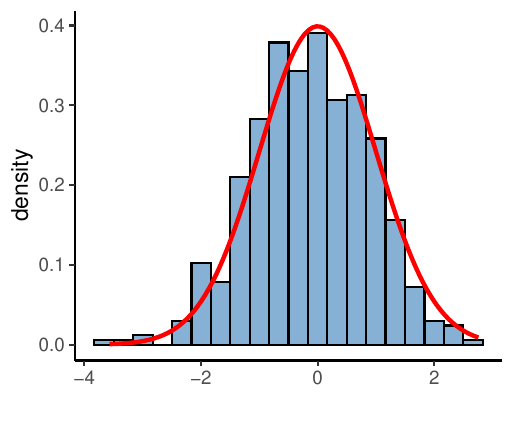}
        \caption{$\inp{M_0}{e_1e_5^{\top} - e_2e_2^{\top}}$}
    \end{subfigure}
	
    \caption{Empirical distributions under $\gamma=1/3$ with $T=30,000$, $T_0=8,000$ (top), and $T_0=5,000$ (bottom). The red curve represents the p.d.f. of standard normal distributions.}
    \label{T05} 
\end{figure}

\subsection{Regret analysis}
Next, we assess the regret of the matrix completion bandit algorithm under identical settings. For $\gamma=1/3$, the time horizon $T$ ranges from 40,000 to 80,000 in increments of 5,000, with a fixed $T_0=13.5T^{1-\gamma}$. Similarly, for $\gamma=\frac{1}{4}$, $T$ varies from 20,000 to 60,000 in increments of 5,000, and $T_0=4.5T^{1-\gamma}$ for each $T$. We conduct 100 simulation trials and plot the average cumulative regret against $T^{2/3}$ when $\gamma=1/3$ and $T^{3/4}$ for $\gamma=1/4$. As shown in Figure~\ref{regretplot}, the cumulative regret exhibits linear scaling, following an $O(T^{2/3})$ trend for $\gamma=1/3$ and an $O(T^{3/4})$ trend for $\gamma=1/4$. These results align with our theoretical findings in Theorem~\ref{thm:regret}.
\begin{figure} [h]
	\centering
		\includegraphics[scale=0.82]{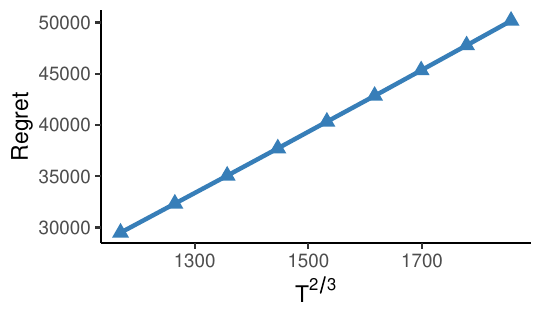}
		\includegraphics[scale=0.82]{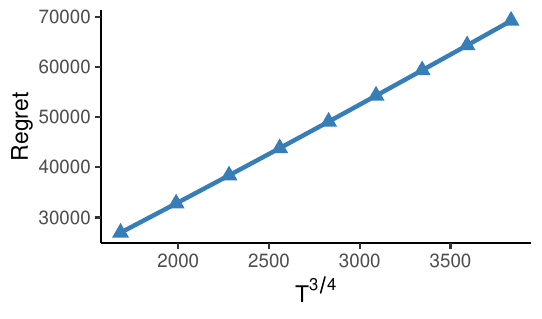}
    \caption{Average empirical cumulative regret against $T^{2/3}$ for $\gamma=\frac{1}{3}$ (left) and $T^{3/4}$ for $\gamma=\frac{1}{4}$ (right) under 100 trials.}
    \label{regretplot} 
\end{figure}

\section{Real Data Analysis}
\subsection{\emph{SFpark} parking pricing project} \label{data:sfpark}

Starting in April 2011, San Francisco Municipal Transportation Agency (SFMTA) adopted the \emph{SFpark} pilot project, which marked the largest pricing reform for on-street parking in San Francisco. The objective of this project was to effectively manage parking towards availability targets. Prior to that, parking meters in San Francisco charged a single hourly price regardless of the time of day or location within a zone. Consequently, certain blocks experienced severe overcrowding especially during peak hours, while others remained underutilized. Recognizing this issue, extensive literature  has identified the on-street parking pricing as the dominant factor that determines parking behavior (\cite{aljohani2021survey}). In general, higher prices encourage a few drivers to move away from the most crowded blocks, while lower prices attract the drivers to the underoccupied blocks. 

\emph{SFpark} thus allocates parking spaces more efficiently than uniform prices can by adjusting the price varied by time of day, day of week and from block to block. The occupancy rate of a block of a given hour is 
defined to be  divided by 3600 seconds multiplied by the number of spaces in the block. SFMTA established the desired target occupancy rate for every block at each hour at between 60\% and 80\%. Numerous previous studies have evaluated the effectiveness of \emph{SFpark} (\cite{pierce2013getting,millard2014curb}). We would like to apply our matrix completion bandit algorithm on \emph{SFpark} data, and determine the optimal parking prices for each block in a specific hour, with the aim of achieving the target occupancy rates for a greater number of blocks and time periods.  

We collected the data from \cite{sfpark_evaluation}. The dataset includes the hourly occupancy rate and price for each block at every hour, covering seven pilot zones and two control zones, across each day throughout the study period from April 2011 to July 2013. The pilot zones implemented a new pricing policy and involved significant data collection, while the control zones had no new policy but still underwent substantial data collection. During the study period, SFMTA executed ten demand-responsive parking price adjustments in the pilot zones, where price for each block at every hour was adjusted gradually and periodically according to the following rules: if the average occupancy rate after last adjustment was between 60\% and 80\%, the hourly price would remain unchanged; if the average occupancy rate was larger than 80\%, the hourly price would be raised; if the occupancy rate was less than 60\%, the hourly price would be lowered. In the following analysis, we will focus on the data from Downtown, which is one of the designated pilot zones, covering the time period from April 1st, 2011, to March 27th, 2012, as an illustrative example. We choose to include only one year of data because the occupancy rate after 2012 can be influenced by confounding factors such as changes in garage policies and overall city traffic developments. During the selected time period, a total of four pricing adjustments occurred on the following dates: August 1st, 2011; October 11th, 2011; December 13th, 2011; and February 14th, 2012. These adjustments allow us to divide the entire time period into five distinct periods for further analysis. 

Most parking meters in Downtown are operational from 7am to 6pm, and the traffic conditions differ between weekdays and weekends. To capture this variation, we define $d_2=22$, representing each hour from 7 am to 6 pm on both weekdays and weekends. Here the action 1 corresponds to a high-price policy, where the hourly price is larger than or equal to \$3.5. Conversely, action 0 represents a low-price policy, with an hourly price below \$3.5. We identified $d_1=34$ blocks that had at least one hour for both actions during the entire time period. The underlying matrices $M_1$ and $M_0$ represent the expected deviation of the real occupancy rate from the target range $[0.6, 0.8]$ across different blocks and hours, under high price and low price, respectively. For instance, when making a decision for block $j_1\in[d_1]$ at hour $j_2\in[d_1]$, the contextual matrix can be represented as $X_t=e_{j_1}e_{j_2}^{\top}$. If we further choose $a_t=1$ for the high price and observe an occupancy rate of 0.7, the reward $r_t$ would be calculated as $r_t=0$; if the actual occupancy rate is 0.5 or 0.9, then $r_t=-0.1$. To simulate the online data collection procedure, we evaluate each observation by comparing the decision made by our algorithm with the actual observed action. If they align, we retain the sample to fit the model and update the online estimators. Otherwise, we discard it and proceed to the next sample. Each hour for every block on each day throughout the entire time period is treated as an observation, resulting in a time horizon $T=105,825$. Given that we have a nearly equal amount of data for each block and hour, we can consider the context $X_t$ to be uniformly distributed, meeting our requirements. We obtain the initial estimates $\widehat{M}_{1,0}$ and $\widehat{M}_{0,0}$ by randomly selecting a subset of samples. Subsequently, scree plots were generated from these estimates to determine the optimal low rank, resulting in $r=5$, as illustrated in Figure~\ref{screesfpark}.

\begin{figure} [h]
	\centering
		\includegraphics[scale=0.7]{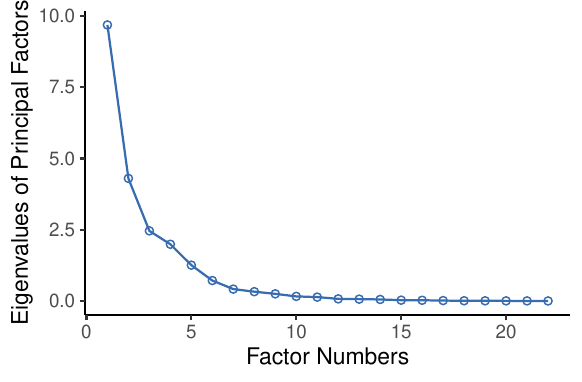}
		\includegraphics[scale=0.7]{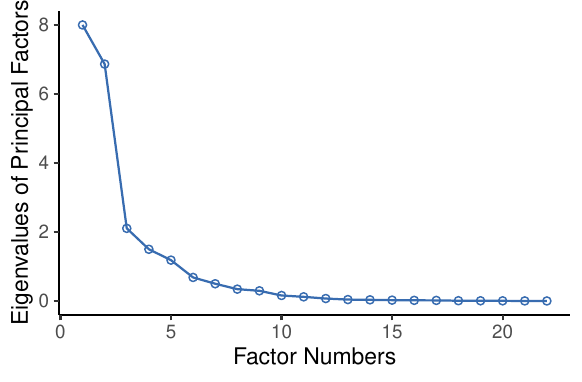}
    \caption{\emph{SFpark}: Scree plots of eigenvalues against factor numbers for $\widehat{M}_{1,0}$ (left) and $\widehat{M}_{0,0}$ (right).}
    \label{screesfpark}
\end{figure}

We applied our learning and inference strategy to two representative blocks: 02ND ST 200 on weekdays and BATTERY ST 400 on weekends, to determine optimal hourly pricing and assess its impact on occupancy rates. Figure~\ref{periodplot} illustrates the average hourly occupancy rates across different periods under the original \emph{SFpark} policy. For comparison, we employed an offline method using all available data to obtain initial estimates for $M_1$ and $M_0$ via the SoftImpute algorithm \citep{mazumder2010spectral}, followed by the debiasing and inference technique from \cite{ma2023multiple}. The results are summarized in Tables~\ref{No2} and \ref{No3}.

\begin{table} [h]
    \caption{\label{No2} The p-values at 95\% confidence level of $H_0:\inp{M_0-M_1}{e_{j_1}e_{j_2}^{\top}}\leq 0 \  vs\  H_1:\inp{M_0-M_1}{e_{j_1}e_{j_2}^{\top}}> 0$ (top three rows) or $H_0:\inp{M_1-M_0}{e_{j_1}e_{j_2}^{\top}}\leq 0 \  vs\  H_1:\inp{M_1-M_0}{e_{j_1}e_{j_2}^{\top}}> 0$ (bottom three rows), where $j_1$ denotes block 02ND ST 200 and $j_2$ denotes weekday 7:00 to 11:00 (top three rows) or  12:00 to 17:00 (bottom three rows).}
    \centering
    \begin{tabular}{ccccccc} 
        \toprule 
        \textbf{Hour} & 7 & 8 & 9 & 10 & 11 \\  
        \midrule 
        p-value(MCBantit) & 0.264 & $<$0.001 & 0.001 & 0.002 & 0.410 \\
        p-value(Offline) & $<$0.001 & $<$0.001 & $<$0.001 & $<$0.001 & $<$0.001 \\
        \hline\hline
        \textbf{Hour} & 12 & 13 & 14 & 15 & 16 & 17 \\  
        \midrule 
        p-value(MCBantit) & $<$0.001 & $<$0.001 & $<$0.001 & 0.783 & $<$0.001 & 0.983 \\
        p-value(Offline) & 0.999 & 0.999 & 0.999 & 0.265 & 0.035 & $<$0.001 \\
        \bottomrule
    \end{tabular}%
\end{table}

\begin{table}  [h]
    \caption{\label{No3} The p-values at 95\% confidence level of $H_0:\inp{M_0-M_1}{e_{j_1}e_{j_2}^{\top}}\leq 0 \quad vs \quad H_1:\inp{M_0-M_1}{e_{j_1}e_{j_2}^{\top}}> 0$, where $j_1$ denotes block BATTERY ST 400 and $j_2$ denotes weekend 7:00 to 11:00 (top three rows) or 12:00 to 17:00 (bottom three rows).}
    \centering
     \begin{tabular}{ccccccc}  
        \toprule 
        \textbf{Hour} & 7 & 8 & 9 & 10 & 11 \\  
        \midrule 
        p-value(MCBantit) & 0.002 & $<$0.001 & $<$0.001 & 0.001 & $<$0.001 \\
        p-value(Offline) & 0.050 & 0.001 & $<$0.001 & $<$0.001 & $<$0.001 \\
        \hline\hline
        \textbf{Hour} & 12 & 13 & 14 & 15 & 16 & 17 \\  
        \midrule 
        p-value(MCBantit) & $<$0.001 & 0.006 & $<$0.001 & $<$0.001 & $<$0.001 & $<$0.001 \\
        p-value(Offline) & $<$0.001 & $<$0.001 & $<$0.001 & $<$0.001 & $<$0.001 & 0.006 \\
        \bottomrule
    \end{tabular}%
\end{table}

At 02ND ST 200, occupancy rates initially exceeded 80\%, particularly after 12 pm. The \emph{SFpark} policy maintained high prices throughout five periods, effectively reducing occupancy rates. However, early morning rates consistently remained below 60\%, as shown in Figure~\ref{periodplot}. Table~\ref{No2} reveals that both our matrix completion bandit method and the offline approach identified action 0 (low pricing) as significant during these hours, suggesting that lower prices before noon could optimize parking space utilization. For hours after noon, our results differ from the offline method for some hours. However,
we consider our results more reliable, as Figure \ref{periodplot} demonstrates that the high price adjustment
successfully reduced the occupancy rate from over 80\% to nearly 60\% from period 1 to 5.

Conversely, BATTERY ST 400 exhibited consistently low occupancy rates below 60\% on weekends. Implementing low prices starting in period 2 led to increasing occupancy rates from periods 2 to 5, as depicted in Figure~\ref{periodplot}. Table~\ref{No3} confirms that both our method and the offline approach found action 0 statistically significant for all weekend hours, validating the effectiveness of low-price adjustments in enhancing parking space utilization.

\begin{figure} [h]
	\centering
		\includegraphics[scale=0.78]{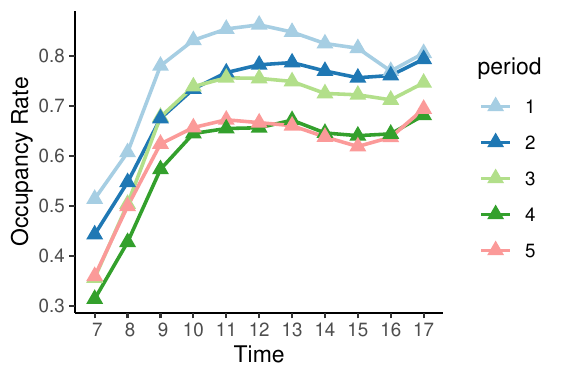}
		\includegraphics[scale=0.78]{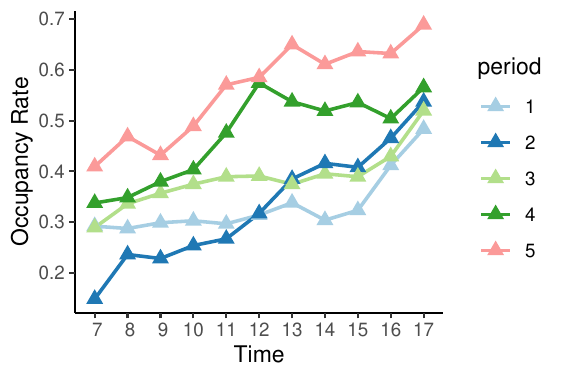}
    \caption{Average occupancy rates of parking space in different time of day during 5 periods. Left: weekday of 02ND ST 200; right: weekend of BATTERY ST 400.}
    \label{periodplot}
\end{figure}

Finally, we evaluate the overall performance of matrix completion bandit. This time,
to generate the online data, we follow the same method as described earlier, keeping the
samples where our decision coincides with the actual action. However, in cases where the
actual action contradicts our decision, we utilize the occupancy rate of the current block and time with the matching action from the nearest date to fit the model, instead of simply dropping the sample. We only discard a sample when no other samples are available for the same actions at the given block and time.
For comparison, we apply both the original \emph{SFpark} policy and the machine learning pricing strategy \citep{simhon2017smart}, referred to as the \emph{neighborhood approach}. The latter accounts for the impact of price adjustments in one block on the demand in adjacent blocks by optimizing prices across all blocks simultaneously. This approach aims to minimize the root mean square error (RMSE) between predicted and target occupancy rates for all blocks.

Figure \ref{compare} displays the percentage of all the selected time and blocks that achieves the target occupancy rate in each period under \emph{SFpark} policy, our matrix completion bandit and the neighborhood approach. Across all time periods, there is an overall improvement in achieving the target occupancy rate by all the methods, and our matrix completion bandit almost consistently outperforms the other two methods. The original \emph{SFpark} policy explored too much and may adopt sub-optimal policies in the beginning, sometimes maintaining unsuitable prices for more than one period. In contrast, the matrix completion bandit could adjust the price as soon as it detects that the current policy is not optimal enough.

\begin{figure} [h]
    \centering
    \includegraphics[scale=0.9]{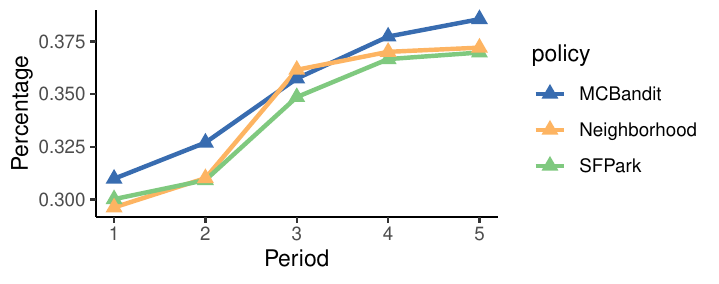}
     \caption{Comparison of percentage of blocks and times achieving target occupancy rate over periods by matrix completion bandit, \emph{SFpark} policy, and neighborhood approach. }
     \label{compare}
\end{figure}

\subsection{Superstore discount strategy}\label{data:discount}
Now we apply our matrix completion bandit approach to another dataset related to superstore discount strategies. The data is collected from \cite{vivek468_superstore}. It originates from a superstore where numerous customers order various products from the website daily. When a customer orders a product, the superstore decides whether to offer a discount. Providing a discount reduces the unit profit but might increase the quantity ordered. Conversely, not offering a discount results in higher unit profit but may dampen the customer's interest in the product. Therefore, the objective is to determine the appropriate discount for each product and customer to maximize overall profit.

After data cleaning—such as removing products that receive excessive discounts leading to negative profit—the dataset comprises $d_1=790$ customers and $d_2=1196$ products, with a total of $T=6393$ orders. This presents a high-dimensional feature space where $d_1d_2\gg T$. The data includes information on discounts, quantities, sales prices, and profits for each order. We set arm 1 as offering a discount on the product to the customer, while arm 0 represents not providing a discount. Furthermore, we define the reward as the profit rate multiplied by the sales quantity, i.e., the profit divided by the total sales price times the sales quantity. Similarly as Section \ref{data:sfpark}, to simulate the online data collection process, for each observation, if the decision made by our algorithm matches the actual action, we retain the observation for model fitting; otherwise, we discard it and move to the next sample. 


Taking a customer who ordered a total of 7 products as an example, we construct the point estimators and confidence intervals for the expected difference in reward between offering a discount and not offering a discount for each of these products. To evaluate our bandit algorithm, we compare our results with offline matrix completion method, which involves first applying the SoftImpute method from \cite{mazumder2010spectral} to fit all the offline data, followed by the approach in \citep{ma2023multiple} to obtain the debiased estimator and construct the corresponding confidence intervals. As shown in Figure \ref{resultdiscount}, it is evident that for most products, our estimated values closely align with those from the offline estimator. The confidence intervals constructed by the matrix completion bandit estimator are slightly wider than those from the offline method due to the influence of IPW, as indicated by our theoretical findings.

Among the 7 products, products OFF-BI-10000088 and OFF-BI-10000545 show negative significance in the expected reward difference for both methods, suggesting that not offering a discount would generate more profit for these products with this customer. In the dataset, the actual actions for these two products also involved not offering a discount, resulting in reward profits of 0.98 and 6.5, respectively. For the remaining 5 products, neither method shows positive or negative significance, indicating that offering or not offering a discount would not lead to a significant difference in the final reward profit.

\begin{figure}[H]
    \centering
    \includegraphics[scale=0.2]{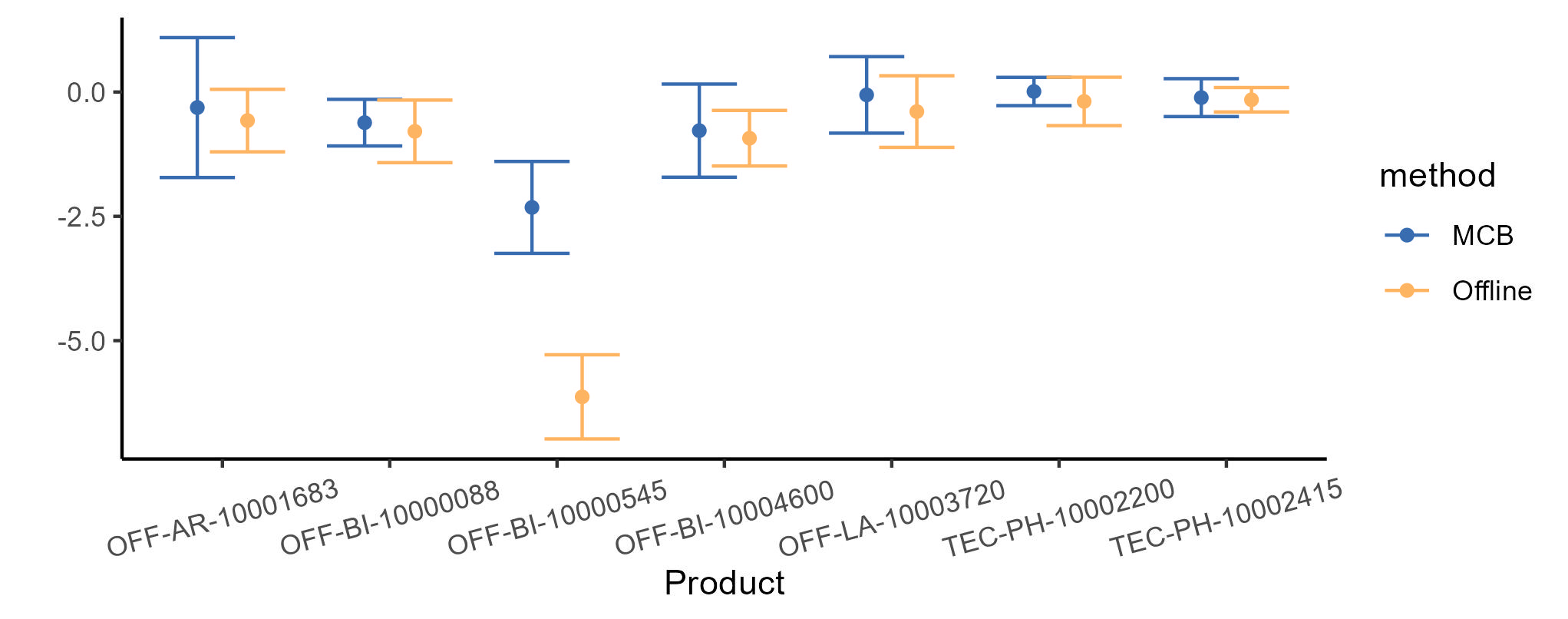}
     \caption{Superstore discount: point estimators and 95\% confidence intervals for the expected difference in reward between offering a discount and not offering a discount for a specific user and 7 products.}
     \label{resultdiscount}
\end{figure}

\section*{Acknowledgment}
Congyuan Duan and Dong Xia’s research was partially supported by Hong Kong RGC Grant GRF 16302323.

\section*{Supplementary Materials}
We provide additional numerical simulation results in Appendix A. The extension to $k$-armed bandits, algorithmic acceleration using efficient SVD implementation, adaptation to non-uniform sampling distributions, and discussions of the upper confidence bound (UCB) and Thompson sampling (TS) algorithms are presented in Appendix B. All proofs and technical lemmas are included in Appendices C and D.

\bibliography{reference}
\bibliographystyle{plainnat}


\appendix
\setcounter{figure}{0}
\renewcommand{\thefigure}{A.\arabic{figure}}
\newpage

\begin{center}
{\bf\LARGE Supplement to ``Online Policy Learning and Inference by Matrix Completion "}

\smallskip

Congyuan Duan$^1$, Jingyang Li$^2$, and Dong Xia$^1$

\smallskip

$^1$ Department of Mathematics, Hong Kong University of Science and Technology\\
$^2$ Department of Statistics, University of Michigan, Ann Arbor
\end{center}
\smallskip

\section{Additional Simulation Experiments}\label{app:numerical}

\subsection{Varying model and algorithmic parameters}
We investigate the effects of varying the matrix rank $r$, phase-I exploration probability $\varepsilon$, and noise levels $\sigma_1^2$ and $\sigma_0^2$. Unless otherwise specified, matrices are generated as described in Section~\ref{sec:simulation} with parameters $r=2$, $d_1=d_2=300$, $T=30000$, $T_0=10000$, $\gamma=1/3$ and $\sigma_1^2=\sigma_0^2=1$.For each scenario, we assess the asymptotic normality of: (1) $\inp{M_1}{e_1e_5^{\top}}$, (2) $\inp{M_0}{e_1e_5^{\top}}$, (3)$\inp{M_0-M_1}{e_1e_5^{\top}}$ and (4) $\inp{M_0}{e_1e_5^{\top} - e_2e_2^{\top}}$ across 500 independent runs. The results are presented in Figures~\ref{r10}-\ref{sigma05}.
\begin{figure} [H]
	\centering
	\begin{subfigure} {0.24\linewidth} 
		\includegraphics[width=1.1\linewidth]{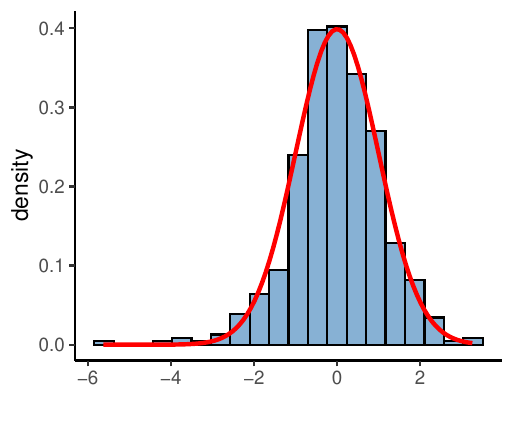}
        \caption{$\inp{M_1}{e_1e_5^{\top}}$}
    \end{subfigure}
    \begin{subfigure} {0.24\linewidth} 
		\includegraphics[width=1.1\linewidth]{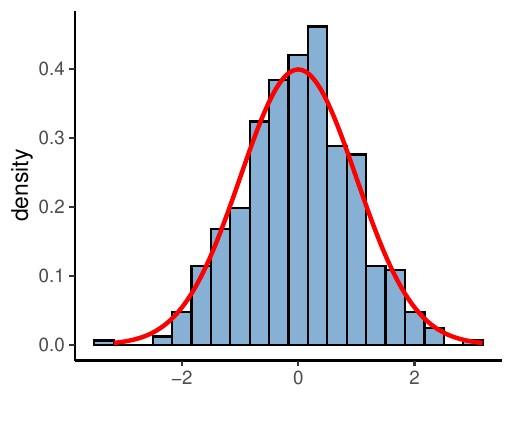}
        \caption{$\inp{M_0}{e_1e_5^{\top}}$}
    \end{subfigure}
    \begin{subfigure} {0.24\linewidth} 
		\includegraphics[width=1.1\linewidth]{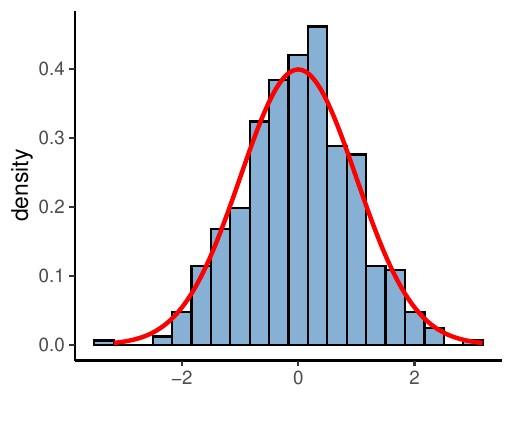}
        \caption{$\inp{M_0-M_1}{e_1e_5^{\top}}$}
    \end{subfigure}
    \begin{subfigure} {0.24\linewidth} 
		\includegraphics[width=1.1\linewidth]{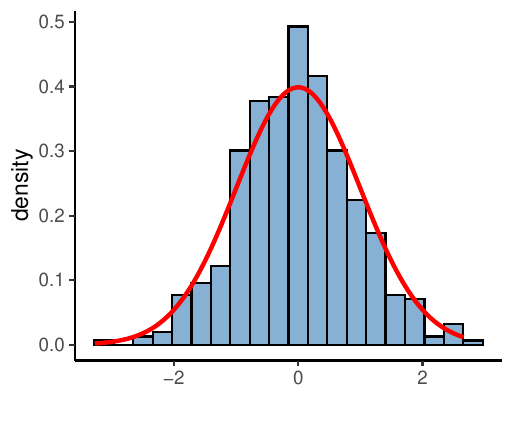}
        \caption{$\inp{M_0}{e_1e_5^{\top} - e_2e_2^{\top}}$}
    \end{subfigure}

	\begin{subfigure} {0.24\linewidth} 
		\includegraphics[width=1.1\linewidth]{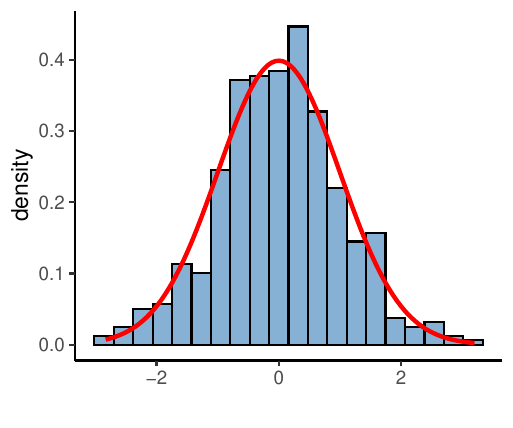}
        \caption{$\inp{M_1}{e_1e_5^{\top}}$}
    \end{subfigure}
    \begin{subfigure} {0.24\linewidth} 
		\includegraphics[width=1.1\linewidth]{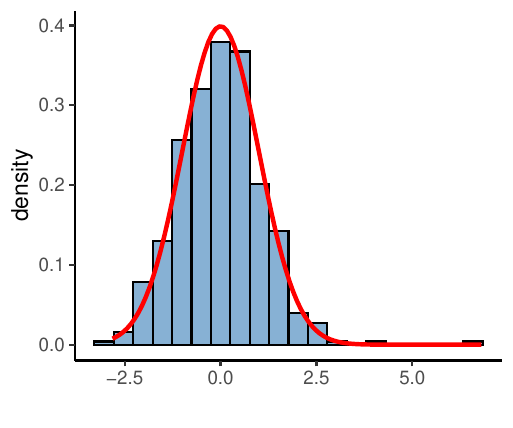}
        \caption{$\inp{M_0}{e_1e_5^{\top}}$}
    \end{subfigure}
    \begin{subfigure} {0.24\linewidth} 
		\includegraphics[width=1.1\linewidth]{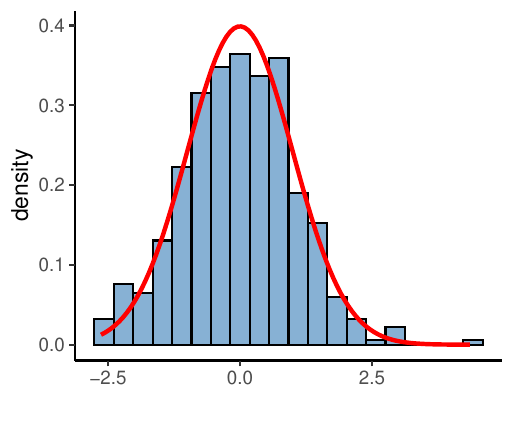}
        \caption{$\inp{M_0-M_1}{e_1e_5^{\top}}$}
    \end{subfigure}
    \begin{subfigure} {0.24\linewidth} 
		\includegraphics[width=1.1\linewidth]{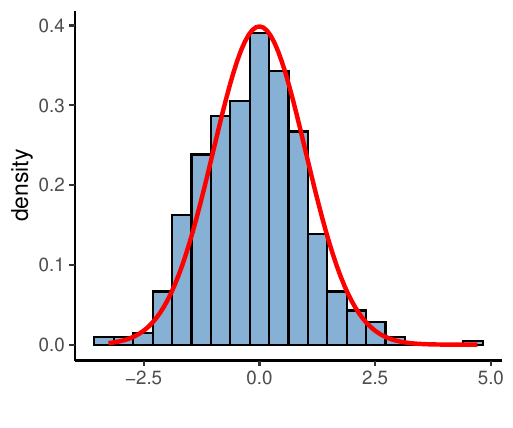}
        \caption{$\inp{M_0}{e_1e_5^{\top} - e_2e_2^{\top}}$}
    \end{subfigure}
	
    \caption{(\textbf{Varying rank $r$}) Empirical distributions under $r=5$ (top four plots) and $r=10$ (bottom four plots). The red curve represents the p.d.f. of standard normal distributions.}
    \label{r10} 
\end{figure}

\begin{figure} [H]
	\centering
	\begin{subfigure} {0.24\linewidth} 
		\includegraphics[width=1.1\linewidth]{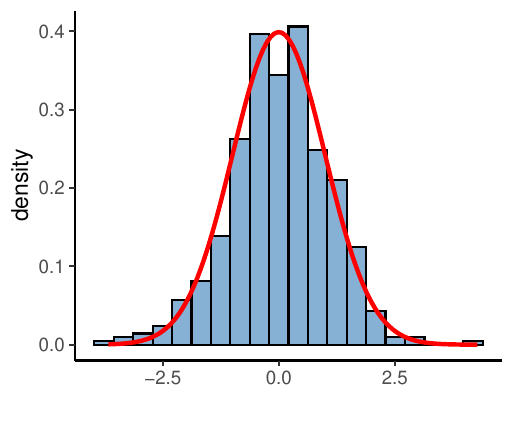}
        \caption{$\inp{M_1}{e_1e_5^{\top}}$}
    \end{subfigure}
    \begin{subfigure} {0.24\linewidth} 
		\includegraphics[width=1.1\linewidth]{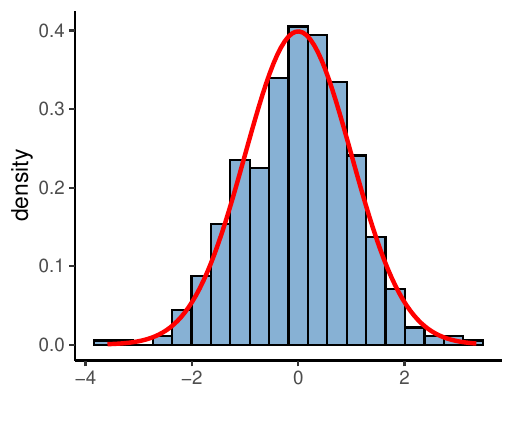}
        \caption{$\inp{M_0}{e_1e_5^{\top}}$}
    \end{subfigure}
    \begin{subfigure} {0.24\linewidth} 
		\includegraphics[width=1.1\linewidth]{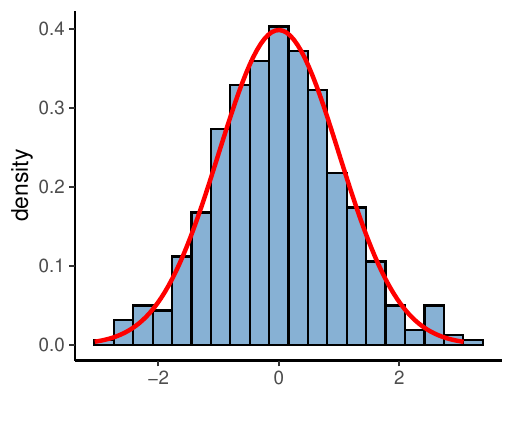}
        \caption{$\inp{M_0-M_1}{e_1e_5^{\top}}$}
    \end{subfigure}
    \begin{subfigure} {0.24\linewidth} 
		\includegraphics[width=1.1\linewidth]{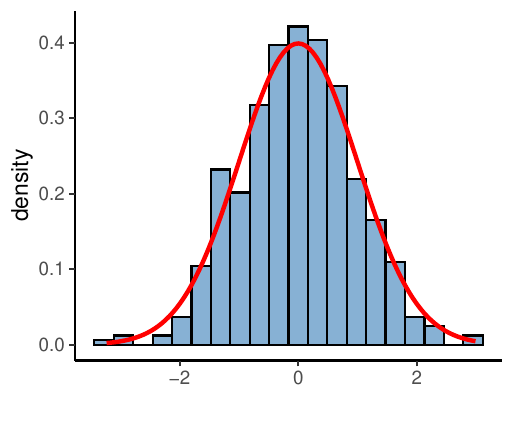}
        \caption{$\inp{M_0}{e_1e_5^{\top} - e_2e_2^{\top}}$}
    \end{subfigure}
	
	\begin{subfigure} {0.24\linewidth} 
		\includegraphics[width=1.1\linewidth]{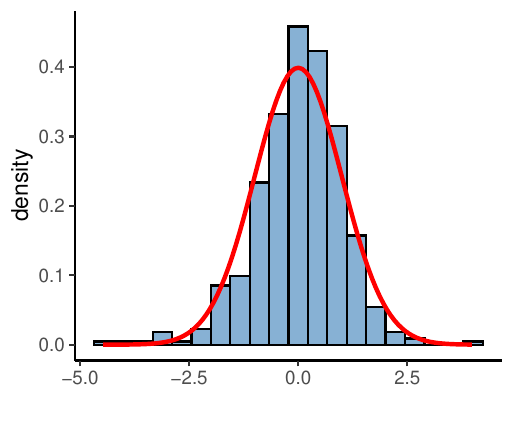}
        \caption{$\inp{M_1}{e_1e_5^{\top}}$}
    \end{subfigure}
    \begin{subfigure} {0.24\linewidth} 
		\includegraphics[width=1.1\linewidth]{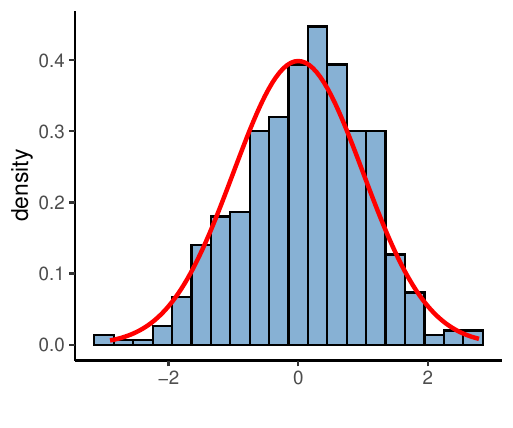}
        \caption{$\inp{M_0}{e_1e_5^{\top}}$}
    \end{subfigure}
    \begin{subfigure} {0.24\linewidth} 
		\includegraphics[width=1.1\linewidth]{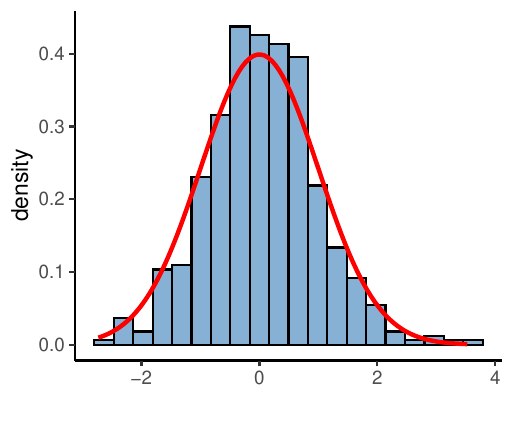}
        \caption{$\inp{M_0-M_1}{e_1e_5^{\top}}$}
    \end{subfigure}
    \begin{subfigure} {0.24\linewidth} 
		\includegraphics[width=1.1\linewidth]{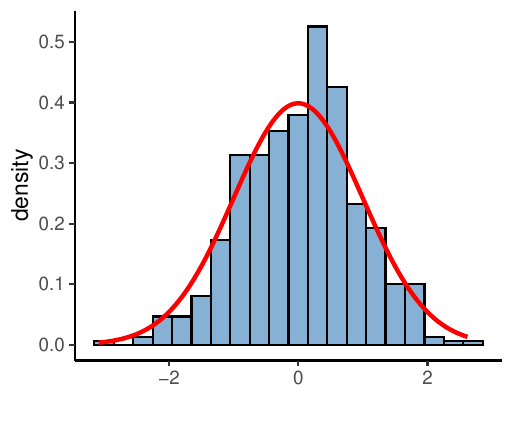}
        \caption{$\inp{M_0}{e_1e_5^{\top} - e_2e_2^{\top}}$}
    \end{subfigure}
	
    \caption{(\textbf{Varying phase-I exploration probability $\varepsilon$}) Empirical distributions under $\varepsilon=0.4$ (top four plots) and $\varepsilon=0.2$ (bottom four plots). The red curve represents the p.d.f. of standard normal distributions.}
    \label{epsilon2} 
\end{figure}

\begin{figure} [H]
	\centering
	\begin{subfigure} {0.24\linewidth} 
		\includegraphics[width=1.1\linewidth]{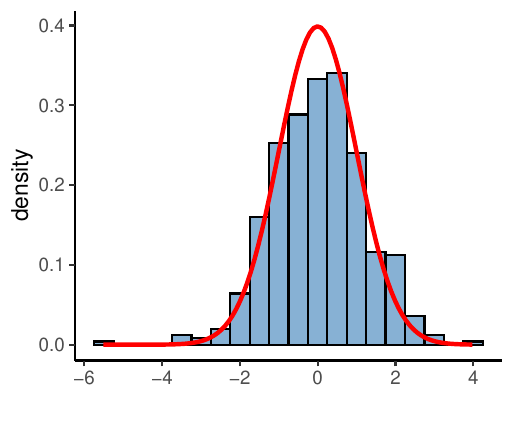}
        \caption{$\inp{M_1}{e_1e_5^{\top}}$}
    \end{subfigure}
    \begin{subfigure} {0.24\linewidth} 
		\includegraphics[width=1.1\linewidth]{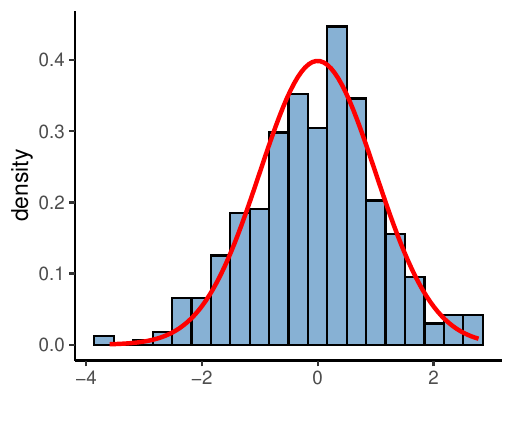}
        \caption{$\inp{M_0}{e_1e_5^{\top}}$}
    \end{subfigure}
    \begin{subfigure} {0.24\linewidth} 
		\includegraphics[width=1.1\linewidth]{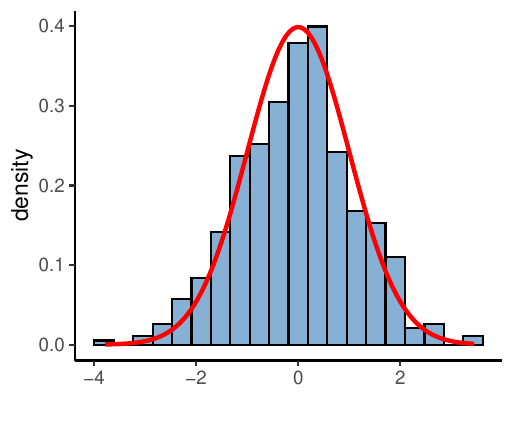}
        \caption{$\inp{M_0-M_1}{e_1e_5^{\top}}$}
    \end{subfigure}
    \begin{subfigure} {0.24\linewidth} 
		\includegraphics[width=1.1\linewidth]{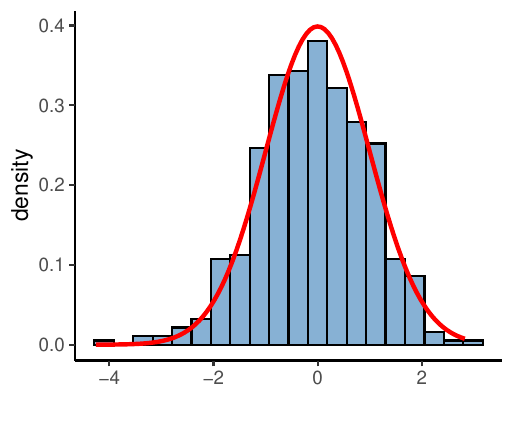}
        \caption{$\inp{M_0}{e_1e_5^{\top} - e_2e_2^{\top}}$}
    \end{subfigure}

	\begin{subfigure} {0.24\linewidth} 
		\includegraphics[width=1.1\linewidth]{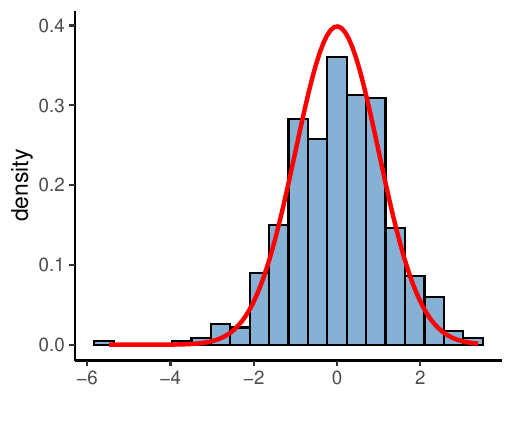}
        \caption{$\inp{M_1}{e_1e_5^{\top}}$}
    \end{subfigure}
    \begin{subfigure} {0.24\linewidth} 
		\includegraphics[width=1.1\linewidth]{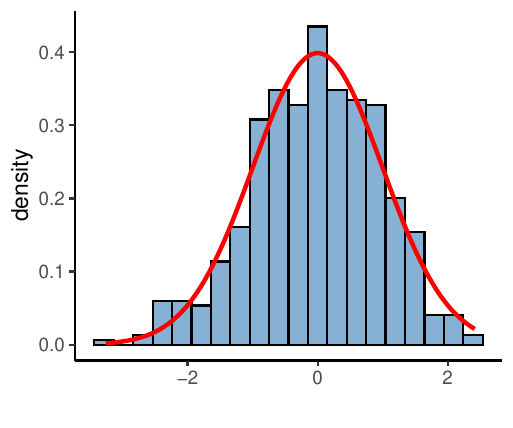}
        \caption{$\inp{M_0}{e_1e_5^{\top}}$}
    \end{subfigure}
    \begin{subfigure} {0.24\linewidth} 
		\includegraphics[width=1.1\linewidth]{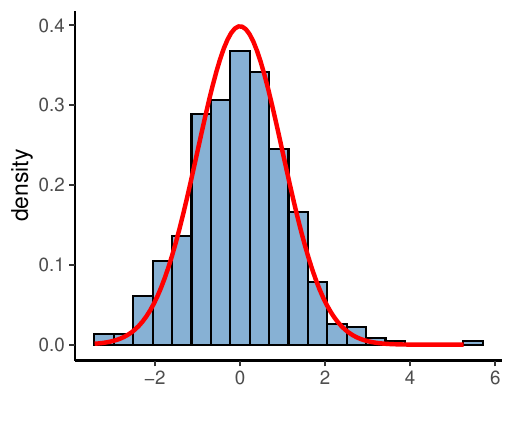}
        \caption{$\inp{M_0-M_1}{e_1e_5^{\top}}$}
    \end{subfigure}
    \begin{subfigure} {0.24\linewidth} 
		\includegraphics[width=1.1\linewidth]{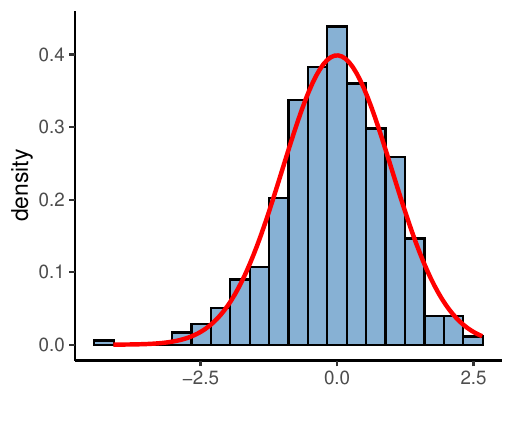}
        \caption{$\inp{M_0}{e_1e_5^{\top} - e_2e_2^{\top}}$}
    \end{subfigure}
	
    \caption{(\textbf{Varying noise levels $\sigma_1^2$ and $\sigma_0^2$}) Empirical distributions under $\sigma_1^2=1$, $\sigma_0^2=4$ (top four plots), and $\sigma_0^2=0.25$ (bottom four plots). The red curve represents the p.d.f. of standard normal distributions.}
    \label{sigma05} 
\end{figure}

\subsection{Misspecified ranks}
Finally, we examine the impact of misspecified matrix ranks $r$. Specifically, we evaluate three scenarios: 1) the true rank $r=3$ incorrectly specified as 5; 2) the true rank $r=5$ incorrectly specified as 7; 3) the true rank $r=8$ in correctly specified as 10.
\begin{figure} [H]
	\centering
	\begin{subfigure} {0.24\linewidth} 
		\includegraphics[width=1.1\linewidth]{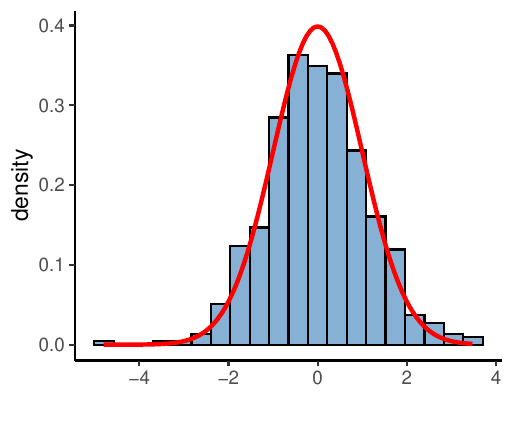}
        \caption{$\inp{M_1}{e_1e_5^{\top}}$}
    \end{subfigure}
    \begin{subfigure} {0.24\linewidth} 
		\includegraphics[width=1.1\linewidth]{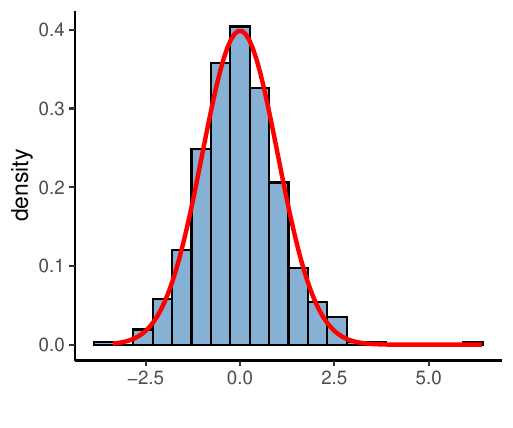}
        \caption{$\inp{M_0}{e_1e_5^{\top}}$}
    \end{subfigure}
    \begin{subfigure} {0.24\linewidth} 
		\includegraphics[width=1.1\linewidth]{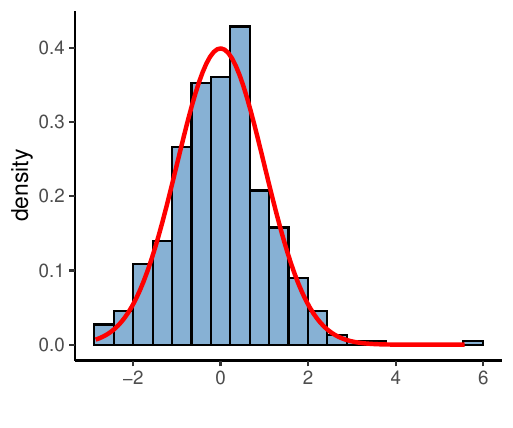}
        \caption{$\inp{M_0-M_1}{e_1e_5^{\top}}$}
    \end{subfigure}
    \begin{subfigure} {0.24\linewidth} 
		\includegraphics[width=1.1\linewidth]{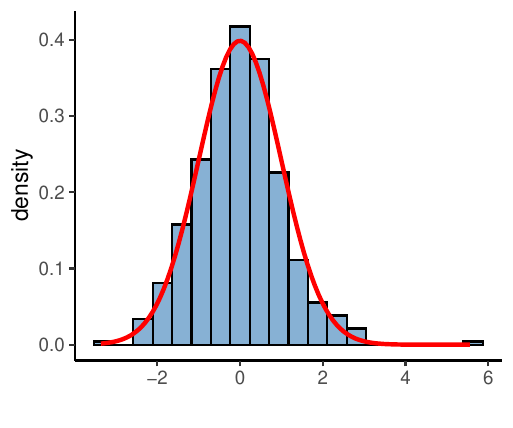}
        \caption{$\inp{M_0}{e_1e_5^{\top} - e_2e_2^{\top}}$}
    \end{subfigure}

	\begin{subfigure} {0.24\linewidth} 
		\includegraphics[width=1.1\linewidth]{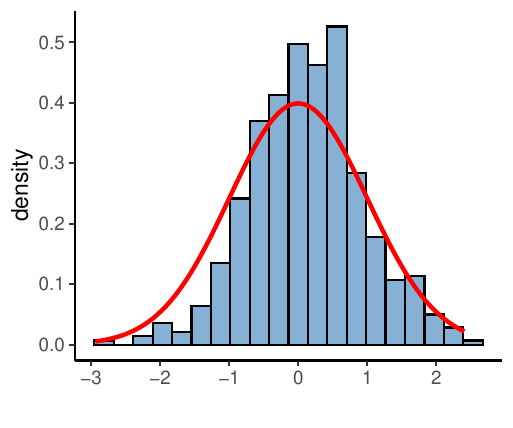}
        \caption{$\inp{M_1}{e_1e_5^{\top}}$}
    \end{subfigure}
    \begin{subfigure} {0.24\linewidth} 
		\includegraphics[width=1.1\linewidth]{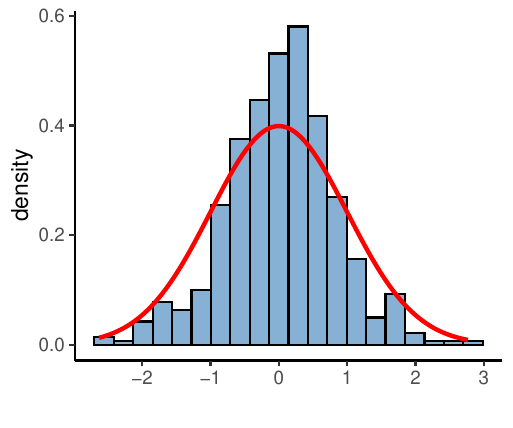}
        \caption{$\inp{M_0}{e_1e_5^{\top}}$}
    \end{subfigure}
    \begin{subfigure} {0.24\linewidth} 
		\includegraphics[width=1.1\linewidth]{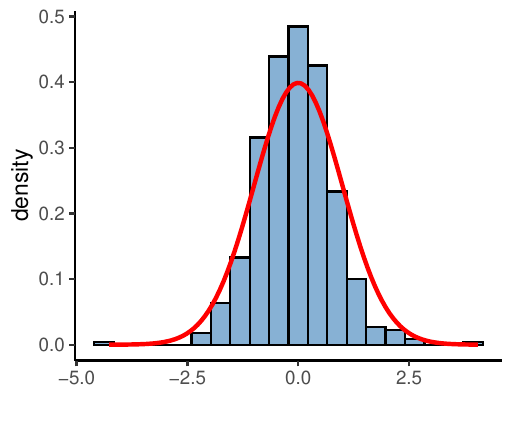}
        \caption{$\inp{M_0-M_1}{e_1e_5^{\top}}$}
    \end{subfigure}
    \begin{subfigure} {0.24\linewidth} 
		\includegraphics[width=1.1\linewidth]{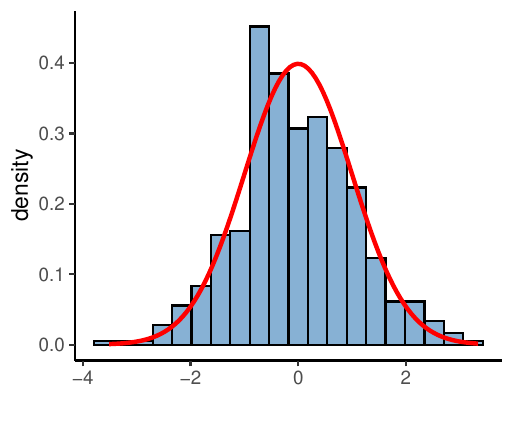}
        \caption{$\inp{M_0}{e_1e_5^{\top} - e_2e_2^{\top}}$}
    \end{subfigure}
	
	\begin{subfigure} {0.24\linewidth} 
		\includegraphics[width=1.1\linewidth]{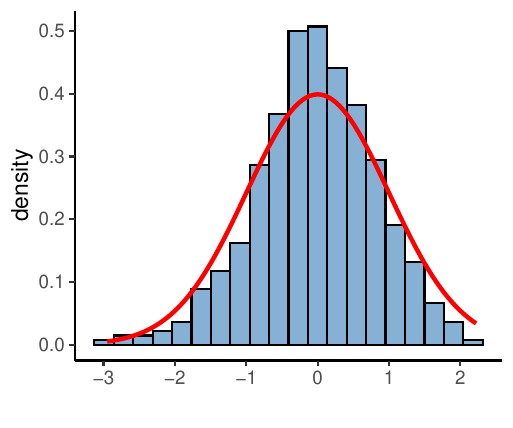}
        \caption{$\inp{M_1}{e_1e_5^{\top}}$}
    \end{subfigure}
    \begin{subfigure} {0.24\linewidth} 
		\includegraphics[width=1.1\linewidth]{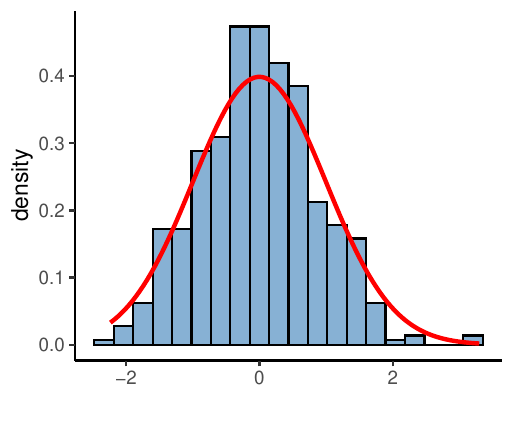}
        \caption{$\inp{M_0}{e_1e_5^{\top}}$}
    \end{subfigure}
    \begin{subfigure} {0.24\linewidth} 
		\includegraphics[width=1.1\linewidth]{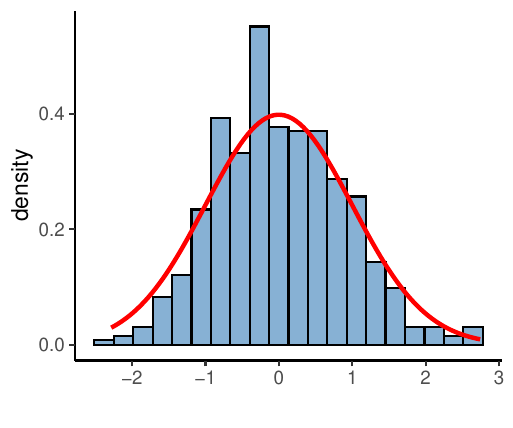}
        \caption{$\inp{M_0-M_1}{e_1e_5^{\top}}$}
    \end{subfigure}
    \begin{subfigure} {0.24\linewidth} 
		\includegraphics[width=1.1\linewidth]{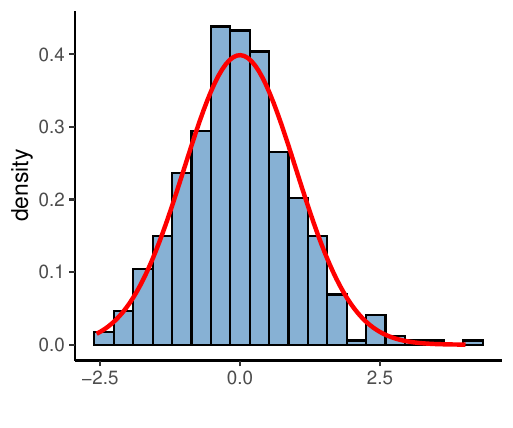}
        \caption{$\inp{M_0}{e_1e_5^{\top} - e_2e_2^{\top}}$}
    \end{subfigure}
	
    \caption{(\textbf{Misspecified ranks}) Empirical distributions under true rank $r=3$ and misspecified as 5 (top four plots), true rank $r=5$ and misspecified as 7 (middle four plots),  true rank $r=8$ and misspecified as 10 (bottom four plots). The red curve represents the p.d.f. of standard normal distributions.}
    \label{r8to10} 
\end{figure}

The comprehensive experimental results demonstrate that our method consistently delivers reliable statistical inference across varying tuning parameters and even when the matrix rank $r$ is misspecified.

\section{Extensions and Discussions}

\subsection{K-armed matrix completion bandit}\label{sec:karm}
\setcounter{Theorem}{0}
\newtheorem{thm}{Theorem}[section]
\newtheorem{assump}{Assumption}[section]
\newtheorem{cor}{Corollary}[section]

In real-world applications, decision makers often face multiple action choices. Accordingly, this section extends our results to the general $K$-armed bandit setting. We begin by introducing the $\varepsilon$-greedy 
$K$-arm matrix completion bandit algorithm. Next, we derive the convergence properties of the online estimators and establish an upper bound on the cumulative regret. Finally, we present the online debiasing procedure and establish a central limit theorem for the IPW-based debiased estimator.

Let $M_a$ denote the bandit parameter for arm $a\in[K]$, and let its SVD be $M_a=L_a\Lambda_aR_a^{\top}$. Here, $\Lambda_a = \text{diag}(\lambda_{a,1},\cdots,\lambda_{a,r})$ is a diagonal matrix containing the singular values of $M_a$ in non-increasing order. The signal strength is defined by $\lambda_{\min} := \min_{a\in [K]}\lambda_{a,r}$, and we set $\bar\sigma:=\max_{a\in [K]} \sigma_a$. Moreover, the incoherence parameter of $M_a$ is defined by
$$
\mu(M_a):=\max\bigg\{\sqrt{\frac{d_1}{r}}\|L_a\|_{2,\max},\ \sqrt{\frac{d_2}{r}}\|R_a\|_{2,\max} \bigg\}.
$$
We assume that the maximum incoherence across all arms satisfies  $\max_a\{\mu(M_a)\}\leq \mu$. 

For the policy learning algorithm, the primary modification involves updating the propensity scores. Instead of using $\pi_t=\PP(a_t=1|\mathcal{F}_{t-1})$ and $1-\pi_t$ in binary action selection, we replace them with $\pi_{a,t}=\PP(a_t=a|\mathcal{F}_{t-1})=(1-\varepsilon_t)\mathbbm{1}(a=a^{*})+\frac{\varepsilon_t}{K}$ for each $a\in [K]$, where $a^{*}=\argmax_{a\in [K]} \inp{\hat M_{a,t-1}}{X_t}$.

\begin{algorithm}
		\caption{$\eps$-greedy K-arm MCB with online gradient descent}\label{alg:mcb3}
\begin{algorithmic}
\STATE{\textbf{Input}: exploration probabilities $\{\varepsilon_t\}_{t\geq 1}$; step sizes $\{\eta_t\}_{t\geq 1}$; initializations with balanced factorization ${\hat M}_{k,0}=\hat U_{k,0}\hat V_{k,0}^{\top}, k\in[K]$.}
\STATE{\textbf{Output}: $\hat{M}_{k,T}, k\in[K]$.}
			\FOR{$t= 1,2,\cdots, T$}  
            \STATE Observe a new request $X_t$ \\
            \STATE $a^{*}=\argmax_{a\in [K]} \inp{\hat M_{a,t-1}}{X_t}$ \\
            \STATE Calculate $\pi_{a,t}=(1-\varepsilon_t)\mathbbm{1}(a=a^{*})+\frac{\varepsilon_t}{K}$ for $a\in [K]$ \\
            \STATE Sample an action $a_t\sim \text{Categorical}(1,\cdots,K; \pi_{1,t},\cdots,\pi_{K,t})$ and get a reward $r_t$ \\
            \FOR{$a= 1,2,\cdots, K$}
            \STATE Update by
            \begin{align*}
             \begin{pmatrix}
        \tilde U_{a,t} \\
        \tilde V_{a,t}
        \end{pmatrix} = \begin{pmatrix}
        \hat U_{a,t-1} \\
        \hat V_{a,t-1} 
        \end{pmatrix} - \frac{\mathbbm{1}(a_t=a)\eta_t}{\pi_{a,t}}\cdot\begin{pmatrix}
        \big(\inp{\hat U_{a,t-1}\hat V_{a,t-1}^{\top}}{X_t}-r_t\big)X_t\hat V_{a,t-1} \\
        \big(\inp{\hat U_{a,t-1}\hat V_{a,t-1}^{\top}}{X_t}-r_t)X_t^{\top}\hat U_{a,t-1}
        \end{pmatrix}.
        \end{align*}	
        Set $\hat{U}_{a,t}=\hat L_{a,t}\hat \Lambda_{a,t}^{1/2}$ and $\hat{V}_{a,t}=\hat R_{a,t}\hat \Lambda_{a,t}^{1/2}$, where $\hat L_{a,t}\hat \Lambda_{a,t}\hat R_{a,t}^{\top}$ is the thin SVD of $\hat {M}_{a,t}=\tilde U_{a,t}\tilde V_{a,t}^{\top}$. 
        \ENDFOR
        \ENDFOR  
\end{algorithmic}
\end{algorithm}

\begin{thm}\label{thm:MCB-convK}
Suppose that the horizon $T\leq d_1^{100}$ and the initializations are incoherent satisfying $\max_a{\|\hat M_{a,0}-M_a\|_{\rm F}}\leq c_0\lambda_{\min}$ for some small constant $c_0>0$.
        Fix some $\gamma\in [0,1)$, $\eps\in(0, 1/2)$, and set $T_0:=C_0T^{1-\gamma}$ for a large constant $C_0>0$.  When $t\leq T_0$, set $\eps_t\equiv\eps$ and $\eta_t\equiv \eta:=c_1d_1d_2\log(d_1)/(T^{1-\gamma}\lambda_{\max})$; when $T_0<t\leq T$, set $\eps_t=c_2t^{-\gamma}$ and $\eta_t=\eps_t \eta$, where $c_1, c_2>0$ are numerical constants. Suppose that the horizon and signal-to-noise ratio satisfy
        $$
        T\geq C_1r^3(Kd_1)^{1/(1-\gamma)} \log^2 d_1\quad {\rm and}\quad  \frac{\lambda_{\min}^2}{\max_a{\sigma_a^2}}\geq C_2K\frac{rd_1^2d_2\log^2d_1}{T^{1-\gamma}},
        $$
        for some large constants $C_1, C_2>0$ depending on $C_0, c_0,  c_1,c_2$ only. There exist constant $C_3,C_4,C_5 >0$ such that, for any $a\in [K]$, with probability at least $1-8td_1^{-200}$, the output of Algorithm 2 satisfies 
        \begin{align*}
            &\big\|\hat M_{a, T} - M_a \big\|_{\rm F}^2\leq C_3\fro{\widehat{M}_{a,0}-M_a}^2 \left(1-\frac{c_1\log d_1}{4\kappa T^{1-\gamma}}\right)^t + C_4t\cdot \sigma_a^2K\frac{rd_1^2d_2\log^4 d_1}{T^{2-2\gamma}}, \\
            &\big\|\hat M_{a,T} - M_a\big\|_{\max}^2\leq  C_3\frac{\lambda_{\min}^2r^3}{d_1d_2} \left(1-\frac{c_1\log d_1}{4\kappa T^{1-\gamma}}\right)^t +   C_4t\cdot \sigma_a^2K\frac{rd_1\log^4 d_1}{T^{2-2\gamma}},
        \end{align*}
for all $t\leq T_0$. Moreover, for all $t\geq T_0$, with probability at least $1-8d_1^{-100}$, we have
        \begin{align*}
            \big\|\hat M_{a, t} - M_a \big\|_{\rm F}^2\leq C_5\sigma_a^2K\frac{rd_1^2d_2\log^4 d_1}{T^{1-\gamma}} \quad {\rm and}\quad 
            \big\|\hat M_{a,t} - M_a\big\|_{\max}^2\leq C_5\sigma_a^2K\frac{rd_1\log^4 d_1}{T^{1-\gamma}}.
        \end{align*}
\end{thm} 
We also derive the regret upper bound for Algorithm \ref{alg:mcb3}.
\begin{thm} \label{thm:regretK}
  Suppose that  the conditions in Theorem \ref{thm:MCB-convK} hold, and denote by $\bar{m}:=\max_{k,j \in [K]}\|M_j-M_k\|_{\max}$. 
  Then there exists a numerical constant $C_6>0$ such that  the regret is upper bounded by 
\begin{align*}
        R_T\leq & C_6K\bigg(\bar m r\cdot T^{1-\gamma}+\bar\sigma\cdot T^{(1+\gamma)/2}\sqrt{rd_1}\log^2 d_1 \bigg).
 \end{align*}
\end{thm}

For policy inference, we construct the debiased estimators by replacing the propensity scores $\pi_t$ or $1-\pi_t$ with $\pi_{a,t}$ as follows:
\begin{align}\label{eq:debias-awK}
\hat M_a^{\uipw}:&= \frac{1}{T-T_0}\sum_{t=T_0+1}^ T\hat M_{a, t-1}+\frac{d_1d_2}{T-T_0}\sum_{t=T_0+1}^T \frac{\mathbbm{1}(a_t=a)}{\pi_{a,t}} \big(r_t-\langle \hat M_{a, t-1}, X_t\rangle\big)X_t,\quad a\in[K].
\end{align}

We extend Assumption \ref{assump:arm-opt} to the $K$-armed case. For any constant $\delta>0$, we define $\Omega_a(\delta):= \{X\in \mathcal{X}: \min_{k\in [k]\backslash \{a\}} \inp{M_a-M_k}{X}>\delta\}$ and $\Omega_{\emptyset}(\delta):=(\cup_{a\in [K]} \Omega_a(\delta))^c$. Let $\mathcal{P}_{\Omega_a}(M)$ denote the operator which zeros out the entries of $M$ except those in the set $\Omega_a$. We also require the variance contributed by samples in $\Omega_{\emptyset}$ to be negligible compared to others. 
\begin{assump}
\label{assump:arm-optK}(Arm Optimality)
There exists a gap $\delta>0$ such that the sets $\Omega_a$ for $a\in [K]$ and $\Omega_{\emptyset}$ ensure 
$\fro{\calP_{\Omega_{\emptyset}}\calP_{M_a}(Q)}^2/\min_k{\fro{\calP_{\Omega_{k}}\calP_{M_a}(Q)}^2}=o(1)$ as $d_1, d_2\to\infty$ for $a\in [K]$.
\end{assump}
The gap specified in Assumption \ref{assump:arm-optK} is allowed to diminish as $d_1,d_2\rightarrow \infty$. Finally, the best rank-$r$ approximation of $\widehat{M}_a^{\uipw}$ is $\hat M_a=\widehat{L}_a\widehat{L}_a^{\top}\widehat{M}_a^{\uipw}\widehat{R}_a\widehat{R}_a^{\top}$, where $\widehat{L}_a$ and $\widehat{R}_a$ consist of the left and right top-$r$ singular vectors of $\widehat{M}_a^{\uipw}$, respectively. Then the asymptotic normality for $\inp{\widehat{M}_a}{Q}$ is described by
\begin{thm}\label{thm:CLTK}
Suppose that the conditions in Theorem \ref{thm:MCB-convK} and Assumption~\ref{assump:arm-optK} hold with $\delta^2:=\delta_T^2:=C_1(\sigma_0^2+\sigma_1^2)(rd_1/T^{1-\gamma})\log^4d_1$ for some large constant $C_1>0$ in Assumption \ref{assump:arm-optK}. Define $S_a^2=T^{-\gamma}\fro{\calP_{\Omega_a}\calP_{M_a}(Q)}^2+\sum_{k\neq a, k\in [K]} C_{\gamma}\fro{\calP_{\Omega_k}\calP_{M_a}(Q)}^2$ for $a\in [K]$, where $C_{\gamma}=K/(c_2(1+\gamma))$ and $c_2$ is defined in Theorem \ref{thm:MCB-convK}. Assume that  
\begin{align}
\frac{Kr^2d_1\log^5d_1}{T^{1-\gamma}} + \frac{\sigma_aK}{\lambda_{\min}}\sqrt{\frac{rd_1^2\log d_1}{T^{1-\gamma}}} \frac{\|Q\|_{\ell_1}}{S_a} + \frac{\sigma_a}{\lambda_{\min}}\sqrt{\frac{Krd_1^2d_2\log d_1}{T^{1-\gamma}}} \rightarrow 0
\end{align}
as $d_1, d_2, T\to\infty$. Then, 
\begin{align}
    \frac{\langle \hat M_a, Q\rangle - \langle M_a, Q\rangle}{\sigma_aS_a \sqrt{d_1d_2/T^{1-\gamma}}}\rightarrow N(0,1)
\end{align}
as $d_1,d_2, T\to\infty$. 
\end{thm}
We propose the following estimators for $\sigma_a^2$ and $S_a^2$, which enable us to derive the studentized central limit theorem.
\begin{align*}
    &\hat\sigma_a^2:=\frac{1}{T-T_0}\sum_{t=T_0+1}^T \frac{\mathbbm{1}(a_t=a)}{\pi_{a,t}} \big(r_t-\langle \hat M_{0, t-1}, X_t\rangle\big)^2 \\
    &\widehat{S}_a^2:= \left(1/T^{\gamma}\fro{\calP_{\widehat{\Omega}_{a,T}}\calP_{\hat M_a}(Q)}^2+\sum_{k\neq a, k\in [K]} C_{\gamma}\fro{\calP_{\widehat{\Omega}_{k,T}}\calP_{\hat M_a}(Q)}^2\right)b_T.
\end{align*}
where $\widehat{\Omega}_{a,T}=\{X\in \mathcal{X}: \argmax_{k\in [K]}\inp{\widehat{M}_{k,T}}{X}=a\}$ for $a\in [K]$ and $b_T=T/(T-T_0)$. 
\begin{thm}
\label{thm:studentized-CLTK}
Suppose that the conditions in Theorem~\ref{thm:CLTK} hold. Then for $a\in [K]$, 
\begin{align*}
    \frac{\langle \hat M_a, Q\rangle - \langle M_a, Q\rangle}{\hat\sigma_a \widehat{S}_a \sqrt{d_1d_2/T^{1-\gamma}}}\rightarrow N(0,1)
\end{align*}
as $d_1,d _2, T\to\infty$.  
\end{thm}
The following Corollary enables the evaluation in outcomes between any two actions. 
\begin{cor}
    \label{cor:m1-m0K}
     Suppose that the conditions in Theorem \ref{thm:CLTK} hold. Then, for any $g,h\in [K]$, 
        \begin{align*}
        \frac{\langle \hat M_g -\hat M_h, Q\rangle - \langle M_g-M_h, Q\rangle}{\sqrt{\big(\hat\sigma_g^2\widehat{S}_g^2+\hat\sigma_h^2\widehat{S}_h^2 \big)d_1d_2/T^{1-\gamma}}} \longrightarrow N(0,1),
        \end{align*}
        as $d_1,d _2, T\to\infty$. 
\end{cor}

The proof of Theorem \ref{thm:MCB-convK}, \ref{thm:studentized-CLTK} and Corollary \ref{cor:m1-m0K} is  almost the same with Theorem \ref{thm:MCB-conv}, \ref{thm:studentized-CLT} and Corollary \ref{cor:m1-m0}. We provide the proof of Theorem \ref{thm:regretK} and \ref{thm:CLTK} in Appendix \ref{sec:proofregretK} and \ref{sec:proofCLTK}.

\subsection{Efficient SVD implementation} \label{sec:practical}
\newtheorem{lem}{Lemma}[section]

Although Algorithm~\ref{alg:mcb1} enables straightforward theoretical analysis of convergence, it is computationally intensive due to the necessity of performing SVD on $d_1\times d_2$ matrices at each time step. To address this, we introduce an alternative algorithm that offers greater computational efficiency for practical implementation.

\begin{algorithm}
		\caption{$\eps$-greedy two-arm MCB with online gradient descent (fast SVD)}\label{alg:mcb2}
\begin{algorithmic}
\STATE{\textbf{Input}: exploration probabilities $\{\varepsilon_t\}_{t\geq 1}$; step sizes $\{\eta_t\}_{t\geq 1}$; initializations with balanced factorization ${\hat M}_{0,0}=\hat U_{0,0}\hat V_{0,0}^{\top}$, ${\hat M}_{1,0}=\hat U_{1,0}\hat V_{1,0}^{\top}$.}
\STATE{\textbf{Output}: $\hat{M}_{0,T}$, $\hat{M}_{1,T}$.}
			\FOR{$t= 1,2,\cdots, T$}  
        \STATE Observe a new request $X_t$; \\
        \STATE Calculate $\pi_t=(1-\varepsilon_t)\mathbbm{1}\big(\inp{{\hat M}_{1,t-1} - {\hat M}_{0,t-1}}{X_t}>0\big)+\frac{\varepsilon_t}{2}$; \\
        \STATE Sample an action $a_t\sim \text{Bernoulli}(\pi_t)$ and get a reward $r_t$; \\
         \FOR{$i= 0,1$}
            \STATE $R_UD_UR_U^{\top} \leftarrow \text{SVD}(\widehat{U}_{i,t-1}^{\top}\widehat{U}_{i,t-1})$,
            \STATE $R_VD_VR_V^{\top} \leftarrow \text{SVD}(\widehat{V}_{i,t-1}^{\top}\widehat{V}_{i,t-1})$,
            \STATE $Q_UDQ_V \leftarrow \text{SVD}(D_U^{1/2}R_U^{\top}R_VD_V^{1/2})$,
		      \STATE Update by 
		      \begin{align*}
        \begin{pmatrix}
        \hat U_{i,t} \\
        \hat V_{i,t}
        \end{pmatrix} = \begin{pmatrix}
        \hat U_{i,t-1} \\
        \hat V_{i,t-1} 
        \end{pmatrix} - &\eta_t\frac{\mathbbm{1}(a_t=i)}{i\pi_{t}+(1-i)(1-\pi_t)}\cdot \\ &\begin{pmatrix}
            \big(\inp{\hat U_{i,t-1}\hat V_{i,t-1}^{\top}}{X_t}-r_t\big)X_t\hat V_{a,t-1}R_VD_V^{-1/2}Q_VQ_U^{\top}D_U^{1/2}R_U^{\top} \\
            \big(\inp{\hat U_{i,t-1}\hat V_{i,t-1}^{\top}}{X_t}-r_t)X_t^{\top}\hat U_{a,t-1}R_UD_U^{-1/2}Q_UQ_V^{\top}D_V^{1/2}R_V^{\top} 
        \end{pmatrix},
        \end{align*}
        Set $\hat {M}_{i,t}=\hat U_{i,t}\hat V_{i,t}^{\top}$.
        \ENDFOR
        \ENDFOR  
\end{algorithmic}
\end{algorithm}

Algorithm \ref{alg:mcb2} can be shown equivalent to Algorithm \ref{alg:mcb1} in the following sense:
\begin{lem}[\cite{jin2016provable} Lemma 3.2]
    Given the same observations and algorithmic parameters, the outputs of every time step $t$ in Algorithm \ref{alg:mcb1}, i.e., $\widetilde{U}_{i,t}$, $\widetilde{V}_{i,t}$ for $i=0,1$, and those for Algorithm \ref{alg:mcb2}, i.e., $\widehat{U}_{i,t}$ and $\widehat{V}_{i,t}$, satisfy $\widetilde{U}_{i,t}\widetilde{V}_{i,t}^{\top}=\widehat{U}_{i,t}\widehat{V}_{i,t}^{\top}$.
\end{lem}

Algorithm \ref{alg:mcb2} only requires computing the SVD of the $r\times r$ matrices $\widehat{U}_{i,t-1}^{\top}\widehat{U}_{i,t-1}$ and $\widehat{V}_{i,t-1}^{\top}\widehat{V}_{i,t-1}$. This reduces the computational complexity of each SVD step from $O(d_1d_2^2)$ to $O(r^3)$, and the computational complexity for matrix multiplication step is $O(d_1r^2)$. Since the rank $r$ is significantly smaller than the dimensions $d_1$ and $d_2$ and is typically a constant in many applications, the computational complexity remains manageable even as the dimensions scale rapidly. Consequently, the computational burden remains low in high-dimensional settings.

\subsection{Non-uniform sampling} \label{non-uniform}
Our previous results assume that the covariates $X_t$ are uniformly sampled from the orthonormal basis $\mathcal{X}=\{e_{j_1}e_{j_2}^{\top}: j_1\in [d_1],j_2\in[d_2] \}$, ensuring equal opportunity to make decisions for each product or block at every week or hour. While this assumption facilitates theoretical analysis, it may not hold in certain applications. For instance, in online recommendation systems, some products may be infrequently searched or visited by users, violating uniform sampling. To address this, we consider covariates sampled with known probabilities $p_{X}$ for each $X\in \mathcal{X}$. Our methods remain applicable, albeit with more complex formulations for theoretical analysis.

For policy learning, all learning strategies remain the same except that the stochastic gradient descent step in Algorithm \ref{alg:mcb1} should be modified to:
\begin{align*}
        \begin{pmatrix}
            \tilde U_{1,t} \\
            \tilde V_{1,t}
        \end{pmatrix} = \begin{pmatrix}
            \hat U_{1,t-1} \\
            \hat V_{1,t-1} 
        \end{pmatrix} - \frac{\eta_t}{\pi_tp_{X_t}}\cdot\begin{pmatrix}
            \big(\inp{\hat U_{1,t-1}\hat V_{1,t-1}^{\top}}{X_t}-r_t\big)X_t\hat V_{1,t-1} \\
            \big(\inp{\hat U_{1,t-1}\hat V_{1,t-1}^{\top}}{X_t}-r_t)X_t^{\top}\hat U_{1,t-1}
        \end{pmatrix}
\end{align*}	
and
\begin{align*}
        \begin{pmatrix}
            \tilde U_{0,t} \\
            \tilde V_{0,t}
        \end{pmatrix} = \begin{pmatrix}
            \hat U_{0,t-1} \\
            \hat V_{0,t-1} 
        \end{pmatrix} - \frac{\eta_t}{(1-\pi_t)p_{X_t}}\cdot\begin{pmatrix}
            \big(\inp{\hat U_{0,t-1}\hat V_{0,t-1}^{\top}}{X_t}-r_t\big)X_t\hat V_{0,t-1} \\
            \big(\inp{\hat U_{0,t-1}\hat V_{0,t-1}^{\top}}{X_t}-r_t)X_t^{\top}\hat U_{0,t-1}
        \end{pmatrix}.
\end{align*}	
Define $p_{L}=\min_{X\in \mathcal{X}}p_{X}$, and set $T_0=C_0T^{1-\gamma}$ for large $C_0$. When $t\leq T_0$, set a constant exploration probability $\varepsilon_t=\varepsilon$ and the step size $\eta_t:=\eta=\frac{c_1\log d_1}{T^{1-\gamma}\lambda_{\max}}$ for some constant $c_1$; when $T_0<t\leq T$, set $\varepsilon_t\asymp t^{-\gamma}$ and $\eta_t=\varepsilon_t\eta$. Assume that the sample size and SNR satisfy $T^{1-\gamma}\gg r\log (d_1)/(p_Ld_2)$ and $\lambda_{\min}^2/(\sigma_0^2+\sigma_1^2)\gg r\log(d_1)/(d_2p_L^2T^{1-\gamma})$. Under these settings, we can achieve the following convergence results for the online estimators:
\begin{align*}
    &\big\|\hat M_{i, t} - M_i \big\|_{\rm F}^2\leq C_3\fro{\widehat{M}_{i,0}-M_i}^2 \left(1-\frac{c_1\log d_1}{4\kappa T^{1-\gamma}}\right)^t + C_4t\cdot \sigma_i^2\frac{rd_1\log^4 d_1}{p_LT^{2-2\gamma}}, \quad\forall t\leq T_0; \\
    &\big\|\hat M_{i, t} - M_i \big\|_{\rm F}^2\leq C_5\sigma_i^2\frac{rd_1\log^4 d_1}{p_LT^{1-\gamma}}, \quad \forall t>T_0.
\end{align*}
For policy inference, the debiasing procedure would require slight modification:
\begin{align*}
    \hat M_0^{\uipw}:&= \frac{1}{T-T_0}\sum_{t=T_0+1}^T \hat M_{0, t-1}+\frac{1}{T-T_0}\sum_{t=T_0+1}^T \frac{\mathbbm{1}(a_t=0)}{(1-\pi_t)p_{X_t}} \big(r_t-\langle \hat M_{0, t-1}, X_t\rangle\big)X_t \\
    \hat M_1^{\uipw}:&= \frac{1}{T-T_0}\sum_{t=T_0+1}^T \hat M_{1, t-1}+\frac{1}{T-T_0}\sum_{t=T_0+1}^T \frac{\mathbbm{1}(a_t=1)}{\pi_tp_{X_t}} \big(r_t-\langle \hat M_{1, t-1}, X_t\rangle\big)X_t.  
\end{align*}
Moreover, if the arm-optimality sets satisfy 
\begin{align*}
    \frac{\sum_{X\in \Omega_{\emptyset}}1/p_{X}\inp{X}{\calP_{M_i}(Q)}^2}{\min\{\sum_{X\in \Omega_1}1/p_{X}\inp{X}{\calP_{M_i}(Q)}^2, \sum_{X\in \Omega_0}1/p_{X}\inp{X}{\calP_{M_i}(Q)}^2\}}=o(1), \quad i\in\{0,1\}, 
\end{align*}
then the following asymptotic normality holds:
\begin{align*}
 \frac{\langle \hat M_0, Q\rangle - \langle M_0, Q\rangle}{\sigma_0S_0 \sqrt{1/T^{1-\gamma}}}\rightarrow N(0,1)\quad {\rm and}\quad 
 \frac{\langle \hat M_1, Q\rangle - \langle M_1, Q\rangle}{\sigma_1S_1 \sqrt{1/T^{1-\gamma}}}\rightarrow N(0,1),
\end{align*}
as $d_1,d_2, T\to\infty$,  where
\begin{align*}
 &S_1^2=T^{-\gamma}\sum_{X\in \Omega_1}1/p_{X}\inp{X}{\calP_{M_1}(Q)}^2+C_{\gamma}\sum_{X\in \Omega_0}1/p_{X}\inp{X}{\calP_{M_1}(Q)}^2 \\
&S_0^2=T^{-\gamma}\sum_{X\in \Omega_0}1/p_{X}\inp{X}{\calP_{M_0}(Q)}^2+C_{\gamma}\sum_{X\in \Omega_1}1/p_{X}\inp{X}{\calP_{M_0}(Q)}^2.
\end{align*}

\subsection{Other bandit algorithms}
We discuss two common bandit algorithms: upper confidence bound (UCB) (\cite{li2010contextual}) and Thompson sampling (TS) (\cite{agrawal2013thompson}) for matrix completion bandit. 
Although both algorithms are well-established in the broader bandit literature, their applications to high-dimensional linear bandits remain relatively limited, let alone matrix completion bandits.


\noindent \textbf{(1) Upper confidence bound (UCB) }
The UCB algorithm effectively balances the trade-off between exploration and exploitation through the principle of \emph{optimism in the face of uncertainty}. The fundamental idea behind UCB is to maintain an optimistic estimate of each arm's potential reward, thereby encouraging the selection of arms that may yield higher rewards based on both observed data and inherent uncertainty. The uncertainty of estimated bandit parameters can be quantified by the sup-norm convergence rate. Therefore, at time $t$, we define
\begin{align*}
     \widehat{p}_{i,t}:= C_0\hat{\sigma}_i\sqrt{\frac{rd_1\log^4d_1}{T^{1-\gamma}}}, \quad \forall i\in\{0,1\}\ {\rm and}\ t\geq T_0,
\end{align*}
where $C_0>0$ is a large constant. Then the action $a_t$ is selected by
\begin{align*}
    a_t:=\mathbbm{1}\left(\inp{X_t}{\widehat{M}_{1,t-1}}+\widehat{p}_{1,t}>\inp{X_t}{\widehat{M}_{0,t-1}}+\widehat{p}_{0,t}\right),\quad \forall \ t\geq T_0.
\end{align*}
For simplicity, we still adopt a constant exploration probability during phase I when $t<T_0:\asymp T^{1-\gamma}$ for some $\gamma\in(0,1)$. We can show that UCB algorithm achieves a regret of $\widetilde{O}\big(\bar{m}\cdot T^{1-\gamma}+\bar\sigma\cdot T^{(1+\gamma)/2}\sqrt{rd_1}\log^2d_1\big)$. However, policy inference under UCB algorithm is significantly challenging because of the inherent bias in $\hat M_{i,t}$'s and that the action $a_t$ is deterministic conditioned on $\calF_{t-1}$ and $X_t$. In contrast, the $\varepsilon$-greedy algorithm randomizes the action $a_t$ and facilitates the IPW-based debiaing technique.

\noindent \textbf{(2) Thompson sampling (TS) }
Unlike UCB which takes a deterministic approach to balancing exploration and exploitation, Thompson sampling is a probabilistic strategy that uses Bayesian principles to guide decision-making. It works by maintaining a posterior distribution for the reward of each arm, which represents the uncertainty about the true reward. 

We follow the Bayesian matrix completion method in \cite{alquier2014bayesian}. Denote the $h$-th row of $U_i$ as $U_i^{(h)\top}$, the $g$-th row of $V_i$ as $V_i^{(g)\top}$, and the $l$-th column of $U_i$ and $V_i$ as $u_{i,l}$ and $v_{i,l}$. Given some constant vector $\delta=(\delta_1,\cdots, \delta_r)$, we assign the Gaussian priors for $u_{i,l}$ and $v_{i,l}$ as $u_{i,l}\sim N(0, \delta_l I_{d_1})$ and $v_{i,l}\sim N(0, \delta_l I_{d_2})$.

At each iteration $t$, suppose the context is $X_t=e_{i_t}e_{j_t}^{\top}$. We first draw $U_{i,t}^{(i_t)\top}$ and $V_{i,t}^{(j_t)\top}$ from their posteriors at time $t-1$ for $i=0,1$. The action $a_t$ is selected by
\begin{align*}
    a_t:=\mathbbm{1}\left(U_{1,t}^{(i_t)\top}V_{1,t}^{(j_t)} > U_{0,t}^{(i_t)\top}V_{0,t}^{(j_t)}\right).
\end{align*}
Then we can update the posterior distribution of $U_{a_t}^{(i_t)\top}$ and $V_{a_t}^{(j_t)\top}$ at time $t$ as
\begin{align*}
    U_{a_t}^{(i_t)\top}\big| \mathcal{F}_t,\delta \sim N(\bar{U}_{a_t}^{(i_t)\top}, \mathcal{V}_{a_t,i_t}) \quad\quad  V_{a_t}^{(j_t)\top}\big| \mathcal{F}_t,\delta \sim N(\bar{V}_{a_t}^{(j_t)\top}, \mathcal{W}_{a_t,j_t}), 
\end{align*}
where 
\begin{align*}
    &\mathcal{V}_{a_t,i_t}^{-1}= \text{diag}(\delta)^{-1} + 2\beta V_{a_t,t}^{(j_t)}V_{a_t,t}^{(j_t)\top}, \quad \bar{U}_{a_t}^{(i_t)\top}= 2r_t\beta V_{a_t,t}^{(j_t)\top}\mathcal{V}_{a_t,i_t}, \\
    & \mathcal{W}_{a_t,j_t}^{-1}= \text{diag}(\delta)^{-1} + 2\beta U_{a_t,t}^{(i_t)}U_{a_t,t}^{(i_t)\top}, \quad \bar{V}_{a_t}^{(j_t)\top}= 2r_t\beta U_{a_t,t}^{(i_t)\top}\mathcal{W}_{a_t,j_t},
\end{align*}
and $\beta$ is a tuning parameter depending on $\sigma_i$.

\section{Proofs of Main Results}\label{app:proofs_thms}

\subsection{Proof of Theorem \ref{thm:MCB-conv}} \label{proof:completeconv}
We first state the convergence result under general choices of step size $\eta_t$ and $\varepsilon_t$ and provide the proof for this Theorem. 
\begin{thm} \label{thm:genera-thm1}
Suppose that the horizon $T\leq d_1^{50}$ and the initializations are incoherent,  satisfying $\|\hat M_{0,0}-M_0\|_{\rm F}+\|\hat M_{1,0}-M_1\|_{\rm F}\leq c_0\lambda_{\min}$ for some small constant $c_0>0$. Let $\hat M_{0,t}, \hat M_{1,t}, t\in[T]$ denote the output of Algorithm~\ref{alg:mcb1}. 
There exist absolute constants $C_1, C_2,C_3, C_4>0$ such that if, for any $t\in[T]$, 
\begin{align}\label{eq:thm-MCB-eq1}
\min\bigg\{\frac{\lambda_{\min}^2}{\sigma_0^2+\sigma_1^2},\ d_1d_2\log d_1\bigg\}\geq C_1 \sum_{\tau=1}^t \frac{(\eta_{\tau}\lambda_{\max})^2}{\eps_{\tau}}\cdot \frac{r\log^2d_1}{d_2}\quad {\rm and}\quad \max_{\tau\in [t]}\ \frac{\eta_{\tau}^2}{\eps_{\tau}^2}\leq C_2\sum_{\tau=1}^t\frac{\eta_{\tau}^2}{\eps_{\tau}},
\end{align}
then with probability at least $1-8td^{-200}$,
\begin{align*}
&\|\hat M_{i,t}-M_i\|_{\rm F}^2\leq C_3\|\hat M_{i,0}-M_i\|_{\rm F}^2\cdot \prod_{\tau=1}^t \Big(1-\frac{\eta_{\tau}\lambda_{\min}}{4d_1d_2}\Big) +C_4\sigma_i^2\frac{r\log^2 d_1}{d_2}\cdot  \sum_{\tau=1}^t \frac{(\eta_{\tau}\lambda_{\max})^2}{\eps_{\tau}}\\
&\|\hat M_{i, t}-M_i\|_{\max}^2\leq C_3\frac{\lambda_{\min}^2r^3}{d_1d_2}\prod_{\tau=1}^t\Big(1-\frac{\eta_{\tau}\lambda_{\min}}{4d_1d_2}\Big) +C_4\sigma_i^2\frac{r^3\log^2d_1}{d_1d_2^2}\cdot  \sum_{\tau=1}^t \frac{(\eta_{\tau}\lambda_{\max})^2}{\eps_{\tau}},
\end{align*}
for $i=0, 1$. 
\end{thm}
\begin{proof}[Proof of Theorem~\ref{thm:genera-thm1}]
    Define event $\calE_t$ for $1\leq t\leq T$ as
    \begin{align*}
        \calE_t&=\bigg\{\forall i\in\{0,1\}, \forall \tau<t: \fro{\widehat{M}_{i,\tau}-M_i}^2\lesssim \prod_{k=1}^{\tau} (1-\frac{f(\eta_k)}{2}) \fro{\widehat{M}_{i,0}-M}^2 + \sum_{k=1}^{\tau}\frac{\eta_k^2r\lambda_{\max}^2\sigma^2\log^2d_1}{\varepsilon_kd_2}, \\
        &\hspace{2cm}\|\widehat{M}_{i,\tau} - M_i\|_{\max}^2 \lesssim \prod_{k=1}^{\tau} (1-\frac{f(\eta_k)}{2}) \frac{\text{poly}(\mu,r,\kappa)\lambda_{\min}^2}{d_1d_2}+ \sum_{k=1}^{\tau} \frac{\eta_k^2\text{poly}(\mu,r,\kappa)\lambda_{\max}^2\sigma^2\log^2 d_1}{\varepsilon_k d_1d_2^2}, \\
        &\hspace{2cm} \|e_l^{\top} \widehat{U}_{i,\tau}\widehat{V}_{i,\tau}^{\top}\|^2 \leq \frac{\mu_0r\lambda_{\max}^2}{d_1}, \|e_j^{\top} \widehat{V}_{i,\tau}\widehat{U}_{i,\tau}^{\top}\|^2 \leq \frac{\mu_0r\lambda_{\max}^2}{d_2},\quad \forall l \in [d_1], j \in [d_2] \bigg\},
    \end{align*}
    where $f(\eta_k)=\frac{\eta_k\lambda_{\min}}{d_1d_2}$ and $\mu_0=100\mu$. \\
    Recall in Algorithm \ref{alg:mcb1}, at each step, we update $\widetilde{U}_{i,t}$ and $\widetilde{V}_{i,t}$ by 
    \begin{align*}
        \widetilde{U}_{i,t}&=\widehat{U}_{i,t-1}-\frac{\mathbbm{1}(a_t=i)}{i\pi_t+(1-i)(1-\pi_t)}\eta_t (\inp{\widehat{U}_{i,t-1}\widehat{V}_{i,t-1}^{\top}}{X_t}-r_t)X_t\widehat{V}_{i,t-1}, \\
        \widetilde{V}_{i,t}&=\widehat{V}_{i,t-1}-\frac{\mathbbm{1}(a_t=i)}{i\pi_t+(1-i)(1-\pi_t)}\eta_t (\inp{\widehat{U}_{i,t-1}\widehat{V}_{i,t-1}^{\top}}{X_t}-r_t)X_t^{\top}\widehat{U}_{i,t-1}.
    \end{align*}
    Then
    \begin{align*}
        \widehat{U}_{i,t}\widehat{V}_{i,t}^{\top}=&\widetilde{U}_{i,t}\widetilde{V}_{i,t}^{\top}=\widehat{U}_{i,t-1}\widehat{V}_{i,t-1}^{\top}- \frac{\mathbbm{1}(a_t=i)}{i\pi_t+(1-i)(1-\pi_t)}\eta_t(\inp{U_{i,t-1}V_{i,t-1}^{\top}}{X_t}-y_t)X_t\widehat{V}_{i,t-1}\widehat{V}_{i,t-1}^{\top} \\ &-  \frac{\mathbbm{1}(a_t=i)}{i\pi_t+(1-i)(1-\pi_t)}\eta_t(\inp{\widehat{U}_{i,t-1}\widehat{V}_{i,t-1}^{\top}}{X_t}-r_t)\widehat{U}_{i,t-1}\widehat{U}_{i,t-1}^{\top}X_t \\ &+  \frac{\mathbbm{1}(a_t=i)}{i\pi_t^2+(1-i)(1-\pi_t)^2}\eta_t^2(\inp{\widehat{U}_{i,t-1}\widehat{V}_{i,t-1}^{\top}}{X_t}-r_t)^2X_t\widehat{V}_{i,t-1}\widehat{U}_{i,t-1}^{\top}X_t.
    \end{align*}
    Without loss of generality, in the rest of the proof, we only consider $i=1$ and omit the subscript $i$, then $i=0$ case can be proved similarly. Let
    \begin{align*}
        \Delta_t= &\frac{\mathbbm{1}(a_t=1)}{\pi_t}\eta_t(\inp{\widehat{U}_{t-1}\widehat{V}_{t-1}^{\top} - M}{X_t}-\xi_t)X_t\widehat{V}_{t-1}\widehat{V}_{t-1}^{\top} \\
        &+ \frac{\mathbbm{1}(a_t=1)}{\pi_t}\eta_t(\inp{\widehat{U}_{t-1}\widehat{V}_{t-1}^{\top}- M}{X_t}-\xi_t)\widehat{U}_{t-1}\widehat{U}_{t-1}^{\top}X_t \\
        &- \frac{\mathbbm{1}(a_t=1)}{\pi_t^2}\eta_t^2(\inp{\widehat{U}_{t-1}\widehat{V}_{t-1}^{\top}- M}{X_t}-\xi_t)^2X_t\widehat{V}_{t-1}\widehat{U}_{t-1}^{\top}X_t,
    \end{align*}
    where $\pi_t=\PP(a_t=1|X_t,\mathcal{F}_{t-1})$. In the following steps, we prove that if event $\mathcal{E}_{t-1}$ happens, event $\mathcal{E}_t$ will happen with high probability.\\
    \emph{Step 1: bounding $\fro{\widehat{U}_t\widehat{V}_t^{\top}-M}^2$.} \\
    We have $\widehat{U}_t\widehat{V}_t=\widehat{U}_{t-1}\widehat{V}_{t-1}^{\top}-\Delta_t$ and 
    \begin{align*}
        \EE[\fro{\widehat{U}_t\widehat{V}_t^{\top}-M}^2|\mathcal{F}_{t-1}]=\fro{\widehat{U}_{t-1}\widehat{V}_{t-1}^{\top}-M}^2-2\inp{\EE[\Delta_t|\mathcal{F}_{t-1}]}{\widehat{U}_{t-1}\widehat{V}_{t-1}^{\top}-M}+\EE[\fro{\Delta_t}^2|\mathcal{F}_{t-1}].
    \end{align*}
    We first compute $2\inp{\EE[\Delta_t|\mathcal{F}_{t-1}]}{\widehat{U}_{t-1}\widehat{V}_{t-1}^{\top}-M}$,
    \begin{align*}
        &2\inp{\EE[\Delta_t|\mathcal{F}_{t-1}]}{\widehat{U}_{t-1}\widehat{V}_{t-1}^{\top}-M} =  \frac{2\eta_t}{d_1d_2} \big(\fro{(\widehat{U}_{t-1}\widehat{V}_{t-1}^{\top} - M)\widehat{V}_{t-1}}^2 + \fro{\widehat{U}_{t-1}^{\top}(\widehat{U}_{t-1}\widehat{V}_{t-1}^{\top} - M)}^2\big)\\
        &\quad - 2\eta_t^2\EE\left[\frac{\mathbbm{1}(a_t=1)}{\pi_t^2}\left(\inp{\widehat{U}_{t-1}\widehat{V}_{t-1}^{\top}-M}{X_t}^2 \inp{X_t\widehat{V}_{t-1}\widehat{U}_{t-1}^{\top}X_t}{\widehat{U}_{t-1}\widehat{V}_{t-1}^{\top}-M} \right.\right.\\ 
        &\left.\left. \quad + \xi_t^2\inp{X_t\widehat{V}_{t-1}\widehat{U}_{t-1}^{\top}X_t}{\widehat{U}_{t-1}\widehat{V}_{t-1}^{\top}-M}\right)|\mathcal{F}_{t-1}\right]\\
        &\geq \frac{\eta_t\lambda_{\min}}{d_1d_2}\fro{\widehat{U}_{t-1}\widehat{V}_{t-1}^{\top}-M}^2\\
        &\quad - \frac{4\eta_t^2}{\varepsilon_t}\|\widehat{V}_{t-1}\widehat{U}_{t-1}^{\top}\|_{\max}\|\widehat{U}_{t-1}\widehat{V}_{t-1}^{\top}-M\|_{\max}(\frac{1}{d_1d_2}\fro{\widehat{U}_{t-1}\widehat{V}_{t-1}^{\top}-M}^2+\sigma^2).
    \end{align*}
    It is clear to see that $\fro{(\widehat{U}_{t-1}\widehat{V}_{t-1}^{\top} - M)\widehat{V}_{t-1}}^2 + \fro{\widehat{U}_{t-1}^{\top}(\widehat{U}_{t-1}\widehat{V}_{t-1}^{\top} - M)}^2\geq \frac{\lambda_{\min}}{2}\fro{\widehat{U}_{t-1}\widehat{V}_{t-1}^{\top} - M}^2$, and we apply $\frac{1}{\pi_t}\leq \frac{2}{\varepsilon_t}$ in the latter term. Applying the incoherence property of $\widehat{U}_{t-1}, \widehat{V}_{t-1}$, $U$ and $V$ and Cauchy Schwarz inequality, $\|\widehat{V}_{t-1}\widehat{U}_{t-1}^{\top}\|_{\max}=O(\frac{\mu_0r\lambda_{\max}}{\sqrt{d_1d_2}})$ and $\|\widehat{U}_{t-1}\widehat{V}_{t-1}^{\top} - M\|_{\max}\leq \frac{2\mu_0 r\lambda_{\max}}{\sqrt{d_1d_2}}$ under event $\mathcal{E}_{t-1}$. Therefore,
    \begin{align*}
        2\inp{\EE[\Delta_t|\mathcal{F}_{t-1}]}{\widehat{U}_{t-1}\widehat{V}_{t-1}^{\top}-M}&\geq \frac{\eta_t\lambda_{\min}}{d_1d_2}\fro{\widehat{U}_{t-1}\widehat{V}_{t-1}^{\top}-M}^2 \\ &\quad - \frac{8\eta_t^2}{\varepsilon_t}\frac{\mu_0^2r^2\lambda_{\max}^2}{d_1d_2}(\frac{1}{d_1d_2}\fro{\widehat{U}_{t-1}\widehat{V}_{t-1}^{\top}-M}^2+\sigma^2).
    \end{align*}
    Then we compute $\EE\left[\fro{\Delta_t}^2|\mathcal{F}_{t-1}\right]$. Since $(a+b+c)^2\leq 3(a^2+b^2+c^2)$, 
    \begin{align*}
        &\EE\left[\fro{\Delta_t}^2|\mathcal{F}_{t-1}\right]\leq 3\EE\left[\frac{\mathbbm{1}(a_t=1)}{\pi_t^2}\eta_t^2(\inp{\widehat{U}_{t-1}\widehat{V}_{t-1}^{\top}-M}{X_t}-\xi_t)^2\fro{X_t\widehat{V}_{t-1}\widehat{V}_{t-1}^{\top}}^2 |\mathcal{F}_{t-1}\right] \\ &\quad +3\EE\left[\frac{\mathbbm{1}(a_t=1)}{\pi_t^2}\eta_t^2(\inp{\widehat{U}_{t-1}\widehat{V}_{t-1}^{\top} - M}{X_t}-\xi_t)^2\fro{\widehat{U}_{t-1}\widehat{U}_{t-1}^{\top}X_t}^2 |\mathcal{F}_{t-1}\right] \\ &\quad +3\EE\left[\frac{\mathbbm{1}(a_t=1)}{\pi_t^4}\eta_t^4(\inp{\widehat{U}_{t-1}\widehat{V}_{t-1}^{\top} - M}{X_t}-\xi_t)^4\fro{X_t\widehat{V}_{t-1}\widehat{U}_{t-1}^{\top}X_t}^2 |\mathcal{F}_{t-1}\right].
    \end{align*}
    For the first expectation, we have
    \begin{align*}
        &\quad \EE\left[\frac{\mathbbm{1}(a_t=1)}{\pi_t^2}\eta_t^2(\inp{\widehat{U}_{t-1}\widehat{V}_{t-1}^{\top}-M}{X_t}-\xi_t)^2\fro{X_t\widehat{V}_{t-1}\widehat{V}_{t-1}^{\top}}^2 |\mathcal{F}_{t-1}\right] \\
        &\leq \frac{2\eta_t^2}{\varepsilon_t}\left(\EE\left[\inp{\widehat{U}_{t-1}\widehat{V}_{t-1}^{\top}-M}{X_t}^2|\mathcal{F}_{t-1}\right]\|\widehat{V}_{t-1}\widehat{V}_{t-1}^{\top}\|_{2,\max}^2 + \EE[\xi_t^2]\EE\left[\fro{X_t\widehat{V}_{t-1}\widehat{V}_{t-1}^{\top}}^2|\mathcal{F}_{t-1}\right]\right) \\
        &\leq \frac{2\eta_t^2}{\varepsilon_t}\left(\frac{1}{d_1d_2}\fro{\widehat{U}_{t-1}\widehat{V}_{t-1}^{\top}-M}^2\frac{\mu_0r\lambda_{\max}^2}{d_2} + \sigma^2\frac{1}{d_2}\fro{\widehat{V}_{t-1}\widehat{V}_{t-1}^{\top}}^2\right) \\
        &\leq \frac{2\eta_t^2}{\varepsilon_td_1d_2}\fro{\widehat{U}_{t-1}\widehat{V}_{t-1}^{\top}-M}^2\frac{\mu_0r\lambda_{\max}^2}{d_2} + \frac{2\eta_t^2r\sigma^2}{\varepsilon_td_2}O(\lambda_{\max}^2).
    \end{align*}
    The last inequality holds since $\fro{\widehat{V}_{t-1}\widehat{V}_{t-1}^{\top}}^2\leq r\|\widehat{V}_{t-1}\widehat{V}_{t-1}^{\top}\|^2$ and $\|\widehat{V}_{t-1}\widehat{V}_{t-1}^{\top}\|^2=O(\lambda_{\max}^2)$. \\
    By symmetry, $\EE\left[\frac{\mathbbm{1}(a_t=1)}{\pi_t^2}\eta_t^2(\inp{\widehat{U}_{t-1}\widehat{V}_{t-1}^{\top} - M}{X_t}-\xi_t)^2\fro{\widehat{U}_{t-1}\widehat{U}_{t-1}^{\top}X_t}^2 |\mathcal{F}_{t-1}\right]\leq \frac{2\eta_t^2}{\varepsilon_td_1d_2}\fro{\widehat{U}_{t-1}\widehat{V}_{t-1}^{\top}-M}^2\frac{\mu_0r\lambda_{\max}^2}{d_1} + \frac{2\eta_t^2r\sigma^2}{\varepsilon_td_1}O(\lambda_{\max}^2)$. Similarly, the last expectation can be written as
    \begin{align*}
        &\quad \EE\left[\frac{\mathbbm{1}(a_t=1)}{\pi_t^4}\eta_t^4(\inp{\widehat{U}_{t-1}\widehat{V}_{t-1}^{\top} - M}{X_t}-\xi_t)^4\fro{X_t\widehat{V}_{t-1}\widehat{U}_{t-1}^{\top}X_t}^2 |\mathcal{F}_{t-1}\right] \\
        &\leq \frac{8\eta_t^4}{\varepsilon_t^3}\left(\EE\left[\inp{\widehat{U}_{t-1}\widehat{V}_{t-1}^{\top}-M}{X_t}^4|\mathcal{F}_{t-1}\right]\|\widehat{V}_{t-1}\widehat{U}_{t-1}^{\top}\|_{\max}^2 + \EE(\xi_t^4)\|\widehat{V}_{t-1}\widehat{U}_{t-1}^{\top}\|_{\max}^2 \right) \\
        &\lesssim \frac{\eta_t^4}{\varepsilon_t^3}\left(\frac{1}{d_1d_2}\|\widehat{U}_{t-1}\widehat{V}_{t-1}^{\top}-M\|_{\max}^2\sum(\widehat{U}_{t-1}\widehat{V}_{t-1}^{\top}-M)_{ij}^2\right)\frac{\mu_0^2r^2\lambda_{\max}^2}{d_1d_2} + \frac{\eta_t^4\sigma^4}{\varepsilon_t^3}\frac{\mu_0^2r^2\lambda_{\max}^2}{d_1d_2} \\
        &\lesssim \frac{\eta_t^4}{\varepsilon_t^3}\frac{\mu_0^2r^2\lambda_{\max}^4}{d_1^3d_2^3}\fro{\widehat{U}_{t-1}\widehat{V}_{t-1}^{\top}-M}^2 + \frac{\eta_t^4\sigma^4}{\varepsilon_t^3}\frac{\mu_0^2r^2\lambda_{\max}^2}{d_1d_2}.
    \end{align*}
    As long as $\eta_t\lesssim \frac{\varepsilon_td_1d_2^{1/2}}{\sqrt{\mu_0r}\lambda_{\max}}$ and $\frac{\lambda_{\min}^2}{\sigma^2}\gg \frac{\eta_t^2}{\varepsilon_t^2}\frac{\lambda_{\max}^2}{d_1r\kappa^2}$, the second order terms of $\eta_t$ above will dominate the fourth order terms. Above all, for some positive constants $c_1$ and $c_2$, 
    \begin{align*}
        &\quad - 2\EE\left(\inp{\Delta_t}{\widehat{U}_{t-1}\widehat{V}_{t-1}^{\top}-M}|\mathcal{F}_{t-1}\right) + \EE\left(\fro{\Delta_t}^2|\mathcal{F}_{t-1}\right) \\
        &\leq (-\frac{\eta_t\lambda_{\min}}{d_1d_2} + \frac{c_1\eta_t^2\mu_0r\lambda_{\max}^2}{\varepsilon_td_1d_2^2})\fro{\widehat{U}_{t-1}\widehat{V}_{t-1}^{\top}-M}^2+\frac{c_2\eta_t^2r\lambda_{\max}^2\sigma^2}{\varepsilon_td_2} \\
        &\leq -f(\eta_t)\fro{\widehat{U}_{t-1}\widehat{V}_{t-1}^{\top}-M}^2+\frac{c_2\eta_t^2r\lambda_{\max}^2\sigma^2}{\varepsilon_td_2},
    \end{align*}
    where $f(\eta_t)=\frac{\eta_t\lambda_{\min}}{2d_1d_2}$. The second order term $\frac{c_1\eta_t^2\mu_0r\lambda_{\max}^2}{\varepsilon_td_1d_2^2}$ is dominated as long as $\eta_t\lesssim \frac{\varepsilon_t d_2}{\mu_0r\kappa\lambda_{\max}}$.\\
    Next, we use telescoping to derive the relationship between $\fro{\widehat{U}_{t-1}\widehat{V}_{t-1}^{\top}-M}^2$ and $\fro{\widehat{U}_{0}\widehat{V}_{0}^{\top}-M}^2$. Notice that,
    \begin{align*}
        \fro{\widehat{U}_{t}\widehat{V}_{t}^{\top} - M}^2 &= \fro{\widehat{U}_{t-1}\widehat{V}_{t-1}^{\top} - \Delta_t - M}^2 = \fro{\widehat{U}_{t-1}\widehat{V}_{t-1}^{\top} - M}^2 - 2\inp{\Delta_t}{\widehat{U}_{t-1}\widehat{V}_{t-1}^{\top} - M} +\fro{\Delta_t}^2\\
        & = \fro{\widehat{U}_{t-1}\widehat{V}_{t-1}^{\top} - M}^2 - 2(\inp{\Delta_t}{\widehat{U}_{t-1}\widehat{V}_{t-1}^{\top}-M} -\EE(\inp{\Delta_t}{\widehat{U}_{t-1}\widehat{V}_{t-1}^{\top}-M}|\mathcal{F}_{t-1})) \\
        &\quad+ (\fro{\Delta_t}^2 - \EE(\fro{\Delta_t}^2|\mathcal{F}_{t-1})) - 2\EE(\inp{\Delta_t}{\widehat{U}_{t-1}\widehat{V}_{t-1}^{\top}-M}|\mathcal{F}_{t-1}) + \EE(\fro{\Delta_t}^2|\mathcal{F}_{t-1})\\
        & \leq (1-f(\eta_t))\fro{\widehat{U}_{t-1}\widehat{V}_{t-1}^{\top} - M}^2 + \frac{c_2\eta_t^2r\lambda_{\max}^2\sigma^2}{\varepsilon_td_2}  \\
        &\quad + \underbrace{2(\EE(\inp{\Delta_t}{\widehat{U}_{t-1}\widehat{V}_{t-1}^{\top}-M}|\mathcal{F}_{t-1}) -\inp{\Delta_t}{\widehat{U}_{t-1}\widehat{V}_{t-1}^{\top}-M})}_{A_{t}} +\underbrace{(\fro{\Delta_t}^2 - \EE(\fro{\Delta_t}^2|\mathcal{F}_{t-1}))}_{B_{t}} \\
        &= (1-f(\eta_t))(1-f(\eta_{t-1})) \fro{\widehat{U}_{t-2}\widehat{V}_{t-2}^{\top} - M}^2 + \frac{c_2\eta_t^2r\lambda_{\max}^2\sigma^2}{\varepsilon_td_2} + (1-f(\eta_{t}))\frac{c_2\eta_{t-1}^2r\lambda_{\max}^2\sigma^2}{\varepsilon_{t-1}d_2} \\
        &\quad + (1-f(\eta_t))(A_{t-1}+B_{t-1})+(A_{t}+B_{t}) \\
        &= \prod_{k=1}^{t}(1-f(\eta_k)) \fro{\widehat{U}_{0}\widehat{V}_{0}^{\top} - M}^2 + \sum_{\tau=1}^{t-1}\prod_{k=\tau+1}^{t}(1-f(\eta_k))\frac{c_2\eta_{\tau}^2r\lambda_{\max}^2\sigma^2}{\varepsilon_{\tau}d_2} + \frac{c_2\eta_t^2r\lambda_{\max}^2\sigma^2}{\varepsilon_td_2}  \\
        &\quad + \sum_{\tau=1}^{t-1}\prod_{k=\tau+1}^{t} (1-f(\eta_k))A_{\tau} + \sum_{\tau=1}^{t}\prod_{k=\tau+1}^{t} (1-f(\eta_k))B_{\tau} + A_{t} + B_{t} \\
        &\leq \prod_{k=1}^{t}(1-f(\eta_k)) \fro{\widehat{U}_{0}\widehat{V}_{0}^{\top} - M}^2 + \sum_{k=1}^{t}\frac{c_2\eta_k^2r\lambda_{\max}^2\sigma^2\log^2d_1}{\varepsilon_kd_2} \\
        &\quad + \sum_{\tau=1}^{t-1}\prod_{k=\tau+1}^{t} (1-f(\eta_k))A_{\tau} + \sum_{\tau=1}^{t-1}\prod_{k=\tau+1}^{t} (1-f(\eta_k))B_{\tau} + A_{t} + B_{t}. 
    \end{align*}
    In Lemma \ref{lemmafro}, we show $\sum_{\tau=1}^{t-1}\prod_{k=\tau+1}^{t} (1-f(\eta_k))A_{\tau}+A_t$ and $\sum_{\tau=1}^{t}\prod_{k=\tau+1}^{t} (1-f(\eta_k))B_{\tau}+B_{t}$ will be dominated.
    \begin{Lemma}
        \label{lemmafro} For any $1\leq t\leq T$, under the conditions in Theorem \ref{thm:MCB-conv} and event $\mathcal{E}_{t-1}$, with probability at least $1-td_1^{-200}$,
        \begin{align*}
            &\sum_{\tau=1}^{t-1}\prod_{k=\tau+1}^{t} (1-f(\eta_k))A_{\tau}+A_t\lesssim \gamma\prod_{k=1}^{t}(1-\frac{f(\eta_k)}{2}) \fro{\widehat{U}_{0}\widehat{V}_{0}^{\top} - M}^2 + \gamma\sum_{k=1}^{t}\frac{\eta_k^2r\lambda_{\max}^2\sigma^2\log^2d_1}{\varepsilon_kd_2}, \\
            &\sum_{\tau=1}^{t-1}\prod_{k=\tau+1}^{t} (1-f(\eta_k))B_{\tau}+B_t\lesssim \gamma\prod_{k=1}^{t}(1-\frac{f(\eta_k)}{2}) \fro{\widehat{U}_{0}\widehat{V}_{0}^{\top} - M}^2 + \gamma\sum_{k=1}^{t}\frac{\eta_k^2r\lambda_{\max}^2\sigma^2\log^2d_1}{\varepsilon_kd_2}.
        \end{align*}
    \end{Lemma}
    \noindent Combine with the union bound, with probability at least $1-td_1^{-200}$,
    \begin{align*}
        \fro{U_tV_t^{\top} - M}^2&\lesssim \prod_{k=1}^{t} (1-\frac{f(\eta_k)}{2}) \fro{U_0V_0^{\top}-M}^2 + \sum_{k=1}^{t}\frac{\eta_k^2r\lambda_{\max}^2\sigma^2\log^2d_1}{\varepsilon_kd_2}.
    \end{align*}
    \emph{Step 2: proof of incoherence.} \\
    In Lemma \ref{incoherence}, we prove the incoherence property holds true for every $U_t$ and $V_t$.
    \begin{Lemma}
        \label{incoherence} 
        For any $1\leq t\leq T$, under the conditions in Theorem \ref{thm:MCB-conv} and event $\mathcal{E}_{t-1}$, with probability at least $1-td_1^{-200}$,
        \begin{align*}
            \|e_l^{\top} \widehat{U}_{i,\tau}\widehat{V}_{i,\tau}^{\top}\|^2 \leq \frac{\mu_0r\lambda_{\max}^2}{d_1}\quad  \|e_j^{\top} \widehat{V}_{i,\tau}\widehat{U}_{i,\tau}^{\top}\|^2 \leq \frac{\mu_0r\lambda_{\max}^2}{d_2},\quad \forall l \in [d_1], j \in [d_2],
        \end{align*}
        where $\mu_0=100\mu$.
    \end{Lemma}
    \noindent \emph{Step 3: bounding $\|\widehat{U}_t\widehat{V}_t^{\top}-M\|_{\max}^2$.} \\
    In order to bound $\|\widehat{U}_t\widehat{V}_t^{\top}-M\|_{\max}^2$, we need to bound $\|e_l^{\top}(\widehat{U}_t\widehat{V}_t^{\top}-M)\|^2$ and $\|e_j^{\top}(\widehat{V}_t\widehat{U}_t^{\top}-M^{\top})\|^2$ for any $l \in [d_1], j\in [d_2]$ first.
    \begin{Lemma}
        \label{lemma2max}
        For any $1\leq t\leq T$, under the conditions in Theorem \ref{thm:MCB-conv} and event $\mathcal{E}_{t-1}$, with probability at least $1-d_1^{-200}$, 
        \begin{align*}
            &\|e_l^{\top}(\widehat{U}_t\widehat{V}_t^{\top}-M)\|^2\lesssim \prod_{k=1}^{t}(1-\frac{f(\eta_k)}{2}) \frac{\text{poly}(\mu,r,\kappa)\lambda_{\min}^2}{d_1} +   \sum_{k=1}^{t} \frac{\eta_k^2\text{poly}(\kappa,\mu,r)\lambda_{\max}^2\sigma^2\log^2d_1}{\varepsilon_k d_1d_2}, \\
            &\|e_j^{\top}(\widehat{V}_t\widehat{U}_t^{\top}-M^{\top})\|^2\lesssim \prod_{k=1}^{t}(1-\frac{f(\eta_k)}{2}) \frac{\text{poly}(\mu,r,\kappa)\lambda_{\min}^2}{d_2} +  \sum_{k=1}^{t} \frac{\eta_k^2\text{poly}(\kappa,\mu,r)\lambda_{\max}^2\sigma^2\log^2d_1}{\varepsilon_k d_2^2}
        \end{align*}
        for $\forall l \in [d_1], j\in [d_2]$.
    \end{Lemma}
    \noindent Now we are ready to prove the upper bound for $\|\widehat{U}_t\widehat{V}_t^{\top}-M\|_{\max}^2$. For simplicity, denote $\widehat{M}_t=\widehat{U}_t\widehat{V}_t^{\top}$ .The SVD of $\widehat{M}_t$ and $M$ are $M_t=\widehat{L}_t\widehat{\Lambda}_t\widehat{R}_t$ and $M=L\Lambda R$, respectively. Then, $\widehat{M}_t - M$ can be written as
    \begin{align*}
        \widehat{M}_t - M&= \widehat{L}_t\widehat{L}_t^{\top}\widehat{M}_t\widehat{R}_t\widehat{R}_t^{\top} - LL^{\top}MRR^{\top} \\
        &= \widehat{L}_t\widehat{L}_t^{\top}(\widehat{M}_t-M)\widehat{R}_t\widehat{R}_t^{\top} + (\widehat{L}_t\widehat{L}_t^{\top}-LL^{\top})M\widehat{R}_t\widehat{R}_t^{\top} + LL^{\top}M(\widehat{R}_t\widehat{R}_t^{\top}-RR^{\top})
    \end{align*}
    As a result, for any $l\in [d_1]$, $j\in [d_2]$, $e_l^{\top}(M_t - M)e_j$ can be decomposed as
    \begin{align*}
        \|e_l^{\top}(\widehat{M}_t - M)e_j\|^2&\lesssim \underbrace{\|e_l^{\top}\widehat{L}_t\widehat{L}_t^{\top}(\widehat{M}_t-M)\widehat{R}_t\widehat{R}_t^{\top}e_j\|^2}_{\uppercase\expandafter{\romannumeral1}}  + \underbrace{\|e_l^{\top}(\widehat{L}_t\widehat{L}_t^{\top}-LL^{\top})M\widehat{R}_t\widehat{R}_t^{\top}e_j\|^2}_{\uppercase\expandafter{\romannumeral2}} \\ &\quad + \underbrace{\|e_l^{\top}LL^{\top}M(\widehat{R}_t\widehat{R}_t^{\top}-RR^{\top})e_j\|^2}_{\uppercase\expandafter{\romannumeral3}}.
    \end{align*}
    Term $\uppercase\expandafter{\romannumeral1}$ can be bounded in the following way,
    \begin{align*}
        \uppercase\expandafter{\romannumeral1}&\leq \|e_l^{\top}\widehat{L}_t\widehat{L}_t^{\top}\|^2\fro{M_t-M}^2\|\widehat{R}_t\widehat{R}_t^{\top}e_j\|^2 \\
        &\lesssim \frac{\mu_0r}{d_1}\left(\prod_{k=1}^{t} (1-\frac{f(\eta_k)}{2})\fro{\widehat{M}_0-M}^2 + \sum_{k=1}^{t}\frac{\eta_k^2r\lambda_{\max}^2\sigma^2\log^2 d_1}{\varepsilon_k d_2}\right)\frac{\mu_0r}{d_2} \\
        &\lesssim \prod_{k=1}^{t} (1-\frac{f(\eta_k)}{2}) \frac{\mu_0^2r^2\fro{\widehat{M}_0-M}^2}{d_1d_2} + \sum_{k=1}^{t}\frac{\eta_k^2\mu_0^2r^3\lambda_{\max}^2\sigma^2\log^2 d_1}{\varepsilon_k d_1d_2^2} \\
        &\lesssim \prod_{k=1}^{t} (1-\frac{f(\eta_k)}{2}) \frac{\mu_0^2r^2\lambda_{\min}^2}{d_1d_2} + \sum_{k=1}^{t}\frac{\eta_k^2\mu_0^2r^3\lambda_{\max}^2\sigma^2\log^2 d_1}{\varepsilon_k d_1d_2^2}.
    \end{align*}
    The last inequality holds since $\fro{\widehat{M}_0-M}^2\lesssim \lambda_{\min}^2$. Next, we look at Term $\uppercase\expandafter{\romannumeral2}$. Let $O_{\widehat{L}_t}=\mathop{\arg\min}_{O\in \mathbb{O}_r} \|\widehat{L}_t-LO\|$, where $\mathbb{O}_r$ denotes the set of all $r\times r$ real matrices with orthonormal columns, then
    \begin{align*} 
        \uppercase\expandafter{\romannumeral2}&\lesssim \|e_l^{\top}(\widehat{L}_t - LO_{\widehat{L}_t})\widehat{L}_t^{\top}M\widehat{R}_t\widehat{R}_t^{\top}e_j\|^2 + \|e_l^{\top}LO_{\widehat{L}_t}(\widehat{L}_t^{\top}-O_{\widehat{L}_t}^{\top}L^{\top})M\widehat{R}_t\widehat{R}_t^{\top}e_j\|^2.
    \end{align*}
    The second part can be bounded by,
    \begin{align*}
        \|e_l^{\top}LO_{\widehat{L}_t}(\widehat{L}_t^{\top}-O_{\widehat{L}_t}^{\top}L^{\top})M\widehat{R}_t\widehat{R}_t^{\top}e_j\|^2&\leq \|e_l^{\top}LO_{\widehat{L}_t}\|^2\|\widehat{L}_t^{\top}-O_{\widehat{L}_t}^{\top}L^{\top}\|^2\|M\widehat{R}_t\widehat{R}_t^{\top}e_j\|^2 \\
        &\leq \frac{\mu r}{d_1}\|\widehat{L}_t^{\top}-O_{\widehat{L}_t}^{\top}L^{\top}\|^2\frac{\mu_0r\lambda_{\max}^2}{d_2},
    \end{align*}
    where we use the incoherence property and the consistency of matrix norm and vector norm. By Wedin’s $\sin\Theta$ Theorem (\cite{davis1970rotation,wedin1972perturbation}), 
    \begin{align*}
        \|\widehat{L}_t^{\top}-O_{\widehat{L}_t}^{\top}L^{\top}\|^2&\leq \frac{4\fro{\widehat{M}_t-M}^2}{\lambda_{\min}^2} \\
        &\lesssim \prod_{k=1}^{t} (1-\frac{f(\eta_k)}{2})\frac{\fro{\widehat{M}_0-M}^2}{\lambda_{\min}^2} + \sum_{k=1}^{t}\frac{\eta_k^2r\kappa^2 \sigma^2\log^2 d_1}{\varepsilon_k d_2}\\
        &\lesssim \gamma\prod_{k=1}^{t} (1-\frac{f(\eta_k)}{2}) + \gamma\sum_{k=1}^{t}\frac{\eta_k^2r\kappa^2 \sigma^2\log^2 d_1}{\varepsilon_k d_2}.
    \end{align*}
    The first part can be bounded by the upper bound of $\|e_l^{\top}(\widehat{U}_t\widehat{V}_t^{\top} - M)\|^2$ we have obtained in Lemma \ref{lemma2max},
    \begin{align*}
        \|e_l^{\top}(\widehat{L}_t - LO_{\widehat{L}_t})L_t^{\top}M\widehat{R}_t\widehat{R}_t^{\top}e_j\|^2&\leq \|e_i^{\top}(\widehat{L}_t - LO_{\widehat{L}_t})\|^2\|\widehat{L}_t^{\top}M\|^2\|\widehat{R}_t\widehat{R}_t^{\top}e_j\|^2 \\
        &\leq \|e_i^{\top}(\widehat{L}_t - LO_{\widehat{L}_t})\|^2\frac{\mu_0r\lambda_{\max}^2}{d_2}.
    \end{align*}
    By Theorem3.7 in \cite{cape2019two} and Wedin’s $\sin\Theta$ Theorem again, 
    \begin{align*}
        \|e_l^{\top}(\widehat{L}_t - LO_{\widehat{L}_t})\|&\leq 4\frac{\|(I-LL^{\top})(\widehat{M}_t-M)\|_{2,\max}}{\lambda_{\min}}  + \|\sin\Theta(\widehat{L}_t,L)\|^2\|L\|_{2,\max} \\
        &\leq 4\frac{\|\widehat{M}_t-M\|_{2,\max}+\|LL^{\top}\|\|\widehat{M}_t-M\|_{2,\max}}{\lambda_{\min}} + \|\sin\Theta(\widehat{L}_t,L)\|^2\sqrt{\frac{\mu r}{d_1}} \\
        &\leq 8\frac{\|\widehat{M}_t-M\|_{2,\max}}{\lambda_{\min}} +  2\frac{\fro{\widehat{M}_t-M}}{\lambda_{\min}}\sqrt{\frac{\mu r}{d_1}}.
    \end{align*}
    Then
    \begin{align*}
        \|e_i^{\top}(\widehat{L}_t - LO_{\widehat{L}_t})\|^2&\lesssim \frac{\|\widehat{M}_t-M\|_{2,\max}^2}{\lambda_{\min}^2} + \frac{\fro{\widehat{M}_t-M}^2}{\lambda_{\min}^2}\frac{\mu r}{d_1} \\
        &\lesssim \prod_{k=t}^{t} (1-\frac{f(\eta_k)}{2})\frac{\text{poly}(\mu,r,\kappa)}{d_1} + \sum_{k=1}^{t}\frac{\eta_k^2\sigma^2\log^2 d_1\text{poly}(\mu,r,\kappa)}{\varepsilon_k d_1d_2}.
    \end{align*}
    Above all, 
    \begin{align*}
        \uppercase\expandafter{\romannumeral2}&\lesssim \prod_{k=t}^{t} (1-\frac{f(\eta_k)}{2})\frac{\text{poly}(\mu,r,\kappa)\lambda_{\min}^2}{d_1d_2} + \sum_{k=1}^{t}\frac{\eta_k^2\lambda_{\max}^2\sigma^2\log^2 d_1\text{poly}(\mu,r,\kappa)}{\varepsilon_k d_1d_2^2}.
    \end{align*}
    By symmetry, following the same arguments, \uppercase\expandafter{\romannumeral3} will also have the same upper bound. \\
    Combine the above three terms together as well as the union bound, we can conclude that with probability at least $1-8td_1^{-200}$,
    \begin{align*}
        \|\widehat{M}_t-M\|_{\max}^2\lesssim \prod_{k=1}^{t} (1-\frac{f(\eta_k)}{2})\frac{\text{poly}(\mu,r,\kappa)\lambda_{\min}^2}{d_1d_2} + \sum_{k=1}^{t}\frac{\eta_k^2\text{poly}(\mu,r,\kappa)\lambda_{\max}^2\sigma^2\log^2 d_1}{\varepsilon_k d_1d_2^2}.
    \end{align*}

    Next we derive the upper bound of $\fro{\widehat{M}_{i,t}-M_i}^2$ under specific choices of $\eta_t$ and $\varepsilon_t$ for all $1\leq t\leq T$, and the entrywise error bound can be derived following the same arguments. When $1\leq t\leq T_0=C_0T^{1-\gamma}$, $\eta_t=\eta=\frac{c_1d_1d_2\log d_1}{T^{1-\gamma}\lambda_{\max}^2}$ and $\varepsilon_t=\varepsilon$. Then
    \begin{align*}
    \fro{\widehat{M}_{i,t}-M_i}^2&\lesssim \fro{\widehat{M}_{i,0}-M_i}^2\prod_{\tau=1}^{t} (1-\frac{c_1\log d_1}{4\kappa T^{1-\gamma}}) + \sigma_i^2\frac{r\log^2 d_1}{d_2}\frac{tc_1^2d_1^2d_2^2\log^2d_1}{T^{2(1-\gamma)}\varepsilon} \\
    &\lesssim \fro{\widehat{M}_{i,0}-M_i}^2(1-\frac{c_1\log d_1}{4\kappa T^{1-\gamma}})^t + \frac{trd_1^2d_2\log^4d_1}{T^{2(1-\gamma)}}.
    \end{align*}
    When $T_0+1\leq t\leq T$, $\varepsilon_t=c_2t^{-\gamma}$ and $\eta_t=\varepsilon_t\eta$. Note that the contraction term becomes $(1-\frac{c_1\log d_1}{4\kappa T^{1-\gamma}})^{T_0}\prod_{\tau=T_0+1}^{t} (1-\frac{c_1c_2\varepsilon_t\log d_1}{4\kappa T^{1-\gamma}})\leq (1-\frac{c_1\log d_1}{4\kappa T^{1-\gamma}})^{C_0T^{1-\gamma}}=O(d_1^{-C_0c_1/4\kappa})$. Since $C_0$ can be sufficiently large, $O(d_1^{-C_0c_1/4\kappa})$ can be small enough to make the first computational error term dominated by the second statistical error term. Therefore, 
    \begin{align}
    \fro{\widehat{M}_{i,t}-M_i}^2&\lesssim \sigma_i^2rd_1^2d_2\log^4d_1\frac{T^{1-\gamma}+\sum_{\tau=T_0+1}^{t} \tau^{-\gamma}}{T^{2(1-\gamma)}} \lesssim \frac{\sigma_i^2rd_1^2d_2\log^4d_1}{T^{1-\gamma}},
    \end{align}
    where the second inequality comes from the fact that $\sum_{\tau=T_0+1}^{t} \tau^{-\gamma}\leq \sum_{\tau=1}^{t} \tau^{-\gamma}=O(t^{1-\gamma})\leq O(T^{1-\gamma})$.
    \end{proof}

\subsection{Proof of Theorem \ref{thm:regret}}
\begin{proof}
Without loss of generality, assume at time $t$, arm 1 is optimal, i.e., $\inp{M_1-M_0}{X_t}>0$, then the expected regret is
\begin{align*}
    r_t&=\EE\left[\inp{M_1-M_0}{X_t}\mathbbm{1}(\text{choose arm 0}) \right] \\
    &= \EE\left[\inp{M_1-M_0}{X_t}\mathbbm{1}(\text{choose arm 0}, \inp{\widehat{M}_{1,t-1}-\widehat{M}_{0,t-1}}{X_t}\geq 0)\right] \\ &\quad + \EE\left[\inp{M_1-M_0}{X_t}\mathbbm{1}(\text{choose arm 0}, \inp{\widehat{M}_{0,t-1}-\widehat{M}_{1,t-1}}{X_t}> 0)\right] \\
    &\leq \frac{\varepsilon_t}{2}\|M_1-M_0\|_{\max} + \underbrace{\EE\left[\inp{M_1-M_0}{X_t}\mathbbm{1}(\inp{\widehat{M}_{0,t-1}-\widehat{M}_{1,t-1}}{X_t}> 0)\right]}_{(2)}.
\end{align*}
Define $\delta_t^2$ the entrywise bound in Theorem \ref{thm:MCB-conv} such that $\|U_{i,t}V_{i,t}^{\top} - M\|_{\max}^2\leq \delta_t^2$ with probability at least $1-8td_1^{-200}$, and event $B_t=\{\inp{M_1-M_0}{X_t}>2\delta_t\}$, then the latter expectation can be written as
\begin{align*}
    (2)&\leq \EE\left[\inp{M_1-M_0}{X_t}\mathbbm{1}\left(\inp{\widehat{M}_{0,t-1}-\widehat{M}_{1,t-1}}{X_t}> 0\cap B_t\right)\right] \\ &\quad + \EE\left[\inp{M_1-M_0}{X_t}\mathbbm{1}\left(\inp{\widehat{M}_{0,t-1}-\widehat{M}_{1,t-1}}{X_t}> 0\cap B_t^c\right)\right] \\
    &\leq \|M_1-M_0\|_{\max} \EE\left[\mathbbm{1}\left(\inp{\widehat{M}_{0,t-1}-\widehat{M}_{1,t-1}}{X_t}> 0\cap B_t\right)\right] + 2\delta_t\EE\left[\mathbbm{1}(B_t^c)\right].
\end{align*}
Under event $B_t$ and $\inp{\widehat{M}_{0,t-1}-\widehat{M}_{1,t-1}}{X_t}> 0$, 
\begin{align*}
    0>\inp{\widehat{M}_{1,t-1}-\widehat{M}_{0,t-1}}{X_t} &= \inp{\widehat{M}_{1,t-1}-M_1}{X_t} + \inp{M_0-\widehat{M}_{0,t-1}}{X_t} + \inp{M_1-M_0}{X_t}\\
    &> \inp{\widehat{M}_{1,t-1}-M_1}{X_t} + \inp{M_0-\widehat{M}_{0,t-1}}{X_t} + 2\delta_t.
\end{align*}
This means either $\inp{\widehat{M}_{1,t-1}-M_1}{X_t}<-\delta_t$ or $\inp{M_0-\widehat{M}_{0,t-1}}{X_t}<-\delta_t$, which further implies either $\|\widehat{M}_{1,t-1}-M_1\|_{\max}$ or $\|M_0-\widehat{M}_{0,t-1}\|_{\max}$ should be larger than $\delta_t$. Therefore, 
\begin{align*}
    \EE\left[\mathbbm{1}(\inp{{M}_{0,t-1}-{M}_{1,t-1}}{X_t}> 0)\cap B_t\right]&= \PP\left(\inp{{M}_{0,t-1}-{M}_{1,t-1}}{X_t}> 0 \cap B_t\right) \\
    &\leq \PP\left(\|{M}_{1,t-1}-M_1\|_{\max}>\delta_t\right) + \PP\left(\|{M}_{0,t-1}-M_0\|_{\max}>\delta_t\right) \\
    &\leq \frac{16(t-1)}{d^{200}},
\end{align*}
For the other term, since $\EE\left[\mathbbm{1}(B_t^c)\right]=\PP(\inp{M_1-M_0}{X_t}<2\delta_t)\leq 1$, then
\begin{align*}
    r_t&\leq \varepsilon_t\|M_1-M_0\|_{\max} + \frac{16(t-1)}{d^{200}}\|M_1-M_0\|_{\max}  + 2\delta_t. 
\end{align*}
The total regret up to time $T$ is
\begin{align*}
    R_T&=\sum_{t=1}^{T} r_t \leq \sum_{t=1}^T\varepsilon_t \|M_1-M_0\|_{\max} + \frac{8T(T-1)}{d_1^{200}}\|M_1-M_0\|_{\max}  + 2\sum_{t=1}^{T} \delta_t.
\end{align*}
In Theorem \ref{thm:MCB-conv}, we have shown that $\delta_t^2= C_3\frac{\lambda_{\min}^2r^3}{d_1d_2}  \prod_{\tau=1}^{t}(1-\frac{c_1\log d_1}{4\kappa T^{1-\gamma}}) + C_4\frac{t\cdot rd_1\log^4d_1\sigma_i^2}{T^{2(1-\gamma)}}$ when $1\leq t\leq T_0$ and $\delta_t^2=C_5\frac{rd_1\log^4d_1\bar{\sigma}^2}{T^{1-\gamma}}$ when $T_0+1\leq t\leq T$, for some constants $C_3$ ,$C_4$ and $C_5$. Therefore,
\begin{align*}
    R_T&\lesssim \sum_{t=1}^{T_0}\varepsilon \|M_1-M_0\|_{\max} + \sum_{t=T_0 + 1}^{T}\frac{1}{t^\gamma}\|M_1-M_0\|_{\max}   + \frac{\|M_1-M_0\|_{\max}T^2}{d_1^{200}} \\ &\quad + \sum_{t=1}^{T_0} \sqrt{\frac{\lambda_{\min}^2r^3}{d_1d_2}  \prod_{\tau=1}^{t}(1-\frac{c_1\log d_1}{4\kappa T^{1-\gamma}})} + \sum_{t=1}^{T_0} \sqrt{\frac{t\cdot rd_1\log^4d_1\bar{\sigma}^2}{T^{2(1-\gamma)}}} + \sum_{t=T_0+1}^{T} \sqrt{\frac{rd_1\log^4d_1\bar{\sigma}^2}{T^{1-\gamma}}}.
\end{align*}
Since $T\ll d_1^{50}$, the third term is negligible. As a result, from the fact that $\sum_{t=1}^T t^{-\gamma}\lesssim T^{1-\gamma}$,
\begin{align*}
    R_T&\lesssim T^{1-\gamma}\|M_1-M_0\|_{\max} + T^{1-\gamma}\frac{\lambda_{\min}r^{3/2}}{\sqrt{d_1d_2}}  + T^{(1+\gamma)/2}\sqrt{rd_1\log^4d_1\bar{\sigma}^2} \\
    &\lesssim T^{1-\gamma}(\|M_1\|_{\max} + \|M_0\|_{\max}) + T^{(1+\gamma)/2}\sqrt{rd_1\log^4d_1\bar{\sigma}^2}.
\end{align*}
\end{proof}

\subsection{Proof of Theorem \ref{thm:CLT}}
\begin{proof}
Without loss of generality, we only consider the proof for $i=1$, and omit all the subscript. \\
We follow the arguments in \cite{xia2021statistical} to decompose $\inp{\widehat{M}}{Q}-\inp{M}{Q}$. Notice that
\begin{align*}
    \inp{\widehat{M}}{Q}-\inp{M}{Q}=\inp{\widehat{L}\widehat{L}^{\top}\widehat{Z}\widehat{R}\widehat{R}^{\top}}{Q} + \inp{\widehat{L}\widehat{L}^{\top}M\widehat{R}\widehat{R}^{\top}-M}{Q}.
\end{align*}
Let $b=\frac{T}{T-T_0}$, notice that $b=O(1)$ and will converge to 1 as $T\rightarrow \infty$. We can write $\widehat{M}^{\uipw}$ in (\ref{eq:debias-aw}) into
\begin{align*}
    \widehat{M}^{\uipw}&= \frac{b}{T}\sum_{t=T_0+1}^{T}\widehat{M}_{t-1}+ \frac{b}{T}\sum_{t=T_0+1}^{T} \frac{d_1d_2\mathbbm{1}(a_t=1)}{\pi_t}(y_t-\inp{{M}_{t-1}}{X_t})X_t \\
    &= M_{1} + \underbrace{\frac{b}{T}\sum_{t=T_0+1}^T \frac{d_1d_2\mathbbm{1}(a_t=1)}{\pi_t}\xi_tX_t}_{=:\widehat{Z}_1} + \underbrace{\frac{b}{T}\sum_{t=T_0+1}^{T} \left(\frac{d_1d_2\mathbbm{1}(a_t=1)}{\pi_t}\inp{\widehat{\Delta}_{t-1}}{X_t}X_t-\widehat{\Delta}_{t-1}\right)}_{=:\widehat{Z}_2}.
\end{align*}
Moreover, recall that $\widehat{L}$ and $\widehat{R}$ are the top-$r$ left and right singular vectors of $\widehat{M}^{\uipw}$. Define $(d_1+d_2)\times (2r)$ matrices 
$$
    {\Theta}= \begin{pmatrix}
        L & 0 \\
        0 & R
    \end{pmatrix} \quad 
    \widehat{\Theta}= \begin{pmatrix}
        \widehat{L} & 0 \\
        0 & \widehat{R}
    \end{pmatrix} \quad 
    A = \begin{pmatrix}
        0 & M \\
        M^{\top} & 0
    \end{pmatrix}.
$$
Then, we can write
$$
    \widehat{\Theta}\widehat{\Theta}^{\top}A\widehat{\Theta}\widehat{\Theta}^{\top}-\Theta\Theta^{\top}A\Theta\Theta^{\top}= \begin{pmatrix}
        0 & \widehat{L}\widehat{L}^{\top}M\widehat{R}\widehat{R}^{\top}-M \\
        (\widehat{L}\widehat{L}^{\top}M\widehat{R}\widehat{R}^{\top}-M)^{\top} & 0
    \end{pmatrix}.
$$
We further define
$$
    \widetilde{Q}=\begin{pmatrix}
        0 & Q \\
        0 & 0
    \end{pmatrix} \quad \text{and} \quad 
    \widehat{E}= \begin{pmatrix}
        0 & \widehat{Z} \\
        \widehat{Z}^{\top} & 0
    \end{pmatrix},
$$
where $\widehat{Z}=\widehat{Z}_1+\widehat{Z}_2$. Therefore, we have
\begin{align*}
    \inp{\widehat{L}\widehat{L}^{\top}M\widehat{R}\widehat{R}^{\top}-M}{Q}=\inp{\widehat{\Theta}\widehat{\Theta}^{\top}A\widehat{\Theta}\widehat{\Theta}^{\top}-\Theta\Theta^{\top}A\Theta\Theta^{\top}}{\widetilde{Q}}.
\end{align*}
Define
    \begin{align*}
        \mathfrak{P}^{-s}=\left\{
        \begin{aligned}
            \begin{pmatrix}
                L\Lambda^{-s}L^{\top} & 0 \\
                0 & R\Lambda^{-s}R^{\top}
            \end{pmatrix} \quad \text{if $s$ is even}  \\
            \begin{pmatrix}
                0 &  L\Lambda^{-s}R^{\top}\\
                R\Lambda^{-s}L^{\top} & 0
            \end{pmatrix} \quad \text{if $s$ is odd,} 
        \end{aligned} \right.
    \end{align*}
    and
    \begin{align*}
        \mathfrak{P}^0=\mathfrak{P}^{\perp}=\begin{pmatrix}
            L_{\perp}L_{\perp}^{\top} & 0 \\
            0 & R_{\perp}R_{\perp}^{\top}
        \end{pmatrix}.
    \end{align*}
By Theorem 1 in \cite{xia2021normal}, $\widehat{\Theta}\widehat{\Theta}^{\top}-\Theta\Theta^{\top}$ has explicit representation formula of empirical spectral projector, in the form of 
\begin{align*}
    &\widehat{\Theta}\widehat{\Theta}^{\top}-\Theta\Theta^{\top}=\sum_{k=1}^{\infty}\mathcal{S}_{A,k}(\widehat{E}) \\
    &\mathcal{S}_{A,k}(\widehat{E})=\sum_{s_1+\cdots+s_{k+1}=k}^{k} (-1)^{1+\tau(s)} \mathfrak{P}^{-s_1}\widehat{E}\mathfrak{P}^{-s_2}\cdots\mathfrak{P}^{-s_k}\widehat{E}\mathfrak{P}_{-s_{k+1}}
\end{align*}
where $s_1,\cdots,s_{k+1}\geq 0$ are integers, $\tau(\mathbf{s})=\sum_{i=1}^{k+1}\mathbbm{1}(s_i>0)$. \\
As a result, 
\begin{align*}
    \widehat{\Theta}\widehat{\Theta}^{\top}A\widehat{\Theta}\widehat{\Theta}^{\top}-\Theta\Theta^{\top}A\Theta\Theta^{\top}&= (\mathcal{S}_{A,1}(\widehat{E})A\Theta\Theta^{\top}+\Theta\Theta^{\top}A\mathcal{S}_{A,1}(\widehat{E})) + \sum_{k=2}^{\infty} (\mathcal{S}_{A,k}(\widehat{E})A\Theta\Theta^{\top}+\Theta\Theta^{\top}A\mathcal{S}_{A,k}(\widehat{E})) \\ &\quad + (\widehat{\Theta}\widehat{\Theta}^{\top}-\Theta\Theta^{\top})A(\widehat{\Theta}\widehat{\Theta}^{\top}-\Theta\Theta^{\top}).
\end{align*}
By definition of $\mathcal{S}_{A,1}(\widehat{E})$, 
\begin{align*}
    &\inp{\mathcal{S}_{A,1}(\widehat{E})A\Theta\Theta^{\top}+\Theta\Theta^{\top}A\mathcal{S}_{A,1}(\widehat{E})}{\widetilde{Q}}= \inp{LL^{\top}\widehat{Z}R_{\perp}R_{\perp}^{\top}}{Q} + \inp{L_{\perp}L_{\perp}^{\top}\widehat{Z}RR^{\top}}{Q} \\
    &\quad =\inp{LL^{\top}\widehat{Z}_1R_{\perp}R_{\perp}^{\top}}{Q} + \inp{L_{\perp}L_{\perp}^{\top}\widehat{Z}_1RR^{\top}}{Q}  + \inp{LL^{\top}\widehat{Z}_2R_{\perp}R_{\perp}^{\top}}{Q} + \inp{L_{\perp}L_{\perp}^{\top}\widehat{Z}_2RR^{\top}}{Q}.
\end{align*}
Combine all the terms above, $\inp{\widehat{M}}{Q}-\inp{M}{Q}$ can be decomposed into
\begin{align}
    \inp{\widehat{M}}{Q}-\inp{M}{Q}&= \inp{LL^{\top}\widehat{Z}_1R_{\perp}R_{\perp}^{\top}}{Q} + \inp{L_{\perp}L_{\perp}^{\top}\widehat{Z}_1RR^{\top}}{Q} + \inp{LL^{\top}\widehat{Z}_1RR^{\top}}{Q} \label{main1} \\ &\quad + \inp{LL^{\top}\widehat{Z}_2R_{\perp}R_{\perp}^{\top}}{Q} + \inp{L_{\perp}L_{\perp}^{\top}\widehat{Z}_2RR^{\top}}{Q} + \inp{LL^{\top}\widehat{Z}_2RR^{\top}}{Q}\label{neg1} \\ &\quad + \inp{\widehat{L}\widehat{L}^{\top}\widehat{Z}\widehat{R}\widehat{R}^{\top}}{Q} - \inp{LL^{\top}\widehat{Z}RR^{\top}}{Q} \label{neg2} \\ &\quad + \inp{\sum_{k=2}^{\infty} (\mathcal{S}_{A,k}(\widehat{E})A\Theta\Theta^{\top}+\Theta\Theta^{\top}A\mathcal{S}_{A,k}(\widehat{E}))}{\widetilde{Q}} + \inp{(\widehat{\Theta}\widehat{\Theta}^{\top}-\Theta\Theta^{\top})A(\widehat{\Theta}\widehat{\Theta}^{\top}-\Theta\Theta^{\top})}{\widetilde{Q}}. \label{neg3}
\end{align}
In Lemma \ref{main term}, we aim to show equation (\ref{main1}) is the main term which is asymptotic normal. In Lemma \ref{lemmaneg1}, Lemma \ref{lemmaneg2} and Lemma \ref{lemmaneg3}, we show equation (\ref{neg1}), (\ref{neg2}) and (\ref{neg3}) are negligible higher order terms, respectively. 
\begin{Lemma}
    \label{main term}
    Under the conditions in Theorem \ref{thm:CLT},
    \begin{align*}
        \frac{\left(\inp{LL^{\top}\widehat{Z}_1R_{\perp}R_{\perp}^{\top}}{Q} + \inp{L_{\perp}L_{\perp}^{\top}\widehat{Z}_1RR^{\top}}{Q} + \inp{LL^{\top}\widehat{Z}_1RR^{\top}}{Q} \right)}{\sigma S\sqrt{d_1d_2/T^{1-\gamma}}} \rightarrow N(0,1)
    \end{align*}
    where $S^2=1/T^{\gamma}\fro{\calP_{\Omega_1}\calP_M(Q)}^2+C_{\gamma}\fro{\calP_{\Omega_0}\calP_M(Q)}^2$.
\end{Lemma}
\begin{Lemma}
    \label{lemmaneg1}
    Under the conditions in Theorem \ref{thm:CLT},
    \begin{align*}
        \frac{\left|\inp{LL^{\top}\widehat{Z}_2R_{\perp}R_{\perp}^{\top}}{Q} + \inp{L_{\perp}L_{\perp}^{\top}\widehat{Z}_2RR^{\top}}{Q} + \inp{LL^{\top}\widehat{Z}_2RR^{\top}}{Q}\right|}{\sigma S\sqrt{d_1d_2/T^{1-\gamma}}}\overset{p}{\rightarrow} 0.
    \end{align*}
\end{Lemma}
\begin{Lemma}
    \label{lemmaneg2}
    Under the conditions in Theorem \ref{thm:CLT},
    \begin{align*}
        \frac{\big|\inp{\widehat{L}\widehat{L}^{\top}\widehat{Z}\widehat{R}\widehat{R}^{\top}}{Q} - \inp{LL^{\top}\widehat{Z}RR^{\top}}{Q}\big|}{\sigma S\sqrt{d_1d_2/T^{1-\gamma}}}
        \overset{p}{\rightarrow} 0. 
    \end{align*}
\end{Lemma}

\begin{Lemma}
    \label{lemmaneg3}
    Under the conditions in Theorem \ref{thm:CLT},
    \begin{align*}
        &\frac{\big|\sum_{k=2}^{\infty} \inp{(\mathcal{S}_{A,k}(\widehat{E})A\Theta\Theta^{\top}+\Theta\Theta^{\top}A\mathcal{S}_{A,k}(\widehat{E}))}{\widetilde{Q}}\big|}{\sigma S\sqrt{d_1d_1/T^{1-\gamma}}}\overset{p}{\rightarrow} 0, \\
        &\frac{\big|\inp{(\widehat{\Theta}\widehat{\Theta}^{\top}-\Theta\Theta^{\top})A(\widehat{\Theta}\widehat{\Theta}^{\top}-\Theta\Theta^{\top})}{\widetilde{Q}}\big|}{\sigma S\sqrt{d_1d_1/T^{1-\gamma}}} \overset{p}{\rightarrow} 0.
    \end{align*}
\end{Lemma}
\noindent Combine Lemma \ref{main term}-\ref{lemmaneg3}, we can conclude that
\begin{align*}
    \frac{\inp{\widehat{M}}{Q}-\inp{M}{Q}}{\sigma S\sqrt{d_1d_2/T^{1-\gamma}}}\rightarrow N(0,1).
\end{align*}   
\end{proof}   

\subsection{Proof of Theorem \ref{thm:studentized-CLT}}
\begin{proof}
    Without loss of generality, we only prove the $i=1$ case. As long as we show $\widehat{\sigma}^2\overset{p}{\rightarrow} \sigma^2$ and $\widehat{S}^2\overset{p}{\rightarrow} S^2$, then by Slutsky's Theorem, we can obtain the desired result. \\
    $\mathbf{\widehat{\sigma}^2\overset{p}{\rightarrow} \sigma^2}$ \quad Notice that, $\widehat{\sigma}^2$ can be written as
    \begin{align*}
        \widehat{\sigma}^2&= \underbrace{\frac{b}{T}\sum_{t=T_0+1}^{T} \frac{\mathbbm{1}(a_t=1)}{\pi_t}\inp{M-\widehat{M}_{t-1}}{X_t}^2}_{\uppercase\expandafter{\romannumeral1}} + \underbrace{\frac{2b}{T}\sum_{t=T_0+1}^{T} \frac{\mathbbm{1}(a_t=1)}{\pi_t} \inp{M-\widehat{M}_{t-1}}{X_t}\xi_t}_{\uppercase\expandafter{\romannumeral2}} \\ &\quad + \underbrace{\frac{b}{T}\sum_{t=T_0+1}^{T} \frac{\mathbbm{1}(a_t=1)}{\pi_t} \xi_t^2}_{\uppercase\expandafter{\romannumeral3}}.
    \end{align*}
    For term \uppercase\expandafter{\romannumeral1}, 
    \begin{align*}
        \big|\frac{b}{T} \frac{\mathbbm{1}(a_t=1)}{\pi_t}\inp{M-\widehat{M}_{t-1}}{X_t}^2\big|&\lesssim \frac{t^\gamma}{T}\|\widehat{\Delta}_{t-1}\|_{\max}^2 \lesssim \frac{1}{T^{1-\gamma}}\sigma^2, 
    \end{align*}
    and 
    \begin{align*}
        \EE\left[\frac{b^2}{T^2} \frac{\mathbbm{1}(a_t=1)}{\pi_t^2}\inp{M-\widehat{M}_{t-1}}{X_t}^4\big|\mathcal{F}_{t-1}\right]&\leq \frac{b^2}{T^2}\frac{2}{\varepsilon_t}\|\widehat{\Delta}_{t-1}\|_{\max}^2\frac{1}{d_1d_2}\fro{\widehat{\Delta}_{t-1}}^2 \\
        &\lesssim \frac{t^\gamma}{T^2}\sigma^4.
    \end{align*}
    By martingale Bernstein inequality, with probability at least $1-d_1^{-2}$,
    \begin{align*}
        \uppercase\expandafter{\romannumeral1}&\lesssim \bigg|\frac{b}{T}\sum_{t=T_0+1}^{T} \EE\left[\frac{\mathbbm{1}(a_t=1)}{\pi_t}\inp{M-\widehat{M}_{t-1}}{X_t}^2\big|\mathcal{F}_{t-1}\right]\bigg| + \sqrt{\frac{\log d_1}{T^{1-\gamma}}}\sigma^2 \\
        &\lesssim \frac{d_1r\sigma^2\log^4 d_1\text{poly}(\mu,r,\kappa)}{T^{1-\gamma}} + \sqrt{\frac{\log d_1}{T^{1-\gamma}}}\sigma^2,
    \end{align*}
    which indicates $\uppercase\expandafter{\romannumeral1}\overset{p}{\rightarrow} 0$ when $\frac{d_1\log^4 d_1}{T^{1-\gamma}}\rightarrow 0$. 
    
    Similarly, for term $\uppercase\expandafter{\romannumeral2}$, note that $|\xi_t|$ is bounded with $\sigma\log^{1/2} d_1$ with high probability, 
    \begin{align*}
        \bigg|\frac{2b}{T} \frac{\mathbbm{1}(a_t=1)}{\pi_t} \inp{M-\widehat{M}_{t-1}}{X_t}\xi_t\bigg|\leq \frac{2b}{T}\frac{2}{\varepsilon_t}\|\widehat{\Delta}_{t-1}\|_{\max}\sigma\log^{1/2} d_1 \lesssim \frac{\sigma^2\log^{1/2} d_1}{T^{1-\gamma}},
    \end{align*}
    and
    \begin{align*}
        \EE\left[\frac{4b^2}{T^2} \frac{\mathbbm{1}(a_t=1)}{\pi_t^2} \inp{M-\widehat{M}_{t-1}}{X_t}^2\xi_t^2|\mathcal{F}_{t-1}\right] \lesssim \frac{t^\gamma\sigma^2}{T^2}\|\widehat{\Delta}_{t-1}\|_{\max}^2 \lesssim  \frac{t^\gamma\sigma^4}{T^2}.
    \end{align*}
    By martingale Bernstein inequality, with probability at least $1-d_1^{-2}$, 
    \begin{align*}
        \bigg|\uppercase\expandafter{\romannumeral2} - \frac{2b}{T} \sum_{t=T_0+1}^{T} \EE\left[\frac{\mathbbm{1}(a_t=1)}{\pi_t}\inp{M-\widehat{M}_{t-1}}{X_t} \xi_t|\mathcal{F}_{t-1}\right]\bigg| \lesssim \sqrt{\frac{\log d_1}{T^{1-\gamma}}}\sigma^2,
    \end{align*}
    which indicates $\uppercase\expandafter{\romannumeral2}\overset{p}{\rightarrow} 0$ as long as $\frac{\log d_1}{T^{1-\gamma}} \rightarrow 0$. 
    
    Lastly, for term $\uppercase\expandafter{\romannumeral3}$,
    \begin{align*}
        \|\frac{b}{T} \frac{\mathbbm{1}(a_t=1)}{\pi_t} \xi_t^2\|_{\psi_1}\lesssim \frac{t^\gamma\sigma^2}{T},
    \end{align*}
    and
    \begin{align*}
        \EE\left[\frac{b^2}{T^2} \frac{\mathbbm{1}(a_t=1)}{\pi_t^2} \xi_t^4|\mathcal{F}_{t-1}\right]  \lesssim  \frac{t^\gamma\sigma^4}{T}.
    \end{align*}
    By martingale Bernstein inequality, with probability at least $1-d_1^{-2}$, 
    \begin{align*}
        \bigg|\uppercase\expandafter{\romannumeral3} - \frac{b}{T} \sum_{t=T_0+1}^{T} \EE\left[\frac{\mathbbm{1}(a_t=1)}{\pi_t} \xi_t^2|\mathcal{F}_{t-1}\right]\bigg| \lesssim \sqrt{\frac{\log d_1}{T^{1-\gamma}}}\sigma^2,
    \end{align*}
    which indicates $\uppercase\expandafter{\romannumeral3}\overset{p}{\rightarrow} \sigma^2$ as long as $\frac{\log d_1}{T^{1-\gamma}} \rightarrow 0$. 
    
    {\it Showing $\widehat{S}^2\overset{p}{\rightarrow} S^2$.} \quad  Notice that
    \begin{align*}
        \left|\fro{\mathcal{P}_M(Q)}^2 - \fro{\mathcal{P}_{\widehat{M}}(Q)}^2\right|&=\left|\fro{L^{\top}Q}^2 + \fro{L_{\perp}^{\top}Q R}^2 - \fro{\widehat{L}^{\top}Q}^2 - \fro{\widehat{L}_{\perp}^{\top}Q \widehat{R}}^2\right| \\
        &\leq \left|\fro{L^{\top}Q}^2 - \fro{\widehat{L}^{\top}Q}^2\right| + \left|\fro{L_{\perp}^{\top}Q R}^2 - \fro{\widehat{L}_{\perp}^{\top}Q \widehat{R}}^2\right|. 
    \end{align*}
    Observing that $L$ and $\widehat{L}$ both have orthonormal columns, then
    \begin{align*}
        \left|\fro{L^{\top}Q}^2 - \fro{\widehat{L}^{\top}Q}^2\right|&=\left|\fro{LL^{\top}Q}^2 - \fro{\widehat{L}\widehat{L}^{\top}Q}^2\right| \\
        &\leq \fro{(LL^{\top} - \widehat{L}\widehat{L}^{\top})Q}^2 + 2\left|\inp{(LL^{\top} - \widehat{L}\widehat{L}^{\top})Q}{LL^{\top}Q}\right|.
    \end{align*}
    In the following we shall consider the upper bound under the event $\calE$ defined in the proof of Lemma \ref{induction:new}. 
    The probability of the following event
    \begin{align*}
        &\fro{(LL^{\top} - \widehat{L}\widehat{L}^{\top})Q}^2 \leq \left(\sum_{j_1,j_2} |Q(j_1,j_2)| \|(LL^{\top} - \widehat{L}\widehat{L}^{\top})e_{j_1}\|\right)^2 \\
        &\quad \leq \|Q\|_{\ell_1}^2 \|LL^{\top} - \widehat{L}\widehat{L}^{\top}\|_{2,\max}^2 \lesssim \|Q\|_{\ell_1}^2\frac{\sigma^2}{\lambda_{\min}^2}\frac{\mu r d_1d_2\log d_1}{T^{1-\gamma}}.
    \end{align*}
    approaches 1 as $\|Q\|_{\ell_1}^2\frac{\sigma^2}{\lambda_{\min}^2}\frac{d_1d_2\log d_1}{T^{1-\gamma}}\rightarrow 0$, which can be easily satisfied under the conditions in Theorem \ref{thm:CLT}. \\ 
    Similarly,
    \begin{align*}
        &\left|\inp{(LL^{\top} - \widehat{L}\widehat{L}^{\top})Q}{LL^{\top}Q}\right|\leq \fro{L^{\top}Q}\fro{L^{\top}(LL^{\top} - \widehat{L}\widehat{L}^{\top})Q} \\
        &\quad \leq \fro{L^{\top}Q}\|Q\|_{\ell_1}\|\|LL^{\top} - \widehat{L}\widehat{L}^{\top}\|_{2,\max} \lesssim \fro{Q}\|Q\|_{\ell_1}\frac{\sigma}{\lambda_{\min}}\sqrt{\frac{\mu r d_1d_2\log d_1}{T^{1-\gamma}}}.
    \end{align*}
    Therefore, as long as $\fro{Q}\|Q\|_{\ell_1}\frac{\sigma}{\lambda_{\min}}\sqrt{\frac{\mu r d_1d_2\log d_1}{T^{1-\gamma}}}\rightarrow 0$, the above term converges to 0 in probability.
    
    Next we bound $\left|\fro{L_{\perp}^{\top}Q R}^2 - \fro{\widehat{L}_{\perp}^{\top}Q \widehat{R}}^2\right|$. It can be decomposed into
    \begin{align*}
        \left|\fro{L_{\perp}^{\top}Q R}^2 - \fro{\widehat{L}_{\perp}^{\top}Q \widehat{R}}^2\right|\leq \left|\fro{L_{\perp}^{\top}Q R}^2 - \fro{{L}_{\perp}^{\top}Q \widehat{R}}^2\right| + \left|\fro{L_{\perp}^{\top}Q \widehat{R}}^2 - \fro{\widehat{L}_{\perp}^{\top}Q \widehat{R}}^2\right|.
    \end{align*}
    It follows that
    \begin{align*}
        &\left|\fro{L_{\perp}^{\top}Q R}^2 - \fro{{L}_{\perp}^{\top}Q \widehat{R}}^2\right|\leq \fro{{L}_{\perp}^{\top}Q(RR^{\top}-\widehat{R}\widehat{R}^{\top})}^2 + 2\left|\inp{{L}_{\perp}^{\top}Q(RR^{\top}-\widehat{R}\widehat{R}^{\top})}{L_{\perp}^{\top}Q R}\right| \\
        &\quad \lesssim  \|Q\|_{\ell_1}^2 \frac{\sigma^2}{\lambda_{\min}^2}\frac{\mu rd_1^2\log d_1}{T} + \fro{L_{\perp}^{\top}Q R}\|Q\|_{\ell_1}\frac{\sigma}{\lambda_{\min}}\sqrt{\frac{\mu rd_1^2\log d_1}{T^{1-\gamma}}},
    \end{align*}
    and
    \begin{align*}
        &\quad\left|\fro{L_{\perp}^{\top}Q \widehat{R}}^2 - \fro{\widehat{L}_{\perp}^{\top}Q \widehat{R}}^2\right|\leq \fro{(\widehat{L}\widehat{L}^{\top}-LL^{\top})Q\widehat{R}\widehat{R}^{\top}}^2 + 2\left|\inp{{L}_{\perp}^{\top}Q(RR^{\top}-\widehat{R}\widehat{R}^{\top})}{L_{\perp}^{\top}Q \widehat{R}}\right| \\
        & \lesssim \|Q\|_{\ell_1}^2 \frac{\sigma^2}{\lambda_{\min}^2}\frac{\mu r d_1d_2\log d_1}{T^{1-\gamma}}  + (\fro{L_{\perp}^{\top}Q {R}}+\fro{{L}_{\perp}^{\top}Q(RR^{\top}-\widehat{R}\widehat{R}^{\top})})\|Q\|_{\ell_1} \frac{\sigma}{\lambda_{\min}}\sqrt{\frac{\mu r d_1^2\log d_1}{T^{1-\gamma}}} \\
        & \lesssim \|Q\|_{\ell_1}^2 \frac{\sigma^2}{\lambda_{\min}^2}\frac{\mu r d_1d_2\log d_1}{T^{1-\gamma}} + \fro{L_{\perp}^{\top}Q {R}}\|Q\|_{\ell_1} \frac{\sigma}{\lambda_{\min}}\sqrt{\frac{\mu r d_1^2\log d_1}{T^{1-\gamma}}}  + \|Q\|_{\ell_1}^2 \frac{\sigma^2}{\lambda_{\min}^2}\frac{\mu r d_1^2\log d_1}{T^{1-\gamma}}.
    \end{align*}
    We can show $\left|\fro{L_{\perp}^{\top}Q \widehat{R}}^2 - \fro{\widehat{L}_{\perp}^{\top}Q \widehat{R}}^2\right|$ converges to 0 in probability as long as $\|Q\|_{\ell_1}^2 \frac{\sigma^2}{\lambda_{\min}^2}\frac{d_1^2\log d_1}{T^{1-\gamma}}\rightarrow 0$. 
    As a result, $\fro{\mathcal{P}_{\widehat{M}}(Q)}^2\overset{p}{\rightarrow} \fro{\mathcal{P}_M(Q)}^2$. By Lemma \ref{indicator}, it directly follows $\fro{\mathcal{P}_{\widehat{\Omega}_{i,T}}\mathcal{P}_{\widehat{M}}(Q)}^2\overset{p}{\rightarrow} \fro{\mathcal{P}_{\Omega_i}\mathcal{P}_M(Q)}^2$ for $i=0,1$. Moreover, $b_T\rightarrow 1$ when $T\rightarrow \infty$.
    Then by Slutsky's Theorem, $\widehat{S}^2\overset{p}{\rightarrow} S^2$.
\end{proof}

\subsection{Proof of Corollary \ref{cor:m1-m0}}
\begin{proof}
    Note that
    \begin{align*}
        (\inp{\widehat{M}_1}{Q} - \inp{\widehat{M}_0}{Q}) - (\inp{M_1}{Q} - \inp{M_0}{Q})= (\inp{\widehat{M}_1}{Q} - \inp{M_1}{Q}) - (\inp{\widehat{M}_0}{Q} - \inp{M_0}{Q}),
    \end{align*}
    then $\inp{\widehat{M}_1}{Q} - \inp{M_1}{Q}$ and $\inp{\widehat{M}_0}{Q} - \inp{M_0}{Q}$ can be decomposed into main term and negligible terms in the same way as Theorem \ref{thm:CLT}. The upper bound of all the negligible terms for both $i=0$ and 1 still follow the Lemma \ref{lemmaneg1}-\ref{lemmaneg3}. The main CLT term is
    \begin{align*}
        \frac{b}{T}\sum_{t_1=T_0+1}^{T}\frac{d_1d_2\mathbbm{1}(a_{t_1}=1)}{\pi_{t_1}}\xi_{t_1}\inp{X_{t_1}}{\mathcal{P}_{M_1}(Q)} - \frac{b}{T}\sum_{t_2=T_0+1}^{T}\frac{d_1d_2\mathbbm{1}(a_{t_2}=0)}{1-\pi_{t_2}}\xi_{t_2}\inp{X_{t_2}}{\mathcal{P}_{M_0}(Q)}.
    \end{align*}
    As long as we show $\sum_{t_1=T_0+1}^{T}\frac{d_1d_2\mathbbm{1}(a_{t_1}=1)}{\pi_{t_1}}\xi_{t_1}\inp{X_{t_1}}{\mathcal{P}_{M_1}(Q)}$ is uncorrelated with \\ $\sum_{t_2=T_0+1}^{T}\frac{d_1d_2\mathbbm{1}(a_{t_2}=0)}{1-\pi_{t_2}}\xi_{t_2}\inp{X_{t_2}}{\mathcal{P}_{M_0}(Q)}$, then the asymptotic variance is the sum of their individual variance, i.e., $\sigma_1^2S_1^2+\sigma_0^2S_0^2$. Notice that,
    \begin{align*}
        &\quad \frac{b^2}{T^2}\sum_{t_1=T_0+1}^{T}\frac{d_1d_2\mathbbm{1}(a_{t_1}=1)}{\pi_{t_1}}\xi_{t_1}\inp{X_{t_1}}{\mathcal{P}_{M_1}(Q)} \sum_{t_2=T_0+1}^{T}\frac{d_1d_2\mathbbm{1}(a_{t_2}=0)}{1-\pi_{t_2}}\xi_{t_2}\inp{X_{t_2}}{\mathcal{P}_{M_0}(Q)} \\
        &=\frac{b^2d_1^2d_2^2}{T^2}\sum_{t_1=T_0+1}^{T}\sum_{t_2=T_0+1}^{T} \frac{\mathbbm{1}(a_{t_1}=1)\mathbbm{1}(a_{t_2}=0)}{\pi_{t_1}(1-\pi_{t_2})}\xi_{t_1}\xi_{t_2}\inp{X_{t_1}}{\mathcal{P}_{M_1}(Q)}\inp{X_{t_2}}{\mathcal{P}_{M_0}(Q)}.
    \end{align*}
    When $t_1=t_2$, $\mathbbm{1}(a_{t_1}=1)\mathbbm{1}(a_{t_2}=0)=0$. When $t_1\neq t_2$, $\EE[\frac{\mathbbm{1}(a_{t_1}=1)\mathbbm{1}(a_{t_2}=0)}{\pi_{t_1}(1-\pi_{t_2})}\xi_{t_1}\xi_{t_2}X_{t_1}X_{t_2}]=0$ due to the i.i.d. distributed $\xi_t$. As a result, the two terms are uncorrelated. \\
    Together with Lemma \ref{lemmaneg1}-\ref{lemmaneg3}, all the negligible terms converge to 0 when divided by \\ $\sqrt{(\sigma_1^2S_1^2+\sigma_0^2S_0^2)d_1d_2/T^{1-\gamma}}$. Then similar as Theorem \ref{thm:studentized-CLT}, we can replace $\sigma_1^2S_1^2+\sigma_0^2S_0^2$ with $\widehat{\sigma}_1^2\widehat{S}_1^2+\widehat{\sigma}_0^2\widehat{S}_0^2$ and conclude the proof.
\end{proof}

\section{Proofs of Technical Lemmas}

\subsection{Proof of Lemma \ref{lemmafro}}
\begin{proof}
    \emph{Step 1: uniform bound for $\prod_{k=\tau+1}^{t} (1-f(\eta_k))A_{\tau}$ and $\prod_{k=\tau+1}^{t} (1-f(\eta_k))B_{\tau}$.} \\
    For $\forall$ $1\leq\tau\leq t-1$, under $\calE_{\tau-1}$, by inequality $(a+b+c)^2\leq 3(a^2+b^2+c^3)$,
    \begin{align*}
        \fro{\Delta_\tau}^2 &\leq \frac{24}{\varepsilon_\tau^2}\eta_\tau^2 (\inp{\widehat{U}_{\tau-1}\widehat{V}_{\tau-1}^{\top} - M}{X_\tau}-\xi_t)^2\max(\|\widehat{V}_{\tau-1}\widehat{V}_{\tau-1}^{\top}\|_{2,\max}^2,\|\widehat{U}_{\tau-1}\widehat{U}_{\tau-1}^{\top}\|_{2,\max}^2) \\ &+ \frac{48}{\varepsilon_t^4}\eta_\tau^4 (\inp{\widehat{U}_{\tau-1}\widehat{V}_{\tau-1}^{\top} - M}{X_\tau}-\xi_\tau)^4\|\widehat{V}_{\tau-1}\widehat{U}_{\tau-1}^{\top}\|_{\max}^2. 
    \end{align*}
    Since $\|\widehat{U}_{\tau-1}\widehat{V}_{\tau-1}^{\top} - M\|_{\max}^2\lesssim \frac{\mu_0 r \kappa^2}{d_2}\fro{\widehat{U}_{\tau-1}\widehat{V}_{\tau-1}^{\top} - M}^2\lesssim \mu_0 r \kappa^2 \prod_{k=1}^{\tau-1} (1-\frac{f(\eta_k)}{2})\frac{\fro{\widehat{U}_0\widehat{V}_0^{\top}-M}^2}{d_2} + \\ \sum_{k=1}^{\tau-1}\frac{\mu_0\kappa^2 \eta_k^2r^2\lambda_{\max}^2\sigma^2\log^2d_1}{\varepsilon_k d_2^2}$. And by the property of subGaussian random variable, $\max|\xi_\tau|<C\sigma\log^{1/2}d_1$ for some constant $C$. Therefore,
    \begin{align*}
        &\fro{\Delta_\tau}^2 \lesssim \frac{\eta_\tau^2}{\varepsilon_\tau^2}\frac{\mu_0r\lambda_{\max}^2}{d_2}\left(\mu_0r\kappa^2 \prod_{k=1}^{\tau-1} (1-\frac{f(\eta_k)}{2})\frac{\fro{\widehat{U}_0\widehat{V}_0^{\top}-M}^2}{d_2} + \sum_{k=1}^{\tau-1}\frac{\mu_0\kappa^2 \eta_k^2r^2\lambda_{\max}^2\sigma^2\log^2d_1}{\varepsilon_k d_2^2} + \sigma^2\log d_1 \right)\\ &\quad + \frac{\eta_\tau^4}{\varepsilon_\tau^4}\frac{\mu_0^3r^3\kappa^2\lambda_{\max}^2}{d_1d_2}\left(\frac{\mu_0^2r^2\lambda_{\max}^2}{d_1d_2}\left(\prod_{k=1}^{\tau-1} (1-\frac{f(\eta_k)}{2})\frac{\fro{\widehat{U}_0\widehat{V}_0^{\top}-M}^2}{d_2} + \sum_{k=1}^{\tau-1}\frac{\eta_k^2r\lambda_{\max}^2\sigma^2\log^2d_1}{\varepsilon_k d_2^2}\right) + \sigma^4\log^2 d_1\right),
    \end{align*}
    where we use $\|\widehat{U}_{\tau-1}\widehat{V}_{\tau-1}^{\top}-M\|_{\max}^4=\|\widehat{U}_{\tau-1}\widehat{V}_{\tau-1}^{\top}-M\|_{\max}^2 \|\widehat{U}_{\tau-1}\widehat{V}_{\tau-1}^{\top}-M\|_{\max}^2\lesssim \frac{2\mu_0^3 r^3\kappa^2 \lambda_{\max}^2}{d_1d_2}(\prod_{k=1}^{\tau-1} (1-\frac{f(\eta_k)}{2})\frac{\fro{\widehat{U}_0\widehat{V}_0^{\top}-M}^2}{d_2} + \sum_{k=1}^{\tau-1}\frac{\eta_k^2r\lambda_{\max}^2\sigma^2\log^2d_1}{\varepsilon_k d_2^2})$ in the second line. The fourth order terms will be dominated as long as $\eta_\tau\lesssim \frac{\varepsilon_\tau d_1d_2^{1/2}}{\mu_0r \lambda_{\max}}$ and $\frac{\lambda_{\min}^2}{\sigma^2}\gg \frac{\eta_{\tau}^2}{\varepsilon_{\tau}^2}\frac{\mu_0r\lambda_{\max}^2\log d_1}{d_1\kappa^2}$. So $\fro{\Delta_\tau}^2\lesssim \frac{\eta_\tau^2}{\varepsilon_\tau^2}\frac{\mu_0r\lambda_{\max}^2}{d_2}(\mu_0r\kappa^2 \prod_{k=1}^{\tau-1} (1-\frac{f(\eta_k)}{2})\frac{\fro{U_0V_0^{\top}-M}^2}{d_2} + \sum_{k=1}^{\tau-1}\frac{\mu_0\kappa^2 \eta_k^2r^2\lambda_{\max}^2\sigma^2\log^2d_1}{\varepsilon_k d_2^2} + \sigma^2\log d_1)$. \\
    Similarly, 
    \begin{align*}
        \left|\inp{\Delta_\tau}{\widehat{U}_{\tau-1}\widehat{V}_{\tau-1}^{\top}-M}\right|\leq &\left|\frac{\mathbbm{1}(a_\tau=1)}{\pi_\tau}\eta_\tau (\inp{\widehat{U}_{\tau-1}\widehat{V}_{\tau-1}^{\top} - M}{X_\tau}-\xi_\tau)\inp{X_\tau}{(\widehat{U}_{\tau-1}\widehat{V}_{\tau-1}^{\top} - M)\widehat{V}_{\tau-1}\widehat{V}_{\tau-1}^{\top}}\right|\\ &+ \left|\frac{\mathbbm{1}(a_\tau=1)}{\pi_\tau}\eta_t (\inp{\widehat{U}_{\tau-1}\widehat{V}_{\tau-1}^{\top} - M}{X_\tau}-\xi_\tau)\inp{X_\tau}{\widehat{U}_{\tau-1}\widehat{U}_{\tau-1}^{\top}(\widehat{U}_{\tau-1}\widehat{V}_{\tau-1}^{\top} - M)}\right| \\ &+  \left|\frac{\mathbbm{1}(a_\tau=1)}{\pi_\tau^2}\eta_\tau^2 (\inp{\widehat{U}_{\tau-1}\widehat{V}_{\tau-1}^{\top} - M}{X_\tau}-\xi_\tau)^2\inp{X_\tau \widehat{V}_{\tau-1}\widehat{U}_{\tau-1}^{\top}X_\tau}{\widehat{U}_{\tau-1}\widehat{V}_{\tau-1}^{\top} - M}\right|.
    \end{align*}
    For term $\inp{X_\tau}{(U_{\tau-1}V_{\tau-1}^{\top} - M)\widehat{V}_{\tau-1}\widehat{V}_{\tau-1}^{\top}}$, we have
    \begin{align*}
        &\inp{X_\tau}{(\widehat{U}_{\tau-1}\widehat{V}_{\tau-1}^{\top} - M)\widehat{V}_{\tau-1}\widehat{V}_{\tau-1}^{\top}}= \inp{X_\tau\widehat{V}_{\tau-1}\widehat{V}_{\tau-1}^{\top}}{\widehat{U}_{\tau-1}\widehat{V}_{\tau-1}^{\top} - M} \\
        &\quad \leq \|\widehat{V}_{\tau-1}\widehat{V}_{\tau-1}^{\top}\|_{2,\max}\fro{\widehat{U}_{\tau-1}\widehat{V}_{\tau-1}^{\top} - M} \leq \sqrt{\frac{\mu_0r\lambda_{\max}^2}{d_2}}\fro{\widehat{U}_{\tau-1}\widehat{V}_{\tau-1}^{\top} - M}.
    \end{align*}
    By symmetry, $\inp{X_\tau}{\widehat{U}_{\tau-1}\widehat{U}_{\tau-1}^{\top}(\widehat{U}_{\tau-1}\widehat{V}_{\tau-1}^{\top} - M)}\leq \sqrt{\frac{\mu_0r\lambda_{\max}^2}{d_1}}\fro{\widehat{U}_{\tau-1}\widehat{V}_{\tau-1}^{\top} - M}$. In addition, 
    \begin{align*}
        \inp{X_\tau \widehat{V}_{\tau-1}\widehat{U}_{\tau-1}^{\top}X_\tau}{\widehat{U}_{\tau-1}\widehat{V}_{\tau-1}^{\top} - M}&\leq  \|\widehat{V}_{\tau-1}\widehat{U}_{\tau-1}^{\top}\|_{\max}\fro{\widehat{U}_{\tau-1}\widehat{V}_{\tau-1}^{\top} - M} \leq \frac{\mu_0r\lambda_{\max}}{\sqrt{d_1d_2}}\fro{\widehat{U}_{\tau-1}\widehat{V}_{\tau-1}^{\top} - M}.
    \end{align*}
    Therefore,  
    \begin{align*}
        &\quad |\inp{\Delta_\tau}{\widehat{U}_{\tau-1}\widehat{V}_{\tau-1}^{\top}-M}| \lesssim  \frac{\eta_\tau}{\varepsilon_\tau} (\|\widehat{U}_{\tau-1}\widehat{V}_{\tau-1}^{\top} - M\|_{\max}+\sigma\log^{1/2}d_1)\sqrt{\frac{\mu_0r\lambda_{\max}^2}{d_2}}\fro{\widehat{U}_{\tau-1}\widehat{V}_{\tau-1}^{\top} - M} \\ &\quad\quad + \frac{\eta_\tau^2}{\varepsilon_\tau^2} (\|\widehat{U}_{\tau-1}\widehat{V}_{\tau-1}^{\top} - M\|_{\max}^2 +\sigma^2\log d_1) \frac{\mu_0r\lambda_{\max}}{\sqrt{d_1d_2}}\fro{\widehat{U}_{\tau-1}\widehat{V}_{\tau-1}^{\top} - M} \\
        &\lesssim \frac{\eta_\tau}{\varepsilon_t}\left(\mu_0^{1/2}r^{1/2}\kappa \prod_{k=1}^{\tau-1} (1-\frac{f(\eta_k)}{2})^{1/2}\frac{\fro{\widehat{U}_0\widehat{V}_0^{\top}-M}}{\sqrt{d_2}} + \sqrt{\sum_{k=1}^{\tau-1}\frac{\mu_0\kappa^2 \eta_k^2r^2\lambda_{\max}^2\sigma^2\log^2d_1}{\varepsilon_k d_2^2}} \right.\\ &\left.\quad\quad + \sigma\log^{1/2}d_1\right)\sqrt{\frac{\mu_0r\lambda_{\max}^2}{d_2}}\fro{\widehat{U}_{\tau-1}\widehat{V}_{\tau-1}^{\top} - M} \\
        &\lesssim \frac{\eta_\tau^2}{\varepsilon_\tau^2}\left(\mu_0 r\kappa^2 \prod_{k=1}^{\tau-1} (1-\frac{f(\eta_k)}{2})\frac{\fro{\widehat{U}_0\widehat{V}_0^{\top}-M}^2}{d_2} + \sum_{k=1}^{\tau-1}\frac{\mu_0\kappa^2 \eta_k^2r^2\lambda_{\max}^2\sigma^2\log^2d_1}{\varepsilon_k d_2^2} + \sigma^2\log^2 d_1\right)\frac{\mu_0r\lambda_{\max}^2}{d_2}\log d_1 \\ &\quad\quad + \left(\prod_{k=1}^{\tau-1} (1-\frac{f(\eta_k)}{2})\fro{\widehat{U}_0\widehat{V}_0^{\top}-M}^2 + \sum_{k=1}^{\tau-1}\frac{\eta_k^2r\lambda_{\max}^2\sigma^2\log^2d_1}{\varepsilon_kd_2}\right)\log^{-1}d_1.
    \end{align*}
    By dealing $\|\widehat{U}_{\tau-1}\widehat{V}_{\tau-1}^{\top}-M\|_{\max}^2\lesssim \frac{\mu_0^{3/2}r^{3/2}\kappa \lambda_{\max}}{\sqrt{d_1d_2}}(\prod_{k=1}^{\tau-1} (1-\frac{f(\eta_k)}{2})^{1/2}\frac{\fro{\widehat{U}_0\widehat{V}_0^{\top}-M}}{\sqrt{d_2}} + \sqrt{\sum_{k=1}^{\tau-1}\frac{\eta_k^2r\lambda_{\max}^2\sigma^2\log^2d_1}{\varepsilon_k d_2^2}})$ as before, the second inequality holds since the second order term will be dominated as long as $\eta_\tau\lesssim \frac{\varepsilon_\tau d_1d_2^{1/2}}{\mu_0 r \kappa \lambda_{\max}}$ and $\frac{\lambda_{\min}^2}{\sigma^2}\gg \frac{\eta_{\tau}^2}{\varepsilon_{\tau}^2}\frac{\mu_0 r \lambda_{\max}^2\log d_1}{d_1\kappa^2}$. And the third inequality holds by $2ab\leq a^2+b^2$.  \\
    As a result, we have the uniform bound of $\prod_{k=\tau+1}^{t} (1-f(\eta_k))A_{\tau}$ for $1\leq \tau\leq t-1$, 
    \begin{align*}
        &\prod_{k=\tau+1}^{t} (1-f(\eta_k))A_{\tau}\lesssim \prod_{k=\tau+1}^{t} (1-f(\eta_k))\left(\mu_0r\kappa^2 \prod_{k=1}^{\tau-1} (1-\frac{f(\eta_k)}{2})\frac{\fro{\widehat{U}_0\widehat{V}_0^{\top}-M}^2}{d_2} \right.\\ &\left.\quad + \sum_{k=1}^{\tau-1}\frac{\mu_0\kappa^2 \eta_k^2r^2\lambda_{\max}^2\sigma^2\log^2d_1}{\varepsilon_k d_2^2}  + \sigma^2\log d_1 \right) \frac{\eta_\tau^2}{\varepsilon_\tau^2}\frac{\mu_0r\lambda_{\max}^2}{d_2}\log d_1 \\ &\quad\quad + \prod_{k=\tau+1}^{t} (1-f(\eta_k))\left(\prod_{k=1}^{\tau-1}(1-\frac{f(\eta_{k})}{2}) \fro{\widehat{U}_0\widehat{V}_0^{\top} - M}^2 + \sum_{k=1}^{\tau-1}\frac{\eta_k^2r\lambda_{\max}^2\sigma^2\log^2 d_1}{\varepsilon_kd_2}\right)\log^{-1}d_1 \\
        &\lesssim  \max_{\tau}\frac{\eta_{\tau}^2}{\varepsilon_{\tau}^2}\frac{\mu_0r\lambda_{\max}^2}{d_2}\log d_1 \left(\mu_0r\kappa^2 \prod_{k=1}^{t}(1-\frac{f(\eta_{k})}{2}) \frac{\fro{\widehat{U}_0\widehat{V}_0^{\top} - M}^2}{d_2} + \sum_{k=1}^{t-1}\frac{\mu_0\kappa^2 \eta_k^2r^2\lambda_{\max}^2\sigma^2\log^2 d_1}{\varepsilon_k d_2^2} + \sigma^2\log d_1\right)  \\ &\quad\quad + \left(\prod_{k=1}^{t}(1-\frac{f(\eta_{k})}{2}) \fro{\widehat{U}_0\widehat{V}_0^{\top} - M}^2 + \sum_{k=1}^{t}\frac{\eta_k^2r\lambda_{\max}^2\sigma^2\log^2 d_1}{\varepsilon_kd_2}\right)\log^{-1}d_1:=R_{A}.
    \end{align*}
    Similarly,
    \begin{align*}
        &\prod_{k=\tau+1}^{t} (1-f(\eta_k))B_{\tau}\lesssim \max_{\tau}\frac{\eta_{\tau}^2}{\varepsilon_{\tau}^2}\frac{\mu_0r\lambda_{\max}^2}{d_2}\left(\mu_0r\kappa^2\prod_{k=1}^{t}(1-\frac{f(\eta_{k})}{2}) \frac{\fro{\widehat{U}_0\widehat{V}_0^{\top} - M}^2}{d_2} \right.\\ &\left.\quad +   \sum_{k=1}^{t}\frac{\mu_0\kappa^2\eta_k^2r^2\lambda_{\max}^2\sigma^2\log^2 d_1}{\varepsilon_k d_2^2} + \sigma^2\log d_1\right) :=R_{B}.
    \end{align*}
    Obviously, the uniform bounds for $A_t$ and $B_t$ are the same as above. \\
    \emph{Step 2: calculating the variance.} \\
    Denote the variance of $\prod_{k=\tau+1}^{t} (1-f(\eta_k)) A_{\tau}$ as $\sigma_{A\tau}^2$,
    \begin{align*}
        &\sigma_{A\tau}^2\lesssim \prod_{k=\tau+1}^{t} (1-f(\eta_k))^2\EE\left[\frac{\mathbbm{1}(a_\tau=1)}{\pi_{\tau}^2}\eta_{\tau}^2 (\inp{\widehat{U}_{\tau-1}\widehat{V}_{\tau-1}^{\top} - M}{X_{\tau}}-\xi_{\tau})^2\inp{X_{\tau}}{(\widehat{U}_{\tau-1}\widehat{V}_{\tau-1}^{\top} - M)\widehat{V}_{\tau-1}\widehat{V}_{\tau-1}^{\top}}^2\big|\mathcal{F}_{\tau-1}\right] \\ &\quad + \prod_{k=\tau+1}^{t} (1-f(\eta_k))^2\EE\left[\frac{\mathbbm{1}(a_\tau=1)}{\pi_{\tau}^2}\eta_{\tau}^2 (\inp{\widehat{U}_{\tau-1}\widehat{V}_{\tau-1}^{\top} - M}{X_{\tau}}-\xi_{\tau})^2\inp{X_{\tau}}{\widehat{U}_{\tau-1}\widehat{U}_{\tau-1}^{\top}(\widehat{U}_{\tau-1}\widehat{V}_{\tau-1}^{\top} - M)}^2\big|\mathcal{F}_{\tau-1}\right] \\ &\quad + \prod_{k=\tau+1}^{t} (1-f(\eta_k))^2 \EE\left[\frac{\mathbbm{1}(a_\tau=1)}{\pi_{\tau}^4}\eta_{\tau}^4 (\inp{\widehat{U}_{\tau-1}\widehat{V}_{\tau-1}^{\top} - M}{X_{\tau}}-\xi_{\tau})^4\inp{X_{\tau}\widehat{V}_{\tau-1}\widehat{U}_{\tau-1}^{\top}X_{\tau}}{\widehat{U}_{\tau-1}\widehat{V}_{\tau-1}^{\top} - M}^2\big|\mathcal{F}_{\tau-1}\right]. 
    \end{align*}
    For the first expectation, we have
    \begin{align*}
        &\quad \EE\left[\frac{\mathbbm{1}(a_{\tau}=1)}{\pi_{\tau}^2}\eta_{\tau}^2 (\inp{\widehat{U}_{\tau-1}\widehat{V}_{\tau-1}^{\top} - M}{X_{\tau}}-\xi_{\tau})^2\inp{X_{\tau}}{(\widehat{U}_{\tau-1}\widehat{V}_{\tau-1}^{\top} - M)\widehat{V}_{\tau-1}\widehat{V}_{\tau-1}^{\top}}^2\big|\mathcal{F}_{\tau-1}\right] \\
        &\leq \frac{2\eta_{\tau}^2}{\varepsilon_{\tau}}\EE\left[\inp{\widehat{U}_{\tau-1}\widehat{V}_{\tau-1}^{\top} - M}{X_{\tau}}^2|\mathcal{F}_{t-1}\right]\left|\inp{X_{\tau}\widehat{V}_{\tau-1}\widehat{V}_{\tau-1}^{\top}}{\widehat{U}_{\tau-1}\widehat{V}_{\tau-1}^{\top} - M}\right|^2 \\ &\quad + \frac{2\eta_{\tau}^2}{\varepsilon_{\tau}}\EE[\xi_{\tau}^2]\EE\left(\inp{X_{\tau}}{(\widehat{U}_{\tau-1}\widehat{V}_{\tau-1}^{\top} - M)\widehat{V}_{\tau-1}\widehat{V}_{\tau-1}^{\top}}^2|\mathcal{F}_{\tau-1}\right) \\
        &\leq \frac{2\eta_{\tau}^2}{\varepsilon_{\tau}}\frac{1}{d_1d_2}\fro{\widehat{U}_{\tau-1}\widehat{V}_{\tau-1}^{\top} - M}^2\fro{\widehat{U}_{\tau-1}\widehat{V}_{\tau-1}^{\top} - M}^2\|\widehat{V}_{\tau-1}\widehat{V}_{\tau-1}^{\top}\|_{2,\max}^2 \\ &\quad + \frac{2\eta_{\tau}^2\sigma^2}{\varepsilon_{\tau}}\frac{1}{d_1d_2}\sum_{ij}\left((\widehat{U}_{\tau-1}\widehat{V}_{\tau-1}^{\top} - M)\widehat{V}_{\tau-1}\widehat{V}_{\tau-1}^{\top}\right)_{ij}^2 \\
        &\leq \frac{2\eta_{\tau}^2}{\varepsilon_{\tau}}\frac{1}{d_1d_2}\fro{\widehat{U}_{\tau-1}\widehat{V}_{\tau-1}^{\top} - M}^4\frac{\mu_0r\lambda_{\max}^2}{d_2} + \frac{2\eta_{\tau}^2\sigma^2}{\varepsilon_{\tau}}\frac{1}{d_1d_2}\fro{(\widehat{U}_{\tau-1}\widehat{V}_{\tau-1}^{\top} - M)\widehat{V}_{\tau-1}\widehat{V}_{\tau-1}^{\top}}^2 \\
        &\leq  \frac{2\eta_{\tau}^2}{\varepsilon_{\tau}}\frac{1}{d_1d_2}\fro{\widehat{U}_{\tau-1}\widehat{V}_{\tau-1}^{\top} - M}^4\frac{\mu_0r\lambda_{\max}^2}{d_2} + \frac{2\eta_{\tau}^2\sigma^2}{\varepsilon_{\tau}}\frac{1}{d_1d_2}\|\widehat{V}_{\tau-1}\widehat{V}_{\tau-1}^{\top}\|^2\fro{\widehat{U}_{\tau-1}\widehat{V}_{\tau-1}^{\top} - M}^2 \\
        &\leq \frac{2\eta_{\tau}^2}{\varepsilon_{\tau}}\frac{1}{d_1d_2}\fro{\widehat{U}_{\tau-1}\widehat{V}_{\tau-1}^{\top} - M}^4\frac{\mu_0r\lambda_{\max}^2}{d_2} + \frac{2\eta_{\tau}^2\sigma^2}{\varepsilon_{\tau}}\frac{O(\lambda_{\max}^2)}{d_1d_2}\fro{\widehat{U}_{\tau-1}\widehat{V}_{\tau-1}^{\top} - M}^2.
    \end{align*}
    By symmetry,
    \begin{align*}
        &\quad \EE\left[\frac{\mathbbm{1}(a_\tau=1)}{\pi_{\tau}^2}\eta_{\tau}^2 (\inp{\widehat{U}_{\tau-1}\widehat{V}_{\tau-1}^{\top} - M}{X_{\tau}}-\xi_{\tau})^2\inp{X_{\tau}}{\widehat{U}_{\tau-1}\widehat{U}_{\tau-1}^{\top}(\widehat{U}_{\tau-1}\widehat{V}_{\tau-1}^{\top} - M)}^2\big|\mathcal{F}_{\tau-1}\right] \\
        &\leq \frac{2\eta_{\tau}^2}{\varepsilon_{\tau}}\frac{1}{d_1d_2}\fro{\widehat{U}_{\tau-1}\widehat{V}_{\tau-1}^{\top} - M}^4O(\frac{\mu_0r\lambda_{\max}^2}{d_1}) + \frac{2\eta_{\tau}^2\sigma^2}{\varepsilon_{\tau}}\frac{O(\lambda_{\max}^2)}{d_1d_2}\fro{\widehat{U}_{\tau-1}\widehat{V}_{\tau-1}^{\top} - M}^2.
    \end{align*}
    And the last expectation is
    \begin{align*}
        &\quad \EE\left[\frac{\mathbbm{1}(a_\tau=1)}{\pi_{\tau}^4}\eta_{\tau}^4 (\inp{\widehat{U}_{\tau-1}\widehat{V}_{\tau-1}^{\top} - M}{X_{\tau}}-\xi_{\tau})^4\inp{X_{\tau}\widehat{V}_{\tau-1}\widehat{U}_{\tau-1}^{\top}X_{\tau}}{\widehat{U}_{\tau-1}\widehat{V}_{\tau-1}^{\top} - M}^2\big|\mathcal{F}_{\tau-1}\right] \\
        &\lesssim \frac{\eta_{\tau}^4}{\varepsilon_{\tau}^3}\left(\frac{1}{d_1d_2}\sum_{ij}(\widehat{U}_{\tau-1}\widehat{V}_{\tau-1}^{\top} - M)^4 + \sigma^4\right)\|\widehat{V}_{\tau-1}\widehat{U}_{\tau-1}^{\top}\|_{\max}^2\fro{\widehat{U}_{\tau-1}\widehat{V}_{\tau-1}^{\top} - M}^2 \\
        &\lesssim \frac{\eta_{\tau}^4}{\varepsilon_{\tau}^3}\left(\frac{1}{d_1d_2}\sum_{ij}(\widehat{U}_{\tau-1}\widehat{V}_{\tau-1}^{\top} - M)^4 + \sigma^4\right)\fro{\widehat{U}_{\tau-1}\widehat{V}_{\tau-1}^{\top} - M}^2\frac{\mu_0^2r^2\lambda_{\max}^2}{d_1d_2}.
    \end{align*}
    This gives,
    \begin{align*}
        \sigma_{A\tau}^2&\lesssim \prod_{k=\tau+1}^{t} (1-f(\eta_k))^2\frac{\eta_{\tau}^2}{\varepsilon_{\tau}}\frac{1}{d_1d_2}\fro{\widehat{U}_{\tau-1}\widehat{V}_{\tau-1}^{\top} - M}^4\frac{\mu_0r\lambda_{\max}^2}{d_2} + \prod_{k=\tau+1}^{t} (1-f(\eta_k))^2\frac{\eta_{\tau}^2\lambda_{\max}^2\sigma^2}{\varepsilon_{\tau} d_1d_2}\fro{\widehat{U}_{\tau-1}\widehat{V}_{\tau-1}^{\top} - M}^2.
    \end{align*}
    The fourth order terms is dominated as long as $\eta_{\tau}\lesssim \frac{\varepsilon_{\tau}d_1d_2^{1/2}}{\sqrt{d_1d_2}\lambda_{\max}}$ and $\frac{\lambda_{\min}^2}{\sigma^2}\gg \frac{\eta_{\tau}^2}{\varepsilon_{\tau}^2}\frac{\mu_0r\lambda_{\max}^2}{d_1\kappa^2}$. Since $\fro{\widehat{U}_{\tau-1}\widehat{V}_{\tau-1}^{\top} - M}^2\leq \prod_{k=1}^{\tau-1}(1-\frac{f(\eta_{k})}{2}) \fro{\widehat{U}_0\widehat{V}_0^{\top} - M}^2 + \sum_{k=1}^{\tau-1}\frac{c_2\eta_k^2r\lambda_{\max}^2\sigma^2\log^2d_1}{\varepsilon_kd_2}$, we can further obtain
    \begin{align*}
        \sigma_{A\tau}^2&\lesssim \frac{\eta_{\tau}^2}{\varepsilon_{\tau}}\frac{1}{d_1d_2}\left(\prod_{k=1}^{t}(1-\frac{f(\eta_{k})}{2})^2 \fro{\widehat{U}_{0}\widehat{V}_{0}^{\top} - M}^4 + \left(\sum_{k=1}^{t}\frac{c_2\eta_k^2r\lambda_{\max}^2\sigma^2\log^2d_1}{\varepsilon_kd_2}\right)^2\right)\frac{\mu_0r\lambda_{\max}^2}{d_2} \\ &\quad +  \frac{\eta_{\tau}^2\lambda_{\max}^2\sigma^2}{\varepsilon_{\tau} d_1d_2}  \prod_{k=1}^{t}(1-\frac{f(\eta_{k})}{2})\fro{\widehat{U}_{0}\widehat{V}_{0}^{\top} - M}^2 + \frac{\eta_{\tau}^2\lambda_{\max}^2\sigma^2}{\varepsilon_{\tau} d_1d_2}\sum_{k=1}^{t}\frac{c_2\eta_k^2r\lambda_{\max}^2\sigma^2\log^2d_1}{\varepsilon_kd_2}.
    \end{align*}
    As a result, 
    \begin{align*}
        \sqrt{\sum_{\tau=1}^t \sigma_{A\tau}^2}&\lesssim \sqrt{\sum_{\tau=1}^t \frac{\eta_{\tau}^2}{\varepsilon_{\tau}}}\frac{1}{\sqrt{d_1d_2}}\prod_{k=1}^{t}(1-\frac{f(\eta_{k})}{2}) \fro{\widehat{U}_{0}\widehat{V}_{0}^{\top} - M}^2O(\sqrt{\frac{\mu_0r\lambda_{\max}^2}{d_2}}) \\ &\quad + \sqrt{\sum_{\tau=1}^t \frac{\eta_{\tau}^2}{\varepsilon_{\tau}}}\frac{1}{\sqrt{d_1d_2}}\sum_{k=1}^{t}\frac{c_2\eta_k^2r\lambda_{\max}^2\sigma^2\log^2d_1}{\varepsilon_kd_2}O(\sqrt{\frac{\mu_0r\lambda_{\max}^2}{d_2}}) \\ &\quad + \prod_{k=1}^{t}(1-\frac{f(\eta_{k})}{2})\fro{\widehat{U}_{0}\widehat{V}_{0}^{\top} - M}^2\log^{-1/2}d_1 + \sum_{\tau=1}^t \frac{\eta_{\tau}^2\lambda_{\max}^2\sigma^2\log^{3/2}d_1}{\varepsilon_{\tau} d_1d_2},
    \end{align*}
    where we apply inequality $2ab\leq a^2+b^2$ to term $\sqrt{\sum_{\tau=1}^t \frac{\eta_{\tau}^2}{\varepsilon_{\tau}}}\sqrt{\frac{\lambda_{\max}^2\sigma^2}{d_1d_2}}\prod_{k=1}^{t}(1-\frac{f(\eta_{k})}{2})^{1/2}\fro{\widehat{U}_{0}\widehat{V}_{0}^{\top} - M}$. \\
    Similarly, denote the variance of $\prod_{k=\tau+1}^{t} (1-f(\eta_k)) B_{\tau}$ as $\sigma_{B\tau}^2$, it can be shown that,
    \begin{align*}
        \sigma_{B\tau}^2&\lesssim \prod_{k=\tau+1}^{t} (1-f(\eta_k))^2 \frac{\eta_{\tau}^4}{\varepsilon_{\tau}^3}\left(\EE\left[\inp{\widehat{U}_{{\tau}-1}\widehat{V}_{{\tau}-1}^{\top}-M}{X_{\tau}}^4|\mathcal{F}_{\tau-1}\right]\max\left(\|\widehat{V}_{\tau-1}\widehat{V}_{\tau-1}^{\top}\|_{2,\max}^4, \|\widehat{U}_{\tau-1}\widehat{U}_{\tau-1}^{\top}\|_{2,\max}^4\right) \right.\\ &\left.\quad + \EE[\xi_{\tau}^4]\max\left(\EE\left[\fro{X_{\tau}\widehat{V}_{\tau-1}\widehat{V}_{\tau-1}^{\top}}^4|\mathcal{F}_{\tau-1}\right], \EE\left[\fro{\widehat{U}_{\tau-1}\widehat{U}_{\tau-1}^{\top}X_{\tau}}^4|\mathcal{F}_{\tau-1}\right]\right)\right) \\ 
        &\lesssim \prod_{k=\tau+1}^{t} (1-f(\eta_k))^2 \frac{\eta_{\tau}^4}{\varepsilon_{\tau}^3}\left(\frac{1}{d_1d_2}\fro{\widehat{U}_{{\tau}-1}\widehat{V}_{{\tau}-1}^{\top}-M}^2\|\widehat{U}_{{\tau}-1}\widehat{V}_{{\tau}-1}^{\top}-M\|_{\max}^2\frac{\mu_0^2r^2\lambda_{\max}^4}{d_2^2}  + \sigma^4\frac{\mu_0r \lambda_{\max}^4}{d_2^2}\right) \\
        &\lesssim \frac{\eta_{\tau}^4}{\varepsilon_{\tau}^3}\frac{1}{d_1d_2^2} \left(\prod_{k=1}^{t} (1-\frac{f(\eta_k)}{2})\fro{\widehat{U}_0\widehat{V}_0^{\top}-M} + \sum_{k=1}^{t}\frac{\eta_k^2r\lambda_{\max}^2\sigma^2\log^2d_1}{\varepsilon_kd_2}\right)^2\frac{\mu_0^3r^3\kappa^2\lambda_{\max}^4}{d_2^2} + \frac{\eta_{\tau}^4}{\varepsilon_{\tau}^3}\sigma^4\frac{\mu_0r\lambda_{\max}^4}{d_2^2},
    \end{align*}
    where the second term in the second last inequality comes from
    $$
    \EE\left(\fro{X_{\tau}\widehat{V}_{\tau-1}\widehat{V}_{\tau-1}^{\top}}^4|\mathcal{F}_{\tau-1}\right)\leq \EE\left(\fro{X_{\tau}\widehat{V}_{\tau-1}\widehat{V}_{\tau-1}^{\top}}^2|\mathcal{F}_{\tau-1}\right)\|\widehat{V}_{\tau-1}\widehat{V}_{\tau-1}^{\top}\|_{2,\max}^2\leq \frac{rO(\lambda_{\max}^2)}{d_2}O(\frac{\mu_0r\lambda_{\max}^2}{d_2}).
    $$ 
    And in the last inequality we again use 
    \begin{align*}
    \|\widehat{U}_{\tau-1}\widehat{V}_{\tau-1}^{\top} - M\|_{\max}^2&\lesssim \frac{\mu_0r\kappa^2}{d_2}\fro{\widehat{U}_{\tau-1}\widehat{V}_{\tau-1}^{\top} - M}^2\\
    &\lesssim \mu_0r\kappa^2\prod_{k=1}^{\tau-1} (1-\frac{f(\eta_k)}{2})\frac{\fro{\widehat{U}_0\widehat{V}_0^{\top}-M}}{d_2} + \sum_{k=1}^{\tau-1}\frac{\mu_0\kappa^2\eta_k^2r^2\lambda_{\max}^2\sigma^2\log^2d_1}{\varepsilon_k d_2^2}.
    \end{align*}
    Therefore, 
    \begin{align*}
    \sqrt{\sum_{\tau=1}^t \sigma_{B\tau}^2} \lesssim& \sqrt{\sum_{\tau=1}^t \frac{\eta_{\tau}^4}{\varepsilon_{\tau}^3}}\frac{1}{d_1d_2^{1/2}}(\prod_{k=1}^{t} (1-\frac{f(\eta_k)}{2})\fro{\widehat{U}_0\widehat{V}_0^{\top}-M} \\
    &+ \sum_{k=1}^{t}\frac{\eta_k^2r\lambda_{\max}^2\sigma^2\log^2d_1}{\varepsilon_kd_2})\frac{\mu_0r\lambda_{\max}^2}{d_2} + \sqrt{\sum_{\tau=1}^t \frac{\eta_{\tau}^4}{\varepsilon_{\tau}^3}}\frac{\mu_0r\lambda_{\max}^2\sigma^2}{d_2}.
    \end{align*}
    \emph{Step 3: Proof of concentration.} \\
    From the calculation in \emph{Step 2}, 
    $$
    R_A\log d_1 + \sqrt{\sum_{\tau=1}^t \sigma_{A\tau}^2\log d_1} \lesssim \prod_{k=1}^{t} (1-\frac{f(\eta_k)}{2}) \fro{\widehat{U}_0\widehat{V}_0^{\top}-M}^2 + \sum_{k=1}^{t}\frac{c_2\eta_k^2r\lambda_{\max}^2\sigma^2\log^2d_1}{\varepsilon_kd_2}
    $$
    and 
    $$
    R_B\log d_1 + \sqrt{\sum_{\tau=1}^t \sigma_{B\tau}^2\log d_1} \lesssim \prod_{k=1}^{t} (1-\frac{f(\eta_k)}{2}) \fro{\widehat{U}_0\widehat{V}_0^{\top}-M}^2 + \sum_{k=1}^{t}\frac{c_2\eta_k^2r\lambda_{\max}^2\sigma^2\log^2d_1}{\varepsilon_kd_2},
    $$ 
    as long as: (1). $\max_{\tau\leq t}\frac{\eta_{\tau}^2}{\varepsilon_{\tau}^2}\leq \sum_{\tau=1}^t \frac{\eta_{\tau}^2}{\varepsilon_{\tau}}$; (2). $\eta_{\tau}\lesssim \frac{\varepsilon_{\tau}d_2}{\sqrt{\mu_0r}\lambda_{\max}\log^{1/2}d_1}$; (3). $\sum_{\tau=1}^t \frac{\eta_{\tau}^2}{\varepsilon_{\tau}} \lesssim \frac{d_1d_2^2}{\mu_0r\lambda_{\max}^2\log d_1}$; (4). $\sum_{\tau=1}^t \frac{\eta_{\tau}^4}{\varepsilon_{\tau}^3}\lesssim \frac{d_1^2d_2^3}{\mu_0^2r^2\lambda_{\max}^4\log d_1}$; (5). $\sqrt{\sum_{\tau=1}^t \frac{\eta_{\tau}^4}{\varepsilon_{\tau}^3}} \lesssim \sum_{\tau=1}^t\frac{\eta_{\tau}^2}{\varepsilon_{\tau}}$. 
    Under the stepsize and SNR conditions in Equation \ref{eq:thm-MCB-eq1}, all of them can be satisfied.
    
    By martingale Bernstein inequality, with probability at least $1-d_1^{-200}$,  
    \begin{align*}
        &\sum_{\tau=1}^{t-1}\prod_{k=\tau+1}^{t} (1-f(\eta_k))A_{\tau}+A_t \lesssim \prod_{k=1}^{t} (1-\frac{f(\eta_k)}{2}) \fro{\widehat{U}_0\widehat{V}_0^{\top}-M}^2 + \sum_{k=1}^{t}\frac{\eta_k^2r\lambda_{\max}^2\sigma^2\log^2d_1}{\varepsilon_kd_2}, \\
        &\sum_{\tau=1}^{t-1}\prod_{k=\tau+1}^{t} (1-f(\eta_k))B_{\tau}+B_t \lesssim \prod_{k=1}^{t} (1-\frac{f(\eta_k)}{2}) \fro{\widehat{U}_0\widehat{V}_0^{\top}-M}^2 + \sum_{k=1}^{t}\frac{\eta_k^2r\lambda_{\max}^2\sigma^2\log^2d_1}{\varepsilon_kd_2}.
    \end{align*}
\end{proof}

\subsection{Proof of Lemma \ref{incoherence}}
\begin{proof}
    For any $l$ in $[d_1]$,
\begin{align*}
    \|e_l^{\top}\widehat{U}_t\widehat{V}_t^{\top}\|^2 &= e_l^{\top} (\widehat{U}_{t-1}\widehat{V}_{t-1}^{\top}-\Delta_{t})(\widehat{V}_{t-1}\widehat{U}_{t-1}^{\top} - \Delta_{t}^{\top})e_l = \|e_l^{\top}\widehat{U}_{t-1}\widehat{V}_{t-1}^{\top}\|^2 - 2e_l^{\top}\Delta_{t}\widehat{V}_{t-1}\widehat{U}_{t-1}^{\top}e_l + \|e_l^{\top}\Delta_{t}\|^2.
\end{align*}
This leads to
\begin{align*}
    \EE[\|e_l^{\top}\widehat{U}_t\widehat{V}_t^{\top}\|^2|\mathcal{F}_{t-1}]= \|e_l^{\top}\widehat{U}_{t-1}\widehat{V}_{t-1}^{\top}\|^2 - 2\EE[e_l^{\top}\Delta_{t}\widehat{V}_{t-1}\widehat{U}_{t-1}^{\top}e_l|\mathcal{F}_{t-1}] + \EE[\|e_l^{\top}\Delta_{t}\|^2|\mathcal{F}_{t-1}].
\end{align*}
We first compute $2\EE[e_l^{\top}\Delta_{t}\widehat{V}_{t-1}\widehat{U}_{t-1}^{\top}e_l|\mathcal{F}_{t-1}]$. Denote $\delta_{ij}=\mathbbm{1}(i=j)$, then 
\begin{align*}
    &2\EE(e_l^{\top}\Delta_{t}\widehat{V}_{t-1}\widehat{U}_{t-1}^{\top}e_l|\mathcal{F}_{t-1})= 2\EE\left[e_l^{\top} \frac{\mathbbm{1}(a_t=1)}{\pi_t}\eta_t(\inp{\widehat{U}_{t-1}\widehat{V}_{t-1}^{\top} - M}{X_t}-\xi_t)X_t\widehat{V}_{t-1}\widehat{V}_{t-1}^{\top} \widehat{V}_{t-1}\widehat{U}_{t-1}^{\top}e_l|\mathcal{F}_{t-1}\right] \\ &\quad + 2\EE\left[e_l^{\top} \frac{\mathbbm{1}(a_t=1)}{\pi_t}\eta_t(\inp{\widehat{U}_{t-1}\widehat{V}_{t-1}^{\top}- M}{X_t}-\xi_t)\widehat{U}_{t-1}\widehat{U}_{t-1}^{\top}X_t V_{t-1}U_{t-1}^{\top}e_l|\mathcal{F}_{t-1}\right] \\ &\quad - 2\EE\left[e_l^{\top} \frac{\mathbbm{1}(a_t=1)}{\pi_t^2}\eta_t^2(\inp{U_{t-1}V_{t-1}^{\top}- M}{X_t}-\xi_t)^2X_tV_{t-1}U_{t-1}^{\top}X_t V_{t-1}U_{t-1}^{\top}e_l|\mathcal{F}_{t-1}\right] \\
    &\geq \frac{2\eta_t}{d_1d_2}e_l^{\top}(U_{t-1}V_{t-1}^{\top} - M)\widehat{V}_{t-1}\widehat{V}_{t-1}^{\top}\widehat{V}_{t-1}\widehat{U}_{t-1}^{\top}e_l  + \frac{2\eta_t}{d_1d_2}e_l^{\top}\widehat{U}_{t-1}\widehat{U}_{t-1}^{\top}(\widehat{U}_{t-1}\widehat{V}_{t-1}^{\top} - M)\widehat{V}_{t-1}\widehat{U}_{t-1}^{\top}e_l \\ &\quad - \frac{4\eta_t^2}{\varepsilon_t d_1d_2}\sum_{ij}\delta_{il}(\widehat{U}_{t-1}\widehat{V}_{t-1}^{\top} - M)_{ij}^2(\widehat{U}_{t-1}\widehat{V}_{t-1}^{\top})_{ij}^2 - \frac{4\eta_t^2\sigma^2}{\varepsilon_t d_1d_2}\sum_{ij}\delta_{il}(\widehat{U}_{t-1}\widehat{V}_{t-1}^{\top})_{ij}^2 \\
    &\geq \frac{\eta_t\lambda_{\min}}{d_1d_2}e_l^{\top}\widehat{U}_{t-1}\widehat{V}_{t-1}^{\top}\widehat{V}_{t-1}\widehat{U}_{t-1}^{\top}e_l - \frac{4\eta_t\lambda_{\max}}{d_1d_2}e_l^{\top}M\widehat{V}_{t-1}\widehat{U}_{t-1}^{\top}e_l + \frac{\eta_t\lambda_{\min}}{d_1d_2}e_l^{\top}\widehat{U}_{t-1}\widehat{V}_{t-1}^{\top}\widehat{V}_{t-1}\widehat{U}_{t-1}^{\top}e_l \\ &\quad - \frac{4\eta_t\lambda_{\max}}{d_1d_2}e_l^{\top}M\widehat{V}_{t-1}\widehat{U}_{t-1}^{\top}e_l  - \frac{4\eta_t^2}{\varepsilon_t d_1d_2}\frac{\mu_0^2r^2\lambda_{\max}^4}{d_1^2} - \frac{4\eta_t^2\sigma^2}{\varepsilon_t d_1d_2}\frac{\mu_0r\lambda_{\max}^2}{d_1}\\
    &\geq \frac{2\eta_t\lambda_{\min}}{d_1d_2}\|e_l^{\top}\widehat{U}_{t-1}\widehat{V}_{t-1}^{\top}\|^2 - \frac{8\eta_t\lambda_{\max}}{d_1d_2}\frac{\sqrt{\mu\mu_0} r\lambda_{\max}^2}{d_1}  - \frac{4\eta_t^2}{\varepsilon_t d_1d_2}\frac{\mu_0^2r^2\lambda_{\max}^4}{d_1^2} - \frac{4\eta_t^2\sigma^2}{\varepsilon_t d_1d_2}\frac{\mu_0r\lambda_{\max}^2}{d_1}.
\end{align*}
The second order term in the second last inequality comes from $\sum_{ij}\delta_{il}(\widehat{U}_{t-1}\widehat{V}_{t-1}^{\top} - M)_{ij}^2(\widehat{U}_{t-1}\widehat{V}_{t-1}^{\top})_{ij}^2\leq \|\widehat{U}_{t-1}\widehat{V}_{t-1}^{\top} - M\|_{2,\max}^2\|\widehat{U}_{t-1}\widehat{V}_{t-1}^{\top}\|_{2,\max}^2$. And the other term is similar. The last inequality holds since $e_l^{\top}M\widehat{V}_{t-1}\widehat{U}_{t-1}^{\top}e_l\leq \|e_l^{\top}M\|\|\widehat{V}_{t-1}\widehat{U}_{t-1}^{\top}e_l\|$ and $\|e_l^{\top}M\|\leq \sqrt{\frac{\mu r\lambda_{\max}^2}{d_1}}$ while $\|\widehat{V}_{t-1}\widehat{U}_{t-1}^{\top}e_l\|\leq \sqrt{\frac{\mu_0r\lambda_{\max}^2}{d_1}}$.\\
Then we compute $\EE(\|e_l^{\top}\Delta_{t}\|^2|\mathcal{F}_{t-1})$. Since $(a+b+c)^2\leq 3a^2+3b^2+3c^2$,
\begin{align*}
    \EE(\|e_l^{\top}\Delta_{t}\|^2|\mathcal{F}_{t-1})&\leq 3\EE\left[e_l^{\top} \frac{\mathbbm{1}(a_t=1)}{\pi_t^2}\eta_t^2(\inp{\widehat{U}_{t-1}\widehat{V}_{t-1}^{\top} - M}{X_t}-\xi_t)^2X_t\widehat{V}_{t-1}\widehat{V}_{t-1}^{\top}\widehat{V}_{t-1}\widehat{V}_{t-1}^{\top}X_t^{\top} e_l|\mathcal{F}_{t-1}\right] \\ &\quad + 3\EE\left[e_l^{\top} \frac{\mathbbm{1}(a_t=1)}{\pi_t^2}\eta_t^2(\inp{\widehat{U}_{t-1}\widehat{V}_{t-1}^{\top}- M}{X_t}-\xi_t)^2\widehat{U}_{t-1}\widehat{U}_{t-1}^{\top}X_tX_t^{\top}\widehat{U}_{t-1}\widehat{U}_{t-1}^{\top} e_l|\mathcal{F}_{t-1}\right] \\ &\quad + 3\EE\left[e_l^{\top} \frac{\mathbbm{1}(a_t=1)}{\pi_t^4}\eta_t^4(\inp{\widehat{U}_{t-1}\widehat{V}_{t-1}^{\top}- M}{X_t}-\xi_t)^4X_t\widehat{V}_{t-1}\widehat{U}_{t-1}^{\top}X_tX_t^{\top}\widehat{U}_{t-1}\widehat{V}_{t-1}^{\top}X_t^{\top}  e_l|\mathcal{F}_{t-1} \right].                                       
\end{align*}
The first term can be written as, 
\begin{align*}
    &3\EE\left[e_l^{\top} \frac{\mathbbm{1}(a_t=1)}{\pi_t^2}\eta_t^2(\inp{\widehat{U}_{t-1}\widehat{V}_{t-1}^{\top} - M}{X_t}-\xi_t)^2X_t\widehat{V}_{t-1}\widehat{V}_{t-1}^{\top}\widehat{V}_{t-1}\widehat{V}_{t-1}^{\top}X_t^{\top} e_l|\mathcal{F}_{t-1}\right] \\
    &\leq \frac{6\eta_t^2}{\varepsilon_t d_1d_2}\sum_{ij}(\widehat{U}_{t-1}\widehat{V}_{t-1}^{\top} - M)_{ij}^2\delta_{il}\|e_j^{\top}\widehat{V}_{t-1}\widehat{V}_{t-1}^{\top}\|^2 + \frac{6\eta_t^2}{\varepsilon_t d_1d_2}\sigma^2\sum_{ij} \delta_{il}\|e_j^{\top}\widehat{V}_{t-1}\widehat{V}_{t-1}^{\top}\|^2 \\
    &\leq \frac{6\eta_t^2}{\varepsilon_t d_1d_2}\frac{\mu_0r\lambda_{\max}^2}{d_2}\|e_l^{\top}(\widehat{U}_{t-1}\widehat{V}_{t-1}^{\top} - M)\|^2 + \frac{6\eta_t^2}{\varepsilon_t d_1}\frac{\mu_0r\lambda_{\max}^2\sigma^2}{d_2} \\
    &\leq \frac{6\eta_t^2}{\varepsilon_t d_1d_2}\frac{\mu_0^2r^2\lambda_{\max}^4}{d_1d_2} + \frac{6\eta_t^2}{\varepsilon_t d_1}\frac{\mu_0r\lambda_{\max}^2\sigma^2}{d_2}.
\end{align*}
Similarly, 
\begin{align*}
    &3\EE\left[e_l^{\top} \frac{\mathbbm{1}(a_t=1)}{\pi_t^2}\eta_t^2(\inp{\widehat{U}_{t-1}\widehat{V}_{t-1}^{\top}- M}{X_t}-\xi_t)^2\widehat{U}_{t-1}\widehat{U}_{t-1}^{\top}X_tX_t^{\top}\widehat{U}_{t-1}\widehat{U}_{t-1}^{\top} e_l|\mathcal{F}_{t-1}\right] \\
    &\leq \frac{6\eta_t^2}{\varepsilon_t d_1d_2}\sum_{ij} (\widehat{U}_{t-1}\widehat{V}_{t-1}^{\top} - M)_{ij}^2(\widehat{U}_{t-1}\widehat{U}_{t-1}^{\top})_{li}^2 + \frac{6\eta_t^2\sigma^2}{\varepsilon_t d_1d_2}\sum_{ij} (\widehat{U}_{t-1}\widehat{U}_{t-1}^{\top})_{li}^2 \\
    &\leq \frac{6\eta_t^2}{\varepsilon_t d_1d_2}\frac{\mu_0^2r^2\lambda_{\max}^4}{d_1^2} + \frac{6\eta_t^2\sigma^2}{\varepsilon_t d_1}\frac{\mu_0r\lambda_{\max}^2}{d_1}.
\end{align*}
And the last fourth order term can be bounded in the same way. Moreover, the first order term $\frac{8\eta_t\lambda_{\max}}{d_1d_2}\frac{\sqrt{\mu\mu_0}r\lambda_{\max}^2}{d_1}$will dominate all the higher order terms as long as $\eta_t\lesssim \frac{\varepsilon_t d_2}{\mu_0r\lambda_{\max}}$ and $\frac{\lambda_{\min}^2}{\sigma^2}\gg \frac{\eta_t}{\varepsilon_t}\frac{\lambda_{\max}d_1}{\kappa}$. So the sum of all the higher order terms will be less than $\frac{2\eta_t\lambda_{\max}}{d_1d_2}\frac{\sqrt{\mu\mu_0} r\lambda_{\max}^2}{d_1}$. Recall $f(\eta_t)=\frac{\eta_t\lambda_{\min}}{d_1d_2}$. Overall, we have
\begin{align*}
    \|e_l^{\top}U_tV_t^{\top}\|^2 &= \|e_l^{\top}\widehat{U}_{t-1}\widehat{V}_{t-1}^{\top}\|^2 - 2e_l^{\top}\Delta_{t}\widehat{V}_{t-1}\widehat{U}_{t-1}^{\top}e_l +\|e_l^{\top}\Delta_{t}\|^2\\
	& = \|e_l^{\top}\widehat{U}_{t-1}\widehat{V}_{t-1}^{\top}\|^2 - 2(e_l^{\top}\Delta_{t}\widehat{V}_{t-1}\widehat{U}_{t-1}^{\top}e_l -\EE(e_l^{\top}\Delta_{t}\widehat{V}_{t-1}\widehat{U}_{t-1}^{\top}e_l|\mathcal{F}_{t-1})) \\
	&\quad +(\|e_l^{\top}\Delta_{t}\|^2 - \EE(\|e_l^{\top}\Delta_{t}\|^2|\mathcal{F}_{t-1})) - 2\EE(e_l^{\top}\Delta_{t}\widehat{V}_{t-1}\widehat{U}_{t-1}^{\top}e_l|\mathcal{F}_{t-1}) + \EE( \|e_l^{\top}\Delta_{t}\|^2|\mathcal{F}_{t-1})\\
	&\leq (1-f(\eta_t))\|e_l^{\top}\widehat{U}_{t-1}\widehat{V}_{t-1}^{\top}\|^2 + \frac{10\eta_t\sqrt{\mu\mu_0} r\lambda_{\max}^3}{d_1^2d_2}  \\
	&\quad + \underbrace{2(\EE(e_l^{\top}\Delta_{t}\widehat{V}_{t-1}\widehat{U}_{t-1}^{\top}e_l|\mathcal{F}_{t-1}) -e_l^{\top}\Delta_{t}\widehat{V}_{t-1}\widehat{U}_{t-1}^{\top}e_l)}_{A_{t}} +\underbrace{(\|e_l^{\top}\Delta_{t}\|^2 - \EE(\|e_l^{\top}\Delta_{t}\|^2|\mathcal{F}_{t-1}))}_{B_{t}} \\
    &\leq (1-f(\eta_t))(1-f(\eta_{t-1}))\|e_l^{\top}\widehat{U}_{t-2}\widehat{V}_{t-2}^{\top}\|^2 + (1-f(\eta_{t-1}))\frac{10\eta_{t-1}\sqrt{\mu\mu_0} r\lambda_{\max}^3}{ d_1^2d_2} + \frac{10\eta_t\sqrt{\mu\mu_0} r\lambda_{\max}^3}{d_1^2d_2} \\
    &\quad + (1-f(\eta_t))(A_{t-1}+B_{t-1})+(A_{t}+B_{t}) \\
    &\leq  \prod_{\tau=1}^{t}(1-f(\eta_{\tau})) \|e_l^{\top}\widehat{U}_{0}\widehat{V}_0^{\top}\|^2 +  \sum_{\tau=1}^{t}\frac{10\eta_{\tau}\sqrt{\mu\mu_0} r\lambda_{\max}^3}{d_1^2d_2}  + \sum_{\tau=1}^{t}A_{\tau} + \sum_{\tau=1}^{t}B_{\tau}.
\end{align*}
It can be shown that the first term is less than $\frac{1}{6}\frac{\mu_0r\lambda_{\max}^2}{d_1}$. The second term will be less than $\frac{1}{2}\frac{\mu_0r\lambda_{\max}^2}{d_1}$ as long as $\mu_0=100\mu$ and $\sum_{\tau=1}^{t}\eta_{\tau}\leq \frac{d_1d_2}{2\lambda_{\max}}$, which is satisfied under the stepsize condition in Theorem \ref{thm:MCB-conv}. Then we need to show $\sum_{\tau=1}^{t}A_{\tau}\leq \frac{1}{6}\frac{\mu_0r\lambda_{\max}^2}{d_1}$ and $\sum_{\tau=1}^{t}B_{\tau}\leq \frac{1}{6}\frac{\mu_0r\lambda_{\max}^2}{d_1}$.\\
Similar as the proof of Lemma \ref{lemmafro}, we first derive the uniform bound for $A_{\tau}$ and $B_{\tau}$. Under $\mathcal{E}_{\tau-1}$, obviously, $\|\widehat{U}_{\tau-1}\widehat{V}_{\tau-1}^{\top}-M\|_{\max}^2\lesssim \frac{\lambda_{\max}^2}{d_1d_2}$. We use this trivial bound in the following calculation. 
\begin{align*}
    |e_l^{\top}\Delta_{\tau}\widehat{V}_{\tau-1}\widehat{U}_{\tau-1}^{\top}e_l|&\leq |e_l^{\top} \frac{\mathbbm{1}(a_\tau=1)}{\pi_\tau}\eta_\tau(\inp{\widehat{U}_{\tau-1}\widehat{V}_{\tau-1}^{\top} - M}{X_\tau}-\xi_\tau)X_\tau\widehat{V}_{\tau-1}\widehat{V}_{\tau-1}^{\top} \widehat{V}_{\tau-1}\widehat{U}_{\tau-1}^{\top}e_l| \\ &\quad + |e_l^{\top} \frac{\mathbbm{1}(a_\tau=1)}{\pi_\tau}\eta_\tau(\inp{\widehat{U}_{\tau-1}\widehat{V}_{\tau-1}^{\top}- M}{X_\tau}-\xi_\tau)\widehat{U}_{\tau-1}\widehat{U}_{\tau-1}^{\top}X_\tau \widehat{V}_{\tau-1}\widehat{U}_{\tau-1}^{\top}e_l| \\ &\quad + |e_l^{\top} \frac{\mathbbm{1}(a_\tau=1)}{\pi_\tau^2}\eta_\tau^2(\inp{\widehat{U}_{\tau-1}\widehat{V}_{\tau-1}^{\top}- M}{X_\tau}-\xi_\tau)^2X_\tau \widehat{V}_{\tau-1}\widehat{U}_{\tau-1}^{\top}X_\tau \widehat{V}_{\tau-1}\widehat{U}_{\tau-1}^{\top}e_l| \\
    &\leq \frac{2\eta_\tau}{\varepsilon_\tau}|\inp{\widehat{U}_{\tau-1}\widehat{V}_{\tau-1}^{\top} - M}{X_\tau}-\xi_\tau|\max_{j}|e_j^{\top}\widehat{V}_{\tau-1}\widehat{V}_{\tau-1}^{\top} \widehat{V}_{\tau-1}\widehat{U}_{\tau-1}^{\top}e_l| \\ &\quad + \frac{2\eta_\tau}{\varepsilon_\tau}|\inp{\widehat{U}_{\tau-1}\widehat{V}_{\tau-1}^{\top} - M}{X_\tau}-\xi_\tau|\|e_l^{\top}\widehat{U}_{\tau-1}\widehat{U}_{\tau-1}^{\top}X_\tau\| \|\widehat{V}_{\tau-1}\widehat{U}_{\tau-1}^{\top}e_l\| \\ &\quad + \frac{4\eta_\tau^2}{\varepsilon_\tau^2}|\inp{\widehat{U}_{\tau-1}\widehat{V}_{\tau-1}^{\top} - M}{X_\tau}-\xi_\tau|^2\max_j\|e_j^{\top}\widehat{V}_{\tau-1}\widehat{U}_{\tau-1}^{\top}X_\tau\| \|\widehat{V}_{\tau-1}\widehat{U}_{\tau-1}^{\top}e_l\| \\  
    &\leq \frac{2\eta_\tau}{\varepsilon_\tau}(\|\widehat{U}_{\tau-1}\widehat{V}_{\tau-1}^{\top} - M\|_{\max} + \sigma\log^{1/2}d_1)\frac{\mu_0^2r^2\lambda_{\max}^2}{d_1d_2} \\ &\quad + \frac{2\eta_\tau}{\varepsilon_\tau}(\|\widehat{U}_{\tau-1}\widehat{V}_{\tau-1}^{\top} - M\|_{\max} + \sigma\log^{1/2}d_1)\sqrt{\frac{\mu_0^2r^2\lambda_{\max}^2}{d_1d_2}}\sqrt{\frac
    {\mu_0r\lambda_{\max}^2}{d_1}} \\ &\quad + \frac{4\eta_\tau^2}{\varepsilon_\tau^2}(\|\widehat{U}_{\tau-1}\widehat{V}_{\tau-1}^{\top} - M\|_{\max}^2+\sigma^2\log d_1)\sqrt{\frac{\mu_0^2r^2\lambda_{\max}^2}{d_1d_2}}\sqrt{\frac
    {\mu_0r\lambda_{\max}^2}{d_1}} \\
    &\lesssim \frac{\eta_\tau}{\varepsilon_\tau}(\frac{\lambda_{\max}}{\sqrt{d_1d_2}}+\sigma\log^{1/2}d_1)\sqrt{\frac{\mu_0^2r^2\lambda_{\max}^2}{d_1d_2}}\sqrt{\frac
    {\mu_0r\lambda_{\max}^2}{d_1}}.
\end{align*}
The first order term dominates as long as $\eta_\tau\lesssim \frac{\varepsilon_\tau\sqrt{d_1d_2}}{\lambda_{\max}}$ and $\frac{\lambda_{\min}^2}{\sigma^2}\gg \frac{\eta_\tau^2}{\varepsilon_\tau^2} \frac{\lambda_{\max}^2\log d_1}{\kappa}$. Then
\begin{align*}
    A_{\tau}&\lesssim \max_{\tau}\frac{\eta_{\tau}}{\varepsilon_{\tau}}(\frac{\lambda_{\max}}{\sqrt{d_1d_2}}+\sigma\log^{1/2}d_1)O(\sqrt{\frac{\mu_0^2r^2\lambda_{\max}^2}{d_1d_2}})O(\sqrt{\frac
    {\mu_0r\lambda_{\max}^2}{d_1}}):=R_A.
\end{align*}
For $\|e_l^{\top}\Delta_{\tau}\|^2$, we also have
\begin{align*}
    &\|e_l^{\top}\Delta_{\tau}\|^2\lesssim |e_l^{\top} \frac{\mathbbm{1}(a_\tau=1)}{\pi_\tau^2}\eta_\tau^2(\inp{\widehat{U}_{\tau-1}\widehat{V}_{\tau-1}^{\top} - M}{X_\tau}-\xi_\tau)^2X_\tau\widehat{V}_{\tau-1}\widehat{V}_{\tau-1}^{\top}\widehat{V}_{\tau-1}\widehat{V}_{\tau-1}^{\top}X_\tau^{\top} e_l| \\ &\quad + |e_l^{\top} \frac{\mathbbm{1}(a_\tau=1)}{\pi_\tau^2}\eta_\tau^2(\inp{\widehat{U}_{\tau-1}\widehat{V}_{\tau-1}^{\top}- M}{X_\tau}-\xi_\tau)^2\widehat{U}_{\tau-1}\widehat{U}_{\tau-1}^{\top}X_\tau X_\tau^{\top}\widehat{U}_{\tau-1}\widehat{U}_{\tau-1}^{\top} e_l| \\ &\quad + |e_l^{\top} \frac{\mathbbm{1}(a_\tau=1)}{\pi_\tau^4}\eta_\tau^4(\inp{\widehat{U}_{\tau-1}\widehat{V}_{\tau-1}^{\top}- M}{X_\tau}-\xi_\tau)^4X_\tau \widehat{V}_{\tau-1}\widehat{U}_{\tau-1}^{\top}X_\tau X_\tau^{\top}\widehat{U}_{\tau-1}\widehat{V}_{\tau-1}^{\top}X_\tau^{\top}  e_l| \\
    &\lesssim \frac{\eta_\tau^2}{\varepsilon_\tau^2}|\inp{\widehat{U}_{\tau-1}\widehat{V}_{\tau-1}^{\top} - M}{X_t}-\xi_t|^2\max_j|e_j^{\top}\widehat{V}_{\tau-1}\widehat{V}_{\tau-1}^{\top}\widehat{V}_{\tau-1}\widehat{V}_{\tau-1}^{\top}e_j| \\ &\quad + \frac{\eta_\tau^2}{\varepsilon_\tau^2}|\inp{\widehat{U}_{\tau-1}\widehat{V}_{\tau-1}^{\top} - M}{X_\tau}-\xi_\tau|^2\|e_l^{\top}\widehat{U}_{\tau-1}\widehat{U}_{\tau-1}^{\top}X_\tau\|^2 \\ &\quad + \frac{\eta_\tau^4}{\varepsilon_\tau^4}|\inp{\widehat{U}_{\tau-1}\widehat{V}_{\tau-1}^{\top} - M}{X_\tau}-\xi_\tau|^4\max_j\|e_j^{\top}\widehat{V}_{\tau-1}\widehat{U}_{\tau-1}^{\top}X_\tau\|^2\\
    &\lesssim \frac{\eta_\tau^2}{\varepsilon_\tau^2}(\|\widehat{U}_{\tau-1}\widehat{V}_{\tau-1}^{\top} - M\|_{\max}^2+\sigma^2\log d_1)\frac{\mu_0^2r^2\lambda_{\max}^2}{d_2^2} + \frac{\eta_\tau^2}{\varepsilon_\tau^2}(\|\widehat{U}_{\tau-1}\widehat{V}_{\tau-1}^{\top} - M\|_{\max}^2+\sigma^2\log d_1)\frac{\mu_0^2r^2\lambda_{\max}^2}{d_1^2} \\ &\quad  + \frac{\eta_\tau^4}{\varepsilon_\tau^4}(\|\widehat{U}_{\tau-1}\widehat{V}_{\tau-1}^{\top} - M\|_{\max}^4+\sigma^4\log^2d_1)\frac{\mu_0^2r^2\lambda_{\max}^2}{d_1d_2} \\
    &\lesssim \frac{\eta_\tau^2}{\varepsilon_\tau^2}(\frac{\lambda_{\max}^2}{d_1d_2} +\sigma^2\log d_1)\frac{\mu_0^2r^2\lambda_{\max}^2}{d_2^2}.
\end{align*}
The second order term dominates as long as $\eta_\tau\lesssim \frac{\varepsilon_\tau d_1}{\lambda_{\max}}$ and $\frac{\lambda_{\min}^2}{\sigma^2}\gg \frac{\eta_{\tau}^2}{\varepsilon_\tau^2}\frac{\lambda_{\max}^2\log d_1}{\kappa^2}$. Then 
\begin{align*}
    B_{\tau}&\lesssim \max_{\tau} \frac{\eta_{\tau}^2}{\varepsilon_{\tau}^2}(\frac{\lambda_{\max}^2}{d_1d_2} +\sigma^2\log d_1)O(\frac{\mu_0^2r^2\lambda_{\max}^2}{d_2^2}):= R_{B}.
\end{align*}
The variance of $A_{\tau}$ and $B_{\tau}$, denoted as $\sigma_{A\tau}^2$ and $\sigma_{B\tau}^2$, can be calculated as before. 
\begin{align*}
    \sigma_{A\tau}^2&\lesssim \frac{\eta_{\tau}^2}{\varepsilon_{\tau}d_1d_2}\sum_{ij}(\widehat{U}_{\tau-1}\widehat{V}_{\tau-1}^{\top} - M)_{ij}^2\delta_{il}(\widehat{V}_{\tau-1}\widehat{V}_{\tau-1}^{\top}\widehat{V}_{\tau-1}\widehat{U}_{\tau-1}^{\top})_{jl}^2 \\&\quad + \frac{\eta_{\tau}^2}{\varepsilon_{\tau} d_1d_2}\sum_{ij}(\widehat{U}_{\tau-1}\widehat{V}_{\tau-1}^{\top} - M)_{ij}^2(\widehat{U}_{\tau-1}\widehat{U}_{\tau-1}^{\top})_{li}^2(\widehat{V}_{\tau-1}\widehat{U}_{\tau-1}^{\top})_{jl}^2 \\ &\quad + \frac{\eta_{\tau}^2\sigma^2}{\varepsilon_{\tau}d_1d_2}\sum_{ij}\delta_{il}(\widehat{V}_{\tau-1}\widehat{V}_{\tau-1}^{\top}\widehat{V}_{\tau-1}\widehat{U}_{\tau-1}^{\top})_{jl}^2 + \frac{\eta_{\tau}^2\sigma^2}{\varepsilon_{\tau}d_1d_2}\sum_{ij}(\widehat{U}_{\tau-1}\widehat{U}_{\tau-1}^{\top})_{li}^2(\widehat{V}_{\tau-1}\widehat{U}_{\tau-1}^{\top})_{jl}^2 \\
    &\lesssim \frac{\eta_{\tau}^2}{\varepsilon_{\tau}d_1d_2}\|\widehat{U}_{\tau-1}\widehat{V}_{\tau-1}^{\top} - M\|_{\max}^2\sum_{ij}\delta_{il}(\widehat{V}_{\tau-1}\widehat{V}_{\tau-1}^{\top}\widehat{V}_{\tau-1}\widehat{U}_{\tau-1}^{\top})_{jl}^2 \\&\quad + \frac{\eta_{\tau}^2}{\varepsilon_{\tau} d_1d_2}\|\widehat{U}_{\tau-1}\widehat{V}_{\tau-1}^{\top} - M\|_{\max}^2\|\widehat{U}_{\tau-1}\widehat{U}_{\tau-1}^{\top}\|_{\max}^2\sum_{ij}(\widehat{V}_{\tau-1}\widehat{U}_{\tau-1}^{\top})_{jl}^2 \\ &\quad + \frac{\eta_{\tau}^2\sigma^2}{\varepsilon_{\tau}d_1d_2}\sum_{ij}\delta_{il}(\widehat{V}_{\tau-1}\widehat{V}_{\tau-1}^{\top}\widehat{V}_{\tau-1}\widehat{U}_{\tau-1}^{\top})_{jl}^2 + \frac{\eta_{\tau}^2\sigma^2}{\varepsilon_{\tau}d_1d_2}\|\widehat{U}_{\tau-1}\widehat{U}_{\tau-1}^{\top}\|_{\max}^2\sum_{ij}(\widehat{V}_{\tau-1}\widehat{U}_{\tau-1}^{\top})_{jl}^2 \\
    &\lesssim \frac{\eta_{\tau}^2}{\varepsilon_{\tau}d_1d_2}\frac{\lambda_{\max}^2}{d_1d_2}\frac{\mu_0r\lambda_{\max}^4}{d_1} + \frac{\eta_{\tau}^2}{\varepsilon_{\tau}d_2}\frac{\lambda_{\max}^2}{d_1d_2}\frac{\mu_0^2r^2\lambda_{\max}^2}{d_1^2}\frac{\mu_0r\lambda_{\max}^2}{d_1} \\ &\quad + \frac{\eta_{\tau}^2\sigma^2}{\varepsilon_{\tau}d_1d_2}\frac{\mu_0r\lambda_{\max}^4}{d_1} + \frac{\eta_{\tau}^2\sigma^2}{\varepsilon_{\tau}d_2}\frac{\mu_0^2r^2\lambda_{\max}^2}{d_1^2}\frac{\mu_0r\lambda_{\max}^2}{d_1},
\end{align*}
where we apply $\sum_{ij}\delta_{il}(\widehat{V}_{\tau-1}\widehat{V}_{\tau-1}^{\top}\widehat{V}_{\tau-1}\widehat{U}_{\tau-1}^{\top})_{jl}^2=\|e_l\widehat{U}_{\tau-1}\widehat{V}_{\tau-1}^{\top}\widehat{V}_{\tau-1}\widehat{V}_{\tau-1}^{\top}\|^2\leq \|e_l\widehat{U}_{\tau-1}\widehat{V}_{\tau-1}^{\top}\|^2\|\widehat{V}_{\tau-1}\widehat{V}_{\tau-1}^{\top}\|^2\\ \lesssim \frac{\mu_0r\lambda_{\max}^2}{d_1}\lambda_{\max}^2$ in the last inequality. \\
Similarly,
\begin{align*}
    \sigma_{B\tau}^2&\lesssim \frac{\eta_{\tau}^4}{\varepsilon_{\tau}^3 d_1d_2}\sum_{ij}(\widehat{U}_{\tau-1}\widehat{V}_{\tau-1}^{\top} - M)_{ij}^4\delta_{il}\|e_j^{\top}\widehat{V}_{\tau-1}\widehat{V}_{\tau-1}^{\top}\|^4 + \frac{\eta_{\tau}^4}{\varepsilon_{\tau}^3 d_1d_2}\sigma^4\sum_{ij} \delta_{il}\|e_j^{\top}\widehat{V}_{\tau-1}\widehat{V}_{\tau-1}^{\top}\|^4 \\ &\quad + \frac{\eta_{\tau}^4}{\varepsilon_{\tau}^3 d_1d_2}\sum_{ij} (\widehat{U}_{\tau-1}\widehat{V}_{\tau-1}^{\top} - M)_{ij}^4(\widehat{U}_{\tau-1}\widehat{U}_{\tau-1}^{\top})_{li}^4 + \frac{\eta_{\tau}^4\sigma^4}{\varepsilon_{\tau}^3 d_1d_2}\sum_{ij} (\widehat{U}_{\tau-1}\widehat{U}_{\tau-1}^{\top})_{li}^4 \\
    &\lesssim \frac{\eta_{\tau}^4}{\varepsilon_{\tau}^3 d_1d_2}\|\widehat{U}_{\tau-1}\widehat{V}_{\tau-1}^{\top} - M\|_{\max}^4\sum_{ij}\delta_{il}\frac{\mu_0^2r^2\lambda_{\max}^4}{d_2^2} + \frac{\eta_{\tau}^4\sigma^4}{\varepsilon_{\tau}^3 d_1d_2}\sum_{ij}\delta_{il}\frac{\mu_0^2r^2\lambda_{\max}^4}{d_2^2} \\ &\quad + \frac{\eta_{\tau}^4}{\varepsilon_{\tau}^3 d_1d_2}\|\widehat{U}_{\tau-1}\widehat{V}_{\tau-1}^{\top} - M\|_{\max}^4\|\widehat{U}_{\tau-1}\widehat{U}_{\tau-1}^{\top}\|_{\max}^2\sum_{ij}(\widehat{U}_{\tau-1}\widehat{U}_{\tau-1}^{\top})_{li}^2 \\ &\quad + \frac{\eta_{\tau}^4\sigma^4}{\varepsilon_{\tau}^3 d_1d_2}\|\widehat{U}_{\tau-1}\widehat{U}_{\tau-1}^{\top}\|_{\max}^2\sum_{ij}(\widehat{U}_{\tau-1}\widehat{U}_{\tau-1}^{\top})_{li}^2 \\
    &\lesssim \frac{\eta_{\tau}^4}{\varepsilon_{\tau}^3 d_1}\frac{\lambda_{\max}^4}{d_1^2d_2^2}\frac{\mu_0^2r^2\lambda_{\max}^4}{d_2^2} + \frac{\eta_{\tau}^4\sigma^4}{\varepsilon_{\tau}^3 d_1}\frac{\mu_0^2r^2\lambda_{\max}^4}{d_2^2} \\ &\quad + \frac{\eta_{\tau}^4}{\varepsilon_{\tau}^3 d_1}\frac{\lambda_{\max}^4}{d_1^2d_2^2}\frac{\mu_0^2r^2\lambda_{\max}^2}{d_1^2}\frac{\mu_0r\lambda_{\max}^2}{d_1} + \frac{\eta_{\tau}^4\sigma^4}{\varepsilon_{\tau}^3 d_1}\frac{\mu_0^2r^2\lambda_{\max}^2}{d_1^2}\frac{\mu_0r\lambda_{\max}^2}{d_1}.
\end{align*}
From the calculation above, $R_A\log d + \sqrt{\sum_{\tau=1}^t \sigma_{A\tau}^2\log d_1}\lesssim \frac{\mu_0r\lambda_{\max}^2}{d_1}$ and $R_B\log d + \sqrt{\sum_{\tau=1}^t \sigma_{B\tau}^2\log d_1}\lesssim \frac{\mu_0r\lambda_{\max}^2}{d_1}$ as long as 1.$\max_{\tau}\frac{\eta_{\tau}}{\varepsilon_{\tau}}\lesssim \frac{d_2^{3/2}}{\sqrt{\mu_0r}\log d_1\lambda_{\max}}$, 2.$\frac{\lambda_{\min}^2}{\sigma^2}\gg\max_{\tau} \frac{\eta_{\tau}^2}{\varepsilon_{\tau}^2}\frac{\mu_0r\lambda_{\max}^2\log^3 d_1}{d_2^3\kappa^2}$, 3.$\sum_{\tau=1}^t \frac{\eta_{\tau}^2}{\varepsilon_{\tau}^1}\lesssim \frac{d_1d_2^2}{\mu_0r\lambda_{\max}^4}$, 4.$\frac{\lambda_{\min}^2}{\sigma^2}\gg \sum_{\tau=1}^{t} \frac{\eta_{\tau}^2}{\varepsilon_{\tau}}\frac{\lambda_{\max}^2}{d_2\kappa^2}$, 5.$\sum_{\tau=1}^t \frac{\eta_{\tau}^4}{\varepsilon_{\tau}^3}\lesssim \frac{d_1d_2^4}{\mu_0^2r^2\lambda_{\max}^4}$, 6.$\frac{\lambda_{\min}^2}{\sigma^2}\gg \sqrt{\sum_{\tau=1}^{t} \frac{\eta_{\tau}^4}{\varepsilon_{\tau}^3}\frac{\lambda_{\max}^4}{d_2\kappa^4}}$. All of them can be satisfied under the stepsize and SNR conditions in Equation \ref{eq:thm-MCB-eq1}.\\
As a result, by Bernstein inequality, with probability at least $1-d_1^{-200}$, 
\begin{align*}
    \sum_{\tau=1}^{t}A_{\tau}\leq \frac{1}{6}\frac{\mu_0r\lambda_{\max}^2}{d_1}.
\end{align*} 
And with  probability at least $1-d_1^{-200}$,
\begin{align*}
    \sum_{\tau=1}^{t}B_{\tau}\leq \frac{1}{6}\frac{\mu_0r\lambda_{\max}^2}{d_1}.
\end{align*} 
Finally, combine all the results above and the union bound, we can obtain that with probability at least $1-d_1^{-200}$, 
\begin{align*}
    \|e_l^{\top}\widehat{U}_{t}V_t^{\top}\|^2\leq \frac{\mu_0r\lambda_{\max}^2}{d_1}.
\end{align*}
By symmetry, following the same arguments, we have for $\forall j \in [d_2]$, with the same probability,
\begin{align*}
    \|e_j^{\top}\widehat{V}_{t}\widehat{U}_t^{\top}\|^2\leq \frac{\mu_0r\lambda_{\max}^2}{d_2}.
\end{align*}
\end{proof}

\subsection{Proof of Lemma \ref{lemma2max}}
\begin{proof}
    For any $l$ in $[d_1]$,
\begin{align*}
    \|e_l^{\top}(\widehat{U}_t\widehat{V}_t^{\top}-M)\|^2 &= e_l^{\top} (\widehat{U}_{t-1}\widehat{V}_{t-1}^{\top}-M-\Delta_{t})(\widehat{U}_{t-1}\widehat{V}_{t-1}^{\top} - M^{\top} - \Delta_{t}^{\top})e_l \\
    &= \|e_l^{\top}(\widehat{U}_{t-1}\widehat{V}_{t-1}^{\top} - M)\|^2 - 2e_l^{\top}\Delta_{t}(\widehat{V}_{t-1}\widehat{U}_{t-1}^{\top}-M^{\top})e_l + \|e_l^{\top}\Delta_{t}\|^2.
\end{align*}
This leads to
\begin{align*}
    \EE(\|e_l^{\top}\widehat{U}_t\widehat{V}_t^{\top}\|^2|\mathcal{F}_{t-1})= \|e_l^{\top}(\widehat{U}_{t-1}\widehat{V}_{t-1}^{\top}-M)\|^2 - 2\EE[e_l^{\top}\Delta_{t}(\widehat{V}_{t-1}\widehat{U}_{t-1}^{\top}-M^{\top})e_l|\mathcal{F}_{t-1}] + \EE[\|e_l^{\top}\Delta_{t}\|^2|\mathcal{F}_{t-1}].
\end{align*}
We first compute $2\EE[e_l^{\top}\Delta_{t}(\widehat{V}_{t-1}\widehat{U}_{t-1}^{\top}-M^{\top})e_l|\mathcal{F}_{t-1}]$.
\begin{align*}
    &\quad \quad 2\EE[e_l^{\top}\Delta_{t}(\widehat{V}_{t-1}\widehat{U}_{t-1}^{\top} - M^{\top})e_l|\mathcal{F}_{t-1}] \\
    &= 2\EE\left[e_l^{\top} \frac{\mathbbm{1}(a_t=1)}{\pi_t}\eta_t(\inp{\widehat{U}_{t-1}\widehat{V}_{t-1}^{\top} - M}{X_t}-\xi_t)X_t\widehat{V}_{t-1}\widehat{V}_{t-1}^{\top} (\widehat{V}_{t-1}\widehat{U}_{t-1}^{\top}-M^{\top})e_l|\mathcal{F}_{t-1}\right] \\ &\quad + 2\EE\left[e_l^{\top} \frac{\mathbbm{1}(a_t=1)}{\pi_t}\eta_t(\inp{\widehat{U}_{t-1}\widehat{V}_{t-1}^{\top}- M}{X_t}-\xi_t)\widehat{U}_{t-1}\widehat{U}_{t-1}^{\top}X_t (\widehat{V}_{t-1}\widehat{U}_{t-1}^{\top}-M^{\top})e_l|\mathcal{F}_{t-1}\right] \\ &\quad - 2\EE\left[e_l^{\top} \frac{\mathbbm{1}(a_t=1)}{\pi_t^2}\eta_t^2(\inp{\widehat{U}_{t-1}\widehat{V}_{t-1}^{\top}}{X_t}-\xi_t)^2X_t\widehat{V}_{t-1}\widehat{U}_{t-1}^{\top}X_t (\widehat{V}_{t-1}\widehat{U}_{t-1}^{\top}-M^{\top})e_l|\mathcal{F}_{t-1}\right] \\
    &\geq \frac{2\eta_t}{d_1d_2}e_l^{\top}(\widehat{U}_{t-1}\widehat{V}_{t-1}^{\top} - M)\widehat{V}_{t-1}\widehat{V}_{t-1}^{\top}(\widehat{V}_{t-1}\widehat{U}_{t-1}^{\top}-M^{\top})e_l \\ &\quad + \frac{2\eta_t}{d_1d_2}e_l^{\top}\widehat{U}_{t-1}\widehat{U}_{t-1}^{\top}(\widehat{U}_{t-1}\widehat{V}_{t-1}^{\top} - M)(\widehat{V}_{t-1}\widehat{U}_{t-1}^{\top}-M^{\top})e_l \\ &\quad - \frac{4\eta_t^2}{\varepsilon_t d_1d_2}\sum_{ij}\delta_{il}(\widehat{U}_{t-1}\widehat{V}_{t-1}^{\top} - M)_{ij}^3(\widehat{U}_{t-1}\widehat{V}_{t-1}^{\top})_{ij} - \frac{4\eta_t^2\sigma^2}{\varepsilon_t d_1d_2}\sum_{ij}\delta_{il}(\widehat{U}_{t-1}\widehat{V}_{t-1}^{\top})_{ij}(\widehat{U}_{t-1}\widehat{V}_{t-1}^{\top} - M)_{ij} \\
    &\geq \frac{\eta_t\lambda_{\min}}{d_1d_2}e_l^{\top}(\widehat{U}_{t-1}\widehat{V}_{t-1}^{\top}-M)(\widehat{V}_{t-1}\widehat{U}_{t-1}^{\top}-M^{\top})e_l + \frac{\eta_t\lambda_{\min}}{d_1d_2}e_l^{\top}(\widehat{U}_{t-1}\widehat{V}_{t-1}^{\top}-M)(\widehat{V}_{t-1}\widehat{U}_{t-1}^{\top}-M^{\top})e_l \\ &\quad - \frac{4\eta_t^2}{\varepsilon_t d_1d_2}\frac{\mu_0^2r^2\lambda_{\max}^2}{d_1d_2}\|e_l^{\top}(\widehat{U}_{t-1}\widehat{V}_{t-1}^{\top}-M)\|^2 - \frac{4\eta_t^2\sigma^2}{\varepsilon_t d_1}\frac{\mu_0^2r^2\lambda_{\max}^2}{d_1d_2}\\
    &\geq \frac{2\eta_t\lambda_{\min}}{d_1d_2}\|e_l^{\top}(\widehat{U}_{t-1}\widehat{V}_{t-1}^{\top}-M)\|^2 - \frac{4\eta_t^2}{\varepsilon_t d_1d_2}O(\frac{\mu_0^2r^2\lambda_{\max}^2}{d_1d_2})\|e_l^{\top}(\widehat{U}_{t-1}\widehat{V}_{t-1}^{\top}-M)\|^2 - \frac{4\eta_t^2\sigma^2}{\varepsilon_t d_1}\frac{\mu_0^2r^2\lambda_{\max}^2}{d_1d_2}.  
\end{align*}
The first second order term in the second last inequality comes from $\sum_{ij}\delta_{il}(\widehat{U}_{t-1}\widehat{V}_{t-1}^{\top} - M)_{ij}^3(\widehat{U}_{t-1}\widehat{V}_{t-1}^{\top})_{ij}\leq \|\widehat{U}_{t-1}\widehat{V}_{t-1}^{\top} - M\|_{\max}\|\widehat{U}_{t-1}\widehat{V}_{t-1}^{\top}\|_{\max}\sum_{ij}\delta_{il}(\widehat{U}_{t-1}\widehat{V}_{t-1}^{\top} - M)_{ij}^2\leq O(\frac{\mu_0^2r^2\lambda_{\max}^2}{d_1d_2})\|e_l^{\top}(\widehat{U}_{t-1}\widehat{V}_{t-1}^{\top}-M)\|^2$. And the other term is similar. \\
For $\EE(\|e_l^{\top}\Delta_{t}\|^2|\mathcal{F}_{t-1})$, analogous to the calculation before, we also have
\begin{align*}
    \EE(\|e_l^{\top}\Delta_{t}\|^2|\mathcal{F}_{t-1})\lesssim \frac{\eta_t^2}{\varepsilon_t d_1d_2}\frac{\mu_0r\lambda_{\max}^2}{d_1}\|e_l^{\top}(\widehat{U}_{t-1}\widehat{V}_{t-1}^{\top}-M)\|^2 + \frac{\eta_t^2\sigma^2\mu_0r\lambda_{\max}^2}{\varepsilon_t d_1d_2}.
\end{align*}
Recall $f(\eta_t)=\frac{\eta_t\lambda_{\min}}{2d_1d_2}$. Then for some constant $c_2$,
\begin{align*}
    -2\EE(e_l^{\top}\Delta_{t}\widehat{V}_{t-1}\widehat{U}_{t-1}^{\top}e_l|\mathcal{F}_{t-1}) + \EE(\|e_l^{\top}\Delta_{t}\|^2|\mathcal{F}_{t-1})&\leq -2f(\eta_t)\|e_l^{\top}(\widehat{U}_{t-1}\widehat{V}_{t-1}^{\top}-M)\|^2 + \frac{c_2\eta_{t}^2\sigma^2\mu_0r\lambda_{\max}^2}{\varepsilon_{t}d_1d_2} \\
    &\leq -f(\eta_t)\|e_l^{\top}(\widehat{U}_{t-1}\widehat{V}_{t-1}^{\top}-M)\|^2 + \frac{c_2\eta_{t}^2\sigma^2\mu_0r\lambda_{\max}^2}{\varepsilon_{t}d_1d_2}.
\end{align*}
We again use telescoping to derive
\begin{align*}
    &\|e_l^{\top}(U_tV_t^{\top}-M)\|^2 = \|e_l^{\top}(\widehat{U}_{t-1}\widehat{V}_{t-1}^{\top}-M)\|^2 - 2e_l^{\top}\Delta_{t}(\widehat{V}_{t-1}\widehat{U}_{t-1}^{\top}-M^{\top})e_l +\|e_l^{\top}\Delta_{t}\|^2\\
	&\quad = \|e_l^{\top}(\widehat{U}_{t-1}\widehat{V}_{t-1}^{\top}-M)\|^2 - 2(e_l^{\top}\Delta_{t}(\widehat{V}_{t-1}\widehat{U}_{t-1}^{\top}-M^{\top})e_l -\EE(e_l^{\top}\Delta_{t}(\widehat{V}_{t-1}\widehat{U}_{t-1}^{\top}-M^{\top})e_l|\mathcal{F}_{t-1})) \\ &\quad\quad +\left(\|e_l^{\top}\Delta_{t}\|^2 - \EE[\|e_l^{\top}\Delta_{t}\|^2|\mathcal{F}_{t-1}]\right)
	 - 2\EE[e_l^{\top}\Delta_{t}\widehat{V}_{t-1}\widehat{U}_{t-1}^{\top}e_l|\mathcal{F}_{t-1}] + \EE[\|e_l^{\top}\Delta_{t}\|^2|\mathcal{F}_{t-1}]\\
	&\quad \leq (1-f(\eta_t))\|e_l^{\top}(\widehat{U}_{t-1}\widehat{V}_{t-1}^{\top}-M)\|^2 +  \frac{c_2\eta_t^2\sigma^2\mu_0r\lambda_{\max}^2}{\varepsilon_t d_1d_2} \\
	&\quad\quad + \underbrace{2(\EE(e_l^{\top}\Delta_{t}(\widehat{V}_{t-1}\widehat{U}_{t-1}^{\top}-M^{\top})e_l|\mathcal{F}_{t-1}) -e_l^{\top}\Delta_{t}(\widehat{V}_{t-1}\widehat{U}_{t-1}^{\top}-M^{\top})e_l)}_{A_{t}} +\underbrace{(\|e_l^{\top}\Delta_{t}\|^2 - \EE[\|e_l^{\top}\Delta_{t}\|^2|\mathcal{F}_{t-1}])}_{B_{t}} \\
    &\quad \leq (1-f(\eta_t))(1-f(\eta_{t-1}))\|e_l^{\top}(\widehat{U}_{t-2}\widehat{V}_{t-2}^{\top}-M)\|^2 + (1-f(\eta_t))\frac{c_2\eta_{t-1}^2\sigma^2\mu_0r\lambda_{\max}^2}{\varepsilon_{t-1}d_1d_2}  \\
    &\quad\quad + \frac{c_2\eta_{t}^2\sigma^2\mu_0r\lambda_{\max}^2}{\varepsilon_{t}d_1d_2}  + (1-f(\eta_t))(A_{t-1}+B_{t-1})+(A_{t}+B_{t}) \\
    &\quad \leq \prod_{\tau=1}^{t}(1-f(\eta_{\tau})) \|e_l^{\top}(\widehat{U}_{0}\widehat{V}_0^{\top}-M)\|^2 + \sum_{\tau=1}^{t-1}\prod_{k=\tau+1}^{t}(1-f(\eta_{k}))\frac{c_2\eta_{\tau}^2\sigma^2\mu_0r\lambda_{\max}^2}{\varepsilon_{\tau} d_1d_2}  \\
    &\quad\quad + \frac{c_2\eta_{t}^2\sigma^2\mu_0r\lambda_{\max}^2}{\varepsilon_t d_1d_2} + \sum_{\tau=1}^{t-1}\prod_{k=\tau+1}^{t} (1-f(\eta_k))A_{\tau} + A_t + \sum_{\tau=1}^{t-1}\prod_{k=\tau+1}^{t} (1-f(\eta_{k}))B_{\tau} + B_t  \\
    &\quad \leq \prod_{\tau=1}^{t}(1-f(\eta_{\tau})) \frac{\mu_0r\lambda_{\max}^2}{d_1}  + \sum_{\tau=1}^{t}\frac{c_2\eta_{\tau}^2\sigma^2\mu_0r\lambda_{\max}^2 \log^2d_1}{\varepsilon_{\tau} d_1d_2}\\
    &\quad \quad + \sum_{\tau=1}^{t-1}\prod_{k=\tau+1}^{t} (1-f(\eta_k))A_{\tau} + A_t + \sum_{\tau=1}^{t-1}\prod_{k=\tau+1}^{t} (1-f(\eta_k))B_{\tau} + B_t. 
\end{align*}
The first term $\|e_l^{\top}(\widehat{U}_{0}\widehat{V}_0^{\top}-M)\|^2\lesssim \frac{\mu_0r\lambda_{\max}^2}{d_1}$ under the incoherence condition of $U_0$ and $V_0$. \\
Next, we derive the uniform bound for $\prod_{k=\tau+1}^{t} (1-f(\eta_k))A_{\tau} $ and $\prod_{k=\tau+1}^{t} (1-f(\eta_k))B_{\tau}$. 
\begin{align*}
    &|e_l^{\top}\Delta_{\tau}(\widehat{V}_{\tau-1}\widehat{U}_{\tau-1}^{\top}-M^{\top})e_l|\leq |e_l^{\top} \frac{\mathbbm{1}(a_{\tau}=1)}{\pi_{\tau}}\eta_{\tau}(\inp{\widehat{U}_{\tau-1}\widehat{V}_{\tau-1}^{\top} - M}{X_{\tau}}-\xi_{\tau})X_{\tau}\widehat{V}_{\tau-1}\widehat{V}_{\tau-1}^{\top} (\widehat{V}_{\tau-1}\widehat{U}_{\tau-1}^{\top}-M^{\top})e_l| \\ &\quad\quad + |e_l^{\top} \frac{\mathbbm{1}(a_{\tau}=1)}{\pi_{\tau}}\eta_{\tau}(\inp{\widehat{U}_{\tau-1}\widehat{V}_{\tau-1}^{\top}- M}{X_{\tau}}-\xi_{\tau})\widehat{U}_{\tau-1}\widehat{U}_{\tau-1}^{\top}X_{\tau} (\widehat{V}_{\tau-1}\widehat{U}_{\tau-1}^{\top}-M^{\top})e_l| \\ &\quad\quad + |e_l^{\top} \frac{\mathbbm{1}(a_{\tau}=1)}{\pi_{\tau}^2}\eta_{\tau}^2(\inp{\widehat{U}_{\tau-1}\widehat{V}_{\tau-1}^{\top}- M}{X_t}-\xi_{\tau})^2X_{\tau}\widehat{V}_{\tau-1}\widehat{U}_{\tau-1}^{\top}X_{\tau} (\widehat{V}_{\tau-1}\widehat{U}_{\tau-1}^{\top}-M^{\top})e_l| \\
    &\quad\leq \frac{2\eta_{\tau}}{\varepsilon_{\tau}}(\|\widehat{U}_{\tau-1}\widehat{V}_{\tau-1}^{\top} - M\|_{\max}+\sigma\log^{1/2} d_1)\|\widehat{V}_{\tau-1}\widehat{V}_{\tau-1}^{\top}\|_{\max} \|\widehat{V}_{\tau-1}\widehat{U}_{\tau-1}^{\top}-M^{\top}\|_{\max} \\ &\quad\quad + \frac{2\eta_{\tau}}{\varepsilon_{\tau}}(\|\widehat{U}_{\tau-1}\widehat{V}_{\tau-1}^{\top} - M\|_{\max}+\sigma\log^{1/2} d_1)\|\widehat{U}_{\tau-1}\widehat{U}_{\tau-1}^{\top}\|_{\max} \|\widehat{V}_{\tau-1}\widehat{U}_{\tau-1}^{\top}-M\|_{\max} \\ &\quad\quad + \frac{4\eta_{\tau}^2}{\varepsilon_{\tau}^2}|(\|\widehat{U}_{\tau-1}\widehat{V}_{\tau-1}^{\top} - M\|_{\max}^2+\sigma^2\log d_1)\|\widehat{V}_{\tau-1}\widehat{U}_{\tau-1}^{\top}\|_{\max} \|\widehat{V}_{\tau-1}\widehat{U}_{\tau-1}^{\top}-M\|_{\max} \\
    &\quad\lesssim \frac{\eta_{\tau}}{\varepsilon_{\tau}}\|\widehat{U}_{\tau-1}\widehat{V}_{\tau-1}^{\top} - M\|_{\max}^2\frac{\mu_0r\lambda_{\max}}{d_2} +\frac{\eta_{\tau}}{\varepsilon_{\tau}}\sigma\log^{1/2} d_1\|\widehat{U}_{\tau-1}\widehat{V}_{\tau-1}^{\top} - M\|_{\max}\frac{\mu_0r\lambda_{\max}}{d_2} \\ &\quad\quad + \frac{\eta_{\tau}^2}{\varepsilon_{\tau}^2}\|\widehat{V}_{\tau-1}\widehat{U}_{\tau-1}^{\top}-M\|_{\max}^2\frac{\mu_0^2r^2\lambda_{\max}^2}{d_1d_2}  + \frac{\eta_{\tau}^2}{\varepsilon_{\tau}^2}\sigma^2\log d_1\|\widehat{V}_{\tau-1}\widehat{U}_{\tau-1}^{\top}-M\|_{\max}\frac{\mu_0r\lambda_{\max}}{\sqrt{d_1d_2}} \\
    &\quad\lesssim \frac{\eta_{\tau}}{\varepsilon_{\tau}}\|\widehat{U}_{\tau-1}\widehat{V}_{\tau-1}^{\top} - M\|_{\max}^2\frac{\mu_0r\lambda_{\max}}{d_2}  +\frac{\eta_{\tau}}{\varepsilon_{\tau}}\sigma\log^{1/2} d_1\|\widehat{U}_{\tau-1}\widehat{V}_{\tau-1}^{\top} - M\|_{\max}\frac{\mu_0r\lambda_{\max}}{d_2}.
\end{align*}
The first order terms dominate as long as $\eta_{\tau}\lesssim \frac{\varepsilon_{\tau} d_1}{\mu_0r\lambda_{\max}}$ and $\frac{\lambda_{\min}^2}{\sigma^2}\gg \frac{\eta_{\tau}^2}{\varepsilon_{\tau}^2}\frac{\lambda_{\max}^2\log d_1}{\kappa^2}$. \\
Since $\|\widehat{U}_{\tau-1}\widehat{V}_{\tau-1}^{\top} - M\|_{\max}^2\lesssim \prod_{k=1}^{\tau-1} (1-\frac{f(\eta_k)}{2}) \frac{\text{poly}(\kappa,\mu,r)\lambda_{\min}^2}{d_1d_2} + \sum_{k=1}^{\tau-1} \frac{\eta_k^2\text{poly}(\kappa,\mu,r)\lambda_{\max}^2\sigma^2\log^2d_1}{\varepsilon_k d_1d_2^2}$ under event $\mathcal{E}_{\tau-1}$, and applying $2ab\leq a^2+b^2$ to term $\frac{\eta_{\tau}}{\varepsilon_{\tau}}\sigma\log^{1/2} d_1\|\widehat{U}_{\tau-1}\widehat{V}_{\tau-1}^{\top}  - M\|_{\max}\frac{\mu_0r\lambda_{\max}}{d_2}$, we have
\begin{align*}
    &\prod_{k=\tau+1}^{t} (1-\frac{f(\eta_k)}{2})|e_l^{\top}\Delta_{\tau}(\widehat{V}_{\tau-1}\widehat{U}_{\tau-1}^{\top}-M^{\top})e_l| \lesssim  \max_{\tau}\frac{\eta_{\tau}}{\varepsilon_{\tau}}\left(\prod_{k=1}^{t} (1-\frac{f(\eta_k)}{2}) \frac{\text{poly}(\kappa,\mu,r)\lambda_{\min}^2}{d_1d_2} \right.\\ &\left.\quad + \sum_{k=1}^{t} \frac{\eta_{k}^2\lambda_{\max}^2\sigma^2\log^2d_1\text{poly}(\kappa,\mu,r)}{\varepsilon_{k} d_1d_2^2}\right)\frac{\mu_0r\lambda_{\max}}{d_2} + \prod_{k=1}^{t} (1-\frac{f(\eta_k)}{2}) \frac{\text{poly}(\kappa,\mu,r)\lambda_{\min}^2}{d_1d_2} \\ &\quad  + \sum_{k=1}^{t} \frac{\eta_{k}^2\lambda_{\max}^2\sigma^2\log^2d_1\text{poly}(\kappa,\mu,r)}{\varepsilon_{k} d_1d_2^2} +  \max_{\tau}\frac{\eta_{\tau}^2}{\varepsilon_{\tau}^2}\sigma^2\log d_1 \frac{\mu_0^2r^2\lambda_{\max}^2}{d_2^2}:= R_A.
\end{align*}
Similarly, we can obtain 
\begin{align*}
    &\prod_{k=\tau+1}^{t} (1-\frac{f(\eta_k)}{2})\|e_l^{\top}\Delta_{\tau}\|^2\lesssim \frac{\eta_{\tau}^2}{\varepsilon_{\tau}^2}(\|\widehat{U}_{\tau-1}\widehat{V}_{\tau-1}^{\top} - M\|_{\max}^2+\sigma^2\log d_1)(\|\widehat{V}_{\tau-1}\widehat{V}_{\tau-1}^{\top}\|_{\max}^2 + \|\widehat{U}_{\tau-1}\widehat{U}_{\tau-1}^{\top}\|_{\max}^2) \\
    &\quad \lesssim \max_{\tau}\frac{\eta_{\tau}^2}{\varepsilon_{\tau}^2} \left(\prod_{k=1}^{t}(1-\frac{f(\eta_k)}{2})\frac{\text{poly}(\kappa,\mu,r)\lambda_{\min}^2}{d_1d_2} O(\frac{\mu_0^2r^2\lambda_{\max}^2}{d_2^2}) + \sum_{k=1}^{t} \frac{\eta_{k}^2\text{poly}(\kappa,\mu,r)\lambda_{\max}^2\sigma^2\log^2d_1}{\varepsilon_{k} d_1d_2^2}\right) \frac{\mu_0^2r^2\lambda_{\max}^2}{d_2^2} \\ &\quad\quad + \max_{\tau}\frac{\eta_{\tau}^2}{\varepsilon_{\tau}^2}\sigma^2\log d_1O(\frac{\mu_0^2r^2\lambda_{\max}^2}{d_2^2}) := R_B.
\end{align*}
Then we calculate their conditional variance, respectively. Denote the variance of $\prod_{k=\tau+1}^{t}(1-f(\eta_{\tau}))A_{\tau}$ and $\prod_{k=\tau+1}^{t}(1-f(\eta_{\tau}))B_{\tau}$ as $\sigma_{A\tau}^2$ and $\sigma_{B\tau}^2$, we have  
\begin{align*}
    \sigma_{A\tau}^2&\lesssim \prod_{k=\tau+1}^{t}(1-f(\eta_{\tau}))^2\left(\frac{\eta_{\tau}^2}{\varepsilon_{\tau}d_1d_2}\sum_{ij}(\widehat{U}_{\tau-1}\widehat{V}_{\tau-1}^{\top} - M)_{ij}^2\delta_{il}(\widehat{V}_{\tau-1}\widehat{V}_{\tau-1}^{\top}(\widehat{V}_{\tau-1}\widehat{U}_{\tau-1}^{\top}-M^{\top}))_{jl}^2 \right.\\ &\left.\quad + \frac{\eta_{\tau}^2}{\varepsilon_{\tau} d_1d_2}\sum_{ij}(\widehat{U}_{\tau-1}\widehat{V}_{\tau-1}^{\top} - M)_{ij}^2(\widehat{U}_{\tau-1}\widehat{U}_{\tau-1}^{\top})_{li}^2(\widehat{V}_{\tau-1}\widehat{U}_{\tau-1}^{\top}-M^{\top})_{jl}^2 \right.\\ &\left.\quad + \frac{\eta_{\tau}^2\sigma^2}{\varepsilon_{\tau}d_1d_2}\sum_{ij}\delta_{il}(\widehat{V}_{\tau-1}\widehat{V}_{\tau-1}^{\top}(\widehat{V}_{\tau-1}\widehat{U}_{\tau-1}^{\top}-M^{\top}))_{jl}^2 + \frac{\eta_{\tau}^2\sigma^2}{\varepsilon_{\tau}d_1d_2}\sum_{ij}(\widehat{U}_{\tau-1}\widehat{U}_{\tau-1}^{\top})_{li}^2(\widehat{V}_{\tau-1}\widehat{U}_{\tau-1}^{\top}-M^{\top})_{jl}^2\right)\\
    &\lesssim \prod_{k=\tau+1}^{t}(1-f(\eta_{\tau}))^2\left(\frac{\eta_{\tau}^2}{\varepsilon_{\tau}d_1d_2}\|\widehat{U}_{\tau-1}\widehat{V}_{\tau-1}^{\top} - M\|_{\max}^2\|\widehat{V}_{\tau-1}\widehat{V}_{\tau-1}^{\top}\|^2\|e_l(\widehat{U}_{\tau-1}\widehat{V}_{\tau-1}^{\top}-M)\|^2 \right.\\ &\left.\quad + \frac{\eta_{\tau}^2}{\varepsilon_{\tau} d_1d_2}\|\widehat{U}_{\tau-1}\widehat{V}_{\tau-1}^{\top} - M\|_{\max}^4\sum_{ij}(\widehat{U}_{\tau-1}\widehat{U}_{\tau-1}^{\top})_{li}^2 \right.\\ &\left.\quad + \frac{\eta_{\tau}^2\sigma^2}{\varepsilon_{\tau}d_1d_2}\|\widehat{V}_{\tau-1}\widehat{V}_{\tau-1}^{\top}\|^2\|e_l(\widehat{U}_{\tau-1}\widehat{V}_{\tau-1}^{\top}-M)\|^2  + \frac{\eta_{\tau}^2\sigma^2}{\varepsilon_{\tau}d_1d_2}\|\widehat{U}_{\tau-1}\widehat{V}_{\tau-1}^{\top} - M\|_{\max}^2\sum_{ij}(\widehat{U}_{\tau-1}\widehat{U}_{\tau-1}^{\top})_{li}^2\right).
\end{align*}
Notice that, $\|\widehat{U}_{\tau-1}\widehat{V}_{\tau-1}^{\top} - M\|_{\max}^2\|e_l(\widehat{U}_{\tau-1}\widehat{V}_{\tau-1}^{\top}-M)\|^2\lesssim \frac{1}{d_2}(\prod_{k=1}^{t}(1-\frac{f(\eta_k)}{2}) \frac{\text{poly}(\mu,r,\kappa)\lambda_{\min}^2}{d_1} + \\ \sum_{k=1}^{t} \frac{\eta_k^2\text{poly}(\kappa,\mu,r)\lambda_{\max}^2\sigma^2\log^2d_1}{\varepsilon_k d_1d_2})^2$ and $\|\widehat{U}_{\tau-1}\widehat{V}_{\tau-1}^{\top} - M\|_{\max}^4\lesssim \frac{1}{d_2^2}(\prod_{k=1}^{\tau-1}(1-\frac{f(\eta_k)}{2}) \frac{\text{poly}(\mu,r,\kappa)\lambda_{\min}^2}{d_1} \\ + \sum_{k=1}^{\tau-1} \frac{\eta_k^2\text{poly}(\kappa,\mu,r)\lambda_{\max}^2\sigma^2\log^2d_1}{\varepsilon_k d_1d_2})^2$ under event $\mathcal{E}_{\tau-1}$. As a result,
\begin{align*}
    \sigma_{A\tau}^2&\lesssim \frac{\eta_{\tau}^2\lambda_{\max}^2}{\varepsilon_{\tau}d_1d_2^2}\left(\prod_{k=1}^{t}(1-\frac{f(\eta_k)}{2}) \frac{\text{poly}(\mu,r,\kappa)\lambda_{\min}^2}{d_1} +  \sum_{k=1}^{t} \frac{\eta_k^2\text{poly}(\kappa,\mu,r)\lambda_{\max}^2\sigma^2\log^2d_1}{\varepsilon_k d_1d_2}\right)^2 \\ &\quad + \frac{\eta_{\tau}^2\sigma^2\lambda_{\max}^2}{\varepsilon_{\tau}d_1d_2}\left(\prod_{k=1}^{t}(1-\frac{f(\eta_k)}{2}) \frac{\text{poly}(\mu,r,\kappa)\lambda_{\min}^2}{d_1} +  \sum_{k=1}^{t} \frac{\eta_k^2\text{poly}(\kappa,\mu,r)\lambda_{\max}^2\sigma^2\log^2d_1}{\varepsilon_k d_1d_2}\right).
\end{align*}
Similarly,
\begin{align*}
    \sigma_{B\tau}^2&\lesssim \prod_{k=\tau+1}^{t}(1-f(\eta_{\tau}))^2 \left(\frac{\eta_{\tau}^4}{\varepsilon_{\tau}^3 d_1d_2}\sum_{ij}(\widehat{U}_{\tau-1}\widehat{V}_{\tau-1}^{\top} - M)_{ij}^4\delta_{il}\|e_j^{\top}\widehat{V}_{\tau-1}\widehat{V}_{\tau-1}^{\top}\|^4 + \frac{\eta_{\tau}^4}{\varepsilon_{\tau}^3 d_1d_2}\sigma^4\sum_{ij} \delta_{il}\|e_j^{\top}\widehat{V}_{\tau-1}\widehat{V}_{\tau-1}^{\top}\|^4 \right.\\ &\left.\quad + \frac{\eta_{\tau}^4}{\varepsilon_{\tau}^3 d_1d_2}\sum_{ij} (\widehat{U}_{\tau-1}\widehat{V}_{\tau-1}^{\top} - M)_{ij}^4(\widehat{U}_{\tau-1}\widehat{U}_{\tau-1}^{\top})_{li}^4 + \frac{\eta_{\tau}^4\sigma^4}{\varepsilon_{\tau}^3 d_1d_2}\sum_{ij} (\widehat{U}_{\tau-1}\widehat{U}_{\tau-1}^{\top})_{li}^4\right) \\
    &\lesssim \prod_{k=\tau+1}^{t}(1-f(\eta_{\tau}))^2 \left(\frac{\eta_{\tau}^4}{\varepsilon_{\tau}^3 d_1d_2}\|\widehat{U}_{\tau-1}\widehat{V}_{\tau-1}^{\top} - M\|_{\max}^4\sum_{ij}\delta_{il}\|e_j^{\top}\widehat{V}_{\tau-1}\widehat{V}_{\tau-1}^{\top}\|^4 + \frac{\eta_{\tau}^4\sigma^4}{\varepsilon_{\tau}^3 d_1d_2}\sum_{ij} \delta_{il}\|e_j^{\top}\widehat{V}_{\tau-1}\widehat{V}_{\tau-1}^{\top}\|^4 \right.\\ &\left.\quad + \frac{\eta_{\tau}^4}{\varepsilon_{\tau}^3 d_1d_2}\|\widehat{U}_{\tau-1}\widehat{V}_{\tau-1}^{\top} - M\|_{\max}^4\sum_{ij} (\widehat{U}_{\tau-1}\widehat{U}_{\tau-1}^{\top})_{li}^4 + \frac{\eta_{\tau}^4\sigma^4}{\varepsilon_{\tau}^3 d_1d_2}\sum_{ij} (\widehat{U}_{\tau-1}\widehat{U}_{\tau-1}^{\top})_{li}^4\right) \\
    &\lesssim \frac{\eta_{\tau}^4}{\varepsilon_{\tau}^3}\frac{\mu_0^2r^2\lambda_{\max}^4}{d_1 d_2^4}\left(\prod_{k=1}^{t}(1-\frac{f(\eta_k)}{2}) \frac{\text{poly}(\mu,r,\kappa)\lambda_{\min}^2}{d_1}  +  \sum_{k=1}^{t} \frac{\eta_k^2\text{poly}(\kappa,\mu,r)\lambda_{\max}^2\sigma^2\log^2d_1}{\varepsilon_k d_1d_2}\right)^2 \\ &\quad + \frac{\eta_{\tau}^4}{\varepsilon_{\tau}^3}\frac{\sigma^4\mu_0^2r^2\lambda_{\max}^4}{d_1d_2^3}.
\end{align*}
From the calculation above, $R_A\log d_1+\sqrt{\sum_{\tau=1}^{t}\sigma_{A\tau}^2\log d_1}\lesssim \prod_{k=1}^{t}(1-\frac{f(\eta_k)}{2}) \frac{\text{poly}(\mu,r,\kappa)\fro{\widehat{U}_0\widehat{V}_0^{\top}-M}^2}{d_1}  +  \sum_{k=1}^{t} \frac{\eta_k^2\lambda_{\max}^2\sigma^2\log^2d_1\text{poly}(\kappa,\mu,r)}{\varepsilon_k d_1d_2}$ and  $R_B\log d_1+\sqrt{\sum_{\tau=1}^{t}\sigma_{B\tau}^2\log d_1}\lesssim \prod_{k=1}^{t}(1-\frac{f(\eta_k)}{2}) \frac{\text{poly}(\mu,r,\kappa)\fro{\widehat{U}_0\widehat{V}_0^{\top}-M}^2}{d_1}  +  \sum_{k=1}^{t} \frac{\eta_k^2\lambda_{\max}^2\sigma^2\log^2d_1\text{poly}(\kappa,\mu,r)}{\varepsilon_k d_1d_2}$ as long as 1. $\eta_{\tau}\lesssim \frac{\varepsilon_{\tau}d_2^{3/2}}{\mu_0r\lambda_{\max}}$, 2. $\max \frac{\eta_{\tau}^2}{\varepsilon_{\tau}^2}\leq \sum_{\tau=1}^{t}\frac{\eta_{\tau}^2}{\varepsilon_{\tau}}$, 3. $\sum_{\tau=1}^{t} \frac{\eta_{\tau}^2}{\varepsilon_{\tau}}\lesssim \frac{d_1d_2^2}{\lambda_{\max}^2\log d_1}$, 4. $\sum_{\tau=1}^{t} \frac{\eta_{\tau}^4}{\varepsilon_{\tau}^3}\lesssim \frac{d_1d_2^4}{\mu_0^2r^2\lambda_{\max}^4\log d_1}$, 5. $\sum_{\tau=1}^{t} \frac{\eta_{\tau}^4}{\varepsilon_{\tau}^3}\leq (\sum_{\tau=1}^{t} \frac{\eta_{\tau}^2}{\varepsilon_{\tau}})^2$. All of them can be satisfied under the stepsize and SNR conditions in Equation \ref{eq:thm-MCB-eq1}. \\
By Bernstein inequality, with probability at least $1-d_1^{-2}$, $\sum_{\tau=1}^{t-1}\prod_{k=\tau+1}^{t}(1-f(\eta_{\tau}))A_{\tau}+A_t\lesssim \prod_{k=1}^{t}(1-\frac{f(\eta_k)}{2}) \frac{\text{poly}(\mu,r,\kappa)\lambda_{\min}^2}{d_1}  +  \sum_{k=1}^{t} \frac{\eta_k^2\text{poly}(\kappa,\mu,r)\lambda_{\max}^2\sigma^2\log^2d_1}{\varepsilon_k d_1d_2}$ and $\sum_{\tau=1}^{t-1}\prod_{k=\tau+1}^{t}(1-f(\eta_{\tau}))B_{\tau}+B_t\lesssim \prod_{k=1}^{t}(1-\frac{f(\eta_k)}{2}) \frac{\text{poly}(\mu,r,\kappa)\lambda_{\min}^2}{d_1}  +  \sum_{k=1}^{t} \frac{\eta_k^2\text{poly}(\kappa,\mu,r)\lambda_{\max}^2\sigma^2\log^2d_1}{\varepsilon_k d_1d_2}$. \\
Finally, we conclude that, with probability at least $1-td_1^{-200}$,
\begin{align*}
    \|e_l^{\top}(\widehat{U}_t\widehat{V}_t^{\top}-M)\|^2\lesssim \prod_{k=1}^{t}(1-\frac{f(\eta_k)}{2}) \frac{\text{poly}(\mu,r,\kappa)\lambda_{\min}^2}{d_1} +   \sum_{k=1}^{t} \frac{\eta_k^2\text{poly}(\kappa,\mu,r)\lambda_{\max}^2\sigma^2\log^2d_1}{\varepsilon_k d_1d_2}.
\end{align*}
By symmetry, following the same arguments, we can prove with same probability, for any $j\in [d_2]$,
\begin{align*}
    \|e_j^{\top}(\widehat{V}_t\widehat{U}_t^{\top}-M^{\top})\|^2\lesssim \prod_{k=1}^{t}(1-\frac{f(\eta_k)}{2}) \frac{\text{poly}(\mu,r,\kappa)\lambda_{\min}^2}{d_2} +  \sum_{k=1}^{t} \frac{\eta_k^2\text{poly}(\kappa,\mu,r)\lambda_{\max}^2\sigma^2\log^2d_1}{\varepsilon_k d_2^2}.
\end{align*}
\end{proof}

\subsection{Proof of Lemma \ref{main term}}
\begin{proof}
    Denote $S^2=1/T^{\gamma}\fro{P_{\Omega_1}(P_M(Q))}^2+C_{\gamma}\fro{P_{\Omega_0}(P_M(Q))}^2$. By definition, 
    \begin{align*}
        &\frac{\sqrt{\frac{T^{1-\gamma}}{d_1d_2}}\left(\inp{LL^{\top}\widehat{Z}_1R_{\perp}R_{\perp}^{\top}}{Q} + \inp{L_{\perp}L_{\perp}^{\top}\widehat{Z}_1RR^{\top}}{Q} + \inp{LL^{\top}\widehat{Z}_1RR^{\top}}{Q} \right)}{\sigma S} \\ &\quad = \frac{\sqrt{\frac{1}{T^{1+\gamma}}} \sum_{t=T_0+1}^T \frac{b\sqrt{d_1d_2}\mathbbm{1}(a_t=1)}{\pi_t}\xi_t\inp{X_t}{\mathcal{P}_M(Q)}}{\sigma S}.
    \end{align*}
    Next, we apply Theorem 3.2 and Corollary 3.1 in \cite{hall2014martingale}, the Martingale Central Limit Theorem to show the asymptotic normality. \\
    \emph{Step 1: checking Lindeberg condition.} \\
    For any $\delta>0$,
    \begin{align*}
        &\quad \sum_{t=T_0+1}^T \EE\left[\frac{b^2d_1d_2\mathbbm{1}(a_t=1)}{\sigma^2S^2T^{1+\gamma}\pi_t^2}\xi_t^2\inp{X_t}{\mathcal{P}_M(Q)}^2 \times  \mathbbm{1}\left(\left|\xi_t\frac{\sqrt{d_1d_2}\mathbbm{1}(a_t=1)\inp{X_t}{\mathcal{P}_M(Q)}}{\sigma S\sqrt{T^{1+\gamma}}\pi_t}\right|>\delta \right)\bigg|\mathcal{F}_{t-1}\right] \\
        &\leq \frac{b^2d_1d_2}{\sigma^2S^2T^{1+\gamma}}\sum_{t=T_0+1}^T \EE\left[\frac{\mathbbm{1}(a_t=1)}{\pi_t^2}\xi_t^2\inp{X_t}{\mathcal{P}_M(Q)}^2 \times  \mathbbm{1}\left(\left|\xi_t\inp{X_t}{\mathcal{P}_M(Q)}\right|>\frac{\sigma S\delta\sqrt{T^{1+\gamma}}\varepsilon_t}{2\sqrt{d_1d_2}} \right)\bigg|\mathcal{F}_{t-1}\right] \\
        &\lesssim \frac{b^2d_1d_2}{\sigma^2S^2T^{1+\gamma}}  \sum_{t=T_0+1}^T \frac{2}{\varepsilon_t}\max_{X\in \mathcal{X}} \inp{X}{\mathcal{P}_M(Q)}^2 \times \sqrt{\EE\left[\mathbbm{1}\left(\left|\xi_t\inp{X_t}{\mathcal{P}_M(Q)}\right|>\frac{\sigma S\delta\sqrt{T^{1+\gamma}}\varepsilon_t}{2\sqrt{d_1d_2}} \right)\right]},
    \end{align*}
    where in the last inequality we use $\EE[XY]\leq \sqrt{\EE[X^2]\EE[Y^2]}$.
    By incoherence condition, 
    \begin{align*}       
    \max_{X\in \mathcal{X}} \big|\inp{X}{\mathcal{P}_M(Q)}\big|\lesssim \sqrt{\frac{\mu r}{d_2}}\fro{\mathcal{P}_M(Q)}.
    \end{align*}
    Moreover, since $\xi_t$ is a subGaussian random variable, so the product $\xi_t\inp{X_t}{\mathcal{P}_M(Q)}$ has subGaussian tail probability
    \begin{align*}
        \PP\left(\left|\xi_t\inp{X_t}{\mathcal{P}_M(Q)}\right|> \frac{\sigma S\delta\sqrt{T^{1+\gamma}}\varepsilon_t}{2\sqrt{d_1d_2}}\right)< 2e^{-\frac{\sigma^2S^2\delta^2T^{1+\gamma}\varepsilon_t^2}{8d_1d_2\nu^2}},    
    \end{align*}
    where $\nu$ is the subGaussian parameter of order $O(\sqrt{\frac{\mu r}{d_2}}\fro{\mathcal{P}_M(Q)}\sigma)$. Notice that $T^{1+\gamma}\varepsilon_t^2\leq T^{1+\gamma}\varepsilon_T^2=O(T^{1-\gamma})$. By property of exponential function, $\frac{d_1d_2\sigma^2}{\sigma^2S^2T^{1+\gamma}} \frac{\mu r}{d_2}\fro{\mathcal{P}_M(Q)}^2\sum_{t=T_0+1}^{T} \frac{2}{\varepsilon_t}2e^{-\frac{\sigma^2S^2\delta^2T^{1+\gamma}\varepsilon_t^2}{16d_1d_2\nu^2}}$ converges to 0 as long as $\frac{d_1}{T^{1-\gamma}} \rightarrow 0$. Then the Lindeberg condition is satisfied. \\
    \emph{Step 2: calculating the variance.} \\
    Next, we show the conditional variance equals to 1. Recall the definition $\Omega_1=\{X\in \mathcal{X}: \inp{M_1-M_0}{X}> \delta_{T,d_1,d_2}\}$, $\Omega_0=\{X\in \mathcal{X}: \inp{M_1-M_0}{X}< \delta_{T,d_1,d_2}\}$ and $\Omega_{\emptyset}=\mathcal{X}\cup (\Omega_1\cup\Omega_0)^c$. Then
    \begin{align*}
        &\frac{b^2d_1d_2}{\sigma^2S^2T^{1+\gamma}}\sum_{t=T_0+1}^T \EE\left[\frac{\mathbbm{1}(a_t=1)}{\pi_t^2}\xi_t^2\inp{X_t}{\mathcal{P}_M(Q)}^2\bigg|\mathcal{F}_{t-1}\right] \\ 
        &\quad = \frac{b^2}{\sigma^2S^2T^{1+\gamma}}\sigma^2 \sum_{t=T_0+1}^{T} \sum_{X\in \Omega_1} \indicator(\inp{\widehat{M}_{1,t-1}-\widehat{M}_{0,t-1}}{X}>0)\frac{1}{1-\frac{\varepsilon_t}{2}}\inp{X}{\mathcal{P}_M(Q)}^2 \\ &\quad\quad\quad\quad +  \indicator(\inp{\widehat{M}_{1,t-1}-\widehat{M}_{0,t-1}}{X}<0)\frac{1}{\frac{\varepsilon_t}{2}}\inp{X}{\mathcal{P}_M(Q)}^2 \\ &\quad + \sum_{X\in \Omega_0} \indicator(\inp{\widehat{M}_{1,t-1}-\widehat{M}_{0,t-1}}{X}>0)\frac{1}{1-\frac{\varepsilon_t}{2}}\inp{X}{\mathcal{P}_M(Q)}^2  +  \indicator(\inp{\widehat{M}_{1,t-1}-\widehat{M}_{0,t-1}}{X}<0)\frac{1}{\frac{\varepsilon_t}{2}}\inp{X}{\mathcal{P}_M(Q)}^2 \\
        &\quad + \sum_{X\in \Omega_\emptyset} \indicator(\inp{\widehat{M}_{1,t-1}-\widehat{M}_{0,t-1}}{X}>0)\frac{1}{1-\frac{\varepsilon_t}{2}}\inp{X}{\mathcal{P}_M(Q)}^2  +  \indicator(\inp{\widehat{M}_{1,t-1}-\widehat{M}_{0,t-1}}{X}<0)\frac{1}{\frac{\varepsilon_t}{2}}\inp{X}{\mathcal{P}_M(Q)}^2.
    \end{align*}
    The next Lemma shows the convergence of the indicator function.
    \begin{Lemma} 
        \label{indicator}
        Given any $X\in \Omega_1\cup\Omega_0$, we have for any $T_0+1\leq t\leq T$, with probability at least $1-8td_1^{-200}$,
        \begin{align*}
            \mathbbm{1}(\inp{\widehat{M}_{i,t}-\widehat{M}_{1-i,t}}{X}> 0)= \mathbbm{1}(\inp{M_i-M_{1-i}}{X}> 0).
        \end{align*}
    \end{Lemma}
    Under this result, $\sum_{X\in \Omega_i} \indicator(\inp{\widehat{M}_{i,t-1}-\widehat{M}_{1-i,t-1}}{X}>0)\inp{X}{\mathcal{P}_M(Q)}^2\rightarrow \fro{P_{\Omega_i}(P_{M}(Q))}^2$ and $\sum_{X\in \Omega_{1-i}} \indicator(\inp{\widehat{M}_{i,t-1}-\widehat{M}_{1-i,t-1}}{X}>0)\inp{X}{\mathcal{P}_M(Q)}^2\rightarrow 0$ for $i=0,1$. \\
    Note that $\frac{b}{T}\sum_{t=T_0+1}^{T} 1/(1-\varepsilon_t/2)\rightarrow 1$ and $\frac{b}{T^{1+\gamma}}\sum_{t=T_0+1}^{T} 2/\varepsilon_t\rightarrow \frac{2}{c_2(1+\gamma)}$ when $\varepsilon_t=c_2t^{-\gamma}$. And by Assumption \ref{assump:arm-opt}, $\sum_{X\in \Omega_\emptyset} \indicator(\inp{\widehat{M}_{1,t-1}-\widehat{M}_{0,t-1}}{X}>0)\frac{1}{1-\frac{\varepsilon_t}{2}}\inp{X}{\mathcal{P}_M(Q)}^2  +  \indicator(\inp{\widehat{M}_{1,t-1}-\widehat{M}_{0,t-1}}{X}<0)\frac{1}{\frac{\varepsilon_t}{2}}\inp{X}{\mathcal{P}_M(Q)}^2\leq 2/\varepsilon_t \fro{\calP_{\Omega_{\emptyset}}(\calP_{M}(Q))}^2$ is negligible. Then the conditional variance will converge in probability to 1. It follows
    \begin{align*}
        \frac{\sum_{t=T_0+1}^T \frac{d_1d_2\mathbbm{1}(a_t=1)}{\pi_t}\xi_t\inp{X_t}{\mathcal{P}_M(Q)}}{\sigma S\sqrt{d_1d_2/T^{1-\gamma}}} \rightarrow N(0,1)
    \end{align*}
    when $T, d_1, d_2\rightarrow \infty$.
\end{proof}

\subsection{Proof of Lemma \ref{lemmaneg1}}
\begin{proof}
    By definition,
    \begin{align*}
       &\inp{LL^{\top}\widehat{Z}_2R_{\perp}R_{\perp}^{\top}}{Q} + \inp{L_{\perp}L_{\perp}^{\top}\widehat{Z}_2RR^{\top}}{Q} + \inp{LL^{\top}\widehat{Z}_2RR^{\top}}{Q} = \inp{\widehat{Z}_2}{\mathcal{P}_M(Q)} \\
       &\quad = \frac{b}{T}\sum_{t=T_0+1}^{T}\frac{d_1d_2\mathbbm{1}(a_t=1)}{\pi_t}\inp{\widehat{\Delta}_{t-1}}{X_t}\inp{X_t}{\mathcal{P}_M(Q)}-\inp{\widehat{\Delta}_{t-1}}{\mathcal{P}_M(Q)}.
    \end{align*}
    Then we prove its upper bound. For any $T_0+1\leq t\leq T$,
    \begin{align*}
        &\frac{b}{T}\left|\frac{d_1d_2\mathbbm{1}(a_t=1)}{\pi_t}\inp{\widehat{\Delta}_{t-1}}{X_t}\inp{X_t}{\mathcal{P}_M(Q)}-\inp{\widehat{\Delta}_{t-1}}{\mathcal{P}_M(Q)}\right| \\
        &\quad \leq \frac{bd_1d_2}{T}\frac{2}{\varepsilon_t}\|\widehat{\Delta}_{t-1}\|_{\max}\sqrt{\frac{\mu r}{d_2}}\fro{\mathcal{P}_M(Q)} + \frac{1}{T}\sqrt{d_1d_2}\|\widehat{\Delta}_{t-1}\|_{\max}\fro{\mathcal{P}_M(Q)} \\
        &\quad \lesssim  \frac{d_1d_2^{1/2}\sqrt{\mu r}}{T^{1-\gamma}}\|\widehat{\Delta}_{t-1}\|_{\max}\fro{\mathcal{P}_M(Q)} \lesssim \frac{d_1d_2^{1/2}\sqrt{\mu r}\sigma}{T^{1-\gamma}}\fro{\mathcal{P}_M(Q)}.
    \end{align*}
    Moreover,
    \begin{align*}
        &\EE\left[\frac{b^2d_1^2d_2^2\mathbbm{1}(a_t=1)}{T^2\pi_t^2}\inp{\widehat{\Delta}_{t-1}}{X_t}^2\inp{X_t}{\mathcal{P}_M(Q)}^2\big|\mathcal{F}_{t-1}\right] \leq \frac{b^2d_1d_2}{T^2}\|\widehat{\Delta}_{t-1}\|_{\max}^2 \EE\left[\frac{1}{\pi_t} \inp{X_t}{\mathcal{P}_M(Q)}^2\big|\mathcal{F}_{t-1}\right] \\
        &\quad = \frac{b^2d_1d_2}{T^2}\|\widehat{\Delta}_{t-1}\|_{\max}^2 \sum_{X\in \mathcal{X}} \indicator(\inp{\widehat{M}_{1,t-1}-\widehat{M}_{0,t-1}}{X}>0)\frac{1}{1-\frac{\varepsilon_t}{2}}\inp{X}{\mathcal{P}_M(Q)}^2 \\ &\quad\quad +  \indicator(\inp{\widehat{M}_{1,t-1}-\widehat{M}_{0,t-1}}{X}<0)\frac{1}{\frac{\varepsilon_t}{2}}\inp{X}{\mathcal{P}_M(Q)}^2 \\
        &\quad  \lesssim \frac{d_1^2d_2r\log^4 d_1\sigma^2}{T^{3-\gamma}} \sum_{X\in \mathcal{X}} \indicator(\inp{\widehat{M}_{1,t-1}-\widehat{M}_{0,t-1}}{X}>0)\frac{1}{1-\frac{\varepsilon_t}{2}}\inp{X}{\mathcal{P}_M(Q)}^2 \\ &\quad\quad +  \indicator(\inp{\widehat{M}_{1,t-1}-\widehat{M}_{0,t-1}}{X}<0)\frac{1}{\frac{\varepsilon_t}{2}}\inp{X}{\mathcal{P}_M(Q)}^2.
    \end{align*}
    Denote $S_T^2:=\frac{1}{T^{1+\gamma}}\sum_{t=T_0+1}^{T} \sum_{X\in \mathcal{X}} \indicator(\inp{\widehat{M}_{1,t-1}-\widehat{M}_{0,t-1}}{X}>0)\frac{1}{1-\frac{\varepsilon_t}{2}}\inp{X}{\mathcal{P}_M(Q)}^2  +  \indicator(\inp{\widehat{M}_{1,t-1}-\widehat{M}_{0,t-1}}{X}<0)\frac{1}{\frac{\varepsilon_t}{2}}\inp{X}{\mathcal{P}_M(Q)}^2$. By martingale Bernstein inequality, with probability at least $1-d_1^{-10}$,
    \begin{align*}
        \left|\frac{1}{T}\sum_{t=T_0+1}^{T}\frac{d_1d_2\mathbbm{1}(a_t=1)}{\pi_t}\inp{\widehat{\Delta}_{t-1}}{X_t}\inp{X_t}{\mathcal{P}_M(Q)}-\inp{\widehat{\Delta}_{t-1}}{\mathcal{P}_M(Q)}\right|\lesssim \sqrt{\frac{d_1^2d_2\log^5 d_1}{T^{2-2\gamma}}}\sigma S_T.
    \end{align*}
    As we have shown in the proof of Lemma \ref{main term}, $S_T\rightarrow S$ when $T,d_1,d_2\rightarrow \infty$. Then \\ $\frac{\sum_{t=T_0+1}^{T}\frac{d_1d_2\mathbbm{1}(a_t=1)}{\pi_t}\inp{\widehat{\Delta}_{t-1}}{X_t}\inp{X_t}{\mathcal{P}_M(Q)}-\inp{\widehat{\Delta}_{t-1}}{\mathcal{P}_M(Q)}}{\sigma S\sqrt{d_1d_2/T^{1-\gamma}}}\overset{p}{\rightarrow} 0$ as long as $\frac{d_1\log^5d_1}{T^{1-\gamma}}\rightarrow 0$.
\end{proof}

\subsection{Proof of Lemma \ref{lemmaneg2}}
\begin{proof}
    We first state the following Lemmas.
    \begin{Lemma}
        \label{Zbound:new}
        Under the conditions in Theorem \ref{thm:CLT}, with probability at least $1-d_1^{-3}$,
        \begin{align*}
        \|\widehat{Z}\|\lesssim \sqrt{\frac{d_1^2d_2\log d_1}{T^{1-\gamma}}}\sigma.
        \end{align*}
    \end{Lemma}


    \begin{Lemma}
        \label{induction:new}
        Under the conditions in Theorem \ref{thm:CLT}, with probability at least $1-5d_1^{-2}$,
        \begin{align*}
            &\|\widehat{L}\widehat{L}^{\top}-LL^{\top}\|_{2,\max}\lesssim \frac{\sigma}{\lambda_{\min}}\sqrt{\frac{d_1^2d_2\log d_1}{T^{1-\gamma}}}\sqrt{\frac{\mu r}{d_1}} \\
            &\|\widehat{R}\widehat{R}^{\top}-RR^{\top}\|_{2,\max}\lesssim \frac{\sigma}{\lambda_{\min}}\sqrt{\frac{d_1^2d_2\log d_1}{T^{1-\gamma}}}\sqrt{\frac{\mu r}{d_2}}. 
        \end{align*}
    \end{Lemma}
\noindent Notice that 
    \begin{align*}
        |\inp{\widehat{L}\widehat{L}^\top \widehat{Z} \widehat{R}\widehat{R}^\top - LL^\top \widehat{Z}RR^\top} {Q}| \le \|Q\|_{\ell_1} \|\widehat{L}\widehat{L}^\top \widehat{Z} \widehat{R}\widehat{R}^\top - LL^\top \widehat{Z}RR^\top\|_{\max}.
    \end{align*}
    Now by triangular inequality, 
    \begin{align*}
        \|\widehat{L}\widehat{L}^\top \widehat{Z} \widehat{R}\widehat{R}^\top - LL^\top \widehat{Z}RR^\top\|_{\max} \le  &   \|(\widehat{L}\widehat{L}^\top-LL^\top) \widehat{Z}RR^\top\|_{\max} + \|LL^\top \widehat{Z}( \widehat{R}\widehat{R}^\top -RR^\top)\|_{\max} \\
    &+  \|(\widehat{L}\widehat{L}^\top-LL^\top)  \widehat{Z} ( \widehat{R}\widehat{R}^\top- RR^\top)\|_{\max}\\
    \le & \|\widehat{Z}\| \left(\|\widehat{L} \widehat{L}^{\top}-LL^{\top}\|_{2,  { \max }}\|R\|_{2, { \max }} +\|\widehat{R}\widehat{R}^\top- RR^\top\|_{2,\max} \|L\|_{2, { \max }} \right) \\
    & + \|\widehat{Z}\| \|\widehat{L} \widehat{L}^{\top}-LL^{\top}\|_{2,  { \max }}\|\widehat{R}\widehat{R}^\top- RR^\top\|_{2,\max} \\
    \lesssim & \frac{\sigma^2}{\lambda_{\min}}\frac{\mu rd_1^{3/2}d_2^{1/2}\log d_1}{T^{1-\gamma}}.
    \end{align*}
    Then as long as $\frac{\|Q\|_{\ell 1}}{\sigma S\sqrt{d_1d_2/T^{1-\gamma}}}\cdot  \frac{\sigma^2}{\lambda_{\min}}\frac{\mu rd_1^{3/2}d_2^{1/2}\log d_1}{T^{1-\gamma}}=\frac{\|Q\|_{\ell_1}}{S}\frac{\sigma}{\lambda_{\min}}\sqrt{\frac{r^2d_1^2\log d_1}{T^{1-\gamma}}}\rightarrow 0$, we prove the claimed results. 
\end{proof}

\subsection{Proof of Lemma \ref{lemmaneg3}}
\begin{proof}
    Similarly as the proof of Lemma \ref{lemmaneg2}, we again define events $\mathcal{E}_0=\{\|\widehat{Z}\|\lesssim \sqrt{\frac{d_1^2d_2\log d_1}{T^{1-\gamma}}}\sigma:=\delta\}$, $\mathcal{E}_1=\{\|\widehat{L}\widehat{L}^{\top}-LL^{\top}\|_{2,\max}\lesssim \frac{\sigma}{\lambda_{\min}}\sqrt{\frac{d_1^2d_2\log d_1}{T^{1-\gamma}}}\sqrt{\frac{\mu r}{d_1}}, \|\widehat{R}\widehat{R}^{\top}-RR^{\top}\|_{2,\max}\lesssim \frac{\sigma}{\lambda_{\min}}\sqrt{\frac{d_1^2d_2\log d_1}{T^{1-\gamma}}}\sqrt{\frac{\mu r}{d_2}}\}$. Then by Lemma 5 and Lemma 6 in \cite{xia2021statistical}, under the $\mathcal{E}_0$ and $\mathcal{E}_1$, 
    \begin{align*}
        \big|\sum_{k=2}^{\infty} \inp{(\mathcal{S}_{A,k}(\widehat{E})A\Theta\Theta^{\top}+\Theta\Theta^{\top}A\mathcal{S}_{A,k}(\widehat{E}))}{\widetilde{Q}}\big|&\lesssim \|Q\|_{\ell_1}\frac{\mu r}{\lambda_{\min}\sqrt{d_1d_2}}\delta^2 \\
        &=\|Q\|_{\ell_1}\frac{\sigma^2}{\lambda_{\min}}\frac{\mu rd_1^{3/2}d_2^{1/2}\log d_1}{T^{1-\gamma}},
    \end{align*}
    and
    \begin{align*}
        \big|\inp{(\widehat{\Theta}\widehat{\Theta}^{\top}-\Theta\Theta^{\top})A(\widehat{\Theta}\widehat{\Theta}^{\top}-\Theta\Theta^{\top})}{\widetilde{Q}}\big|&\lesssim \|Q\|_{\ell_1}\|\Lambda\|\|\widehat{L} \widehat{L}^{\top}-LL^{\top}\|_{2,  { \max }}\|\widehat{R}\widehat{R}^\top- RR^\top\|_{2,\max} \\
        &\leq \|Q\|_{\ell_1}\frac{\sigma^2}{\lambda_{\min}}\frac{\mu r\kappa d_1^{3/2}d_2^{1/2}\log d_1}{T^{1-\gamma}}.
    \end{align*}
    Then as long as $\frac{\|Q\|_{\ell_1}}{S}\frac{\sigma}{\lambda_{\min}}\sqrt{\frac{r^2d_1^2\log d_1}{T^{1-\gamma}}}\rightarrow 0$, we prove the claimed results. 
\end{proof}

\subsection{Proof of Lemma \ref{indicator}}
\begin{proof}
    Without loss of generality, we just prove the $i=1$ case and $i=0$ case can be shown following the same arguments. \\
    Recall $\delta$ defined in Theorem \ref{thm:CLT}. Suppose $\inp{M_1-M_0}{X}>0$, then we have for any $T_0+1\leq t\leq T$,
    \begin{align*}
       \inp{\widehat{M}_{1,t}-\widehat{M}_{0,t}}{X}&=\inp{\widehat{M}_{1,t} - M_1}{X} + \inp{M_0-\widehat{M}_{0,t}}{X} + \inp{M_1-M_0}{X} \\
       &\geq \inp{M_1-M_0}{X} - 2\delta > 0
    \end{align*}
    with probability at least $1-8td_1^{-200}$. The second inequality comes from the gap condition in Assumption \ref{assump:arm-opt}. As a result, with the same probability, $\mathbbm{1}(\inp{\widehat{M}_{1,t}-\widehat{M}_{0,T}}{X}>0)=\mathbbm{1}(\inp{M_1-M_0}{X}>0)$.
\end{proof}

\subsection{Proof of Lemma \ref{Zbound:new}}
\begin{proof}
    We first look at $\|\widehat{Z}_1\|=\|\frac{b}{T}\sum_{t=T_0+1}^{T}\frac{d_1d_2\indicator(a_t=1)}{\pi_t}\xi_tX_t\|$.\\
    Since $\xi_t$ is subGaussian random variable, then for any $T_0+1\leq t\leq T$,
    \begin{align*}
        \bigg\|\|\frac{b}{T}\frac{d_1d_2\indicator(a_t=1)}{\pi_t}\xi_tX_t\|\bigg\|_{\psi_2}\leq \frac{2bd_1d_2}{T\varepsilon_t}\|\xi_t\|_{\psi_2}\lesssim \frac{d_1d_2\sigma}{T^{1-\gamma}}.
    \end{align*}
    Meanwhile,
    \begin{align*}
        &\sum_{t=T_0+1}^{T} \EE\left[\frac{b^2}{T^2}\frac{d_1^2d_2^2\indicator(a_t=1)}{\pi_t^2}\xi_t^2X_tX_t^{\top}\bigg|\mathcal{F}_{t-1}\right]\leq \frac{b^2d_1^2d_2^2}{T^2}\sum_{t=T_0+1}^{T}\frac{2}{\varepsilon_t}\sigma^2\EE[X_tX_t^{\top}] \\
        &\quad = \frac{b^2d_1^2d_2^2}{T^2}\sum_{t=1}^{T}\frac{2}{\varepsilon_t}\sigma^2\frac{1}{d_1}\boldsymbol{I}_{d_1} \lesssim \frac{d_1d_2^2\sigma^2}{T^{1-\gamma}}\boldsymbol{I}_{d_1}.
    \end{align*}
    By symmetry, we can have
    \begin{align*}
        &\max\left\{\bigg\|\sum_{t=T_0+1}^{T} \EE\left[\frac{b^2}{T^2}\frac{d_1^2d_2^2\indicator(a_t=1)}{\pi_t^2}\xi_t^2X_tX_t^{\top}\bigg|\mathcal{F}_{t-1}\right]\bigg\|, \right.\\ & \left.\quad\quad\quad \bigg\|\sum_{t=T_0+1}^{T} \EE\left[\frac{b^2}{T^2}\frac{d_1^2d_2^2\indicator(a_t=1)}{\pi_t^2}\xi_t^2X_t^{\top}X_t\bigg|\mathcal{F}_{t-1}\right]\bigg\|\right\} \lesssim \frac{d_1^2d_2\sigma^2}{T^{1-\gamma}}.
    \end{align*}
    Then by matrix Bernstein inequality, with probability at least $1-d_1^{-2}$,
    \begin{align*}
        \|\widehat{Z}_1\|&\lesssim \frac{d_1d_2\sigma}{T^{1-\gamma}}\log d_1 + \sqrt{\frac{d_1^2d_2\sigma^2\log d_1}{T^{1-\gamma}}} \lesssim \sqrt{\frac{d_1^2d_2\log d_1}{T^{1-\gamma}}}\sigma.
    \end{align*}
    The first term will be dominated by the second term since $T\gg d_1^{1/(1-\gamma)}\log^3d_1$. \\
    To prove the upper bound for $\|\widehat{Z}_2\|=\|\frac{b}{T}\sum_{t=T_0+1}^{T}\frac{d_1d_2\indicator(a_t=1)}{\pi_t}\inp{\widehat{\Delta}_{t-1}}{X_t}X_t-\widehat{\Delta}_{t-1}\|$, we follow the same arguments by observing that for any $T_0+1\leq t\leq T$,
    \begin{align*}
        \bigg\|\frac{b}{T}\left(\frac{d_1d_2\indicator(a_t=1)}{\pi_t}\inp{\widehat{\Delta}_{t-1}}{X_t}X_t-\widehat{\Delta}_{t-1}\right)\bigg\|&\leq\frac{bd_1d_2}{T}\frac{2}{\varepsilon_t}\|\widehat{\Delta}_{t-1}\|_{\max}+\frac{b}{T}\|\widehat{\Delta}_{t-1}\| \\
        &\lesssim \frac{d_1d_2}{T^{1-\gamma}}\|\widehat{\Delta}_{t-1}\|_{\max} + \frac{\sqrt{d_1d_2}}{T}\|\widehat{\Delta}_{t-1}\|_{\max} \lesssim \frac{d_1d_2\sigma}{T^{1-\gamma}}.
    \end{align*}
    The second inequality comes from the fact that $\|\widehat{\Delta}_{t-1}\|\leq \sqrt{d_1d_2}\|\widehat{\Delta}_{t-1}\|_{\max}$, and in the last inequality we again use $\|\widehat{\Delta}_{t-1}\|_{\max}\lesssim \sigma$. \\
    Meanwhile,
    \begin{align*}
        &\quad \bigg\|\frac{b^2}{T^2}\sum_{t=T_0+1}^{T}\EE\left[\left(\frac{d_1d_2\indicator(a_t=1)}{\pi_t}\inp{\widehat{\Delta}_{t-1}}{X_t}X_t-\widehat{\Delta}_{t-1}\right)\left(\frac{d_1d_2\indicator(a_t=1)}{\pi_t}\inp{\widehat{\Delta}_{t-1}}{X_t}X_t-\widehat{\Delta}_{t-1}\right)^{\top}\big|\mathcal{F}_{t-1}\right]\bigg\| \\
        &\leq \frac{b^2}{T^2}\sum_{t=T_0+1}^T \bigg\|\EE\left[\frac{d_1^2d_2^2\indicator(a_t=1)}{\pi_t^2}\inp{\widehat{\Delta}_{t-1}}{X_t}^2X_tX_t^{\top}\big|\mathcal{F}_{t-1}\right]\bigg\| + \|\widehat{\Delta}_{t-1}\|^2 \\
        &\leq \frac{b^2}{T^2}\sum_{t=T_0+1}^T \frac{2d_1^2d_2^2}{\varepsilon_t}\|\widehat{\Delta}_{t-1}\|_{\max}^2\|\EE(X_tX_t^{\top})\| + d_1d_2\|\widehat{\Delta}_{t-1}\|_{\max}^2 \lesssim \frac{d_1d_2^2\sigma^2}{T^{1-\gamma}}.
    \end{align*}
    Finally, by matrix Bernstein inequality, with probability $1-d_1^{-2}$,
    \begin{align*}
        \|\widehat{Z}_2\|&\lesssim \sqrt{\frac{d_1^2d_2\log d_1}{T^{1-\gamma}}}\sigma.
    \end{align*}
\end{proof}

\subsection{Proof of Lemma \ref{induction:new}}
\begin{proof}
    Consider the event in Theorem \ref{Zbound:new}
    $$\calE = \bigg\{\|\tilde Z\|\lesssim \sqrt{\frac{d_1^2d_2\log d_1}{T^{1-\gamma}}}\sigma\bigg\}.$$
    The following proof is adapted from Theorem 4 and Lemma 9 in \cite{xia2021statistical}. We adapt it to martingale version.
    Under $\calE$, we have $\|\widehat Z\|\lesssim \sqrt{\frac{d_1^2d_2\log d_1}{T^{1-\gamma}}}\sigma\lesssim \lambda_{\min} =: \delta$. 
    As long as $\sqrt{\frac{d_1^2d_2\log d_1}{T^{1-\gamma}}}\sigma\lesssim \lambda_{\min}$, we have $\|\widehat E\|=\|\widehat Z\| \lesssim \lambda_{\min}$. 
    Recall the $\mathfrak{P}^{-s}$ for $s\geq 0$ and $\mathfrak{P}^{\perp}$ defined in the proof of Theorem \ref{thm:CLT}. Following the proof in \cite{xia2021statistical}, it suffices to show under $\calE$, there exist absolute constants $C_1$, $C_2>0$ so that for all $k\geq 0$, the following bounds hold with probability at least $1-2(k+1)d_1^{-2}$,
    \begin{align*}
        &\max_{j \in [d_1]} \|e_{j}^{\top}\mathfrak{P}^{\perp}(\mathfrak{P}^{\perp}\widehat{E}\mathfrak{P}^{\perp})^k\widehat{E}\Theta\|\leq C_1(C_2\delta)^{k+1}\sqrt{\frac{\mu r}{d_1}}, \\
        &\max_{j \in [d_2]} \|e_{j+d_1}^{\top}\mathfrak{P}^{\perp}(\mathfrak{P}^{\perp}\widehat{E}\mathfrak{P}^{\perp})^k\widehat{E}\Theta\|\leq C_1(C_2\delta)^{k+1}\sqrt{\frac{\mu r}{d_2}}.
    \end{align*}
Here we only present the proof for $k=0$ and the cases $k\geq 1$ can be similarly extended. 
    Clearly, for $j\in [d_1]$,
    \begin{align*}
        \|e_j^{\top}\mathfrak{P}^{\perp}\widehat{E}\Theta\|\leq \|e_j^{\top}\Theta\Theta^{\top}\widehat{E}\Theta\| + \|e_j^{\top}\widehat{E}\Theta\|\leq \delta\sqrt{\frac{\mu r}{d_1}} + \|e_j^{\top}\widehat{E}\Theta\|= \delta\sqrt{\frac{\mu r}{d_1}} + \|e_j^{\top}\widehat{Z}R\|. 
    \end{align*}
    We write
    \begin{align*}
        e_{j}^{\top}\widehat{Z}R &= \frac{b}{T}\left(\sum_{t=T_0+1}^T \frac{d_1d_2\mathbbm{1}(a_t=1)}{\pi_t}\xi_te_j^{\top}X_tR \right.\\ & \left.\quad + \sum_{t=T_0+1}^T \frac{d_1d_2\mathbbm{1}(a_t=1)}{\pi_t}\inp{\widehat{\Delta}_{t-1}}{X_t}e_j^{\top}X_tR - e_j\widehat{\Delta}_{t-1}R\right).
    \end{align*} 
    Obviously, for any $T_0+1\leq t\leq T$,
    \begin{align*}
        \left\|\big\|\frac{bd_1d_2\mathbbm{1}(a_t=1)}{T\pi_t}e_j^{\top}\xi_tX_tR \big\|\right\|_{\psi_2}\lesssim \frac{d_1d_2}{T^{1-\gamma}}\sigma\|R\|_{2,\max}\leq \frac{d_1d_2}{T^{1-\gamma}}\sigma\sqrt{\frac{\mu r}{d_2}}
    \end{align*}
    and
    \begin{align*}
        \EE\left[\frac{b^2d_1^2d_2^2\mathbbm{1}(a_t=1)}{T^2\pi_t^2}\xi_t^2e_j^{\top}X_tRR^{\top}X_t^{\top}e_j \bigg|\mathcal{F}_{t-1}\right]\lesssim \frac{t^{\gamma}\sigma^2d_1d_2}{T^2}\text{tr}(RR^{\top})\leq \frac{t^{\gamma}r\sigma^2d_1d_2}{T^{2}}.
    \end{align*}
    Then by martingale Bernstein inequality, with probability at least $1-d_1^{-3}$,
    \begin{align*}
        \left\|\frac{b}{T}\sum_{t=T_0+1}^T \frac{d_1d_2\mathbbm{1}(a_t=1)}{\pi_t}\xi_te_j^{\top}X_tR \right\|\lesssim \sqrt{\frac{rd_1d_2\log d_1}{T^{1-\gamma}}}\sigma.
    \end{align*}
    Similarly, with the same probability,
    \begin{align*}
        \left\|\frac{b}{T}\sum_{t=T_0+1}^T \frac{d_1d_2\mathbbm{1}(a_t=1)}{\pi_t}\inp{\widehat{\Delta}_{t-1}}{X_t}e_j^{\top}X_tR - e_j\widehat{\Delta}_{t-1}R \right\|\lesssim \sqrt{\frac{rd_1d_2\log d_1}{T^{1-\gamma}}}\sigma.
    \end{align*}
    Then we can conclude that 
    \begin{align*}
        \|e_{j}^{\top}\widehat{Z}R\|\lesssim \sqrt{\frac{rd_1d_2\log d_1}{T^{1-\gamma}}}\sigma
    \end{align*}
    with probability at least $1-3d_1^{-3}$. Taking union bound, 
    \begin{align*}
        \PP\bigg(\max_{j\in [d_1]}\|e_j^{\top}\mathfrak{P}^{\perp}\widehat{E}\Theta\|\gtrsim \delta\sqrt{\frac{\mu r}{d_1}}\bigg)\leq 3d_1^{-2},\\
        \PP\bigg(\max_{j\in [d_2]}\|e_j^{\top}\mathfrak{P}^{\perp}\widehat{E}\Theta\|\gtrsim \delta\sqrt{\frac{\mu r}{d_2}}\bigg)\leq 3d_1^{-2}.
    \end{align*}
Then following the proof of Lemma 4 in \cite{xia2021statistical}, 
    \begin{align*}
        \PP\bigg(\|\widehat{L}\widehat{L}^{\top}-LL^{\top}\|_{2,\max}
        \gtrsim \frac{\sigma}{\lambda_{\min}}\sqrt{\frac{d_1^2d_2\log d_1}{T^{1-\gamma}}}\sqrt{\frac{\mu r}{d_1}}\bigg|\calE\bigg)\leq 5d_1^{-2}\log^2d_1.
    \end{align*}
    And similarly we can prove the statement for $\|\widehat{R}\widehat{R}^{\top}-RR^{\top}\|_{2,\max}$. 

\end{proof}

\subsection{Proof of Theorem \ref{thm:regretK}}\label{sec:proofregretK}
\begin{proof}
Without loss of generality, assume at time $t$, arm 1 is optimal, i.e., for any $k\in [K]\backslash \{1\}$, $\inp{M_1-M_k}{X_t}>0$, then the expected regret is
\begin{align*}
    r_t&=\EE\left[\sum_{k\in [K]\backslash \{1\}} \inp{M_1-M_k}{X_t}\mathbbm{1}(\text{choose arm k}) \right] \\
    &= \EE\left[\sum_{k\in [K]\backslash \{1\}} \inp{M_1-M_k}{X_t}\mathbbm{1}(\text{choose arm k}, \inp{\widehat{M}_{1,t-1}-\widehat{M}_{k,t-1}}{X_t}\geq 0)\right] \\ &\quad + \EE\left[\sum_{k\in [K]\backslash \{1\}}\inp{M_1-M_k}{X_t}\mathbbm{1}(\text{choose arm 0}, \inp{\widehat{M}_{k,t-1}-\widehat{M}_{1,t-1}}{X_t}> 0)\right] \\
    &\leq \frac{(K-1)\varepsilon_t}{K}\bar{m} + \underbrace{\EE\left[\sum_{k\in [K]\backslash \{1\}} \inp{M_1-M_k}{X_t}\mathbbm{1}(\inp{\widehat{M}_{k,t-1}-\widehat{M}_{1,t-1}}{X_t}> 0)\right]}_{(2)}.
\end{align*}
Define $\delta_t^2$ the entrywise bound in Theorem \ref{thm:MCB-convK} such that $\|U_{i,t}V_{i,t}^{\top} - M\|_{\max}^2\leq \delta_t^2$ with probability at least $1-8td_1^{-200}$, and event $B_{k,t}=\{\inp{M_1-M_k}{X_t}>2\delta_t\}$, then the latter expectation can be written as
\begin{align*}
    (2)&\leq \EE\left[\sum_{k\in [K]\backslash \{1\}} \inp{M_1-M_k}{X_t}\mathbbm{1}\left(\inp{\widehat{M}_{k,t-1}-\widehat{M}_{1,t-1}}{X_t}> 0\cap B_{k,t}\right)\right] \\ &\quad + \EE\left[\sum_{k\in [K]\backslash \{1\}} \inp{M_1-M_k}{X_t}\mathbbm{1}\left(\inp{\widehat{M}_{k,t-1}-\widehat{M}_{1,t-1}}{X_t}> 0\cap B_{k,t}^c\right)\right] \\
    &\leq \sum_{k\in [K]\backslash \{1\}}\bar{m} \EE\left[\mathbbm{1}\left(\inp{\widehat{M}_{k,t-1}-\widehat{M}_{1,t-1}}{X_t}> 0\cap B_t\right)\right] + 2\delta_t\EE\left[\sum_{k\in [K]\backslash \{1\}} \mathbbm{1}(B_{k,t}^c)\right].
\end{align*}
For the first term, under event $B_{k,t}$ and $\inp{\widehat{M}_{k,t-1}-\widehat{M}_{1,t-1}}{X_t}> 0$, 
\begin{align*}
    0>\inp{\widehat{M}_{1,t-1}-\widehat{M}_{k,t-1}}{X_t} &= \inp{\widehat{M}_{1,t-1}-M_1}{X_t} + \inp{M_k-\widehat{M}_{k,t-1}}{X_t} + \inp{M_1-M_k}{X_t}\\
    &> \inp{\widehat{M}_{1,t-1}-M_1}{X_t} + \inp{M_k-\widehat{M}_{k,t-1}}{X_t} + 2\delta_t.
\end{align*}
This means either $\inp{\widehat{M}_{1,t-1}-M_1}{X_t}<-\delta_t$ or $\inp{M_k-\widehat{M}_{k,t-1}}{X_t}<-\delta_t$, which further implies either $\|\widehat{M}_{1,t-1}-M_1\|_{\max}$ or $\|M_k-\widehat{M}_{k,t-1}\|_{\max}$ should be larger than $\delta_t$. Therefore, 
\begin{align*}
    \EE\left[\mathbbm{1}(\inp{{M}_{k,t-1}-{M}_{1,t-1}}{X_t}> 0)\cap B_t\right]&= \PP\left(\inp{{M}_{k,t-1}-{M}_{1,t-1}}{X_t}> 0 \cap B_{k,t}\right) \\
    &\leq \PP\left(\|{M}_{1,t-1}-M_1\|_{\max}>\delta_t\right) + \PP\left(\|{M}_{k,t-1}-M_k\|_{\max}>\delta_t\right) \\
    &\leq \frac{16(t-1)}{d^{200}},
\end{align*}
For the other term, since $\EE\left[\mathbbm{1}(B_{k,t}^c)\right]=\PP(\inp{M_1-M_k}{X_t}<2\delta_t)\leq 1$, then
\begin{align*}
    r_t&\leq \varepsilon_t\bar{m} + \frac{16(K-1)(t-1)}{d^{200}}\bar{m}  + 2(K-1)\delta_t. 
\end{align*}
The total regret up to time $T$ is
\begin{align*}
    R_T&=\sum_{t=1}^{T} r_t \leq \sum_{t=1}^T\varepsilon_t \bar{m} + \frac{8KT(T-1)}{d_1^{200}}\bar{m}  + 2K\sum_{t=1}^{T} \delta_t.
\end{align*}
Since $\delta_t^2= C_3\frac{\lambda_{\min}^2r^3}{d_1d_2}  \prod_{\tau=1}^{t}(1-\frac{c_1\log d_1}{4\kappa T^{1-\gamma}}) + C_4\frac{t\cdot rd_1\log^4d_1\bar{\sigma}^2}{T^{2(1-\gamma)}}$ when $1\leq t\leq T_0$ and $\delta_t^2=C_5\frac{rd_1\log^4d_1\bar{\sigma}^2}{T^{1-\gamma}}$ when $T_0+1\leq t\leq T$, for some constants $C_3$ ,$C_4$ and $C_5$. Therefore,
\begin{align*}
    R_T&\lesssim \sum_{t=1}^{T_0}\varepsilon \bar{m} + \sum_{t=T_0 + 1}^{T}\frac{1}{t^\gamma}\bar{m}  + K\frac{\bar{m}T^2}{d_1^{200}} \\ &\quad + K\sum_{t=1}^{T_0} \sqrt{\frac{\lambda_{\min}^2r^3}{d_1d_2}  \prod_{\tau=1}^{t}(1-\frac{c_1\log d_1}{4\kappa T^{1-\gamma}})} + K\sum_{t=1}^{T_0} \sqrt{\frac{t\cdot rd_1\log^4d_1\bar{\sigma}^2}{T^{2(1-\gamma)}}} + K\sum_{t=T_0+1}^{T} \sqrt{\frac{rd_1\log^4d_1\bar{\sigma}^2}{T^{1-\gamma}}}.
\end{align*}
Since $T\ll d_1^{50}$, the third term is negligible. As a result, from the fact that $\sum_{t=1}^T t^{-\gamma}\lesssim T^{1-\gamma}$,
\begin{align*}
    R_T&\lesssim T^{1-\gamma}\bar{m} + KT^{1-\gamma}\frac{\lambda_{\min}r^{3/2}}{\sqrt{d_1d_2}}  + KT^{(1+\gamma)/2}\sqrt{rd_1\log^4d_1\bar{\sigma}^2} \\
    &\lesssim T^{1-\gamma}\bar{m} + KT^{(1+\gamma)/2}\sqrt{rd_1\log^4d_1\bar{\sigma}^2}.
\end{align*}
\end{proof}

\subsection{Proof of Theorem \ref{thm:CLTK}}\label{sec:proofCLTK}
\begin{proof}
    Without loss of generality, we prove $a=1$ case and omit all the subscript. Recall that $\inp{\widehat{M}}{Q}-\inp{M}{Q}$ can be decomposed into
    \begin{align}
    \inp{\widehat{M}}{Q}-\inp{M}{Q}&= \inp{LL^{\top}\widehat{Z}_1R_{\perp}R_{\perp}^{\top}}{Q} + \inp{L_{\perp}L_{\perp}^{\top}\widehat{Z}_1RR^{\top}}{Q} + \inp{LL^{\top}\widehat{Z}_1RR^{\top}}{Q} \label{main1K} \\ &\quad + \inp{LL^{\top}\widehat{Z}_2R_{\perp}R_{\perp}^{\top}}{Q} + \inp{L_{\perp}L_{\perp}^{\top}\widehat{Z}_2RR^{\top}}{Q} + \inp{LL^{\top}\widehat{Z}_2RR^{\top}}{Q}\label{neg1K} \\ &\quad + \inp{\widehat{L}\widehat{L}^{\top}\widehat{Z}\widehat{R}\widehat{R}^{\top}}{Q} - \inp{LL^{\top}\widehat{Z}RR^{\top}}{Q} \label{neg2K} \\ &\quad + \inp{\sum_{k=2}^{\infty} (\mathcal{S}_{A,k}(\widehat{E})A\Theta\Theta^{\top}+\Theta\Theta^{\top}A\mathcal{S}_{A,k}(\widehat{E}))}{\widetilde{Q}} + \inp{(\widehat{\Theta}\widehat{\Theta}^{\top}-\Theta\Theta^{\top})A(\widehat{\Theta}\widehat{\Theta}^{\top}-\Theta\Theta^{\top})}{\widetilde{Q}}, \label{neg3K}
\end{align}
    where $\widehat{Z}_1=\frac{b}{T}\sum_{t=T_0+1}^T \frac{d_1d_2\mathbbm{1}(a_t=1)}{\pi_t}\xi_tX_t$ and $\widehat{Z}_2=\frac{b}{T}\sum_{t=T_0+1}^{T} \left(\frac{d_1d_2\mathbbm{1}(a_t=1)}{\pi_t}\inp{\widehat{\Delta}_{t-1}}{X_t}X_t-\widehat{\Delta}_{t-1}\right)$. The only different analysis is the main term (\ref{main1K}). (\ref{neg1K})-(\ref{neg3K}) can be shown negligible following the same arguments in Lemma \ref{lemmaneg1}-\ref{lemmaneg3}. Next we show the asymptotic normality of (\ref{main1K}).
    
    Denote $S^2=1/T^{\gamma}\fro{\mathcal{P}_{\Omega_1}(\mathcal{P}_M(Q))}^2+C_{\gamma}\fro{\mathcal{P}_{\cup_{k\in [K]\backslash \{1\}} \Omega_k}(\mathcal{P}_{M}(Q))}^2$. By definition, 
    \begin{align*}
        &\frac{\sqrt{\frac{T^{1-\gamma}}{d_1d_2}}\left(\inp{LL^{\top}\widehat{Z}_1R_{\perp}R_{\perp}^{\top}}{Q} + \inp{L_{\perp}L_{\perp}^{\top}\widehat{Z}_1RR^{\top}}{Q} + \inp{LL^{\top}\widehat{Z}_1RR^{\top}}{Q} \right)}{\sigma S} \\ &\quad = \frac{\sqrt{\frac{1}{T^{1+\gamma}}} \sum_{t=T_0+1}^T \frac{b\sqrt{d_1d_2}\mathbbm{1}(a_t=1)}{\pi_t}\xi_t\inp{X_t}{\mathcal{P}_M(Q)}}{\sigma S}.
    \end{align*}
    Similarly, we apply Theorem 3.2 and Corollary 3.1 in \cite{hall2014martingale}, the Martingale Central Limit Theorem to show the asymptotic normality. \\
    \emph{Step 1: checking Lindeberg condition.} \\
    For any $\delta>0$,
    \begin{align*}
        &\quad \sum_{t=T_0+1}^T \EE\left[\frac{b^2d_1d_2\mathbbm{1}(a_t=1)}{\sigma^2S^2T^{1+\gamma}\pi_t^2}\xi_t^2\inp{X_t}{\mathcal{P}_M(Q)}^2 \times  \mathbbm{1}\left(\left|\xi_t\frac{\sqrt{d_1d_2}\mathbbm{1}(a_t=1)\inp{X_t}{\mathcal{P}_M(Q)}}{\sigma S\sqrt{T^{1+\gamma}}\pi_t}\right|>\delta \right)\bigg|\mathcal{F}_{t-1}\right] \\
        &\leq \frac{b^2d_1d_2}{\sigma^2S^2T^{1+\gamma}}\sum_{t=T_0+1}^T \EE\left[\frac{\mathbbm{1}(a_t=1)}{\pi_t^2}\xi_t^2\inp{X_t}{\mathcal{P}_M(Q)}^2 \times  \mathbbm{1}\left(\left|\xi_t\inp{X_t}{\mathcal{P}_M(Q)}\right|>\frac{\sigma S\delta\sqrt{T^{1+\gamma}}\varepsilon_t}{2\sqrt{d_1d_2}} \right)\bigg|\mathcal{F}_{t-1}\right] \\
        &\lesssim \frac{b^2d_1d_2}{\sigma^2S^2T^{1+\gamma}}  \sum_{t=T_0+1}^T \frac{2}{\varepsilon_t}\max_{X\in \mathcal{X}} \inp{X}{\mathcal{P}_M(Q)}^2 \times \sqrt{\EE\left[\mathbbm{1}\left(\left|\xi_t\inp{X_t}{\mathcal{P}_M(Q)}\right|>\frac{\sigma S\delta\sqrt{T^{1+\gamma}}\varepsilon_t}{2\sqrt{d_1d_2}} \right)\right]},
    \end{align*}
    where in the last inequality we use $\EE[XY]\leq \sqrt{\EE[X^2]\EE[Y^2]}$.
    By incoherence condition, 
    \begin{align*}       
    \max_{X\in \mathcal{X}} \big|\inp{X}{\mathcal{P}_M(Q)}\big|\lesssim \sqrt{\frac{\mu r}{d_2}}\fro{\mathcal{P}_M(Q)}.
    \end{align*}
    Moreover, since $\xi_t$ is a subGaussian random variable, so the product $\xi_t\inp{X_t}{\mathcal{P}_M(Q)}$ has subGaussian tail probability
    \begin{align*}
        \PP\left(\left|\xi_t\inp{X_t}{\mathcal{P}_M(Q)}\right|> \frac{\sigma S\delta\sqrt{T^{1+\gamma}}\varepsilon_t}{2\sqrt{d_1d_2}}\right)< 2e^{-\frac{\sigma^2S^2\delta^2T^{1+\gamma}\varepsilon_t^2}{8d_1d_2\nu^2}},    
    \end{align*}
    where $\nu$ is the subGaussian parameter of order $O(\sqrt{\frac{\mu r}{d_2}}\fro{\mathcal{P}_M(Q)}\sigma)$. Notice that $T^{1+\gamma}\varepsilon_t^2\leq T^{1+\gamma}\varepsilon_T^2=O(T^{1-\gamma})$. By property of exponential function, $\frac{d_1d_2\sigma^2}{\sigma^2S^2T^{1+\gamma}} \frac{\mu r}{d_2}\fro{\mathcal{P}_M(Q)}^2\sum_{t=T_0+1}^{T} \frac{2}{\varepsilon_t}2e^{-\frac{\sigma^2S^2\delta^2T^{1+\gamma}\varepsilon_t^2}{16d_1d_2\nu^2}}$ converges to 0 as long as $\frac{d_1}{T^{1-\gamma}} \rightarrow 0$. Then the Lindeberg condition is satisfied. \\
    \emph{Step 2: calculating the variance} \\
    Next, we show the conditional variance equals to 1. Recall the definition $\Omega_1=\{X\in \mathcal{X}: \inp{M_1-M_0}{X}> \delta_{T,d_1,d_2}\}$, $\Omega_0=\{X\in \mathcal{X}: \inp{M_1-M_0}{X}< \delta_{T,d_1,d_2}\}$ and $\Omega_{\emptyset}=\mathcal{X}\cup (\Omega_1\cup\Omega_0)^c$. Then
    \begin{align*}
        &\quad \frac{b^2d_1d_2}{\sigma^2S^2T^{1+\gamma}}\sum_{t=T_0+1}^T \EE\left[\frac{\mathbbm{1}(a_t=1)}{\pi_t^2}\xi_t^2\inp{X_t}{\mathcal{P}_M(Q)}^2\bigg|\mathcal{F}_{t-1}\right] \\ 
        &= \frac{b^2}{\sigma^2S^2T^{1+\gamma}}\sigma^2 \sum_{t=T_0+1}^{T} \sum_{X\in \Omega_1} \indicator\left(\max_{k\in [K]}\inp{\widehat{M}_{k,t-1}}{X}=1\right)\frac{1}{1-\frac{\varepsilon_t}{K}}\inp{X}{\mathcal{P}_M(Q)}^2 \\ &\quad\quad\quad\quad +  \indicator\left(\max_{k\in [K]}\inp{\widehat{M}_{k,t-1}}{X}\neq 1\right)\frac{1}{\frac{\varepsilon_t}{K}}\inp{X}{\mathcal{P}_M(Q)}^2 \\ &\quad + \sum_{k \in [K]\backslash \{1\}}\sum_{X\in \Omega_k} \indicator\left(\max_{k\in [K]}\inp{\widehat{M}_{k,t-1}}{X}=1\right)\frac{1}{1-\frac{\varepsilon_t}{K}}\inp{X}{\mathcal{P}_M(Q)}^2  +  \indicator\left(\max_{k\in [K]}\inp{\widehat{M}_{k,t-1}}{X}\neq 1\right)\frac{1}{\frac{\varepsilon_t}{K}}\inp{X}{\mathcal{P}_M(Q)}^2 \\
        &\quad + \sum_{X\in \Omega_\emptyset} \indicator\left(\max_{k\in [K]}\inp{\widehat{M}_{k,t-1}}{X}=1\right)\frac{1}{1-\frac{\varepsilon_t}{K}}\inp{X}{\mathcal{P}_M(Q)}^2  +  \indicator\left(\max_{k\in [K]}\inp{\widehat{M}_{k,t-1}}{X}\neq 1\right)\frac{1}{\frac{\varepsilon_t}{K}}\inp{X}{\mathcal{P}_M(Q)}^2.
    \end{align*}
    We aim to show the convergence of the indicator function. For any $X\in \cup_{k\in [K]} \Omega_k$ and any $j,w\in [K]$, suppose $\inp{M_j-M_w}{X}>0$, we have for any $T_0+1\leq t\leq T$,
    \begin{align*}
       \inp{\widehat{M}_{j,t}-\widehat{M}_{w,t}}{X}&=\inp{\widehat{M}_{j,t} - M_j}{X} + \inp{M_w-\widehat{M}_{w,t}}{X} + \inp{M_j-M_w}{X} \\
       &\geq \inp{M_j-M_w}{X} - 2\delta >0
    \end{align*}
    with probability at least $1-8td_1^{-200}$, where $\delta$ is defined in Theorem \ref{thm:CLTK}. The second inequality comes from the gap condition in Assumption \ref{assump:arm-optK}. Similarly, if $\inp{M_j-M_w}{X}<0$, we can show $ \inp{\widehat{M}_{j,t}-\widehat{M}_{w,t}}{X}<0$ with the same probability. As a result, with the same probability, $\mathbbm{1}(\inp{\widehat{M}_{j,t}-\widehat{M}_{w,T}}{X}>0)=\mathbbm{1}(\inp{M_j-M_w}{X}>0)$. Combine with a union bound, with probability at least $1-8Ktd_1^{-200}$,
    \begin{align*}
        \indicator\left(\max_{k\in [K]}\inp{\widehat{M}_{k,t-1}}{X}=1\right)&=\prod_{k\in [K]\backslash \{1\}} \indicator\left(\inp{\widehat{M}_{1,t}-\widehat{M}_{k,t}}{X}>0\right) \\
        &= \prod_{k\in [K]\backslash \{1\}} \indicator\left(\inp{{M}_{1,t}-{M}_{k,t}}{X}>0\right) \\
        &= \indicator\left(\max_{k\in [K]}\inp{{M}_{k}}{X}=1\right).
    \end{align*}
    Therefore, $\sum_{X\in \Omega_1} \indicator(\max_{k\in [K]}\inp{\widehat{M}_{k,t-1}}{X}=1)\inp{X}{\mathcal{P}_M(Q)}^2\rightarrow \fro{P_{\Omega_1}(P_{M}(Q))}^2$ and \\ $\sum_{k \in [K]\backslash \{1\}}\sum_{X\in \Omega_k} \indicator(\max_{k\in [K]}\inp{\widehat{M}_{k,t-1}}{X}\neq 1)\inp{X}{\mathcal{P}_M(Q)}^2\rightarrow \fro{P_{\cup_{k\in [K]\backslash \{1\}} \Omega_k}(P_{M}(Q))}^2$. On the other hand, $\sum_{X\in \Omega_1} \indicator(\max_{k\in [K]}\inp{\widehat{M}_{k,t-1}}{X}\neq 1)\inp{X}{\mathcal{P}_M(Q)}^2\rightarrow 0$ and \\ $\sum_{k \in [K]\backslash \{1\}}\sum_{X\in \Omega_k} \indicator(\max_{k\in [K]}\inp{\widehat{M}_{k,t-1}}{X}= 1)\inp{X}{\mathcal{P}_M(Q)}^2\rightarrow 0$.
    Also note that $\frac{b}{T}\sum_{t=T_0+1}^{T} 1/(1-\varepsilon_t/K)\rightarrow 1$ and $\frac{b}{T^{1+\gamma}}\sum_{t=T_0+1}^{T} K/\varepsilon_t\rightarrow \frac{K}{c_2(1+\gamma)}$. And by Assumption \ref{assump:arm-optK}, $\sum_{X\in \Omega_\emptyset} \indicator(\max_{k\in [K]}\inp{\widehat{M}_{k,t-1}}{X}=1)\frac{1}{1-\frac{\varepsilon_t}{K}}\inp{X}{\mathcal{P}_M(Q)}^2  +  \indicator(\max_{k\in [K]}\inp{\widehat{M}_{k,t-1}}{X}\neq 1)\frac{1}{\frac{\varepsilon_t}{K}}\inp{X}{\mathcal{P}_M(Q)}^2\leq K/\varepsilon_t \fro{\calP_{\Omega_{\emptyset}}(\calP_{M}(Q))}^2$ is negligible. Then the conditional variance will converge in probability to 1. It follows
    \begin{align*}
        \frac{\sum_{t=T_0+1}^T \frac{d_1d_2\mathbbm{1}(a_t=1)}{\pi_t}\xi_t\inp{X_t}{\mathcal{P}_M(Q)}}{\sigma S\sqrt{d_1d_2/T^{1-\gamma}}} \rightarrow N(0,1)
    \end{align*}
    when $T, d_1, d_2\rightarrow \infty$.
\end{proof}

\subsection{Proof of Corollary \ref{cor:m1-m0K}}
\begin{proof}
    Note that for any $g,h\in [K]$
    \begin{align*}
        (\inp{\widehat{M}_g}{Q} - \inp{\widehat{M}_h}{Q}) - (\inp{M_g}{Q} - \inp{M_h}{Q})= (\inp{\widehat{M}_g}{Q} - \inp{M_g}{Q}) - (\inp{\widehat{M}_h}{Q} - \inp{M_h}{Q}),
    \end{align*}
    then $\inp{\widehat{M}_g}{Q} - \inp{M_g}{Q}$ and $\inp{\widehat{M}_h}{Q} - \inp{M_h}{Q}$ can be decomposed into main term and negligible terms in the same way as Theorem \ref{thm:CLT}. The upper bound of all the negligible terms follow the Lemma \ref{lemmaneg1}-\ref{lemmaneg3}. The main CLT term is
    \begin{align*}
        \frac{b}{T}\sum_{t_1=T_0+1}^{T}\frac{d_1d_2\mathbbm{1}(a_{t_1}=g)}{\pi_{g,t_1}}\xi_{t_1}\inp{X_{t_1}}{\mathcal{P}_{M_1}(Q)} - \frac{b}{T}\sum_{t_2=T_0+1}^{T}\frac{d_1d_2\mathbbm{1}(a_{t_2}=h)}{\pi_{h,t_2}}\xi_{t_2}\inp{X_{t_2}}{\mathcal{P}_{M_0}(Q)}.
    \end{align*}
    As long as we show $\sum_{t_1=T_0+1}^{T}\frac{d_1d_2\mathbbm{1}(a_{t_1}=g)}{\pi_{g,t_1}}\xi_{t_1}\inp{X_{t_1}}{\mathcal{P}_{M_g}(Q)}$ is uncorrelated with \\ $\sum_{t_2=T_0+1}^{T}\frac{d_1d_2\mathbbm{1}(a_{t_2}=h)}{\pi_{h,t_2}}\xi_{t_2}\inp{X_{t_2}}{\mathcal{P}_{M_h}(Q)}$, then the asymptotic variance is the sum of their individual variance, i.e., $\sigma_g^2S_g^2+\sigma_h^2S_h^2$. Notice that,
    \begin{align*}
        &\quad \frac{b^2}{T^2}\sum_{t_1=T_0+1}^{T}\frac{d_1d_2\mathbbm{1}(a_{t_1}=g)}{\pi_{g,t_1}}\xi_{t_1}\inp{X_{t_1}}{\mathcal{P}_{M_g}(Q)} \sum_{t_2=T_0+1}^{T}\frac{d_1d_2\mathbbm{1}(a_{t_2}=h)}{\pi_{h,t_2}}\xi_{t_2}\inp{X_{t_2}}{\mathcal{P}_{M_h}(Q)} \\
        &=\frac{b^2d_1^2d_2^2}{T^2}\sum_{t_1=T_0+1}^{T}\sum_{t_2=T_0+1}^{T} \frac{\mathbbm{1}(a_{t_1}=g)\mathbbm{1}(a_{t_2}=h)}{\pi_{g,t_1}\pi_{h,t_2}}\xi_{t_1}\xi_{t_2}\inp{X_{t_1}}{\mathcal{P}_{M_g}(Q)}\inp{X_{t_2}}{\mathcal{P}_{M_h}(Q)}.
    \end{align*}
    When $t_1=t_2$, $\mathbbm{1}(a_{t_1}=g)\mathbbm{1}(a_{t_2}=h)=0$. When $t_1\neq t_2$, $\EE[\frac{\mathbbm{1}(a_{t_1}=g)\mathbbm{1}(a_{t_2}=h)}{\pi_{g,t_1}\pi_{h,t_2}}\xi_{t_1}\xi_{t_2}X_{t_1}X_{t_2}]=0$ due to the i.i.d. distributed $\xi_t$. As a result, the two terms are uncorrelated. \\
    Together with Lemma \ref{lemmaneg1}-\ref{lemmaneg3}, all the negligible terms converge to 0 when divided by \\ $\sqrt{(\sigma_g^2S_g^2+\sigma_h^2S_h^2)d_1d_2/T^{1-\gamma}}$. Then similar as Theorem \ref{thm:studentized-CLT}, we can replace $\sigma_g^2S_g^2+\sigma_h^2S_h^2$ with $\widehat{\sigma}_g^2\widehat{S}_g^2+\widehat{\sigma}_h^2\widehat{S}_h^2$ and conclude the proof.
\end{proof}

\end{document}